\documentclass{article}

\usepackage{microtype}
\usepackage{graphicx}
\usepackage{subfigure}
\usepackage{booktabs} 

\usepackage{hyperref}

\usepackage[accepted]{icml2023}

\usepackage{amsmath}
\usepackage{amssymb}
\usepackage{mathtools}
\usepackage{amsthm}

\usepackage{enumerate}

\usepackage{graphics} 
\usepackage{mathptmx} 
\usepackage{multirow}
\usepackage{amsbsy}
\usepackage{array}
\usepackage{calrsfs}
\usepackage{anyfontsize}

\usepackage{color}

\usepackage[capitalize,noabbrev]{cleveref}

\theoremstyle{plain}
\newtheorem{theorem}{Theorem}[section]

\newtheorem{lemma}[theorem]{Lemma}

\theoremstyle{definition}
\newtheorem{definition}[theorem]{Definition}

\theoremstyle{remark}

\usepackage[textsize=tiny]{todonotes}

\icmltitlerunning{Deterministic Online Classification: Non-iteratively Reweighted Recursive Least-Squares}

\begin{document}
	
	\twocolumn[
	\icmltitle{Deterministic Online Classification: Non-iteratively Reweighted Recursive Least-Squares for Binary Class Rebalancing}
	
	
	
	\icmlsetsymbol{equal}{*}
	
	\begin{icmlauthorlist}
		\icmlauthor{Se-In Jang}{yyy}
	\end{icmlauthorlist}
	
	\icmlaffiliation{yyy}{Gordon Center for Medical Imaging and Center for Advanced Medical Computing and Analysis, Massachusetts General Hospital and Harvard Medical School,  Boston, MA, USA}
	
	\icmlcorrespondingauthor{Se-In Jang}{sjang7@mgh.harvard.edu}
	
	\icmlkeywords{Machine Learning, ICML}
	
	]
	
	
	
	\printAffiliationsAndNotice{}  
	
	\begin{abstract}
		Deterministic solutions are becoming more critical for interpretability. Weighted Least-Squares (WLS) has been widely used as a deterministic batch solution with a specific weight design. In the online settings of WLS, exact reweighting is necessary to converge to its batch settings. In order to comply with its necessity, the iteratively reweighted least-squares algorithm is mainly utilized with a linearly growing time complexity which is not attractive for online learning.
		Due to the high and growing computational costs, an efficient online formulation of reweighted least-squares is desired.
		We introduce a new deterministic online classification algorithm of WLS with a constant time complexity  for binary class rebalancing.
		We demonstrate that our proposed online formulation exactly converges to its batch formulation and 
		outperforms existing state-of-the-art stochastic online binary classification algorithms in real-world data sets empirically.
	\end{abstract} \vspace{-0.7cm}
	
	\section{Introduction}
	Online learning  is an essential step to address large-scale learning (e.g., big data) efficiently  and real-time training (e.g., data streaming) in limited computing resources \cite{bottou2004large, bottou2018optimization, cesa2006prediction, hoi2018online, shalev2007online}. 
	%
	In designing online learning for classification, a stochastic based approach is mainly explored with several nonlinear loss functions (e.g., step and hinge loss functions). In the stochastic-based  online classification, the simplest and most popular architecture  is the Perceptron (PE) algorithm \cite{rosenblatt1958perceptron}  which 
	uses the first-order information obtained from the first-order derivative with a step loss function. The Passive-Aggressive (PA) algorithm \cite{crammer2006online} is also a successful stochastic-based  algorithm that aggressively updates when its hinge loss is non-zero.
	
	The first-order algorithms have paid attention due to their simplicity. 	
	However, due to the limited information from the first-order derivative in optimization,  
	the use of the first and second-order information becomes more attractive although it needs more computation than the first-order algorithms \cite{bottou2018optimization}. 
	One of the most successful second-order algorithms for online classification is the Confidence-Weighted (CW) learning \cite{dredze2008confidence}, which follows a Gaussian distribution and uses the Kullback–Leibler divergence to stay close to the previous Gaussian distribution. 
	As an improved version of CW, Adaptive Regularization Of Weighted vectors (AROW) learning \cite{crammer2009adaptive} is developed based on a squared hinge loss function with confidence regularization for handling non-separable data.  
	In \cite{zhao2018adaptive}, an adaptive regularized cost-sensitive online gradient descent algorithm (ACOG) with a weighted sum matric and a weighted cost metric is presented based on AROW. 
	They assume to give weight to a specific class frequently and aggressively.  However, due to this assumption, ACOG  cannot be appropriately performed when there are minimal samples  for the specific class.
	The above online learning algorithms well established the classification goals by nonlinear loss functions, which seek to find local minima under the stochastic nature.
	\begin{table}[t!] \centering \scriptsize
		\caption{An overview of online classification methods.}
		\label{tbl.algorithm_overview}
		\begin{tabular}{cccc}
			\toprule
			Algorithms & Learning Type & Reweighting & Imbalance \\ \midrule
			PE         & Stochastic                                              & -           & $\times$                                                      \\
			PA         & Stochastic                                              & -           & $\times$                                                      \\
			CW         & Stochastic                                              & -           & $\times$                                                      \\
			AROW       & Stochastic                                              & -           & $\times$                                                      \\
			ACOG       & Stochastic                                              & -           & $\surd$                                                      \\
			AR-RLS     & Deterministic                                           & Approx.     & $\surd$                                                      \\
			IR-RLS     & Deterministic                                           & Exact       & $\surd$                                                     \\
			NR-RLS     & Deterministic                                           & Exact       & $\surd$          \\                                            
			\bottomrule 
		\end{tabular} \vspace{-0.75cm}
	\end{table} 
	
	Although such  stochastic settings have been routinely and successfully applied to online classification problems, 
	logical interpretation of such algorithms was not often convincing due to the inability  to settle at the global minimum \cite{molnar2020interpretable}.   
	The Least-Squares  (LS) \cite {legendre1805nouvelles, stigler1981gauss} is the simplest and most well-known algorithm to obtain a global solution using the squared loss function under the deterministic nature \cite{willems2004deterministic}. 
	For an online setting of LS, the Recursive Least-Squares (RLS) algorithm  was developed for regression analysis \cite{plackett1950some, woodbury1950inverting}.
	The RLS algorithm is not only applied to regression problems but also to the classification problems.
	However, the objective of RLS could not be accounted as the true objective of classification.
	In order to achieve the classification objective under the deterministic nature, a quadratic approximation to the step loss function was designed to solve a Total Error Rate  minimization  problem (TER) \cite{toh2008deterministic}. The TER's classification objective was achieved by  simple class-weight changes based on  Weighted Least-Squares (WLS) like cost-sensitive learning \cite{elkan2001foundations, he2009learning},
	where class-weighting  plays an essential role in class imbalance problems.
	Class-weighting has been observed as a valuable  way to adjust decision boundaries   \cite{wang2008probability, wu2010robust, scott2012calibrated, camoriano2017incremental, xu2020class}.
	
	
	
	When making such online settings of WLS,  the exact reweighting  is necessary to converge to its batch settings.
	As an exact reweighting formulation,  
	an iteratively reweighted least-squares algorithm was designed without recursive computation \cite{chartrand2008iteratively}.
	In \cite{camoriano2017incremental}, a recursive reweighting form of the first moment vector was proposed without the inverse of the second moment matrix, which is very significant for the exact convergence to its batch setting.
	Due to the lack of efficient recursive reweighting formulations of WLS for both the first and second moment matrices, an Approximately Reweighted RLS (AR-RLS) algorithm \cite{kim2011onnet}  was developed.
	However, the approximately estimated recursive formulation can cause cumulative approximation errors in optimization. 
	In order to overcome the approximate reweighting, an Iteratively  Reweighted RLS (IR-RLS)  algorithm \cite{jang2017online}  was then demonstrated for an exact reweighting in a recursive form.
	However, the iterative  inversion of IR-RLS still requires a considerable computational effort with an exponentially growing time complexity. 
	Moreover, due to the limitation of the iterative nature, all the previous samples are also inefficiently stored in memory.

	As summarized in Table \ref{tbl.algorithm_overview}, 
	the main contributions of our work thus include: 
	(i) A new class of an online classification formulation, namely Non-iteratively Reweighted RLS (NR-RLS), which exactly converges to the batch setting of the TER method for deterministic class imbalance classification and achieves a constant time complexity  that is preferred for online settings.  
	To the best of our knowledge, this is the first approach that can \textit{non-iteratively, recursively, and exactly} achieve reweighted least-squares in an online setting. Due to this property, this work can be extended to various recursive forms which need a  reweighting strategy.
	(ii) NR-RLS adopts a total-error-rate metric that simultaneously uses two different weights for both positive and negative classes.  {This helps to address unbalanced data distributions for both classes together.}
	(iii) Accumulation of the arriving samples efficiently.
	(iv) Extensive evaluation of the formulation using 31 real-world data sets.   
	\section{Preliminaries}  
	\subsection{Least-Squares (LS) Minimization}
	The Least-Squares (LS) minimization is the most common method for regression and classification problems. The objective function of the LS minimization   is based on the sum of squared errors  distance function that is more relevant to the regression problems as follows:
	\begin{equation} \label{eq.LSE_J}
		\begin{aligned}
			\text{LS: } J\left(\pmb{w }\right) &= \frac{1}{2 }\sum\limits_{i = 1}^n {{{\left( {{y_i} - {\pmb{w }}^T{{\bf{x}}_i}} \right)}^2}} + \frac{b}{2}\left\| \pmb{w } \right\|_2^2,
		\end{aligned}
	\end{equation}
	which provides a deterministic closed-form solution as:
	\begin{equation} \label{eq.LSE_solution}
		\pmb{w} = (\mathbf{X}^T\mathbf{X} +
		b{\bf{I}})^{-1}\mathbf{X}^T\mathbf{y}, 
	\end{equation} 
	where $\mathbf{X} \in \mathbb{R}^{n\times d}$ is the data matrix, $\mathbf{y} \in  \{-1, 1\}$ is the target label vector, $n$ indicates the number of data samples, and $d$ indicates a sample feature dimension. 
	$b$ is a regularization factor and ${\bf I}$ is an identity matrix with a similar dimension as ${\bf X}^T {\bf X}$.

	\subsection{Minimization for Binary Class Imbalance Learning} \label{sec.batcTER}
	{

		Different from the regression objective of LS,
		in \cite{toh2008between}, a classification  objective is presented based on a quadratic approximation to the step function for a Total Error Rate (TER) minimization, which can maximize the classification accuracy like cost-sensitive learning and also can handle binary class imbalance classification as follows:
		\begin{equation} \label{eq.TER_J}
			\begin{aligned} \text{TER: } J\left(\pmb{w }\right) &= \frac{1}{{2n_{}^ - }}\sum\limits_{i = 1}^{n_{}^ - } {{{\left( {y_i^ -  - {\pmb{w }}^T{\bf{x}}_i^ - } \right)}^2}}  \\
				&\phantom{=} + \frac{1}{{2n_{}^ + }}\sum\limits_{j = 1}^{n_{}^ + } {{{\left( {y_j^ +  - {\pmb{w }}^T{\bf{x}}_j^ + } \right)}^2}} + \frac{b}{2}\left\| \pmb{w } \right\|_2^2,\end{aligned}
		\end{equation}
		which also yields a deterministic closed-form solution related to weighted least-squares as:
		\begin{equation} \label{eq.TER_solution}
			\pmb{w} = (\mathbf{X}^T\mathbf{WX} +
			b{\bf{I}})^{-1}\mathbf{X}^T\mathbf{Wy},
		\end{equation}
		where the superscripts $-$ and $+$ indicate the negative class label and the positive class label respectively. 
		$n^-$ and $n^+$ respectively indicate the populations of negatively and positively labeled samples.
		${\bf{X}} = \left[ {{{\bf{X}}^ - },{{\bf{X}}^ + }} \right]^T$ includes data matrices for negative and positive classes, and
		$\mathbf{W} =
		diag\begin{pmatrix}\begin{bmatrix}\frac{1}{n^-}\,,\,\dotsc\,,\,\frac{1}{n^-}\,,\,\frac{1}{n^+}\,,\,\dotsc\,,\,\frac{1}{n^+}
		\end{bmatrix}\end{pmatrix} \in
		\mathbb{R}^{(n^-+n^+)\times (n^-+n^+)}$ is a class-specific weighting matrix.
		$\mathbf{y} = \begin{bmatrix} (\tau - \eta),\,\dotsc,\,(\tau
			- \eta),\,(\tau + \eta),\,\dotsc,\,(\tau + \eta)\end{bmatrix} \in
		\mathbb{R}^{(n^-+n^+)}$ is the target output, which can be adjusted
		by changing the decision threshold $\tau$ and the offset factor
		$\eta$ (see \cite{toh2008between}).
		Here, $\mathbf{y} = \begin{bmatrix} -1,\,\dotsc,\,-1,\,1,\,\dotsc,\,1\end{bmatrix}$ is obtained by  setting $\tau = 0$ and $\eta = 1$. 
		The prediction outputs for the test set are calculated by ${\hat{\mathbf{y}}_\text{test}}   = {{\bf{X}}_\text{test}}{{\pmb{w }}}$.
		The TER solution  is differentiated from the LS solution in the {adoption of} two different class-specific weights, $1 \over n^-$  and $1 \over n^+$, for negative and positive classes.
		This weight change effectively offers misclassification minimization and class rebalancing together.
		
	}
	

	%

	\subsection{Recursive Least-Squares (RLS)}
	%
	%
	
	Recursive least-squares (RLS) learning \cite{plackett1950some, woodbury1950inverting, haykin2013adaptive} has been frequently utilized as a deterministic closed-form online solution inherited by LS.
	The RLS coefficient vector $\pmb w_t$ at time $t$ is estimated using 
	\begin{equation} \label{eq.RLS}
		\text{RLS: } {\pmb{w}}_{t} = {\pmb{w}}_{t-1} +  {\mathbf{R}}_{t}^{-1} \mathbf{x}_t(y_t -
		\mathbf{x}_t^T{{\pmb{w}}_{t-1}}),
	\end{equation}
	\noindent where $\mathbf{x}_t \in \mathbb{R}^d$ and ${y}_t$ denote respectively the newly arrived sample vector and the output value indexed by time $t$, and
	\begin{equation} \label{eq.Rt-1}
		{\mathbf{R}}_{t}^{-1} = {\mathbf{R}}_{t-1}^{-1} -
		{\mathbf{R}}_{t-1}^{-1}\mathbf{x}_t({1} + \mathbf{x}_t^T {\mathbf{R}}_{t-1}^{-1}
		\mathbf{x}_t)^{-1} \mathbf{x}_t^T {\mathbf{R}}_{t-1}^{-1}
	\end{equation}
	\noindent is the recursively accumulated inverse matrix derived from the well-known matrix inversion lemma \cite{woodbury1950inverting, sherman1950adjustment, bronvstejn2013handbook}.

	\section{Non-iteratively Reweighted Recursive Least-Squares (NR-RLS) for Binary Class Rebalancing}
	In this section, we will establish a Non-iteratively Reweighted Recursive Least-Squares formulation (NR-RLS), which can precisely calculate a binary class rebalancing loss function in an online setting. 
	The main goal of NR-RLS is to \textit{non-iteratively, recursively and exactly} estimate the coefficient vector $\pmb{w}_t$ with the two class-specific  weights (e.g., $1 \over n_t^-$ and $1 \over n_t^+$ for the negative and positive classes) which is changed along with the arrival of new samples.
	
	\begin{definition}
		\label{def:inj}
		The batch solution of  \eqref{eq.TER_solution} can be time-indexed and rewritten as follows:
		\begin{equation} \label{eq.RTER_theta} 
			\begin{aligned} 
				{\pmb w _t} &= {\left( {1 \over n_t^-}{{\bf{X}}_t^{-T}  {\bf{X}}_t^{-}  + {1 \over n_t^+} {\bf{X}}_t^{+T} {\bf{X}}_t^{+}   + b{\bf{I}}} \right)^{ - 1}} \\
				&\phantom{=}\times  { {1 \over n_t^-}{{\bf{X}}_t^{-T}  {\bf{y}}_t^{-}  + {1 \over n_t^+} {\bf{X}}_t^{+T} {\bf{y}}_t^{+}  } } \\
				&= {\left( \frac{1}{{n_t^ - }}{\sum\limits_{i = 1}^{n_t^-} {{\mathbf{x}}_i^ - {\mathbf{x}}_i^{ - T}}  + \frac{1}{{n_t^ + }} \sum\limits_{j = 1}^{n_t^ + } {{\mathbf{x}}_j^ + {\mathbf{x}}_j^{ + T}}  + b{\mathbf{I}}} \right)^{ - 1}} \\
				&\phantom{=}\times \left( \frac{1}{{n_t^ - }} {\sum\limits_{i = 1}^{n_t^ - } {{\mathbf{x}}_i^ - y_i^ - }  + \frac{1}{{n_t^ + }} \sum\limits_{j = 1}^{n_t^ + } {{\mathbf{x}}_j^ + y_j^ + } } \right), \\
			\end{aligned}
		\end{equation} 
		\noindent where 
		$\mathbf{x}_i^{-}$ and $\mathbf{x}_j^{+}$ respectively indicate the negative and positive labeled data.
		Next, the  two covariance terms within the inverse operation in equation \eqref{eq.RTER_theta} are written as 
		${\mathbf{S}}_t^ - = \frac{{ {1} }}{{n_t^ - }} {\sum\limits_{i = 1}^{n_t^ - } {{\mathbf{x}}_i^ - {\mathbf{x}}_i^{ - T}} }$ 
		and ${\mathbf{S}}_t^ + = \frac{{ {1} }}{{n_t^ + }} {\sum\limits_{j = 1}^{n_t^ + } {{\mathbf{x}}_j^ + {\mathbf{x}}_j^{ + T}} }$.
		\eqref{eq.RTER_theta} can be simplified as 
		\begin{equation} \label{eq.RTER_theta_simplified}
			\begin{aligned} 
				{\pmb w _t} &= {\left( {{\bf{S}}_t^ -  + {\bf{S}}_t^ +  + b{\bf{I}}} \right)^{ - 1}}\left( {{\bf{z}}_t^ -  + {\bf{z}}_t^ + } \right)\\
				&= {\bf{R}}_t^{ - 1}{\bf{z}}_t^{},
			\end{aligned}
		\end{equation}
		where there is a simple multiplication 
		between the recursive inversion of the weighted second-moment matrix ${\bf{R}}_t^{ - 1}$  
		and the recursion of the weighted first-moment vector ${\bf{z}}_t^{}$. 
	\end{definition}
	
	\begin{theorem}
		\label{thm:recursiveFull}
		The recursive form \eqref{eq.RTER_theta_simplified} is identical to the batch form \eqref{eq.TER_solution} and minimizes the binary class imbalance objective function \eqref{eq.TER_J}.
	\end{theorem}
	
	\begin{proof}
		The proof is in the following subsections.
	\end{proof}

	\subsection{Derivation of NR-RLS }

	
	\subsubsection{Recursive inversion of the second-moment matrix ${\bf{R}}_t^{ - 1}$}
	\begin{theorem}
		\label{thm:recursiveSecond}
		Suppose  ${\bf{R}}_t^{ - 1}$ consisting of two recursive terms and a constant term. The two recursive matrices, ${\bf{S}}_t^{ - }$ and ${\bf{S}}_t^{ + }$, are for accumulation of negative and positive class samples. Since the regularization term $b{\bf{I}}$ is not time-dependent, only the two moment matrices need to be considered in the recursive formulation. 
		Then, ${\bf{R}}_t^{ - 1} = {\left( {{\bf{S}}_t^ -  + {\bf{S}}_t^ +  + b{\bf{I}}} \right)^{ - 1}}$  is identical to 
		$ {\bf{R}}_t^{ - 1} =  {\left( {1 \over n_t^-}{{\bf{X}}_t^{-T}  {\bf{X}}_t^{-}  + {1 \over n_t^+} {\bf{X}}_t^{+T} {\bf{X}}_t^{+}   + b{\bf{I}}} \right)^{ - 1}}$.
	\end{theorem}
	
	
	\begin{proof}

		Suppose the newly arriving sample  ${\mathbf{x}}_t^ -$ comes from the negative category, then
	\begin{equation} \label{eq.S-}
		\begin{aligned} 
			{\mathbf{S}}_t^ -  &= \frac{{n_{t - 1}^ - }}{{n_t^ - }}\underbrace {\sum\limits_{i = 1}^{n_{t - 1}^ - } {\frac{{1}}{{n_{t - 1}^ - }}{\mathbf{x}}_i^ - {\mathbf{x}}_i^{ - T}} }_{{\text{accumulated part}},{\text{ }}{\mathbf{S}}_{t - 1}^ - } + \frac{{1}}{{n_t^ - }}\underbrace {{\mathbf{x}}_t^ - {\mathbf{x}}_t^{ - T}}_{{\text{new sample}}} \hfill \\
			&{=} \frac{{n_{t - 1}^ - }}{{n_t^ - }}{\mathbf{S}}_{t - 1}^ - + \frac{{1}}{{n_t^ - }}{\mathbf{x}}_t^ - {\mathbf{x}}_t^{ - T}  \hfill .
		\end{aligned} 
	\end{equation}
	\noindent 
	Since ${\frac{{n_{t - 1}^ - }}{{n_t^ - }} = \frac{{n_{t -
					1}^ - }}{{n_{t - 1}^ -  + {1}}}  = \left( {1 -
			\frac{{1}}{{n_t^ - }}} \right)}$, 
	we have 
	\begin{equation} \label{eq.S-2}
		\begin{aligned} 
			{\mathbf{S}}_t^ -  &= \left( {1 - \frac{{1}}{{n_t^ - }}} \right){\mathbf{S}}_{t - 1}^ -   + \frac{{1}}{{n_t^ - }}{\mathbf{x}}_t^ - {\mathbf{x}}_t^{ - T} \hfill \\
			&= {\mathbf{S}}_{t - 1}^ - - \frac{{1}}{{n_t^ - }}{\mathbf{S}}_{t - 1}^ - + \frac{{1}}{{n_t^ - }}{\mathbf{x}}_t^ - {\mathbf{x}}_t^{ - T}   \hfill \\
			&= {\bf{S}}_{t - 1}^ -  + \frac{1}{{n_t^ - }}\left( {{\bf{x}}_t^ - {\bf{x}}_t^{ - T} - {\bf{S}}_{t - 1}^ - } \right).  \\
		\end{aligned} 
	\end{equation}
	\noindent
	On the other hand, if the newly arriving sample  ${\mathbf{x}}_t^ +$ comes from the positive category, then
	\begin{equation} \label{eq.S+2}
		\begin{aligned} 
			{\mathbf{S}}_t^ +  &= \frac{{n_{t - 1}^ + }}{{n_t^ + }}\underbrace {\sum\limits_{i = 1}^{n_{t - 1}^ + } {\frac{{1}}{{n_{t - 1}^ + }}{\mathbf{x}}_i^ + {\mathbf{x}}_i^{ + T}} }_{{\text{accumulated part}},{\text{ }}{\mathbf{S}}_{t - 1}^ + } + \frac{{1}}{{n_t^ + }}\underbrace {{\mathbf{x}}_t^ + {\mathbf{x}}_t^{ + T}}_{{\text{new sample}}} \hfill \\
			&= {\mathbf{S}}_{t - 1}^ + - \frac{{1}}{{n_t^ + }}{\mathbf{S}}_{t - 1}^ + + \frac{{1}}{{n_t^ + }}{\mathbf{x}}_t^ + {\mathbf{x}}_t^{ + T} \hfill \\
			&= {\bf{S}}_{t - 1}^ +  + \frac{1}{{n_t^ + }}\left( {{\bf{x}}_t^ + {\bf{x}}_t^{ + T} - {\bf{S}}_{t - 1}^ + } \right).  \\
		\end{aligned} 
	\end{equation}
	By knowing that the newly arriving sample can only belong to one of the two categories,  \eqref{eq.S-2} and \eqref{eq.S+2} are re-written as
	\begin{equation} \label{eq.S-3+3}
		\begin{aligned} 
			{\bf{S}}_t^ -  &= {\bf{S}}_{t - 1}^ -  + \beta _t^ - \left( {{\bf{x}}_t {\bf{x}}_t^{  T} - {\bf{S}}_{t - 1}^ - } \right) ,  \\
			{\bf{S}}_t^ +  &= {\bf{S}}_{t - 1}^ +  + \beta _t^ + \left( {{\bf{x}}_t {\bf{x}}_t^{  T} - {\bf{S}}_{t - 1}^ + } \right) ,
		\end{aligned} 
	\end{equation}
	where $\beta _t^ -  = {\textstyle{{(1 - {y_t})} \over 2n_t^-}}$ and $\beta _t^ +  = {\textstyle{{(1 + {y_t})} \over 2n_t^+}}$ are  indicators to help a selection of either the negative class or the positive class.
	Therefore, {\it the new sample is accumulated} in either ${\mathbf{S}}_t^ -$  or  ${\mathbf{S}}_t^ +$.
	
	{ By combining ${\bf{S}}_t^ -$ and ${\bf{S}}_t^ +$ in  \eqref{eq.S-3+3},} we have
	\begin{equation} \label{eq.full_Rt_term}
		\begin{aligned} 
			{{\mathbf{R}}_t} &= {\mathbf{S}}_t^ -  + {\mathbf{S}}_t^ +  + b{\mathbf{I}} \hfill \\
			&= \underbrace {{\mathbf{S}}_{t - 1}^ -  + {\mathbf{S}}_{t - 1}^ +  + b{\mathbf{I}}}_{{{\mathbf{R}}_{t - 1}}} \\
			&\phantom{=}  - \underbrace {  \beta _t^ - {\mathbf{S}}_{t - 1}^ -  - \beta _t^ + {\mathbf{S}}_{t - 1}^ + }_{  \beta _t {\mathbf{S}}_{t - 1}^{}} \hfill 
			+ \underbrace {\beta _t^ - {\mathbf{x}}_t  {\mathbf{x}}_t^{  T} + \beta _t^ + {\mathbf{x}}_t {\mathbf{x}}_t^{  T}}_{ \beta _t {\mathbf{x}}_t^{}{\mathbf{x}}_t^T} \\
			&= \underbrace {{{\mathbf{R}}_{t - 1}} -  \beta _t{\mathbf{S}}_{t - 1}^{}}_{{{\mathbf{G}}_t}} + { \beta _t}{\mathbf{x}}_t^{}{\mathbf{x}}_t^T, \hfill \\ 
		\end{aligned}
	\end{equation}
	\noindent where
		${\bf{S}}_{t - 1}^{} = {\textstyle{{(1 - {y_t})} \over 2}} {\bf{S}}_{t - 1}^ -  + {\textstyle{{(1 + {y_t})} \over 2}} {\bf{S}}_{t - 1}^ + $,
			${\bf{x}}_t^{} = {\textstyle{{(1 - {y_t})} \over 2}} {\bf{x}}_t^ -  +{\textstyle{{(1 + {y_t})} \over 2}} {\bf{x}}_t^ + $,
				${\beta _t} = \beta _t^ -  + \beta _t^ + $,
				and $y_t \in \{-1, +1\}$. 
				In order to facilitate the utilization of the existing matrix inversion lemma,  \eqref{eq.full_Rt_term} is written  as two summation terms as follows:
				\begin{equation} \label{eq.R_Q_term}
					\begin{aligned}
						{{\bf{G}}_{t}} &= { {{{\mathbf{R}}_{t - 1}} - { \beta _t{\mathbf{S}}_{t - 1}^{}}} } \\
						{{\bf{R}}_{t}} &= {{{\mathbf{G}}_t} +  \beta _t{\mathbf{x}}_t^{}{\mathbf{x}}_t^T}.
					\end{aligned}
				\end{equation}
				Based on   the well-known Sherman-Morrison-Woodbury formulation \cite{henderson1981deriving}: 
				\begin{equation} \label{eq.InversionLemma2}
					{({\mathbf{A}} + {\mathbf{BCD}})^{ - 1}} = {{\mathbf{A}}^{ - 1}} - {{\mathbf{A}}^{ - 1}}{({\mathbf{I}} + {\mathbf{BCD}}{{\mathbf{A}}^{ - 1}})^{ - 1}}{\mathbf{BCD}}{{\mathbf{A}}^{ - 1}},
				\end{equation}
				the inverses of ${\bf{G}}_t$  and ${\bf{R}}_t$ are given by:
				\begin{equation} \label{eq.Qt_inv}
					\begin{aligned}
						{\mathbf{G}}_t^{ - 1} &= {\mathbf{R}}_{t - 1}^{ - 1} + {\mathbf{R}}_{t - 1}^{ - 1}{\left( {{\mathbf{I}} - { \beta _t}{\mathbf{S}}_{t - 1}^{}{\mathbf{R}}_{t - 1}^{ - 1}} \right)^{ - 1}}{ \beta _t}{\mathbf{S}}_{t - 1}^{}{\mathbf{R}}_{t - 1}^{ - 1}, \\
						{\mathbf{R}}_t^{ - 1} &= {\mathbf{G}}_t^{ - 1} - {\mathbf{G}}_t^{ - 1}{\left( {{\mathbf{I}} + { \beta _t}{\mathbf{x}}_t^{}{\mathbf{x}}_t^T{\mathbf{G}}_t^{ - 1}} \right)^{ - 1}}{ \beta _t}{\mathbf{x}}_t^{}{\mathbf{x}}_t^T{\mathbf{G}}_t^{ - 1},
					\end{aligned}
				\end{equation}
				\noindent where
				${{\bf{G}}_{t}^{-1}}$ is derived based on  ${({{{\mathbf{R}}_{t - 1}} - { \beta _t {\mathbf{S}}_{t - 1}^{}}})}$ of  \eqref{eq.R_Q_term} by putting $\mathbf{A} = \mathbf{R}_{t-1}$, $\mathbf{B} = \mathbf{I}$, $\mathbf{C} = -{ \beta _t {\mathbf{S}}_{t - 1}^{}}$ and $\mathbf{D} = \mathbf{I}$. 
				${{\bf{R}}_{t}^{-1}}$ is derived based on $({{{\mathbf{G}}_t} + { \beta _t}{\mathbf{x}}_t^{}{\mathbf{x}}_t^T})$ of  \eqref{eq.R_Q_term} by putting $\mathbf{A} = \mathbf{G}_{t}$, $\mathbf{B} = { \beta _t}{\mathbf{x}}_t^{}$, $\mathbf{C} = \mathbf{I}$ and $\mathbf{D} = {\mathbf{x}}_t^T$. 
				\eqref{eq.Qt_inv} achieves {\it a non-iteratively and exactly reweighting process} by replacing the old weight, $1 \over n_{t-1}$ by the new weight, $1 \over n_t$ for the previously estimated ${{\bf{R}}_{t - 1}^{-1}}$.
			\end{proof} 
			

			\subsubsection{Recursion of the first moment vector ${\bf{z}}_t$}
			
			\begin{theorem}
				\label{thm:recursiveFirst}
				Suppose  ${\bf{z}}_t$ consisting of two recursive terms. The two recursive vectors, ${\bf{z}}_t^ -$ and ${\bf{z}}_t^ +$, are for accumulation of negative and positive class samples. 
				Then, ${\bf{z}}_t = { {{\bf{z}}_t^ -  + {\bf{z}}_t^ + }}$  is identical to 
				$ {\bf{z}}_t = { {1 \over n_t^-}{{\bf{X}}_t^{-T}  {\bf{y}}_t^{-}  + {1 \over n_t^+} {\bf{X}}_t^{+T} {\bf{y}}_t^{+}  } }$.
			\end{theorem}
			
			\begin{proof}
				Similar to \eqref{eq.S-3+3}, 
				these moment vectors can be easily expressed in terms of their previous estimations as
				\begin{equation} \label{eq.t-t+}
					\begin{aligned} 
						{\bf{z}}_t^ -  &= {\bf{z}}_{t - 1}^ -  + \beta _t^ - \left( {{\bf{x}}_t y_t  - {\bf{z}}_{t - 1}^ - } \right) , \\
						{\bf{z}}_t^ +  &= {\bf{z}}_{t - 1}^ +  + \beta _t^ + \left( {{\bf{x}}_t y_t  - {\bf{z}}_{t - 1}^ + } \right) .
					\end{aligned}
				\end{equation}
			\end{proof} 

			\subsection{Summary of the proposed NR-RLS algorithm}
			The proposed NR-RLS algorithm is summarized in the pseudo-code form (see Algorithm~\ref{alg:RTER}).
			The main contribution of the proposed NR-RLS over the existing IR-RLS \cite{jang2017online} lies on the utilization of a vectorized weight matrix update to replace the iterative nature of the sample-wise weight update. 
			Therefore, the proposed NR-RLS achieves a constant time complexity $\mathcal{O}(2d^2)$ similar to the complexity $\mathcal{O}(d^2)$ of RLS.
			This solves the linearly growing computational problem of IR-RLS, which has a growing time complexity of $\mathcal{O}(n_td^2)$ caused by the iterative inversion.
			
			\begin{lemma}
				\label{lem:usefullemma}
				The proposed NR-RLS classifier asymptotically recover the  optimal
				Bayes classifier and can easily be extended to the multiclass classification.
			\end{lemma}
			\begin{proof}
				The proof is in the  Appendix \ref{sec.A} and \ref{sec.C}.
			\end{proof}

			\begin{algorithm}[t!]
				\caption{Non-iteratively Reweighted Recursive Least-Squares}
				\label{alg:RTER}
				\begin{algorithmic}
					
					\STATE {\bfseries Input:} ${\bf{x}}_t \in \mathbb{R}^{d}$, $y_t \in \{-1, +1\}$
					\STATE {\bfseries Initialize:} 
					$n_{0}^ -  = n_{0}^ += 0$,
					${\bf{S}}_0^ -  = {\bf{S}}_0^ + = \mathbf{0} $,  
					${\bf{z}}_0^ -  = {\bf{z}}_0^ + = \mathbf{0} $, 
					${\bf{R}}_0^{ - 1} = {1 \over b}{\bf{I}}$

					\FOR{$t=1,\dots$ }
					\STATE Update	$\scriptsize n_t^ -  = n_{t - 1}^ -  + {\textstyle{{(1 - {y_t})} \over 2 }} $, 
					$\scriptsize n_t^ +  = n_{t - 1}^ +  + {\textstyle{{(1 + {y_t})} \over 2 }} $
					
					\STATE Set the following: \\
					\hskip0.5em ${\beta _t} = \beta _t^ -  + \beta _t^ + $, $\beta _t^ -  = {\textstyle{{(1 - {y_t})} \over 2 n_t^ -}},\beta _t^ +  = {\textstyle{{(1 + {y_t})} \over 2 n_t^ +}}$ \\
					
					\hskip0.5em $\scriptsize {\bf{S}}_{t - 1}^{} = {\textstyle{{(1 - {y_t})} \over 2 }} {\bf{S}}_{t - 1}^ -  + {\textstyle{{(1 + {y_t})} \over 2 }} {\bf{S}}_{t - 1}^ +  $ \\

					\STATE Update the following: \\
					\hskip0.5em $ {\mathbf{G}}_t^{ - 1} = {\mathbf{R}}_{t - 1}^{ - 1} + {\mathbf{R}}_{t - 1}^{ - 1}{\left( {{\mathbf{I}} - {\beta _t}{\mathbf{S}}_{t - 1}^{}{\mathbf{R}}_{t - 1}^{ - 1}} \right)^{ - 1}}{\beta _t}{\mathbf{S}}_{t - 1}^{}{\mathbf{R}}_{t - 1}^{ - 1}$	
					\\
					\hskip0.5em ${\mathbf{R}}_t^{ - 1} = {\mathbf{G}}_t^{ - 1} - {\mathbf{G}}_t^{ - 1}{\left( {{\mathbf{I}} + {\beta _t}{\mathbf{x}}_t^{}{\mathbf{x}}_t^T{\mathbf{G}}_t^{ - 1}} \right)^{ - 1}}{\beta _t}{\mathbf{x}}_t^{}{\mathbf{x}}_t^T{\mathbf{G}}_t^{ - 1}$
					\\
					\hskip0.5em	${\bf{S}}_t^ -  = {\bf{S}}_{t - 1}^ -  + \beta _t^ - \left( {{\bf{x}}_t {\bf{x}}_t^{  T} - {\bf{S}}_{t - 1}^ - } \right)$,
					\\
					\hskip0.5em	${\bf{S}}_t^ +  = {\bf{S}}_{t - 1}^ +  + \beta _t^ + \left( {{\bf{x}}_t {\bf{x}}_t^{  T} - {\bf{S}}_{t - 1}^ + } \right)$
					\\
					\hskip0.5em	$\scriptsize {\bf{z}}_t^ -  = {\bf{z}}_{t - 1}^ -  + \beta _t^ - \left( {{\bf{x}}_t y_t  - {\bf{z}}_{t - 1}^ - } \right)$, 
					\\
					\hskip0.5em	$\scriptsize {\bf{z}}_t^ +  = {\bf{z}}_{t - 1}^ +  + \beta _t^ + \left( {{\bf{x}}_t y_t  - {\bf{z}}_{t - 1}^ + } \right)$
					\\
					\hskip0.5em	${{\pmb{w }}_t} = {\mathbf{R}}_t^{ - 1}{\mathbf{z}}_t^{}$, 
					${\mathbf{z}}_t = {{\mathbf{z}}_t^ -} +{{\mathbf{z}}_t^ +}  $
					
					\ENDFOR
					
				\end{algorithmic} 
			\end{algorithm}

			\begin{table*}[t!] \vspace{-0.5cm}
				\caption{Summary of the 31 real-world data sets for binary class imbalance classification.} \label{tbl.dataSummary} \vspace{-0.3cm}
				\vskip 0.15in
				\begin{center}
					\begin{small}
						\begin{sc}
							\fontsize{6}{7}\selectfont
							\begin{tabular}{l@{\hskip0.5pt}llllll@{\hskip3pt}lllll}
								\toprule
								& No. {\rule{0pt}{3ex}} & Data sets          & Size    & Dimension & Ratio  
								& No. {\rule{0pt}{3ex}} & Data sets          & Size    & Dimension & Ratio   \\ \midrule
								& 1 
								{\rule{0pt}{3ex}}  		   
								& Monks-3     	   & 122        & 6      & 0.98       
								& 17  & Blood-transfusion  & 748        & 4      & 0.31       	\\
								
								& 2   & Monks-1            & 124        & 6      & 1.01       
								& 18  & Pima-diabetes      & 768        & 8      & 0.54        		\\
								
								& 3   & Monks-2            & 169        & 6      & 0.62       
								& 19  & Mammographic       & 830        & 5      & 0.95      	\\
								
								& 4   & Wpbc               & 194        & 33     & 0.31      
								& 20  & Tic-tac-toe        & 958        & 9      & 0.53      				\\
								
								& 5   & Parkinsons         & 195        & 22     & 3.15      
								& 21  & Statlog-german     & 1,000      & 24     & 2.34      	\\
								
								& 6   & Sonar              & 208        & 60     & 1.16       
								& 22  & Ozone-eight        & 1,847      & 72   & 0.07        \\
								
								& 7   & SPECTF-heart       & 267        & 44     & 3.89       
								& 23  & Ozone-one          & 1,848      & 72   & 0.03       \\
								
								& 8   & StatLog-heart      & 270        & 13     & 0.80       
								& 24  & 20News-talk        & 1,848 	    & 3      & 1.04       \\
								
								& 9   & BUPA-liver         & 345        & 6      & 1.39       
								& 25  & 20News-comp        & 1,937      & 3      & 0.98       \\
								
								& 10  & Ionosphere         & 351        & 34     & 0.56       
								& 26  & 20News-sci         & 1,971      & 3      & 1.01       \\
								
								& 11  & Votes              & 435        & 16     & 1.60       
								& 27  & Spambase           & 4,601      & 57     & 0.65       \\
								
								& 12  & Musk-clean-1       & 476        & 166    & 0.77       
								& 28  & Mushroom           & 5,644      & 22   & 1.62       \\
								
								& 13  & Wdbc               & 569        & 30     & 1.69       
								& 29  & Cod-rna            & 59,535     & 8      & 0.50       \\
								
								& 14  & Credit-app         & 653        & 15     & 1.21       
								& 30  & Ijcnn1             & 141,691    & 22     & 0.11       \\
								
								& 15  & Breast-cancer-W    & 683        & 9      & 0.54       
								& 31  & Skin-nonskin       & 245,057    & 3      & 0.26       \\
								
								& 16  & Statlog-australian & 690        & 14     & 0.81       
								&     &      &     &      &      \\ \bottomrule
							\end{tabular} 
						\end{sc}
					\end{small}
				\end{center}
				\vskip -0.1in \vspace{-0.9cm}
			\end{table*}
			
			\begin{figure*}[!t] 
				\vskip 0.2in
				\begin{center} 
					\centering 
					\subfigure[Linear decision boundary (order 1)]{ 
						\includegraphics[width=0.37\textwidth]{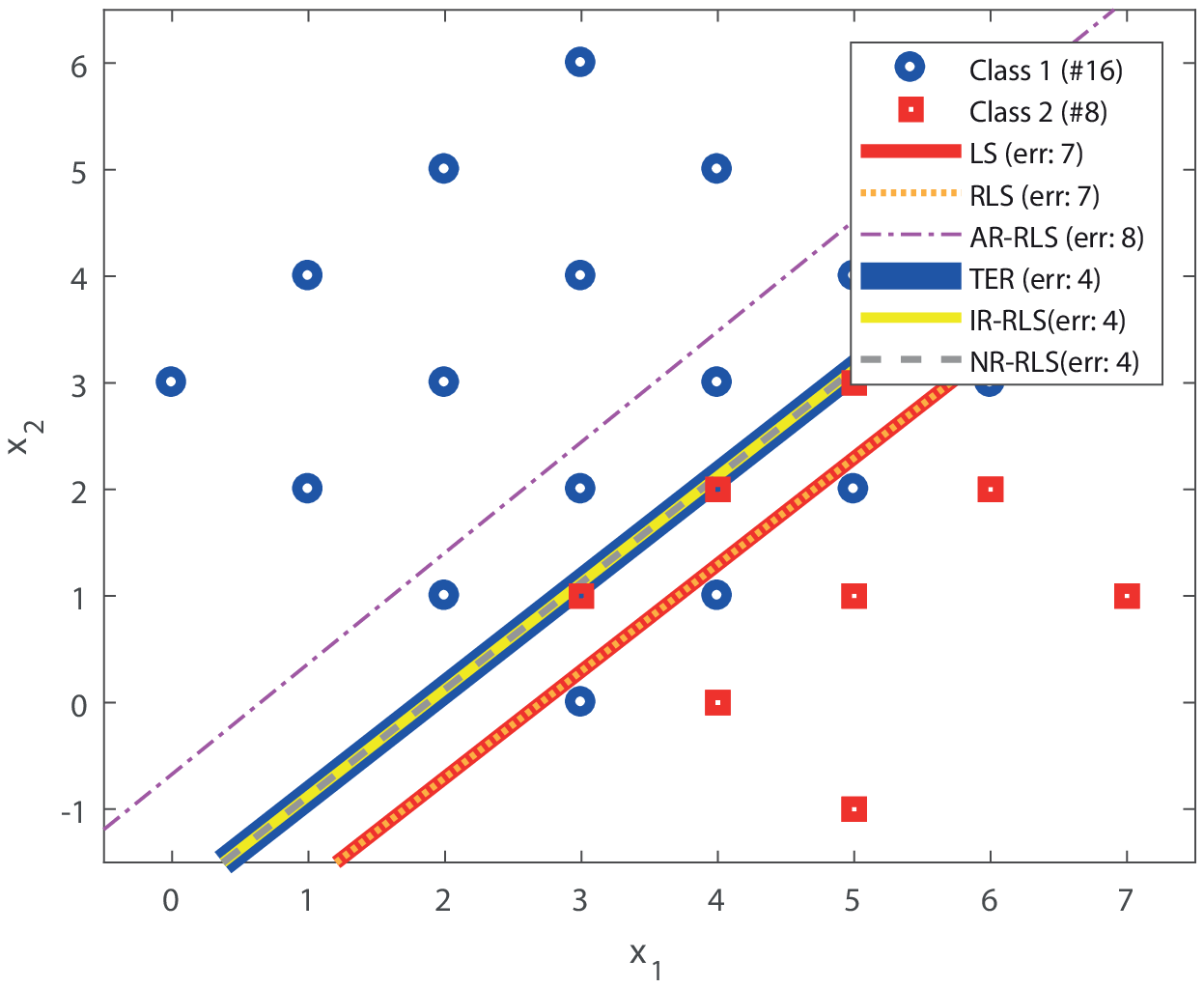}%
					} \hspace{1.5cm}
					\subfigure[Non-linear decision boundary (order 4)]{ 
						\includegraphics[width=0.37\textwidth]{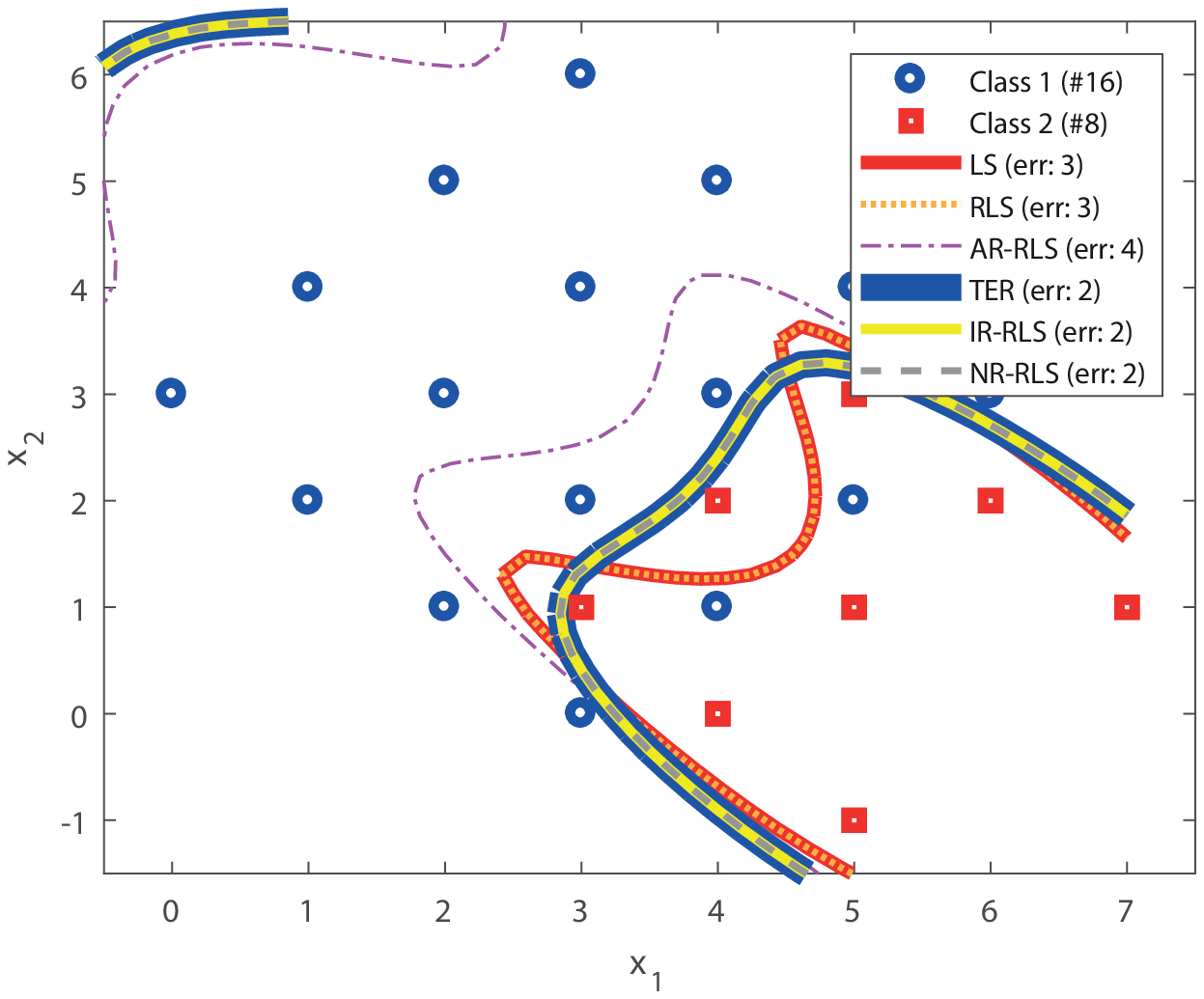}%
					} \vspace{-0.3cm}
					\caption{Decision boundaries of the no weighting based solutions (e.g., LS and RLS) and the class-specific weighting based solutions (e.g., TER, AR-RLS, IR-RLS and NR-RLS) at different polynomial orders: (a) at order 1 and (b) at order 4.}%
					\label{fig.twoDimExample}%
				\end{center}
				\vskip -0.2in
			\end{figure*}

					\section{Experiments} \label{sec.experiments}
					In this section, we perform an empirical evaluation of the proposed Non-iteratively Reweighted Recursive Least-Squares (NR-RLS) based on real-world data sets obtained from 
					the public domain (UCI machine learning repository \cite{Lichman2013} and LIBSVM website \cite{chang2011libsvm}). The data sets are divided into small, medium and large scale data groups to observe the impact of data sizes upon online learning.
					The goals of our experiments are
					(i) to show the impact of the class-specific weights in a synthetic class imbalance data set;
					(ii) to observe the convergence between NR-RLS and TER;
					(iii) to compare the training CPU processing time of NR-RLS with competing algorithms  
					such as RLS, Approximately Reweighted RLS (AR-RLS) and Iteratively Reweighted RLS (IR-RLS);
					(iv) to compare the accuracy performance of  NR-RLS with competing state-of-the-arts such as 
					PErceptron (PE) \cite{rosenblatt1958perceptron}, 
					online Passive-Aggressive learning (PA) \cite{crammer2006online}, 
					online Confident Weighted learning (CW) \cite{dredze2008confidence},
					Adaptive Regularization Of Weights (AROW) \cite{crammer2009adaptive} and
					Adaptive regularized Cost-sensitive Online Gradient descent (ACOG) \cite{zhao2018adaptive}.

					\begin{table*}[t]   \vspace{-0.5cm}
						\caption{Comparison of average G-means and (ranks).} \vspace{-0.7cm}
						\label{tbl.avgacc}
						\vskip 0.2in
						\begin{center}
							\begin{small}
								\begin{sc} \fontsize{6}{6}\selectfont
									\begin{tabular}{l@{\hskip3pt}l@{\hskip3pt}l@{\hskip3pt}l@{\hskip3pt}l@{\hskip3pt}l@{\hskip3pt}l@{\hskip3pt}l@{\hskip3pt}l@{\hskip3pt}}
										\toprule
										\multirow{3}{*}{No.} &   \multicolumn{7}{c}{G-mean $\pm$ std (rank)}  {\rule{0pt}{3ex}}             \\ \cline{2-9} 
										{\rule{0pt}{3ex}} &     \multicolumn{2}{c|}{The First-order}                       &                    \multicolumn{5}{c}{The Second-order}                                                                                 \\ \cline{2-9} 
										{\rule{0pt}{3ex}}                          & {PE}    & {PA}    & {CW}    & {AROW}   & {ACOG}    & {RLS}   & {AR-RLS}  & {NR-RLS and IR-RLS} \\ \midrule
										1   & 0.584   $\pm$ 0.063 (7)  & 0.593   $\pm$ 0.056 (6)  & 0.658   $\pm$ 0.073 (5)  & 0.713   $\pm$ 0.058 (4)  & 0.533   $\pm$ 0.172 (8)  & \textbf{0.771   $\pm$ 0.036 (1.5)} & 0.768   $\pm$ 0.048 (3)  & \textbf{0.771 $\pm$ 0.036} (1.5) \\
										2   & 0.533 $\pm$ 0.069 (7)    & 0.550 $\pm$ 0.063 (6)    & 0.571 $\pm$ 0.079 (5)    & 0.599 $\pm$ 0.040 (4)    & 0.509 $\pm$ 0.121 (8)    & 0.659 $\pm$ 0.065 (2)     & 0.657 $\pm$ 0.058 (3)    & \textbf{0.662 $\pm$ 0.06}1 (1)     \\
										3   & 0.472 $\pm$ 0.046 (5)    & 0.475 $\pm$ 0.046 (4)    & \textbf{0.514 $\pm$ 0.061 (1)}    & 0.246 $\pm$ 0.192 (7)    & 0.218 $\pm$ 0.191 (8)    & 0.330 $\pm$ 0.108 (6)     & 0.492 $\pm$ 0.088 (3)    & 0.511 $\pm$ 0.056 (2)     \\
										4   & 0.440 $\pm$ 0.093 (5)    & 0.440 $\pm$ 0.075 (6)    & 0.522 $\pm$ 0.087 (4)    & 0.285 $\pm$ 0.230 (7)    & 0.114 $\pm$ 0.176 (8)    & 0.639 $\pm$ 0.081 (3)     & 0.687 $\pm$ 0.058 (2)    & \textbf{0.688 $\pm$ 0.047} (1)     \\
										5   & 0.508 $\pm$ 0.058 (6)    & 0.470 $\pm$ 0.045 (7)    & 0.657 $\pm$ 0.044 (4)    & 0.584 $\pm$ 0.103 (5)    & 0.234 $\pm$ 0.254 (8)    & 0.751 $\pm$ 0.067 (3)     & 0.772 $\pm$ 0.055 (2)    & \textbf{0.775 $\pm$ 0.050} (1)     \\
										6   & 0.548 $\pm$ 0.039 (6)    & 0.529 $\pm$ 0.051 (8)    & 0.675 $\pm$ 0.050 (4)    & 0.674 $\pm$ 0.049 (5)    & 0.542 $\pm$ 0.122 (7)    & 0.702 $\pm$ 0.039 (3)     & 0.710 $\pm$ 0.040 (2)    & \textbf{0.729 $\pm$ 0.038} (1)     \\
										7   & 0.423 $\pm$ 0.070 (5)    & 0.407 $\pm$ 0.072 (6)    & 0.540 $\pm$ 0.074 (3)    & 0.077 $\pm$ 0.112 (7)    & 0.046 $\pm$ 0.098 (8)    & 0.450 $\pm$ 0.103 (4)     & 0.630 $\pm$ 0.073 (2)    & \textbf{0.649 $\pm$ 0.055} (1)     \\
										8   & 0.576 $\pm$ 0.033 (7)    & 0.540 $\pm$ 0.031 (8)    & 0.721 $\pm$ 0.039 (5)    & 0.786 $\pm$ 0.029 (4)    & 0.644 $\pm$ 0.176 (6)    & 0.820 $\pm$ 0.028 (3)     & 0.821 $\pm$ 0.026 (2)    & \textbf{0.823 $\pm$ 0.029} (1)     \\
										9   & 0.500 $\pm$ 0.028 (6)    & 0.498 $\pm$ 0.028 (7)    & 0.577 $\pm$ 0.045 (4)    & 0.543 $\pm$ 0.061 (5)    & 0.158 $\pm$ 0.192 (8)    & 0.629 $\pm$ 0.024 (3)     & 0.632 $\pm$ 0.026 (2)    & \textbf{0.639 $\pm$ 0.031} (1)     \\
										10  & 0.532 $\pm$ 0.033 (6)    & 0.509 $\pm$ 0.037 (7)    & 0.762 $\pm$ 0.022 (4)    & 0.731 $\pm$ 0.039 (5)    & 0.494 $\pm$ 0.192 (8)    & 0.786 $\pm$ 0.048 (3)     & 0.795 $\pm$ 0.041 (2)    & \textbf{0.796 $\pm$ 0.040} (1)     \\
										11  & 0.815 $\pm$ 0.025 (8)    & 0.837 $\pm$ 0.025 (7)    & 0.909 $\pm$ 0.023 (5)    & 0.925 $\pm$ 0.014 (4)    & 0.863 $\pm$ 0.079 (6)    & 0.944 $\pm$ 0.011 (3)     & \textbf{0.946 $\pm$ 0.012 (1.5)}  & \textbf{0.946 $\pm$ 0.012} (1.5)     \\
										12  & 0.503 $\pm$ 0.036 (6)    & 0.488 $\pm$ 0.035 (7)    & 0.735 $\pm$ 0.025 (5)    & 0.736 $\pm$ 0.026 (4)    & 0.483 $\pm$ 0.217 (8)    & 0.750 $\pm$ 0.028 (3)     & 0.754 $\pm$ 0.025 (2)    & \textbf{0.806 $\pm$ 0.022} (1)     \\
										13  & 0.767 $\pm$ 0.027 (7)    & 0.718 $\pm$ 0.018 (8)    & 0.939 $\pm$ 0.011 (3)    & 0.934 $\pm$ 0.011 (5)    & 0.768 $\pm$ 0.145 (6)    & 0.938 $\pm$ 0.014 (4)     & 0.951 $\pm$ 0.015 (2)    & \textbf{0.956 $\pm$ 0.011} (1)     \\
										14  & 0.645 $\pm$ 0.058 (7)    & 0.572 $\pm$ 0.028 (8)    & 0.765 $\pm$ 0.023 (5)    & 0.845 $\pm$ 0.015 (4)    & 0.684 $\pm$ 0.168 (6)    & \textbf{0.869 $\pm$ 0.012 (1.5)}   & 0.868 $\pm$ 0.013 (3)    & \textbf{0.869 $\pm$ 0.012} (1.5)   \\
										15  & 0.916 $\pm$ 0.017 (7)    & 0.921 $\pm$ 0.013 (6)    & 0.934 $\pm$ 0.019 (5)    & 0.948 $\pm$ 0.009 (4)    & 0.834 $\pm$ 0.096 (8)    & 0.950 $\pm$ 0.009 (3)     & \textbf{0.957 $\pm$ 0.009 (1)}    & 0.956 $\pm$ 0.009 (2)     \\
										16  & 0.619 $\pm$ 0.038 (7)    & 0.541 $\pm$ 0.030 (8)    & 0.764 $\pm$ 0.022 (5)    & 0.850 $\pm$ 0.013 (4)    & 0.738 $\pm$ 0.115 (6)    & \textbf{0.864 $\pm$ 0.013 (1.5)}   & 0.863 $\pm$ 0.014 (3)    & \textbf{0.864 $\pm$ 0.013} (1.5)   \\
										17  & 0.485 $\pm$ 0.046 (4)    & 0.456 $\pm$ 0.040 (5)    & 0.526 $\pm$ 0.057 (3)    & 0.237 $\pm$ 0.100 (7)    & 0.103 $\pm$ 0.139 (8)    & 0.286 $\pm$ 0.043 (6)     & 0.682 $\pm$ 0.023 (2)    & \textbf{0.685 $\pm$ 0.022} (1)     \\
										18  & 0.555 $\pm$ 0.022 (6)    & 0.502 $\pm$ 0.021 (7)    & 0.625 $\pm$ 0.021 (5)    & 0.667 $\pm$ 0.035 (4)    & 0.435 $\pm$ 0.224 (8)    & 0.698 $\pm$ 0.021 (3)     & 0.739 $\pm$ 0.017 (2)    & \textbf{0.742 $\pm$ 0.014} (1)     \\
										19  & 0.598 $\pm$ 0.019 (7)    & 0.499 $\pm$ 0.031 (8)    & 0.722 $\pm$ 0.023 (5)    & 0.798 $\pm$ 0.010 (4)    & 0.603 $\pm$ 0.218 (6)    & 0.813 $\pm$ 0.019 (2.5)   & \textbf{0.814 $\pm$ 0.017 (1)}    & 0.813 $\pm$ 0.019 (2.5)   \\
										20  & 0.505 $\pm$ 0.025 (5)    & 0.505 $\pm$ 0.017 (4)    & 0.511 $\pm$ 0.024 (3)    & 0.334 $\pm$ 0.093 (7)    & 0.097 $\pm$ 0.137 (8)    & 0.427 $\pm$ 0.049 (6)     & 0.576 $\pm$ 0.016 (2)    & \textbf{0.577 $\pm$ 0.021} (1)     \\
										21  & 0.601 $\pm$ 0.027 (6)    & 0.591 $\pm$ 0.028 (8)    & 0.594 $\pm$ 0.028 (7)    & 0.622 $\pm$ 0.036 (4)    & 0.616 $\pm$ 0.052 (5)    & 0.638 $\pm$ 0.029 (3)     & 0.714 $\pm$ 0.021 (2)    & \textbf{0.715 $\pm$ 0.019} (1)     \\
										22  & 0.364 $\pm$ 0.066 (4)    & 0.296 $\pm$ 0.061 (5)    & 0.571 $\pm$ 0.045 (3)    & 0.072 $\pm$ 0.106 (6)    & 0.000 $\pm$ 0.000 (8)    & 0.007 $\pm$ 0.029 (7)     & 0.807 $\pm$ 0.033 (2)    & \textbf{0.820 $\pm$ 0.018} (1)     \\
										23  & 0.240 $\pm$ 0.104 (4)    & 0.161 $\pm$ 0.101 (5)    & 0.423 $\pm$ 0.100 (3)    & 0.009 $\pm$ 0.041 (6.5)  & 0.009 $\pm$ 0.041 (6.5)  & 0.000 $\pm$ 0.000 (8)     & 0.788 $\pm$ 0.036 (2)    & \textbf{0.819 $\pm$ 0.035} (1)     \\
										24  & \textbf{0.500 $\pm$ 0.014 (1)}    & 0.498 $\pm$ 0.017 (2)    & 0.481 $\pm$ 0.070 (4)    & 0.456 $\pm$ 0.042 (5)    & 0.326 $\pm$ 0.171 (8)    & 0.359 $\pm$ 0.194 (7)     & 0.374 $\pm$ 0.178 (6)    & 0.483 $\pm$ 0.016 (3)     \\
										25  & 0.574 $\pm$ 0.013 (6)    & 0.548 $\pm$ 0.013 (7)    & 0.477 $\pm$ 0.096 (8)    & 0.679 $\pm$ 0.037 (4)    & 0.633 $\pm$ 0.074 (5)    & \textbf{0.695 $\pm$ 0.039 (1.5)}   & 0.688 $\pm$ 0.050 (3)    & \textbf{0.695 $\pm$ 0.039} (1.5)   \\
										26  & 0.714 $\pm$ 0.014 (7)    & 0.689 $\pm$ 0.012 (8)    & 0.722 $\pm$ 0.065 (6)    & 0.850 $\pm$ 0.015 (4)    & 0.745 $\pm$ 0.055 (5)    & \textbf{0.892 $\pm$ 0.017 (1.5)}   & 0.877 $\pm$ 0.032 (3)    & \textbf{0.892 $\pm$ 0.017} (1.5)   \\
										27  & 0.610 $\pm$ 0.036 (8)    & 0.630 $\pm$ 0.008 (7)    & 0.865 $\pm$ 0.007 (4)    & 0.869 $\pm$ 0.006 (3)    & 0.671 $\pm$ 0.119 (6)    & 0.864 $\pm$ 0.010 (5)     & 0.899 $\pm$ 0.007 (2)    & \textbf{0.900 $\pm$ 0.006} (1)     \\
										28  & 0.679 $\pm$ 0.022 (7)    & 0.510 $\pm$ 0.010 (8)    & \textbf{0.984 $\pm$ 0.001 (1)}    & 0.916 $\pm$ 0.009 (5)    & 0.800 $\pm$ 0.081 (6)    & 0.932 $\pm$ 0.009 (4)     & 0.949 $\pm$ 0.004 (3)    & 0.950 $\pm$ 0.003 (2)     \\
										29  & 0.763 $\pm$ 0.002 (6)    & 0.718 $\pm$ 0.005 (7)    & 0.863 $\pm$ 0.002 (5)    & 0.930 $\pm$ 0.001 (3)    & 0.714 $\pm$ 0.304 (8)    & 0.928 $\pm$ 0.001 (4)     & 0.939 $\pm$ 0.001 (2)    & \textbf{0.940 $\pm$ 0.001} (1)     \\
										30  & 0.631 $\pm$ 0.004 (4)    & 0.627 $\pm$ 0.005 (5)    & 0.692 $\pm$ 0.004 (3)    & 0.508 $\pm$ 0.009 (6)    & 0.173 $\pm$ 0.085 (8)    & 0.301 $\pm$ 0.007 (7)     & 0.858 $\pm$ 0.004 (2)    & \textbf{0.859 $\pm$ 0.004} (1)     \\
										31  & 0.768 $\pm$ 0.002 (7)    & 0.783 $\pm$ 0.002 (6)    & 0.454 $\pm$ 0.015 (8)    & 0.899 $\pm$ 0.001 (4)    & 0.784 $\pm$ 0.342 (5)    & 0.904 $\pm$ 0.001 (3)     & 0.957 $\pm$ 0.000 (2)    & \textbf{0.958 $\pm$ 0.000} (1)     \\ Avg.
										& 0.580 $\pm$ 0.037 (5.94) & 0.552 $\pm$ 0.033 (6.48) & 0.670 $\pm$ 0.040 (4.35) & 0.623 $\pm$ 0.050 (4.85) & 0.470 $\pm$ 0.147 (7.02) & 0.664 $\pm$ 0.037 (3.74)  & 0.773 $\pm$ 0.034 (2.29) & \textbf{0.783 $\pm$ 0.024 (1.32)}  \\ \bottomrule 
									\end{tabular}
								\end{sc}
							\end{small}
						\end{center}
						\vskip -0.2in \vspace{-0.7cm}
					\end{table*}

					\begin{figure*}[!t] 
						\vskip 0.2in
						\begin{center} 
							\subfigure[The $L_2$-norm values
							]{ 
								\includegraphics[width=0.3\textwidth]{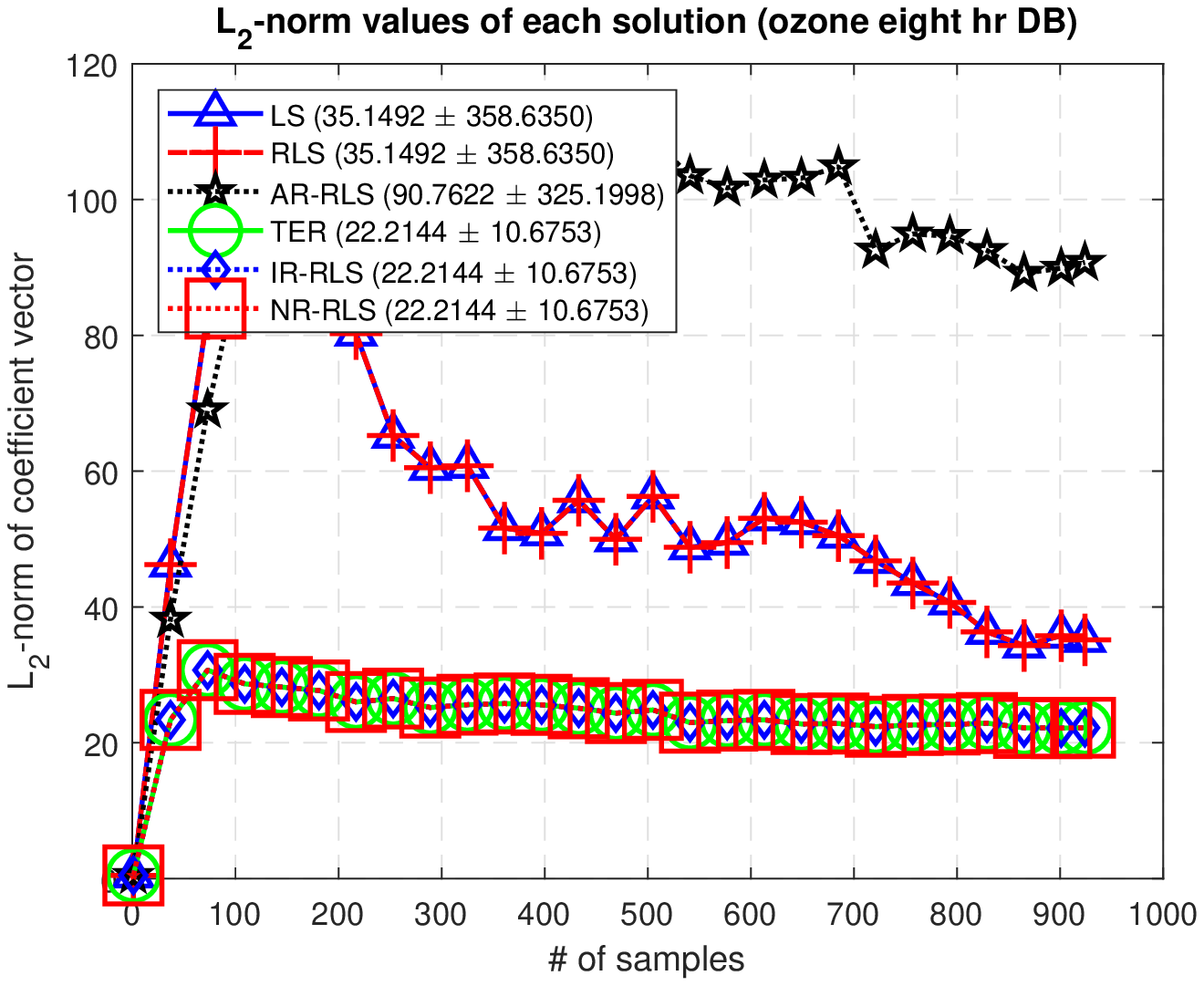}%
							} \hspace{-0.5cm}
							\subfigure[The G-means
							]{ 
								\includegraphics[width=0.3\textwidth]{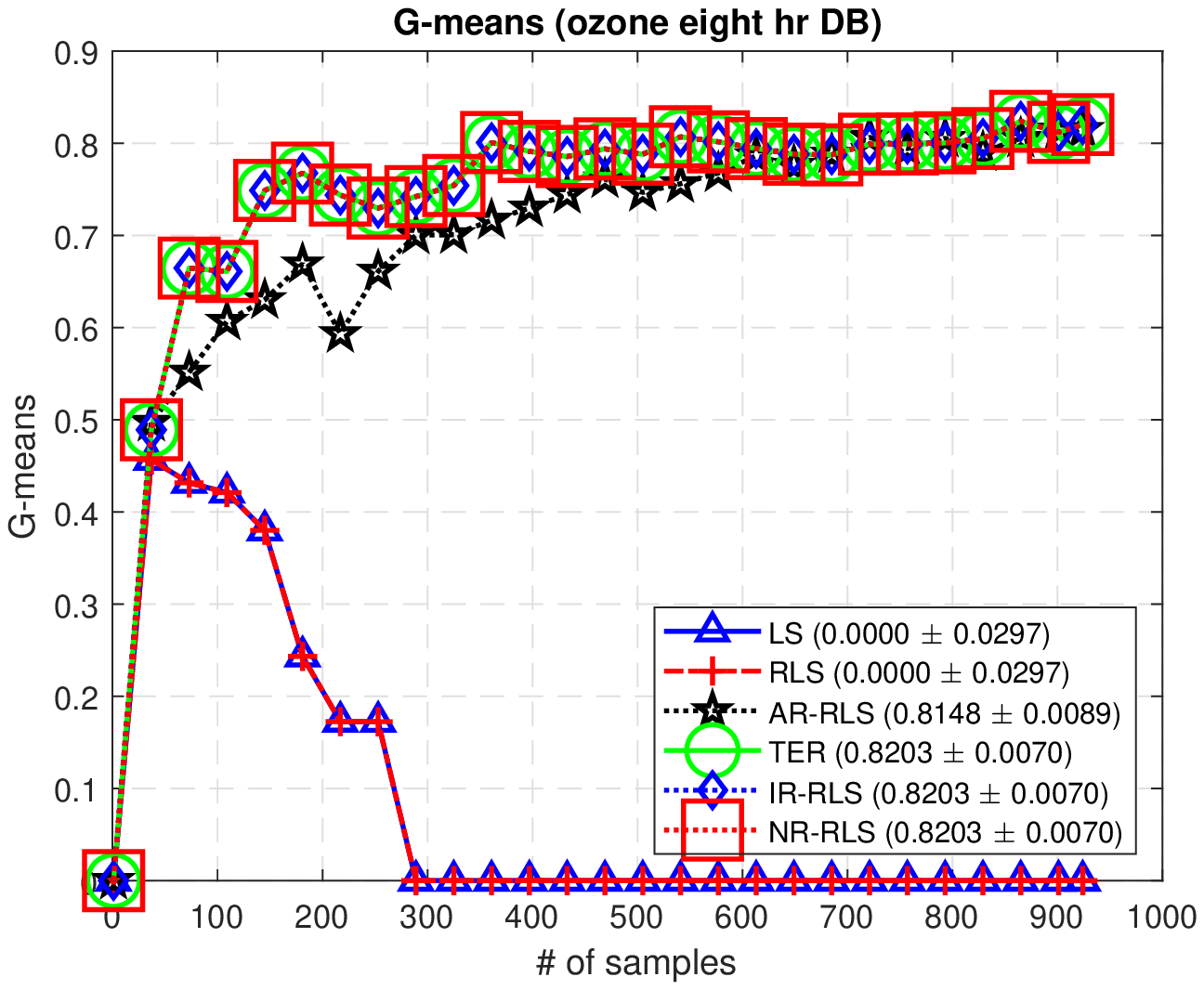}%
							} \hspace{-0.5cm}
							\subfigure[The CPU times
							]{ 
								\includegraphics[width=0.3\textwidth]{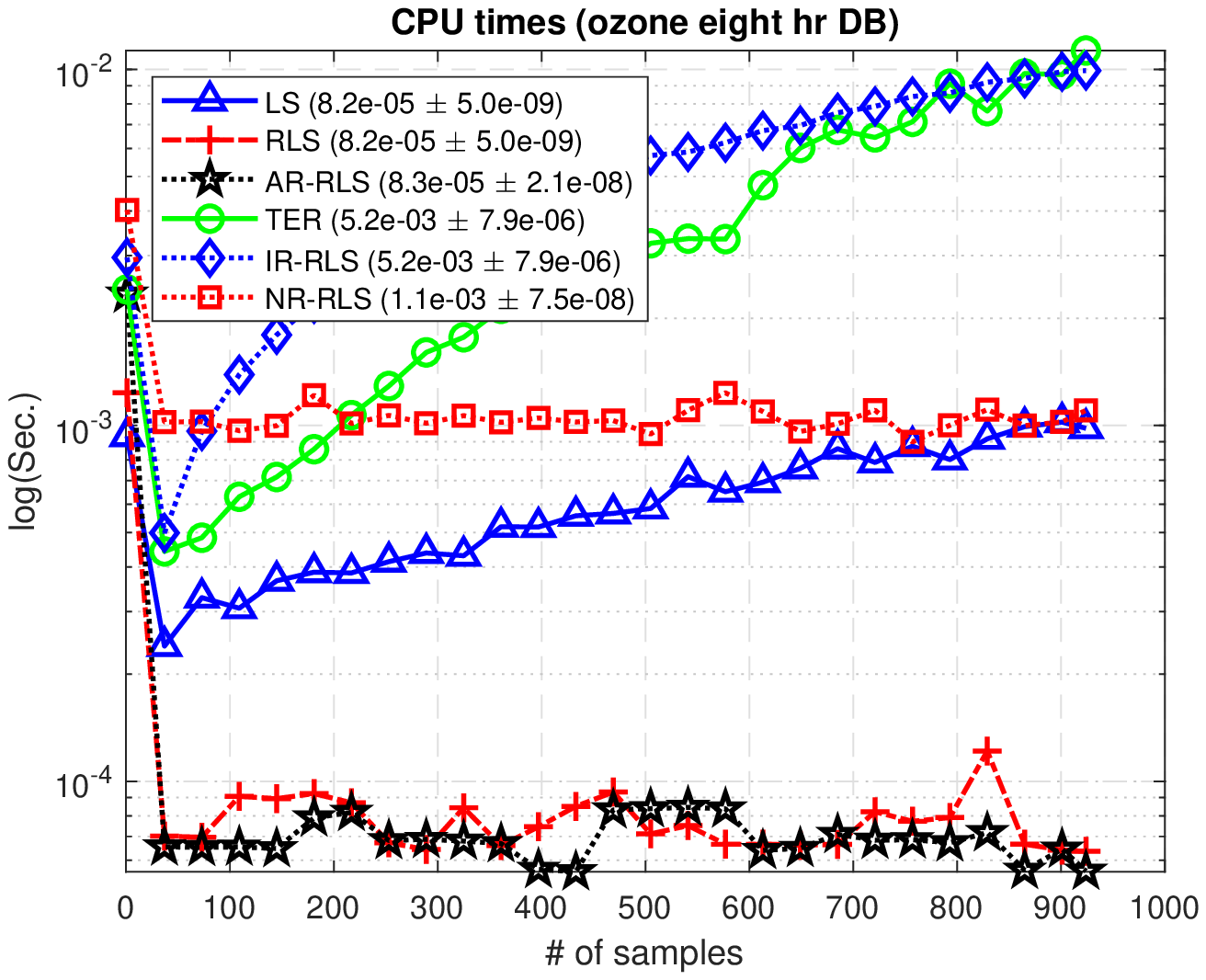}%
							} \vspace{-0.3cm}
							\caption{Comparisons of (a) the $L_2$-norm values, 
								(b) the G-means and
								(c) the CPU times among LS, RLS, TER, AR-RLS, IR-RLS and NR-RLS plotted over different  number of training samples for the `Ozone-eight' data set. Each bracket in legends indicates a mean value and its standard deviation over the number of samples.}%
							\label{fig.l2norm_TER_acc_TER_time_TER}%
						\end{center}
						\vskip -0.2in
					\end{figure*} 
					
					\subsection{Data sets and experimental setup}
					In our experiments, twenty small scale data sets ($<1,000$ samples) are taken from 
					the UCI machine learning repository \cite{Lichman2013}.
					Additionally, eight medium scale data sets ($<10,000$ samples) consist of five data sets from the UCI machine learning repository and three data sets from the 20 Newsgroups which are popular in the NLP community. 
					The sample size of these data set ranges from 122 to 245,057 samples.
					The data imbalance ratio is calculated by ${n^+}\over{n^-}$.
					In addition to these two groups of data, 
					three large scale data sets    from the LIBSVM website \cite{chang2011libsvm} 
					are included in this study. 
					In \cite{lu2016large, jian2017budget, hu2015dependent, ito2017unified}, these data sets are considered as the large scale data sets ($> 50,000$ samples).
					The input data is normalized to the range $[0, 1]$.
					Table \ref{tbl.dataSummary}  summarizes the attributes of the total 31 data sets used in our study. For performance comparison with  representative online algorithms, our proposed NR-RLS is compared with 
					PE \cite{rosenblatt1958perceptron}, 
					PA \cite{crammer2006online}, 
					CW \cite{dredze2008confidence}, 
					AROW \cite{crammer2009adaptive},
					ACOG \cite{zhao2018adaptive},
					RLS \cite{haykin2013adaptive}, 
					AR-RLS  \cite{kim2011onnet} and 
					IR-RLS \cite{jang2017online}.

					Similar to \cite{kim2012online}, our experiments are recorded using ten runs of 2-fold cross-validations for all compared algorithms. 
					We adopt a $G$-mean matric \cite{he2009learning} which evaluates the degree of inductive bias in terms of a ratio of the positive and
					negative accuracies as follows: $ G{\text{-mean}} = \sqrt {\frac{{TP}}{{TP + FN}} \times \frac{{TN}}{{TN + FP}}}$.
					The average $G$-means  for the unseen test data are recorded. 
					In the class-specific weighting based classifiers, namely TER, AR-RLS, IR-RLS and NR-RLS,
					there are two parameters: the class-specific weight and the regularization factor $b$. 
					The RLS classifier has one parameter $b$. 
					Here, the regularization setting $b$ is set at a small value of $b=10^{-4}$ following \cite{kim2012online, toh2014exploiting} since the main objective of this setting is to stabilize the least-squares solution.
					The class-specific weight is varied according to \cite{toh2008between}.
					By setting $n_t^ -  = 1 $ and $n_t^ +  = 1$ without its update, the proposed NR-RLS can take the balanced weight setting, heading to the LS objective.
					To address the nonlinear input-output relationship,
					the multivariate polynomial model  is adopted at different polynomial orders $r \in \left\{ {1,2, \ldots ,6} \right\}$.

					\subsection{Comprehensive analysis of the no weighting and the class-specific reweighting based solutions}
					In order to observe the effectiveness of the class-specific reweighting for binary class rebalancing, 
					the class-specific reweighting based TER, AR-RLS, IR-RLS and NR-RLS are compared with 
					the no weighting based LS and RLS.
					
					\subsubsection{The impact of class-specific weights in class imbalance learning  }
					In order to observe the difference between LS and TER in a class imbalance problem, 
					Fig. \ref{fig.twoDimExample} illustrates the decision boundaries on a synthetic example consisting of 24  unbalanced discrete data points (i.e., 8 negative samples and 16 positive samples) with 7 overlapping data points.
					The decision boundaries are drawn in two different order polynomials
					(e.g., the first and fourth orders).
					Fig. \ref{fig.twoDimExample}(a) shows 7 error counts for `LS and RLS', 4 error counts for `TER, IR-RLS and
					NR-RLS' and 8 error counts for `AR-RLS' in the first order polynomial. 
					The exact reweighting based TER, IR-RLS and NR-RLS achieve lower error counts than the no weighting based solutions and the approximate reweighting based AR-RLS. 
					Fig. \ref{fig.twoDimExample}(b) also shows the error counts in the fourth-order polynomial.
					The error counts for `LS and RLS', `TER, IR-RLS and NR-RLS' and `AR-RLS' are given by 3, 2 and 4. 
					Similarly, the exact reweighting based TER, IR-RLS and NR-RLS are seen to be the best performer compared to all the other solutions.
					In Fig. \ref{fig.twoDimExample}, the equality  between the batch and online settings can be found in (i) LS and RLS, (ii) TER, IR-RLS and NR-RLS, whereas the inequality remains between TER and AR-RLS.
					
					\subsubsection{Observing the convergence trends} \label{sec.convergence}
				In order to observe the convergence trends of the batch and recursive formulations using real-world data, 
				the `Ozone-eight' data set from the UCI repository \cite{Lichman2013} is adopted for this investigation. 
				Fig. \ref{fig.l2norm_TER_acc_TER_time_TER}(a) shows the $L_2$-norm of the coefficient vectors for each algorithm. 
				In this figure, the average $L_2$-norm values show that 
				IR-RLS and NR-RLS converge to the batch setting of TER
				while AR-RLS shows a different convergence. 
				Similarly, RLS shows the same convergence with the batch setting of LS.
				In the  Appendix \ref{sec.B}, we give a figure which includes each value of the coefficient vectors.
				
				Fig. \ref{fig.l2norm_TER_acc_TER_time_TER}(b) shows the estimation trends of the compared algorithms according to each arrival of training samples. 
				We observe that 
				each group of 
				`TER, IR-RLS and NR-RLS' and `LS and RLS' achieves the same $G$-mean and the same standard deviation, whereas AR-RLS shows the different $G$-mean and standard deviation from the exact reweighting group.
				This  also verifies that NR-RLS is converged to TER. 
				Due to the exact convergence of the TER objective, 
				IR-RLS and NR-RLS show a better $G$-mean performance than 
				the approximate reweighting based AR-RLS. 
				Since the `Ozone-eight' data set is highly imbalanced with its ratio, 0.07, shown in Table \ref{tbl.dataSummary}, the no weighting based LS and RLS cannot classify one class completely.
				In Fig. \ref{fig.l2norm_TER_acc_TER_time_TER}(a) and (b), the exact reweighting based TER, IR-RLS, and NR-RLS show a lower standard deviation than LS, RLS, and AR-RLS.


				\begin{table*}[t] \vspace{-0.5cm}

					\caption{Comparison of average training CPU times in seconds} \vspace{-0.3cm}
					
					\label{tbl.cputime}
					\vskip 0.15in
					\begin{center}
						\begin{small}
							\begin{sc} \fontsize{6.7}{4}\selectfont
								\begin{tabular}{lrrrrrrrr}
									\toprule
									\multirow{3}{*}{No.}  & \multicolumn{8}{l}{Training CPU times in seconds}  {\rule{0pt}{3ex}}                                                                                                                                         \\ \cline{2-9} 
									{\rule{0pt}{3ex}} 
									&  \multicolumn{2}{c|}{The First-order}          & \multicolumn{6}{c}{The Second-order }                                                                                  \\ \cline{2-9} 
									{\rule{0pt}{3ex}}
									& \multicolumn{1}{c}{PE} & \multicolumn{1}{c}{PA} & \multicolumn{1}{c}{CW} & \multicolumn{1}{c}{AROW} & \multicolumn{1}{c}{RLS} & \multicolumn{1}{c}{AR-RLS} & \multicolumn{1}{c}{IR-RLS} & \multicolumn{1}{c}{NR-RLS} \\ \midrule
									1 {\rule{0pt}{3ex}}
									& 0.0052 & 0.0026 & 0.0021 & 0.0024 & 0.0053 & 0.0099 & 0.0388     & 0.0254 \\
									2                    & 0.0014 & 0.0013 & 0.0020 & 0.0023 & 0.0007 & 0.0007 & 0.0170     & 0.0068 \\
									3                    & 0.0017 & 0.0019 & 0.0028 & 0.0034 & 0.0009 & 0.0010 & 0.0265     & 0.0036 \\
									4                    & 0.0021 & 0.0021 & 0.0039 & 0.0046 & 0.0016 & 0.0018 & 0.0658     & 0.0123 \\
									5                    & 0.0019 & 0.0021 & 0.0033 & 0.0040 & 0.0014 & 0.0017 & 0.0539     & 0.0087 \\
									6                    & 0.0023 & 0.0022 & 0.0063 & 0.0071 & 0.0035 & 0.0045 & 0.2497     & 0.0382 \\
									7                    & 0.0035 & 0.0034 & 0.0080 & 0.0110 & 0.0036 & 0.0038 & 0.2218     & 0.0307 \\
									8                    & 0.0028 & 0.0030 & 0.0046 & 0.0053 & 0.0014 & 0.0016 & 0.0785     & 0.0096 \\
									9                    & 0.0037 & 0.0038 & 0.0058 & 0.0071 & 0.0018 & 0.0020 & 0.1019     & 0.0081 \\
									10                   & 0.0035 & 0.0041 & 0.0061 & 0.0078 & 0.0031 & 0.0034 & 0.2303     & 0.0347 \\
									11                   & 0.0051 & 0.0050 & 0.0067 & 0.0086 & 0.0027 & 0.0028 & 0.1770     & 0.0126 \\
									12                   & 0.0057 & 0.0084 & 0.0619 & 0.0725 & 0.0433 & 0.0438 & 5.5455     & 0.4465 \\
									13                   & 0.0050 & 0.0058 & 0.0079 & 0.0104 & 0.0046 & 0.0115 & 0.4944     & 0.0305 \\
									14                   & 0.0074 & 0.0079 & 0.0096 & 0.0131 & 0.0034 & 0.0039 & 0.3622     & 0.0160 \\
									15                   & 0.0064 & 0.0064 & 0.0085 & 0.0105 & 0.0037 & 0.0041 & 0.3938     & 0.0139 \\
									16                   & 0.0064 & 0.0074 & 0.0099 & 0.0126 & 0.0038 & 0.0053 & 0.4389     & 0.0190 \\
									17                   & 0.0073 & 0.0074 & 0.0126 & 0.0139 & 0.0035 & 0.0041 & 0.4350     & 0.0142 \\
									18                   & 0.0074 & 0.0088 & 0.0109 & 0.0154 & 0.0050 & 0.0044 & 0.5167     & 0.0152 \\
									19                   & 0.0083 & 0.0083 & 0.0111 & 0.0152 & 0.0040 & 0.0049 & 0.5469     & 0.0145 \\
									20                   & 0.0095 & 0.0096 & 0.0149 & 0.0178 & 0.0052 & 0.0058 & 0.7547     & 0.0215 \\
									21                   & 0.0109 & 0.0111 & 0.0233 & 0.0267 & 0.1160 & 0.1293 & 6.1132     & 1.5705 \\
									22                   & 0.0164 & 0.0173 & 0.0398 & 0.0599 & 0.0503 & 0.0444 & 17.9072    & 0.4350 \\
									23                   & 0.0168 & 0.0163 & 0.0344 & 0.0645 & 0.0437 & 0.0448 & 17.5479    & 0.4225 \\
									{24}                   & { 0.0211} & { 0.0212} & { 0.0364} & { 0.0324} & { 0.0095} & { 0.0121} & { 2.1149}     & { 0.0672} \\
									{25}                   & { 0.0204} & { 0.0218} & { 0.031} & { 0.0336} & { 0.0136} & { 0.0184} & { 1.9262}     & { 0.0592} \\
									{26}                   & { 0.0216} & { 0.0219} & { 0.0302} & { 0.0339} & { 0.0103} & { 0.0133} & { 1.9603}     & { 0.0571} \\
									27                   & 0.0417 & 0.0493 & 0.0924 & 0.1331 & 0.0816 & 0.0891 & 94.5412    & 0.7049 \\
									28                   & 0.0572 & 0.0610 & 0.0641 & 0.1220 & 0.0580 & 0.0668 & 37.2190    & 0.2729 \\
									29                   & 0.5827 & 0.6056 & 0.7635 & 1.0081 & 0.6249 & 0.6699 & 2402.4486  & 1.6872 \\
									30                   & 1.3512 & 1.3686 & 1.9197 & 2.4948 & 2.5199 & 3.0872 & 23977.8040 & 8.3815 \\
									31                   & 2.2627 & 2.4008 & 2.9559 & 4.3889 & 3.1734 & 3.4330 & 39724.2624 & 6.4619 \\
									{\rule{0pt}{3ex}}
									Avg                & { 0.1451} & { 0.1515} & { 0.1997} & { 0.2788} & { 0.2195} & { 0.2493} & { 2138.5353}  & { 0.6742} \\ \bottomrule
								\end{tabular}
							\end{sc}
						\end{small}
					\end{center}
					\vskip -0.1in \vspace{-0.9cm}
				\end{table*} 
				
				\begin{figure*}[!t] 
					\vskip 0.2in
					\begin{center} 
						\subfigure[G-mean rank]{ 
							\includegraphics[width=0.42\textwidth]{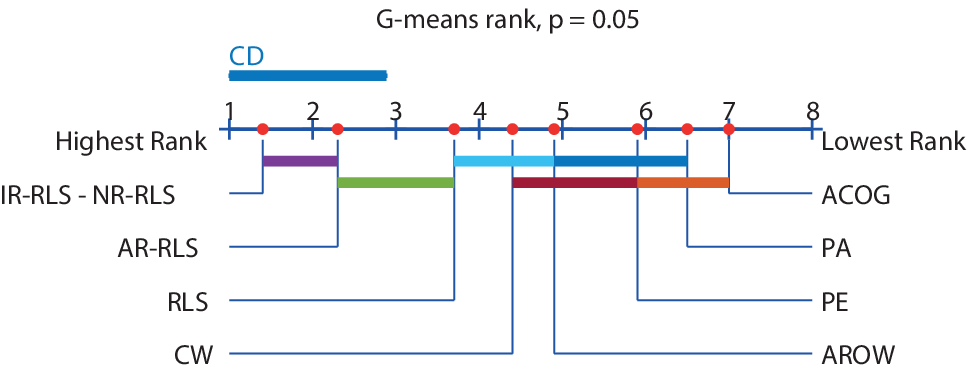}%
						} \hspace{0.9cm}
						\subfigure[Computational rank]{ 
							\includegraphics[width=0.42\textwidth]{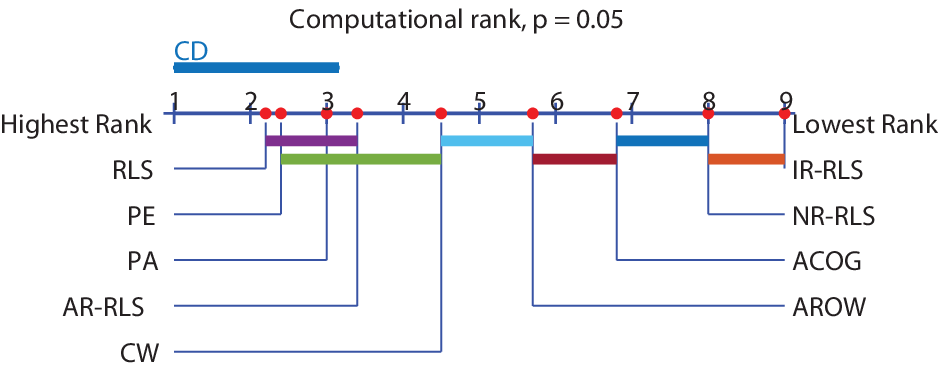}%
						}  \vspace{-0.3cm}
						\caption{Statistical significance among the averaged (a) $G$-means and (b) computational times of the online algorithms according to the Nemenyi test. The connected algorithms by the Critical Difference (CD) are those that their differences in performance are of no statistical significance.}%
						\label{fig.nemenyi_acc_time}%
					\end{center}
					\vskip -0.2in
				\end{figure*} 
				
				\subsubsection{Comparing the CPU processing time}
				The `Ozone-eight' data set is again used for comparing the training computational time. 
				Fig. \ref{fig.l2norm_TER_acc_TER_time_TER}(c) shows the average CPU times of LS, RLS, TER, AR-RLS, IR-RLS, and NR-RLS over 10 runs along with each arriving training sample. 
				Due to the batch mode and iterative reweighting settings, 
				the computational times of LS, TER and IR-RLS are linearly growing according to the increasing number of samples. 
				Different from LS, TER, and IR-RLS, 
				the non-growing computational times of RLS, AR-RLS, and NR-RLS are observed.
				Similar to Fig. \ref{fig.l2norm_TER_acc_TER_time_TER}, we provide all the corresponding figures for each data set in the  Appendix \ref{sec.E}.
				
				\subsection{Performance evaluation}
				\subsubsection{Comparing $G$-mean and CPU time performances among the state-of-the-arts}
				Table \ref{tbl.avgacc}  shows the average $G$-mean results  
				with its standard deviation and ranks of the 31 data sets. 
				In the  Appendix \ref{sec.D}, highly imbalanced data sets (e.g., Ozone and IJCNN) are highlighted.
				The $G$-mean and CPU time values are recorded based on 10 runs of 2-fold cross-validation.
				As competing state-of-the-arts, two first-order algorithms such as 
				PE \cite{rosenblatt1958perceptron} and 
				PA \cite{crammer2006online} are included, 
				and five second-order algorithms, namely 
				CW \cite{dredze2008confidence}, 
				AROW \cite{crammer2009adaptive},
				ACOG \cite{zhao2018adaptive},
				RLS \cite{haykin2013adaptive}, 
				AR-RLS  \cite{kim2011onnet}, and 
				IR-RLS \cite{jang2017online} are adopted.

				Several observations are gathered for this experiment. 
				Firstly, the second-order solutions show higher average $G$-means than the first-order solutions. This is due to the use of more information than the first-order information in the update rules. This consistent observation can also be found in \cite{dredze2008confidence, crammer2009adaptive, hao2018second}.
				Secondly, the exact reweighting based solutions namely, NR-RLS and IR-RLS show the highest average $G$-mean performance than all the other solutions.
				The main reason is due to the direct optimization of the classification error goal 
				with the class imbalance classification design.
				Additionally, NR-RLS and IR-RLS give the lowest standard deviation than all the state-of-the-art methods. 

				Table \ref{tbl.cputime} shows the average CPU time on the 31 data sets. 
				The computational time of NR-RLS is much faster than IR-RLS. The main reason is the replacement of the iterative reweighting with the single-step vectorized reweighting.
				The computational times of the first-order algorithms are seen to be faster than the second-order methods. 
				
				Friedman tests (see \cite{demvsar2006statistical}) on the $G$-mean and CPU time comparisons reject the null hypothesis that all eight compared algorithms are statistically equivalent. 
				These are followed by Nemenyi plots as a post-hoc analysis to show the groups of connected algorithms that are not significantly different at $p = 0.05$. 
				In Fig. \ref{fig.nemenyi_acc_time}(a), the Nemenyi plot for the $G$-mean rank shows six groups of algorithm similarity 
				namely,  
				(i) IR-RLS--NR-RLS--AR-RLS, 
				(ii) AR-RLS--RLS,
				(iii) RLS--CW--AROW,
				(iv) CW--AROW--PE,
				(v) AROW--PE--PA, and 
				(vi) PE--PA--ACOG. 
				In the first group, NR-RLS and IR-RLS achieve the highest rank and significantly differ from all the other algorithms in the lowly-ranked groups.  
				NR-RLS and IR-RLS show the higher $G$-mean performance than all the other algorithms. 
				Since NR-RLS--IR-RLS is not overlapped between the first and second groups, NR-RLS--IR-RLS are seen to be  the best performer.
				In Fig. \ref{fig.nemenyi_acc_time}(b), the Nemenyi plot for the computational rank also shows six groups of algorithm similarity, namely
				(i) RLS--PE--PA--AR-RLS, 
				(ii) PE--PA--AR-RLS--CW, 
				(iii) CW--AROW, 
				(iv) AROW--ACOG, 
				(v) ACOG--NR-RLS, and 
				(vi) NR-RLS--IR-RLS. 
				The proposed NR-RLS is overlapped between the two lowly-ranked   groups. 

				%
				
				Here, we summarize our observations: (i) In terms of the $G$-mean performance, 
				the equivalence between TER and NR-RLS is observed.
				The exact reweighting based  NR-RLS outperformed each of the approximate reweighting based AR-RLS, the no weighting based RLS, and the stochastic based solutions such as PE, PA, CW, AROW, and ACOG  with statistical significance.
				(ii) In terms of the computational time, 
				the non-growing computational trend of NR-RLS is observed 
				whereas IR-RLS showed the growing computational trend.
				NR-RLS is seen to be slower than the first-order and the second-order solutions but comparable to them.
				%
				%
				
				\section{Conclusion}
				This paper presented a new deterministic online learning formulation of the weighted least-squares for binary class rebalancing. 
				Specifically, we proposed a non-iteratively reweighted recursive least-squares algorithm which is designed to replace the old weights with the new ones. 
				We showed that the proposed online formulation converged to the batch setting for binary class imbalance classification and achieved the constant time complexity. 
				We also showed that the proposed algorithm
				outperformed the state-of-the-art online binary classification algorithms effectively and efficiently.
				In the future, we will extend this formulation to nonlinear classifiers in a reproducing kernel Hilbert space.
				%
				

				\bibliography{references.bib}
				\bibliographystyle{icml2023}

				\newpage
				\appendix
				\onecolumn
				\section*{Appendix}  
				Our appendices contain additional details which are omitted from the main text.
				
				In Appendix \ref{sec.A}, 
				we introduce the relationship between the optimal Bayes classifier and the Least-Squares (LS) classifier.
				We also show the weighted version of the optimal Bayes classifier and then build the relationship among the weighted optimal Bayes classifier, the Weighted Least-Squares (WLS) classifier and the Total-Error-Rate (TER) classifier.

				In Appendix \ref{sec.C}, 
				we show that the proposed NR-RLS can easily be extended to the multiclass classification.
				
				In Appendix \ref{sec.B}, 
				we show the learned coefficient vectors of LS, Recursive LS (RLS), TER, Approiximately Reweighted RLS (AR-RLS), Iteratively Reweighted RLS (IR-RLS), Non-iteratively Reweighted RLS (NR-RLS) to experimentally provide the convergence results between the batch and online settings on the Mushroom data set which has enough numbers of feature dimension and data samples to show a good presentation.
				
				In Appendix \ref{sec.E}, 
				we provide all the figures which show $L_2$-norm, $G$-mean and CPU time values for each data set. In the main text, we only showed the figure for the ozone-eight data set as a representative example.
				
				In Appendix \ref{sec.D}, 
				we highlight the data sets highly imbalanced.

				\section{ Relationship with the optimal Bayes classifier}  \label{sec.A}
				Consider a binary classification problem with 
				a finite set of observations $\left\{ {{{\bf{x}}_i},{y_i}} \right\}_i^n$, 
				where $\bf x$ and $y$ are 
				the input and the output that randomly sampled according to a distribution $p$ over $\mathcal{X} \times \{-1, 1 \}$.
				The overall classification error can be minimized by 
				the optimal Bayes classifier as follows:
				\begin{equation} \label{eq.bayes}
					f_{Bayes}^ *  = \mathop {\arg \min }\limits_{f_{Bayes} :\mathcal{X} \to \left\{ { - 1,1} \right\}}  \int_{\mathcal{X} \times \{-1, 1 \}} {{\bf 1}\left( {f_{Bayes}  \left( {\bf{x}} \right) - y} \right)} dp \left( {{\bf{x}},y} \right),
				\end{equation}
				where ${\bf{1}}\left(  \cdot  \right): \mathbb{R} \to \left\{ { 0,1} \right\}$ is the binary function.
				The optimal Bayes classifier satisfies the following equation:
				\begin{equation} \label{eq.bayes_bin}
					f_{Bayes}^ * \left( {\bf{x}} \right) = \left\{ {\begin{array}{*{20}{c}}
							1&{\text{if } p \left( {1|{\bf{x}}} \right) > p \left( { - 1|{\bf{x}}} \right)}\\
							{ - 1}&{{\rm{otherwise}}}
					\end{array}} \right. .
				\end{equation}
				Since large scale data sets are needed for good estimation of $p\left( { y|{\bf{x}}} \right)$, 
				a good surrogate method is required for a good feasible solution in practice.
				The Least-Squares (LS) minimization is a well-known method to asymptotically
				recover the optimal Bayes classifier  as follows:
				\begin{equation} \label{eq.LS}
					f_{LS}^ *  = \mathop {\arg \min }\limits_{f_{LS} :\mathcal{X} \to \mathbb{R}} \int_{\mathcal{X} \times \{-1, 1 \}} {{{\left( {y - f_{LS} \left( {\bf{x}} \right)} \right)}^2}} dp \left( {{\bf{x}},y} \right).
				\end{equation}
				Then, we  have
				\begin{equation}
					\begin{aligned} 
						\int {{{\left( {y - f_{LS} \left( {\bf{x}} \right)} \right)}^2}dp \left( {{\bf{x}},y} \right)} 
						= \int {\int {{{\left( {y - f_{LS} \left( {\bf{x}} \right)} \right)}^2}dp \left( {y|{\bf{x}}} \right)dp \left( {\bf{x}} \right)} } \\
						= \int {\left[ {{{\left( {1 - f_{LS} \left( {\bf{x}} \right)} \right)}^2}p \left( {1|{\bf{x}}} \right) + {{\left( {f_{LS} \left( {\bf{x}} \right) + 1} \right)}^2}p \left( { - 1|{\bf{x}}} \right)} \right]dp \left( {\bf{x}} \right)} ,
					\end{aligned} 
				\end{equation}
				which implies that the minimizer of the equation \eqref{eq.LS} satisfies
				\begin{equation}
					f_{LS}^ * \left( {\bf{x}} \right) = 2p \left( {1|{\bf{x}}} \right) - 1 = p \left( {1|{\bf{x}}} \right) - p \left( { - 1|{\bf{x}}} \right) .
				\end{equation}
				The optimal Bayes classifier can be recovered
				by:
				$f_{Bayes}^ *\left({\bf x}\right)   = sign \left( f_{LS}^ *\left({\bf x}\right) \right)$.
				Indeed, $f_{LS}^ * > 0 $ if and only if $p \left( {1|{\bf{x}}} \right) > p \left( { - 1|{\bf{x}}} \right)$.
				
				Similar to the equation \eqref{eq.bayes}, a weighted version of the optimal Bayes classifier for binary class rebalancing can be defined as:
				\begin{equation} \label{eq.wbayes}
					f_{wBayes}^ *  = \mathop {\arg \min }\limits_{f_{wBayes} :\mathcal{X} \to \left\{ { - 1,1} \right\}}  \int_{\mathcal{X} \times \{-1, 1 \}} {w\left( y \right){\bf{1}}\left( {f_{wBayes}  \left( {\bf{x}} \right) - y} \right)} dp\left( {{\bf{x}},y} \right),
				\end{equation}
				where $w(y)$ indicates a weight.
				Similar to the equation \eqref{eq.bayes_bin}, the solution is as follows:
				\begin{equation}
					f_{wBayes}^ * \left( {\bf{x}} \right) = \left\{ {\begin{array}{*{20}{c}}
							1&{\text{if } p \left( {1|{\bf{x}}} \right) w(1) > p \left( { - 1|{\bf{x}}} \right) w(-1)}\\
							{ - 1}&{{\rm{otherwise}}}
					\end{array}} \right. .
				\end{equation}
				By setting $w(1) = w(-1) = 0.5$, we can have the optimal Bayes classifier without the class rebalancing in the equation \eqref{eq.bayes_bin}. 
				Similar to the equation \eqref{eq.LS}, the Weighted Least-Squares (WLS) minimization problem for the weighted optimal Bayes classifier is as follows: 
				\begin{equation} \label{eq.WLS}
					f_{WLS}^ *  = \mathop {\arg \min }\limits_{f_{WLS} :\mathcal{X} \to \mathbb{R}} \int_{\mathcal{X} \times \{-1, 1 \}} {w\left( y \right){{\left( {y - f_{WLS}^{}\left( {\bf{x}} \right)} \right)}^2}dp\left( {{\bf{x}},y} \right)} .
				\end{equation}
				Then, we have the minimizer of the equation \eqref{eq.WLS} satisfies
				\begin{equation} \label{eq.WLS_bin}
					f_{WLS}^ * \left( {\bf{x}} \right) = \frac{{p\left( {1|{\bf{x}}} \right)w\left( 1 \right) - p\left( { - 1|{\bf{x}}} \right)w\left( { - 1} \right)}}{{p\left( {1|{\bf{x}}} \right)w\left( 1 \right) + p\left( { - 1|{\bf{x}}} \right)w\left( { - 1} \right)}} .
				\end{equation}
				By assuming $w(1) > 0$  and  $w(-1) > 0$, 
				the weighted optimal Bayes classifier can be recovered
				by:
				$f_{wBayes}^ *\left({\bf x}\right)   = sign \left( f_{WLS}^ *\left({\bf x}\right) \right)$.
				Therefore, $f_{WLS}^ * > 0 $ if and only if $p \left( {1|{\bf{x}}} \right)w(1) > p \left( { - 1|{\bf{x}}} \right)w(-1)$.
				
				\begin{figure*}[!t]
					\centering 
					{ 
						\includegraphics[width=1\textwidth]{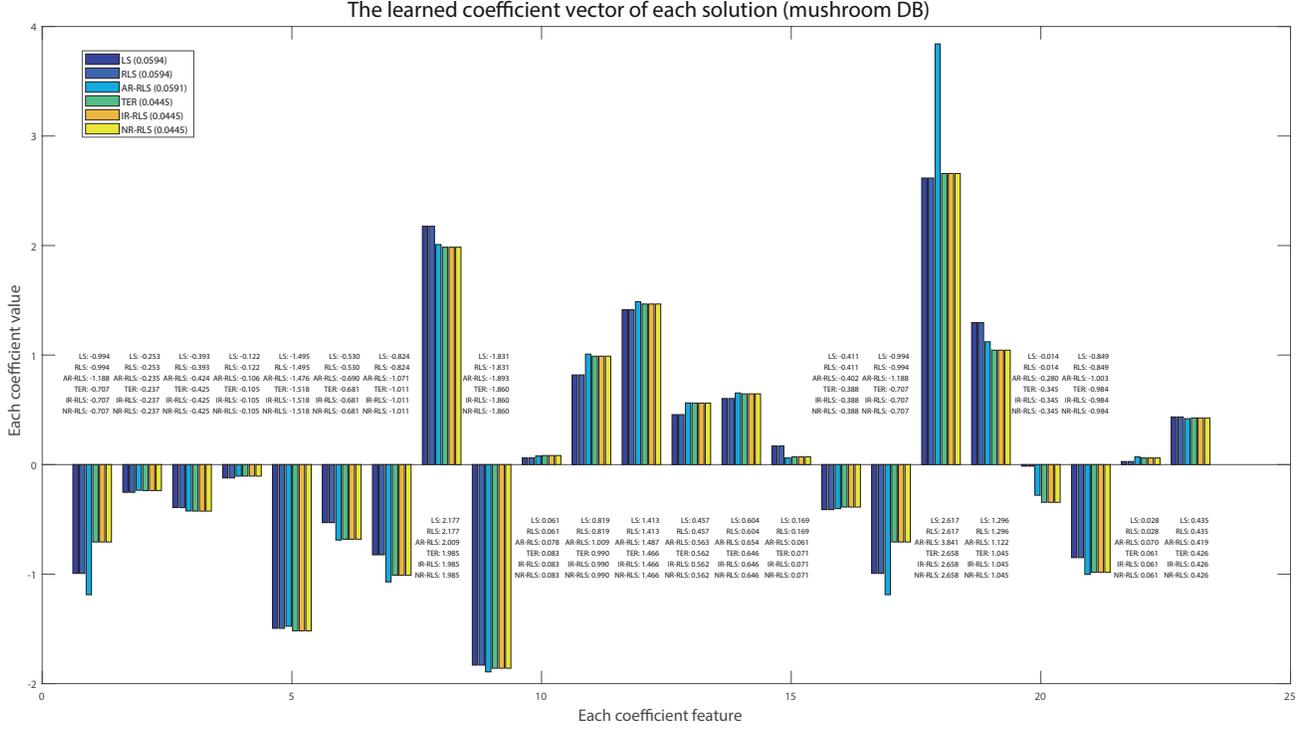}%
					}  
					\caption{The learned coefficient vector of each algorithm in the `Mushroom' data set.}%
					\label{fig.BAR_mushroom}%
				\end{figure*} 
				
				Similar to the equation \eqref{eq.WLS}, 
				the Total-Error-Rate (TER) minimization problem for the weighted optimal Bayes classifier is defined as:
				\begin{equation}
					f_{TER}^ *  = \mathop {\arg \min }\limits_{f_{TER} :\mathcal{X} \to \mathbb{R}} \int_{\mathcal{X} \times \{-1, 1 \}} {w\left( { - 1} \right){{\left( {{y^ - } - f_{TER}^{}\left( {\bf{x}} \right)} \right)}^2} + w\left( 1 \right){{\left( {{y^ + } - f_{TER}^{}\left( {\bf{x}} \right)} \right)}^2}dp\left( {{\bf{x}},y} \right)}  ,
				\end{equation}
				where $y^-$ and $y^+$ respectively are $-1$ and $1$ for negative and positive classes.
				Similar to the equation \eqref{eq.WLS_bin}, 
				we have 
				\begin{equation} \label{eq.TER_bin}
					f_{TER}^ * \left( {\bf{x}} \right) = \frac{{p\left( {1|{\bf{x}}} \right)w\left( 1 \right) - p\left( { - 1|{\bf{x}}} \right)w\left( { - 1} \right)}}{{p\left( {1|{\bf{x}}} \right)w\left( 1 \right) + p\left( { - 1|{\bf{x}}} \right)w\left( { - 1} \right)}} .
				\end{equation}
				Since $w(1) = {1 \over n^+} > 0$  and  $w(-1) = {1 \over n^-}  > 0$ in our work, 
				the weighted optimal Bayes classifier can be recovered
				by:
				$f_{wBayes}^ *\left({\bf x}\right)   = sign \left( f_{TER}^ *\left({\bf x}\right) \right)$.
				Therefore, $f_{TER}^ * > 0 $ if and only if $p \left( {1|{\bf{x}}} \right)w(1) > p \left( { - 1|{\bf{x}}} \right)w(-1)$. 
				Since the proposed NR-RLS classifier could exactly converge to the TER classifier, we conclude that the NR-RLS classifier can asymptotically recover the weighted optimal Bayes classifier.

				\section{ Extension to multiclass classification}  \label{sec.C}
				The multiclass version of NR-RLS can be easily extended by the one-vs-all classification scheme as follows:
				\begin{equation} \label{eq.TER_multi}
					{{\pmb \Theta }_t} = \left[ {{\pmb w}_t^1,{\pmb w}_t^2, \ldots, {\pmb w}_t^c} \right] \in \mathbb{R}^{d\times c},
				\end{equation}
				where $c$ is the number of classes. Each solution, ${\pmb w}_t^i$,  is updated upon the arrival of the new training sample. 
				The time complexity of the multiclass NR-RLS is $\mathcal{O}(2cd^2)$ which still has a constant time complexity since the number of classes will be fixed before training.
				
				\section{The learned coefficient vectors of each algorithm} \label{sec.B}
				In this appendix, we show the learned coefficient vectors after training all the input samples for the compared algorithms in Fig. \ref{fig.BAR_mushroom}. 
				Each bracket in legends indicates the mean value of each coefficient vector.
				In this figure, 
				the no weighting based group (e.g., LS and RLS) and
				the exact reweighting based group (e.g., TER, IR-RLS and NR-RLS)
				show
				equal coefficient values in each group whereas the approximate reweighting based AR-RLS show different coefficient values from the exact reweighting based group. 
				This result further verifies the convergence of NR-RLS to its batch setting for TER minimization.

				\section{$L_2$-norm, $G$-mean and CPU time values for each data set} \label{sec.E}
				
				
				In this appendix, we provide each plot of the $L_2$-norm, the $G$-mean and the CPU time values for each data set which is not shown in the main paper.
				%
				Due to the out-of-memory issue caused by the huge number of samples, the batch mode solutions (e.g., LS and TER) are omitted in the cod-rna, ijcnn1 and skin-nonskin data sets.
				In all the figures, the same convergence results between LS and RLS are shown in the no weighting based group. Similarly, the same convergence results among TER, IR-RLS and NR-RLS are shown in the exact reweighting group whereas the convergence results of the approximate reweighting AR-RLS is different from the exact reweighting group.

				\section{ Highlighted comparison on the highly imbalanced data sets}  \label{sec.D}
				In Table \ref{tbl.dataSummary2} and \ref{tbl.avgacc2}, we selected the six data sets highly imbalanced (e.g., the imbalance ratio $< 0.5$) to feel the impact of the reweighting based online  classifiers. We can see that the exactly reweighted NR-RLS and IR-RLS are seen to be the best performer, while the approximately reweighted AR-RLS is the second best performer. 
				Many of the other algorithms that have no reweighting scheme, suffer from the highly imbalanced data sets (e.g., specially the data set no. 17, 22, 23 and 30).

				\begin{table*}[t!]
					\caption{Highlighted summary of the 6 real-world data sets (e.g., the imbalance ratio $< 0.5$).} \label{tbl.dataSummary2}
					\vskip 0.15in
					\begin{center}
						\begin{small}
							\begin{sc}
								\fontsize{6}{7}\selectfont
								\begin{tabular}{l@{\hskip0.5pt}llllll}
									\toprule
									& No. {\rule{0pt}{3ex}} & Data sets          & Size    & Dimension & Ratio  
									\\ \midrule
									& 4 
									{\rule{0pt}{3ex}}  		   
									& Wpbc               & 194        & 33     & 0.31        \\
									
									& 17  & Blood-transfusion  & 748        & 4      & 0.31    \\

									& 22  & Ozone-eight        & 1,847      & 72   & 0.07       \\
									
									& 23  & Ozone-one          & 1,848      & 72   & 0.03       \\
									
									& 30  & Ijcnn1             & 141,691    & 22     & 0.11       \\
									
									& 31  & Skin-nonskin       & 245,057    & 3      & 0.26       \\ \bottomrule
								\end{tabular} 
							\end{sc}
						\end{small}
					\end{center}
					\vskip -0.1in
				\end{table*}
				
				\begin{table*}[t]  
					\caption{Highlighted comparison of average G-means and (ranks).}
					\label{tbl.avgacc2}
					\vskip 0.2in
					\begin{center}
						\begin{small}
							\begin{sc} \fontsize{6}{6}\selectfont
								\begin{tabular}{l@{\hskip3pt}l@{\hskip3pt}l@{\hskip3pt}l@{\hskip3pt}l@{\hskip3pt}l@{\hskip3pt}l@{\hskip3pt}l@{\hskip3pt}l@{\hskip3pt}}
									\toprule
									\multirow{3}{*}{No.} &   \multicolumn{7}{c}{G-mean $\pm$ std (rank)}  {\rule{0pt}{3ex}}             \\ \cline{2-9} 
									{\rule{0pt}{3ex}} &     \multicolumn{2}{c|}{The First-order}                       &                    \multicolumn{5}{c}{The Second-order}                                                                                 \\ \cline{2-9} 
									{\rule{0pt}{3ex}}                          & {PE}    & {PA}    & {CW}    & {AROW}   & {ACOG}    & {RLS}   & {AR-RLS}  & {NR-RLS and IR-RLS} \\ \midrule
									4   & 0.440 $\pm$ 0.093 (5)    & 0.440 $\pm$ 0.075 (6)    & 0.522 $\pm$ 0.087 (4)    & 0.285 $\pm$ 0.230 (7)    & 0.114 $\pm$ 0.176 (8)    & 0.639 $\pm$ 0.081 (3)     & 0.687 $\pm$ 0.058 (2)    & \textbf{0.688 $\pm$ 0.047} (1)     \\
									17  & 0.485 $\pm$ 0.046 (4)    & 0.456 $\pm$ 0.040 (5)    & 0.526 $\pm$ 0.057 (3)    & 0.237 $\pm$ 0.100 (7)    & 0.103 $\pm$ 0.139 (8)    & 0.286 $\pm$ 0.043 (6)     & 0.682 $\pm$ 0.023 (2)    & \textbf{0.685 $\pm$ 0.022} (1)     \\
									22  & 0.364 $\pm$ 0.066 (4)    & 0.296 $\pm$ 0.061 (5)    & 0.571 $\pm$ 0.045 (3)    & 0.072 $\pm$ 0.106 (6)    & 0.000 $\pm$ 0.000 (8)    & 0.007 $\pm$ 0.029 (7)     & 0.807 $\pm$ 0.033 (2)    & \textbf{0.820 $\pm$ 0.018} (1)     \\
									23  & 0.240 $\pm$ 0.104 (4)    & 0.161 $\pm$ 0.101 (5)    & 0.423 $\pm$ 0.100 (3)    & 0.009 $\pm$ 0.041 (6.5)  & 0.009 $\pm$ 0.041 (6.5)  & 0.000 $\pm$ 0.000 (8)     & 0.788 $\pm$ 0.036 (2)    & \textbf{0.819 $\pm$ 0.035} (1)     \\
									30  & 0.631 $\pm$ 0.004 (4)    & 0.627 $\pm$ 0.005 (5)    & 0.692 $\pm$ 0.004 (3)    & 0.508 $\pm$ 0.009 (6)    & 0.173 $\pm$ 0.085 (8)    & 0.301 $\pm$ 0.007 (7)     & 0.858 $\pm$ 0.004 (2)    & \textbf{0.859 $\pm$ 0.004} (1)     \\
									31  & 0.768 $\pm$ 0.002 (7)    & 0.783 $\pm$ 0.002 (6)    & 0.454 $\pm$ 0.015 (8)    & 0.899 $\pm$ 0.001 (4)    & 0.784 $\pm$ 0.342 (5)    & 0.904 $\pm$ 0.001 (3)     & 0.957 $\pm$ 0.000 (2)    & \textbf{0.958 $\pm$ 0.000} (1)     \\ Avg.
									& 0.488 $\pm$ 0.189 (4.66) & 0.461 $\pm$ 0.223 (5.33) & 0.531 $\pm$ 0.095 (4.00) & 0.335 $\pm$ 0.327 (6.08) & 0.197 $\pm$ 0.295 (7.25) & 0.356 $\pm$ 0.357 (5.66)  & 0.797 $\pm$ 0.105 (2) & \textbf{0.805 $\pm$ 0.104} (1)  \\ \bottomrule 
								\end{tabular}
							\end{sc}
						\end{small}
					\end{center}
					\vskip -0.2in
				\end{table*}

				\begin{figure*}[!h]
					\centering 
					\subfigure[The $L_2$-norm values
					]{ 
						\includegraphics[width=0.3\textwidth]{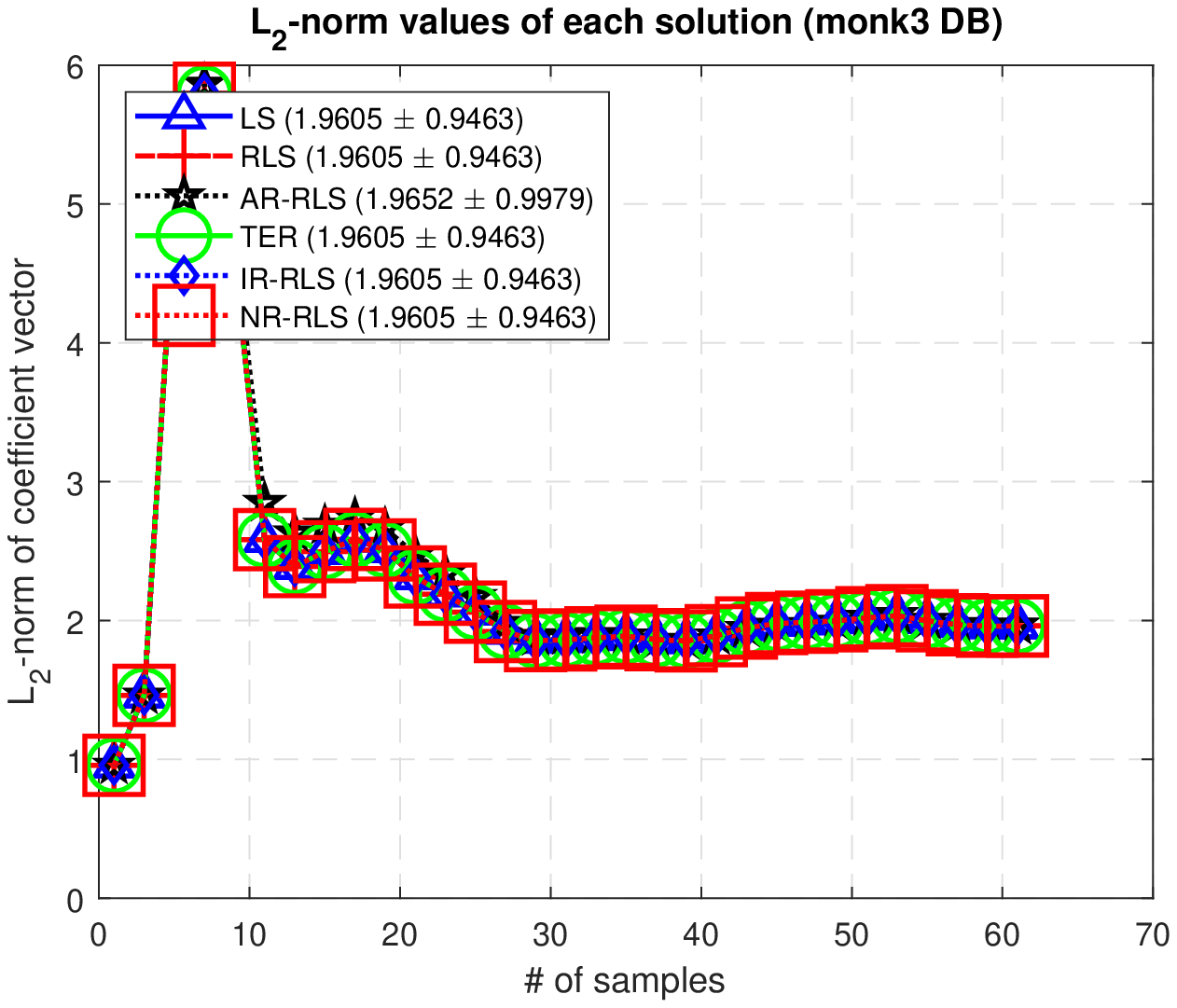}%
					} \hspace{0.4cm}
					\subfigure[The G-means
					]{ 
						\includegraphics[width=0.3\textwidth]{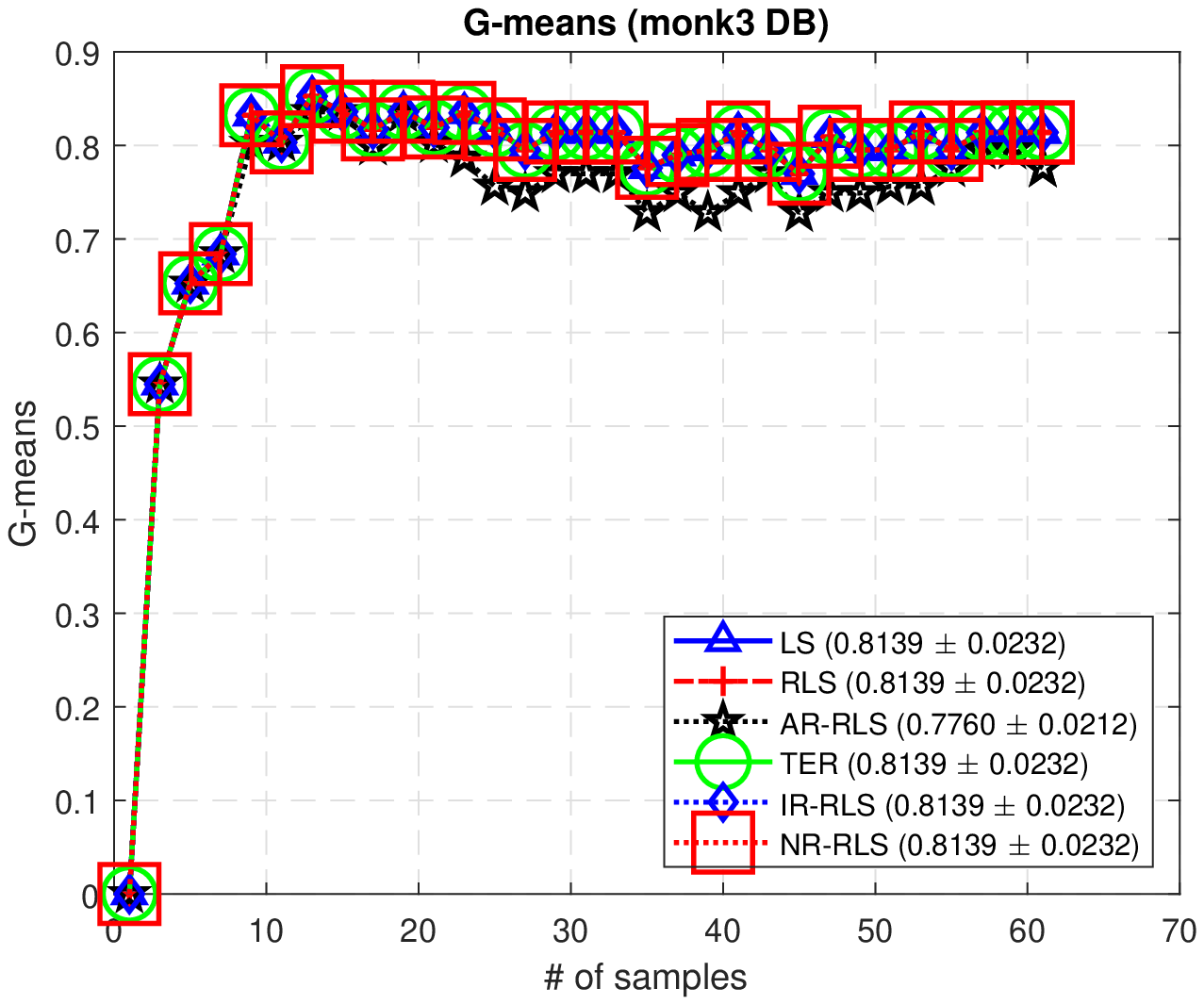}%
					} \hspace{0.4cm}
					\subfigure[The CPU times
					]{ 
						\includegraphics[width=0.3\textwidth]{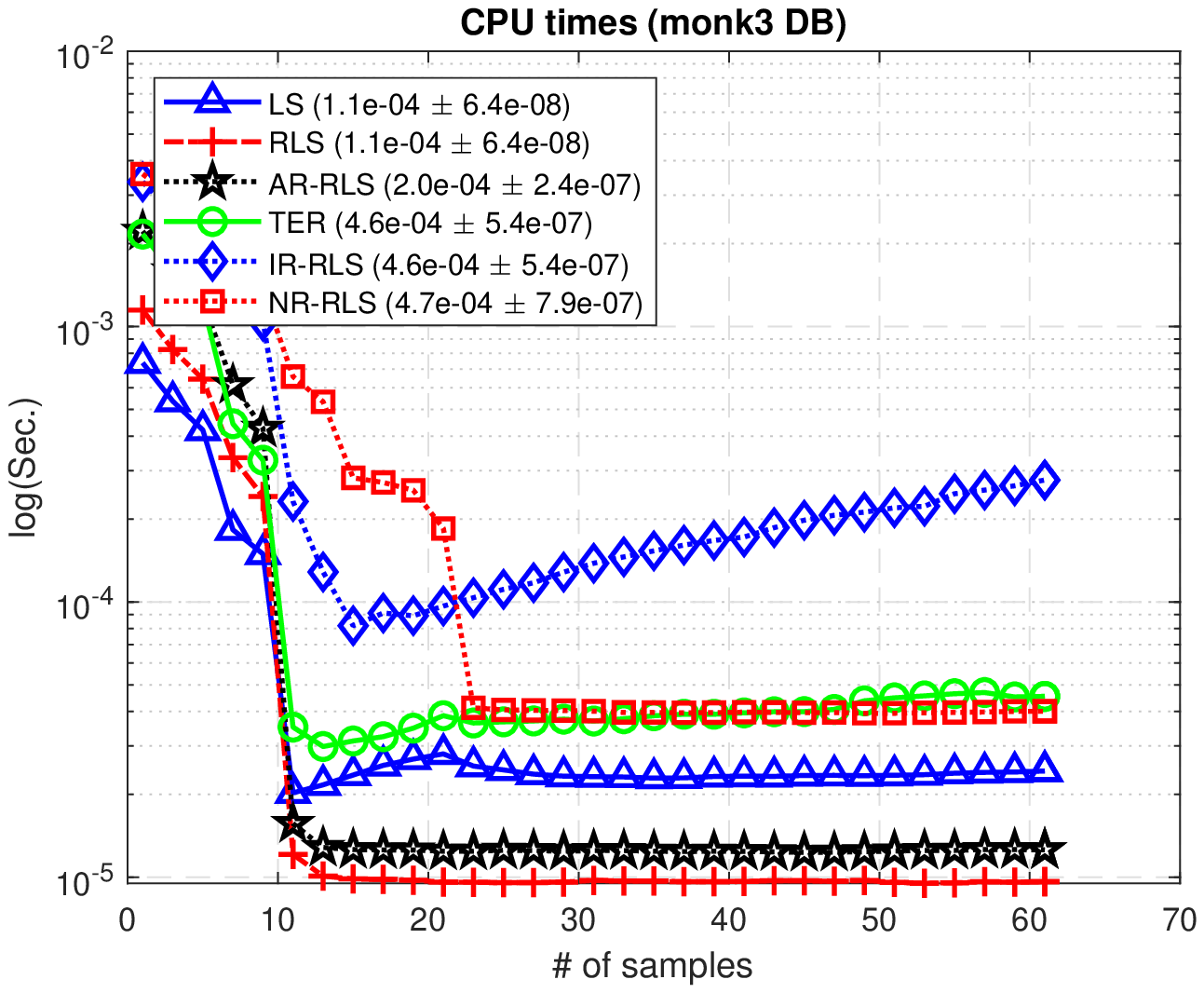}%
					} 
					\caption{Monk-3 DB}%
				\end{figure*} 
				
				\begin{figure*}[!h]
					\centering 
					\subfigure[The $L_2$-norm values
					]{ 
						\includegraphics[width=0.3\textwidth]{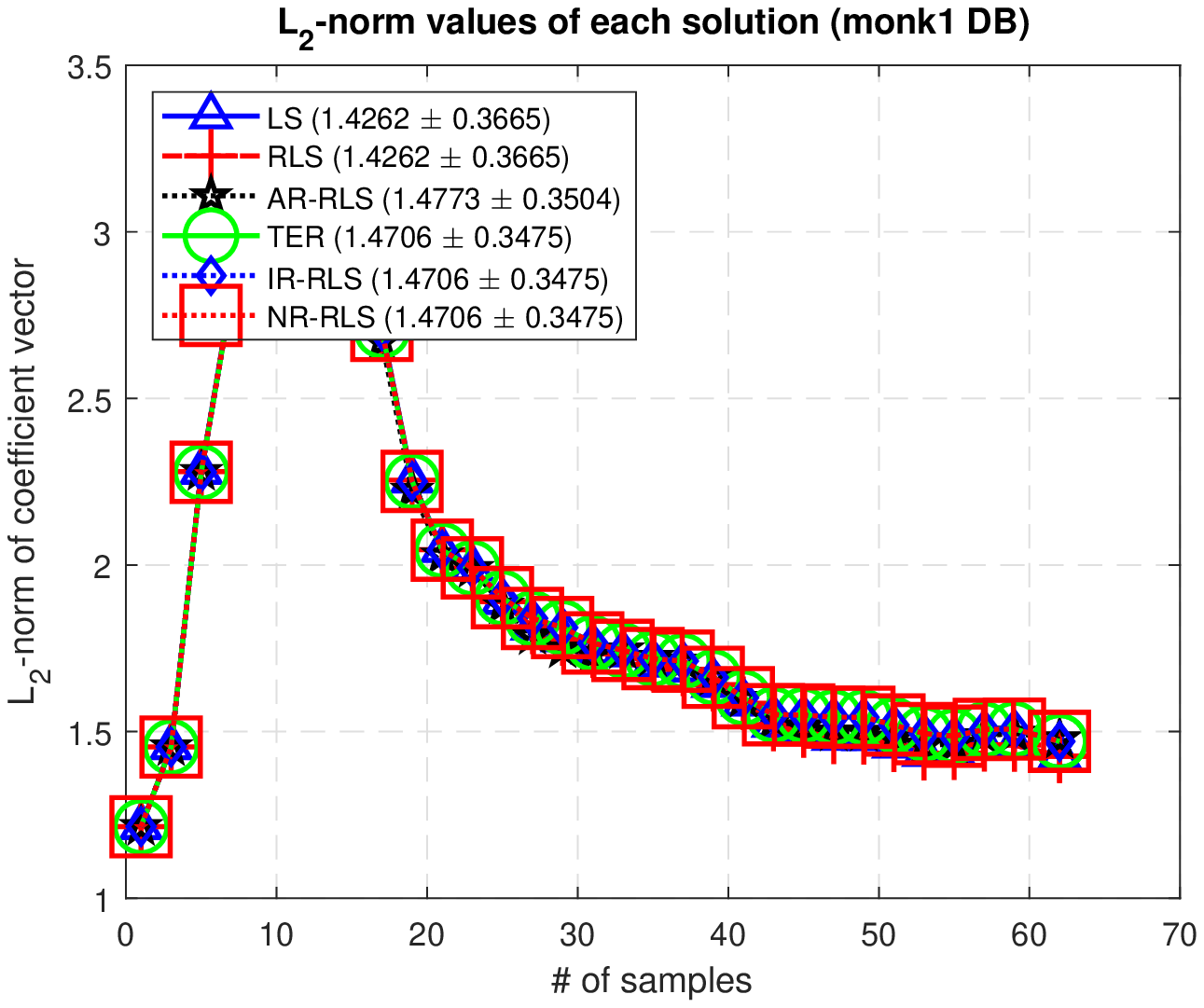}%
					} \hspace{0.4cm}
					\subfigure[The G-means
					]{ 
						\includegraphics[width=0.3\textwidth]{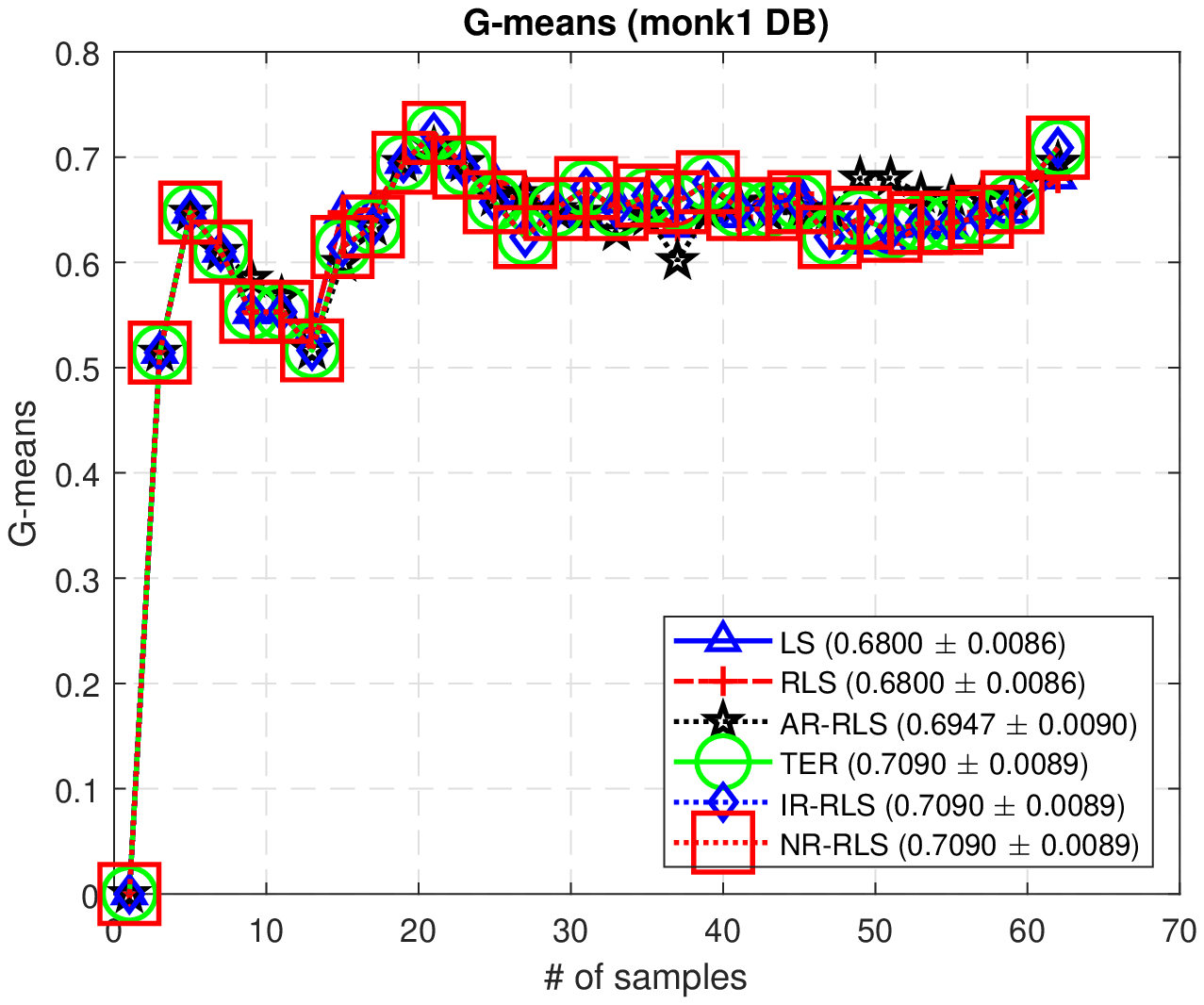}%
					} \hspace{0.4cm}
					\subfigure[The CPU times
					]{ 
						\includegraphics[width=0.3\textwidth]{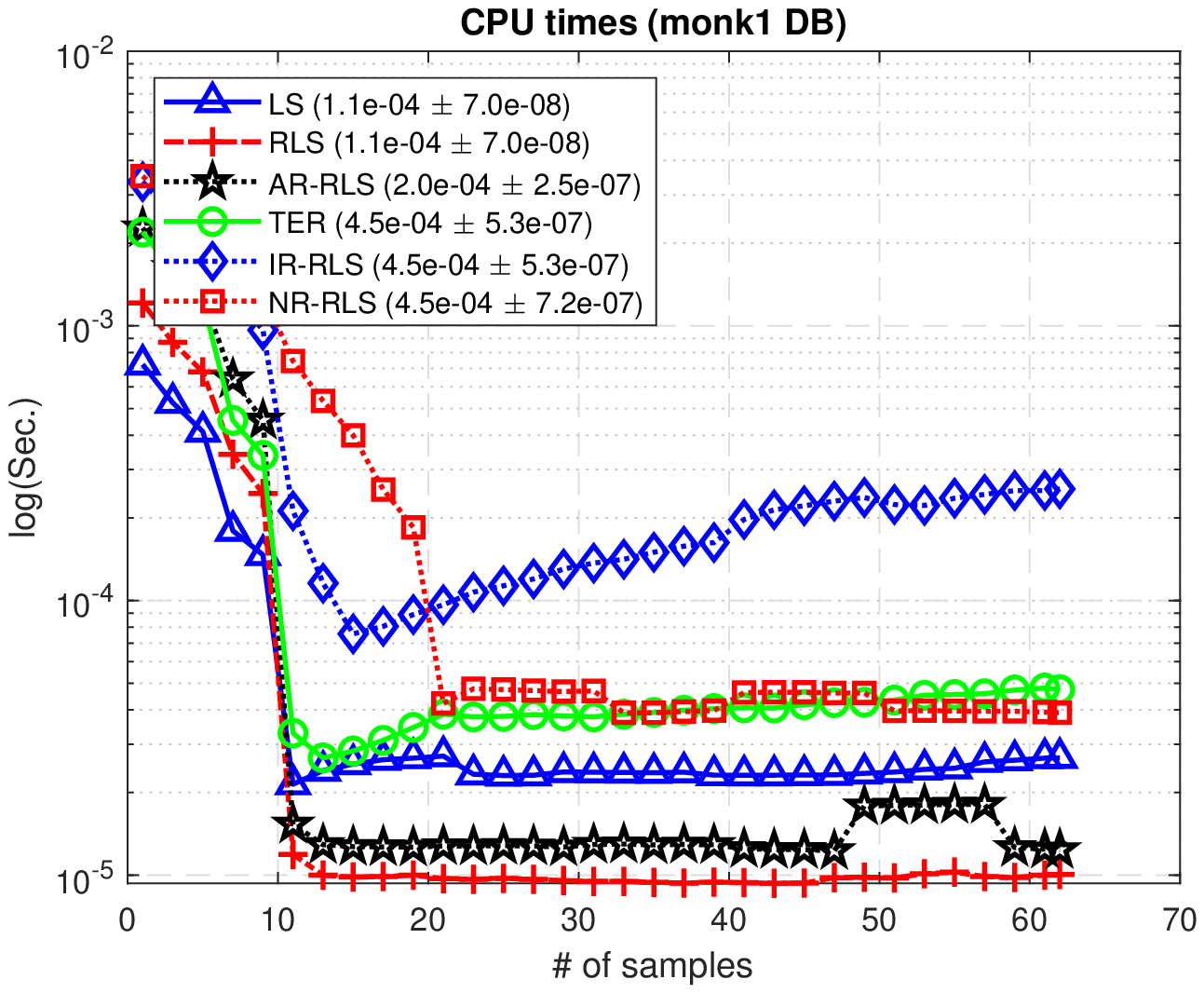}%
					} 
					\caption{Monk-1 DB}%
				\end{figure*} 
				
				\begin{figure*}[!h]
					\centering 
					\subfigure[The $L_2$-norm values
					]{ 
						\includegraphics[width=0.3\textwidth]{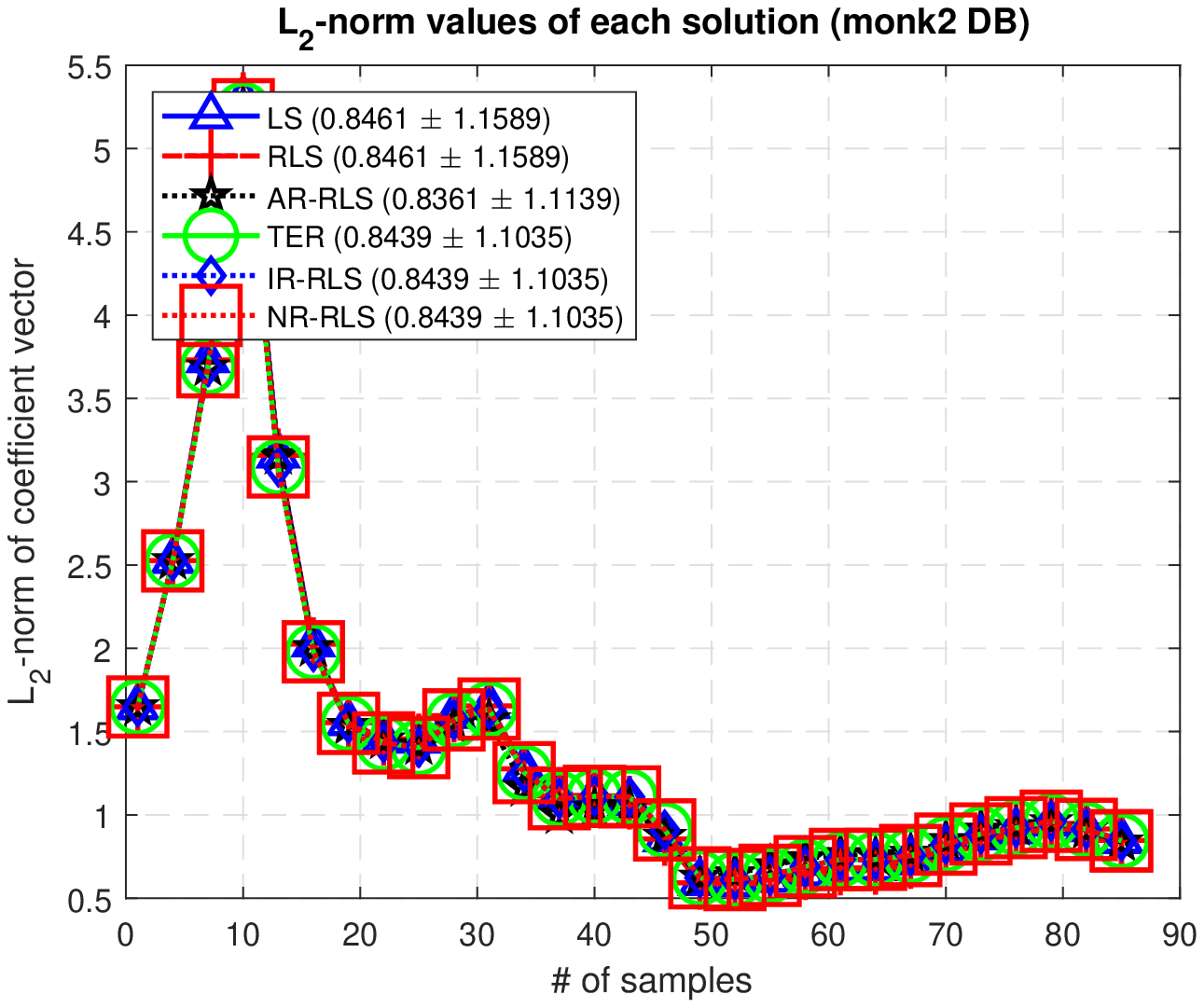}%
					} \hspace{0.4cm}
					\subfigure[The G-means
					]{ 
						\includegraphics[width=0.3\textwidth]{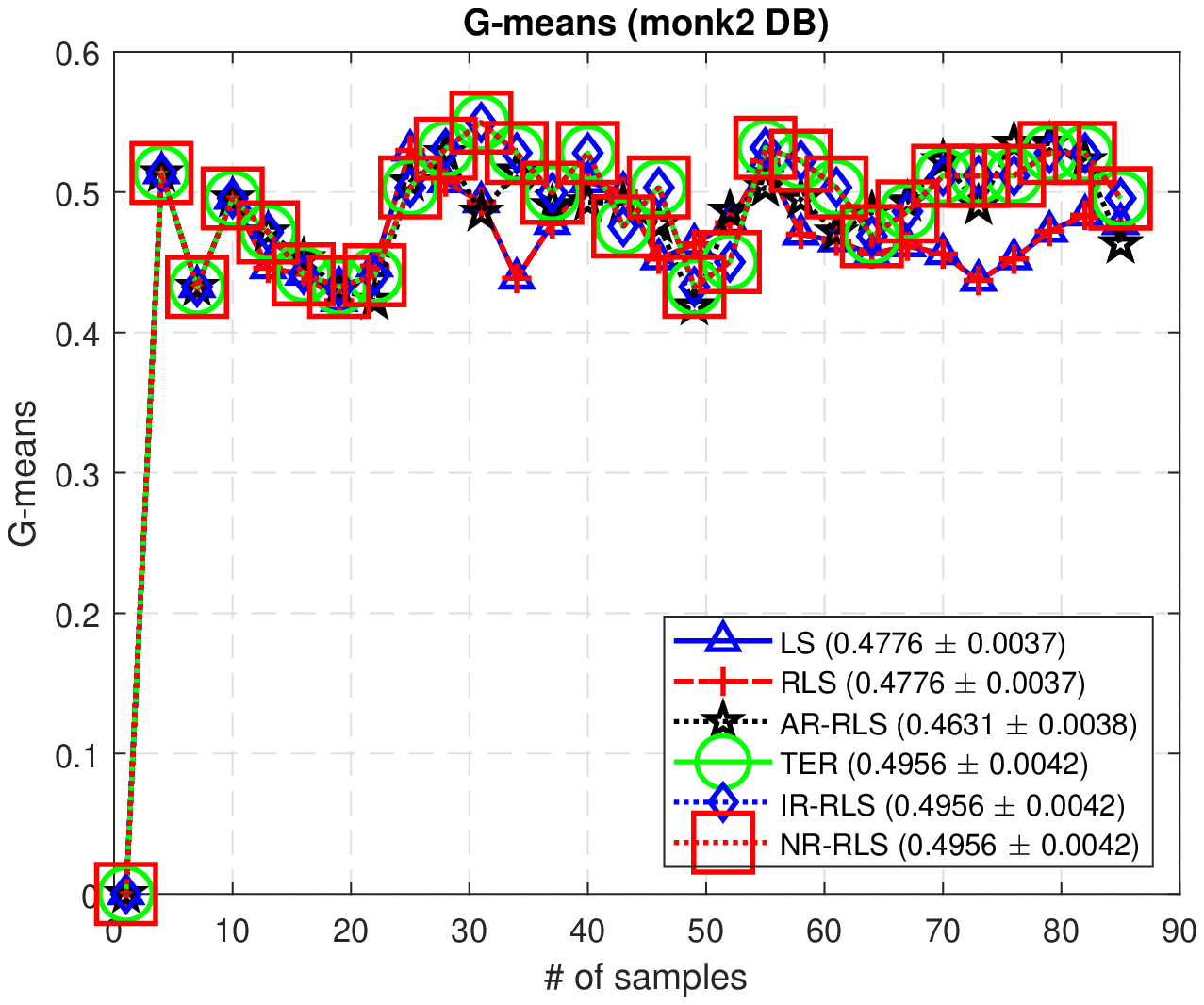}%
					} \hspace{0.4cm}
					\subfigure[The CPU times
					]{ 
						\includegraphics[width=0.3\textwidth]{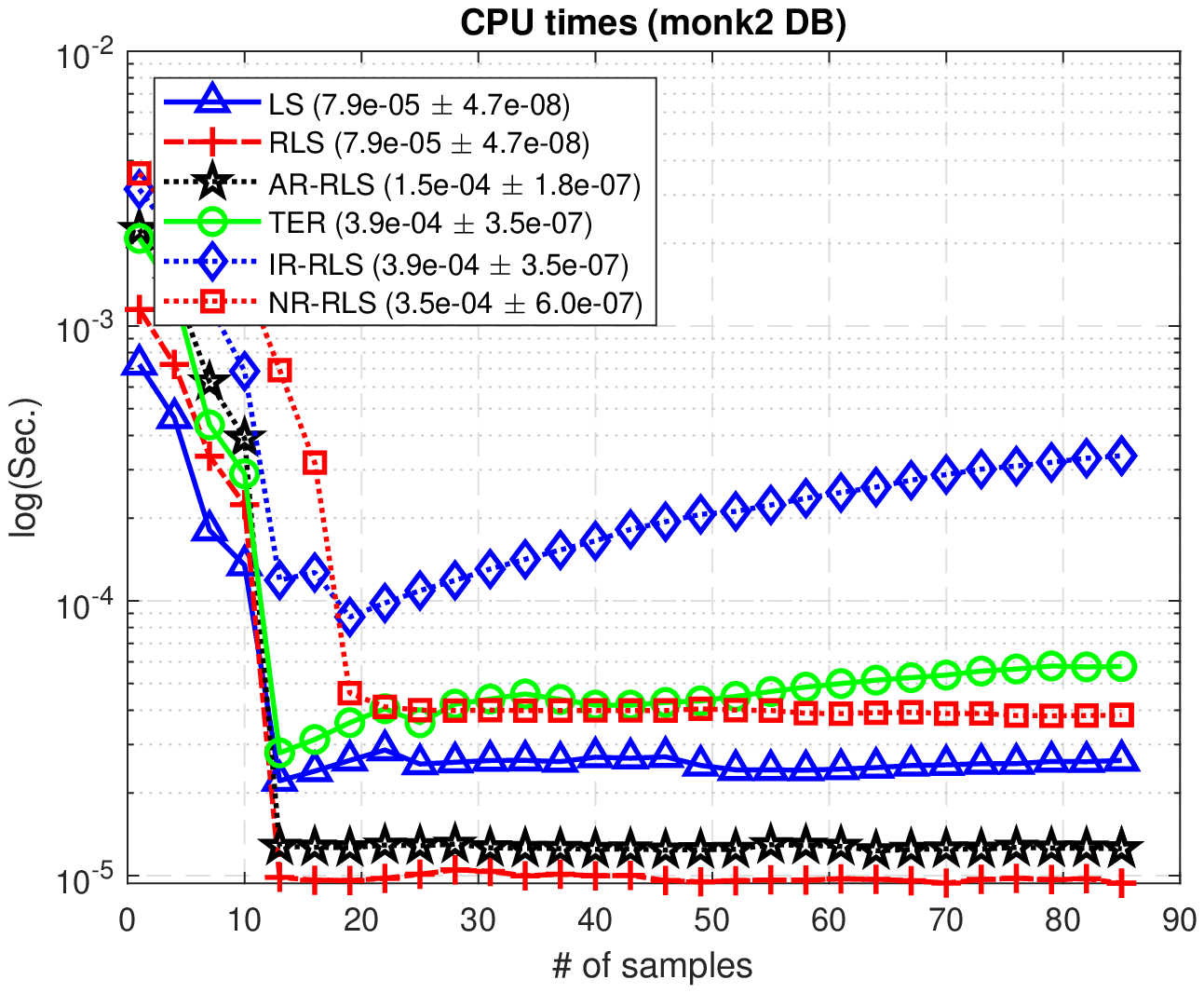}%
					} 
					\caption{Monk-2 DB}%
					
				\end{figure*} 
				
				\begin{figure*}[!h]
					\centering 
					\subfigure[The $L_2$-norm values
					]{ 
						\includegraphics[width=0.3\textwidth]{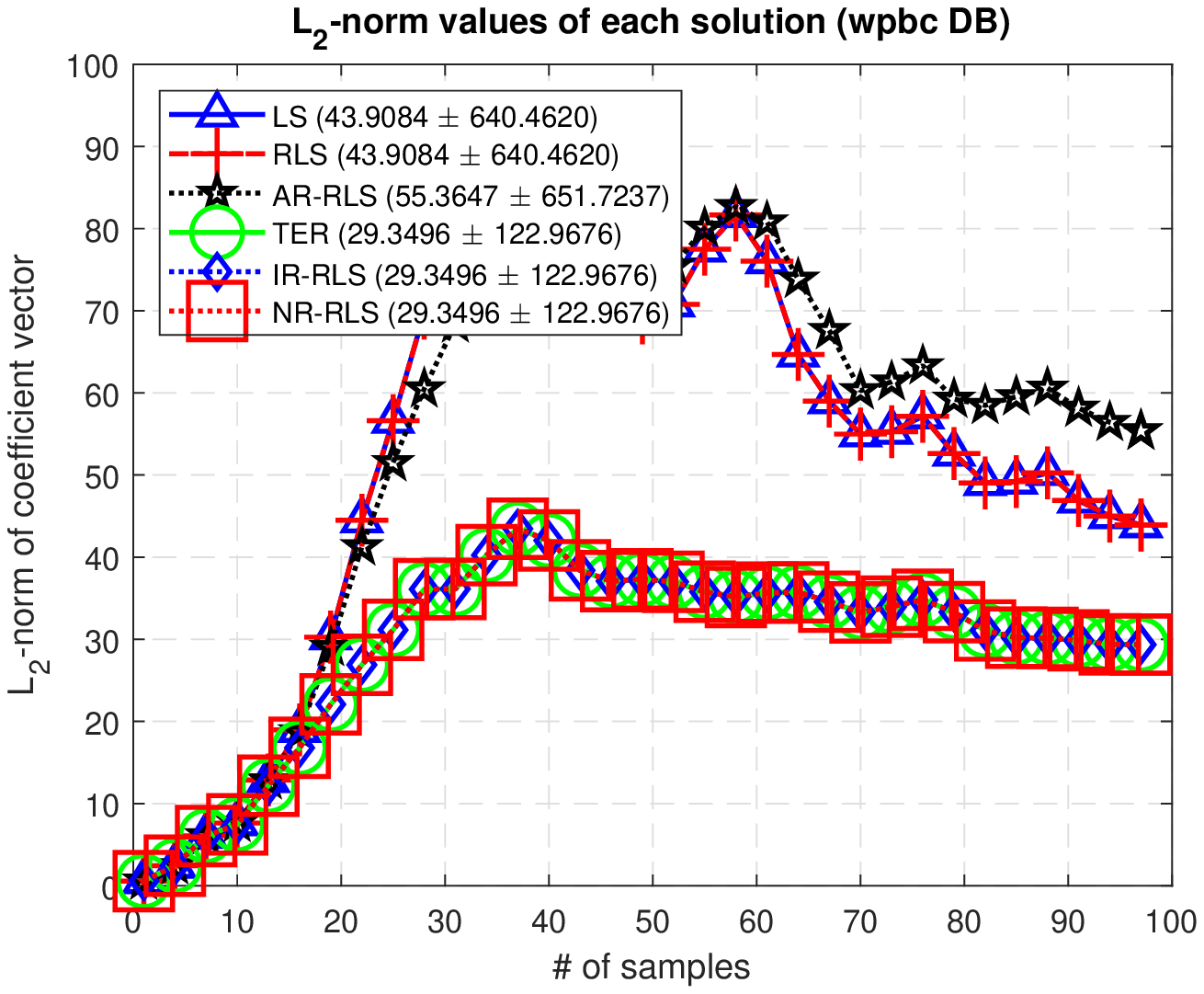}%
					} \hspace{0.4cm}
					\subfigure[The G-means
					]{ 
						\includegraphics[width=0.3\textwidth]{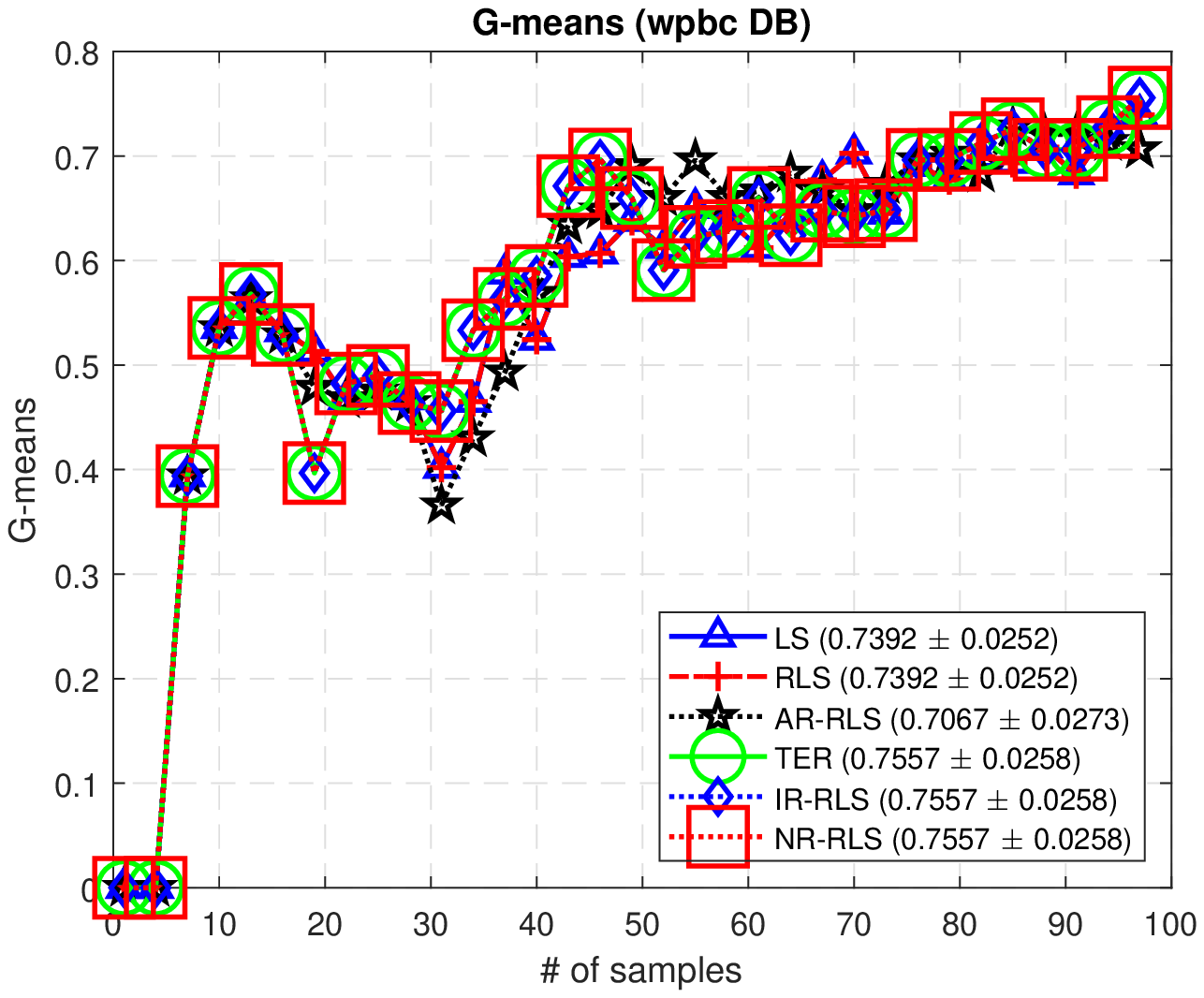}%
					} \hspace{0.4cm}
					\subfigure[The CPU times
					]{ 
						\includegraphics[width=0.3\textwidth]{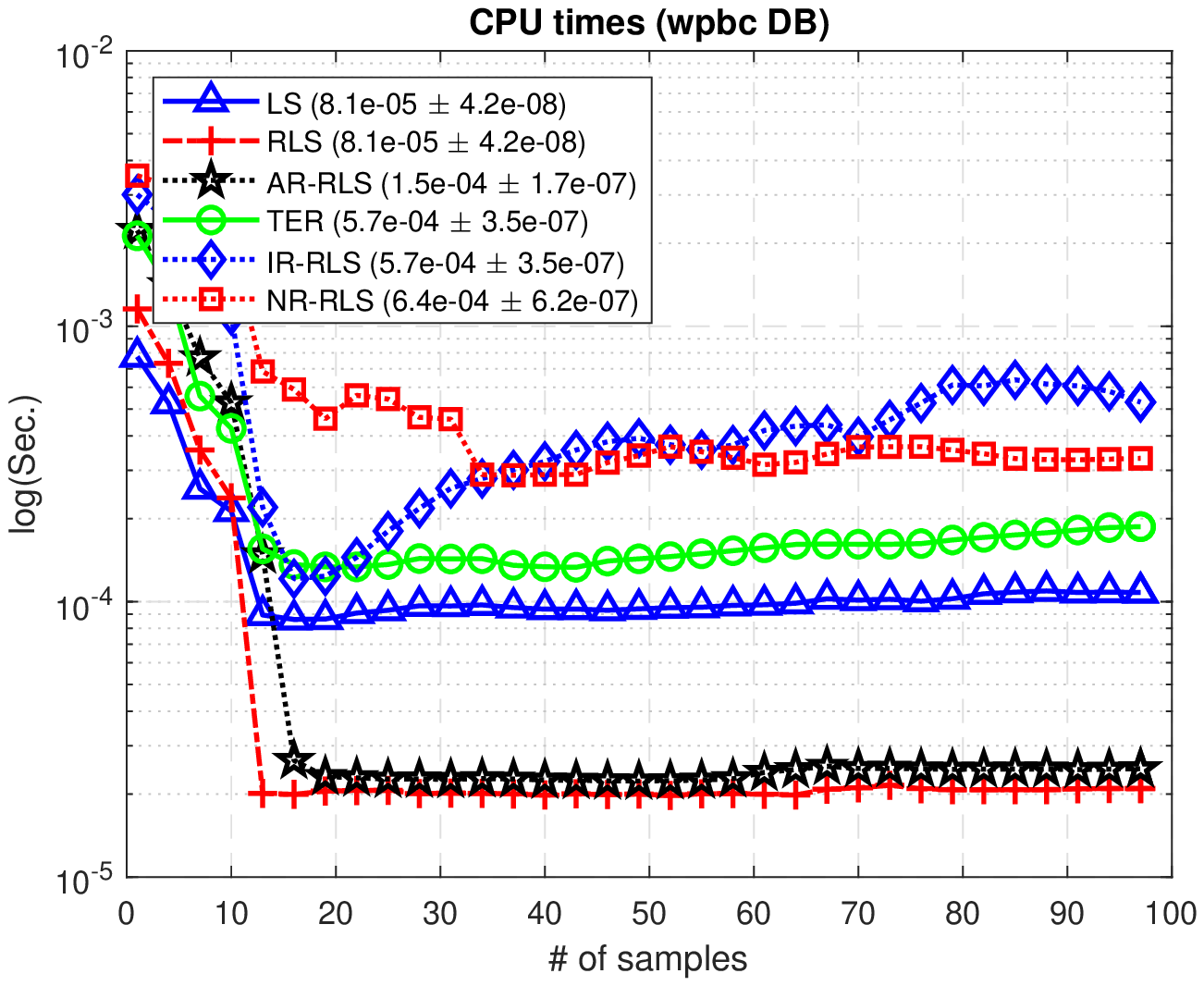}%
					} 
					\caption{Wpbc}%
					
				\end{figure*} 
				\newpage
				\begin{figure*}[!h]
					\centering 
					\subfigure[The $L_2$-norm values
					]{ 
						\includegraphics[width=0.3\textwidth]{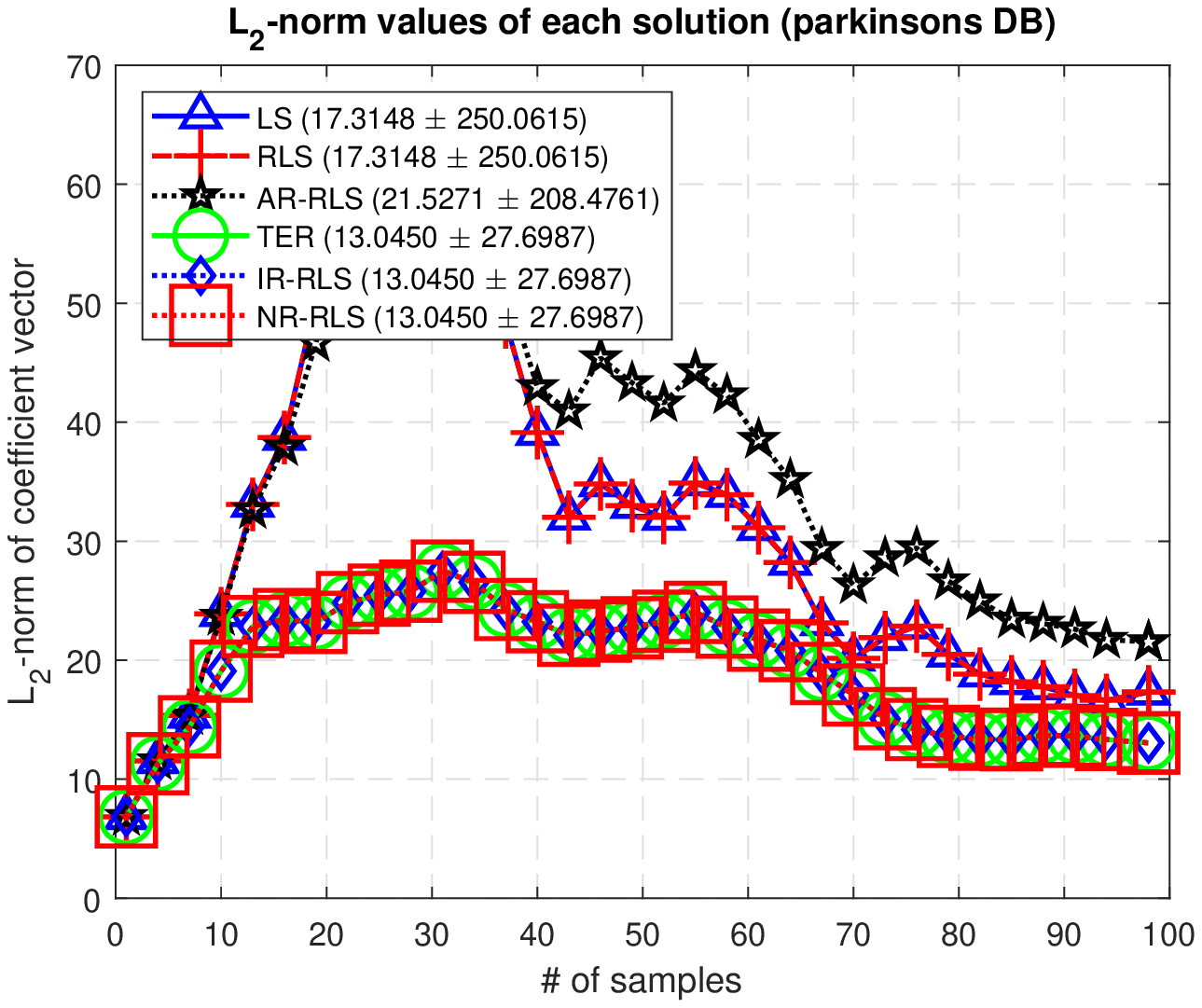}%
					} \hspace{0.4cm}
					\subfigure[The G-means
					]{ 
						\includegraphics[width=0.3\textwidth]{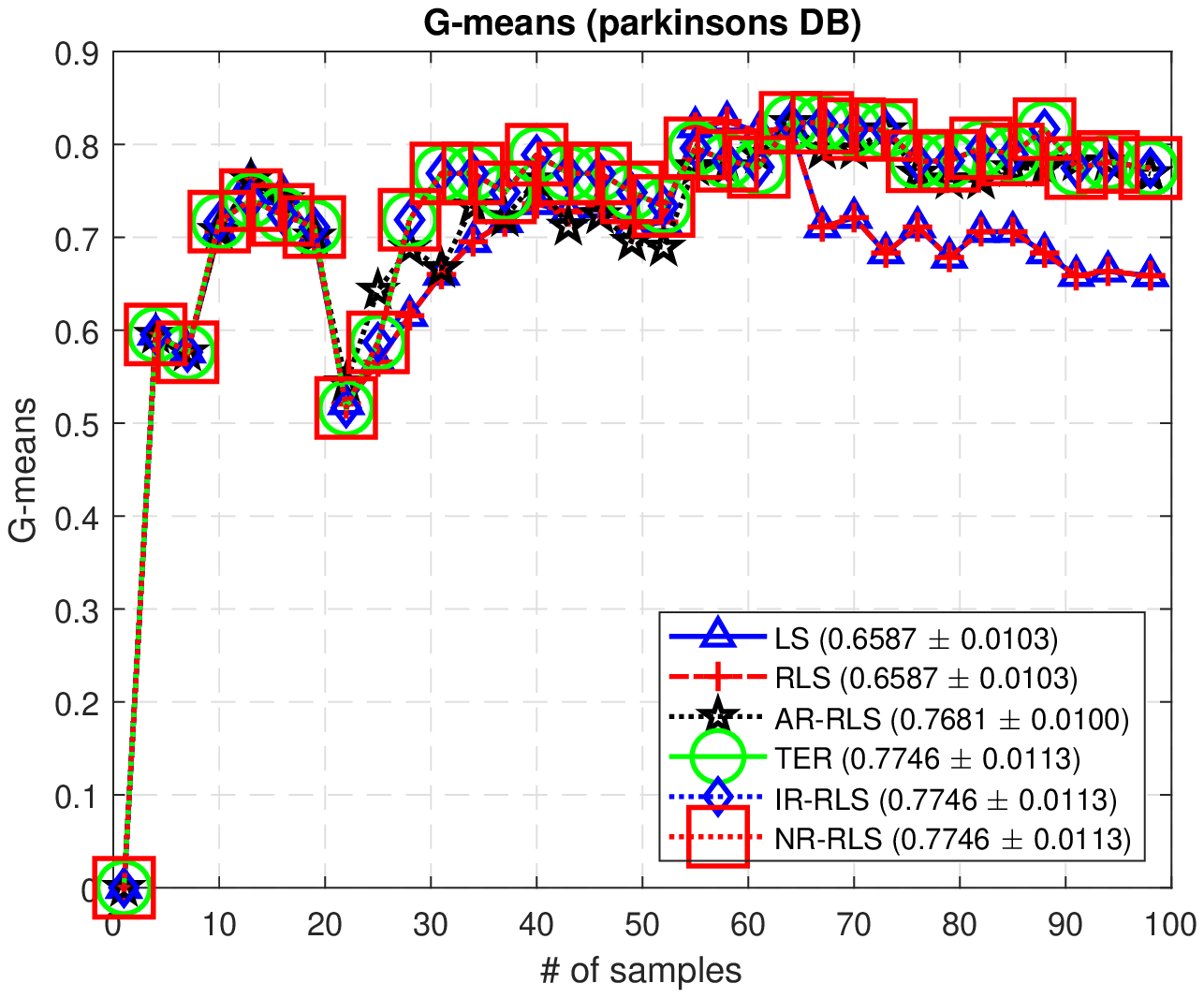}%
					} \hspace{0.4cm}
					\subfigure[The CPU times
					]{ 
						\includegraphics[width=0.3\textwidth]{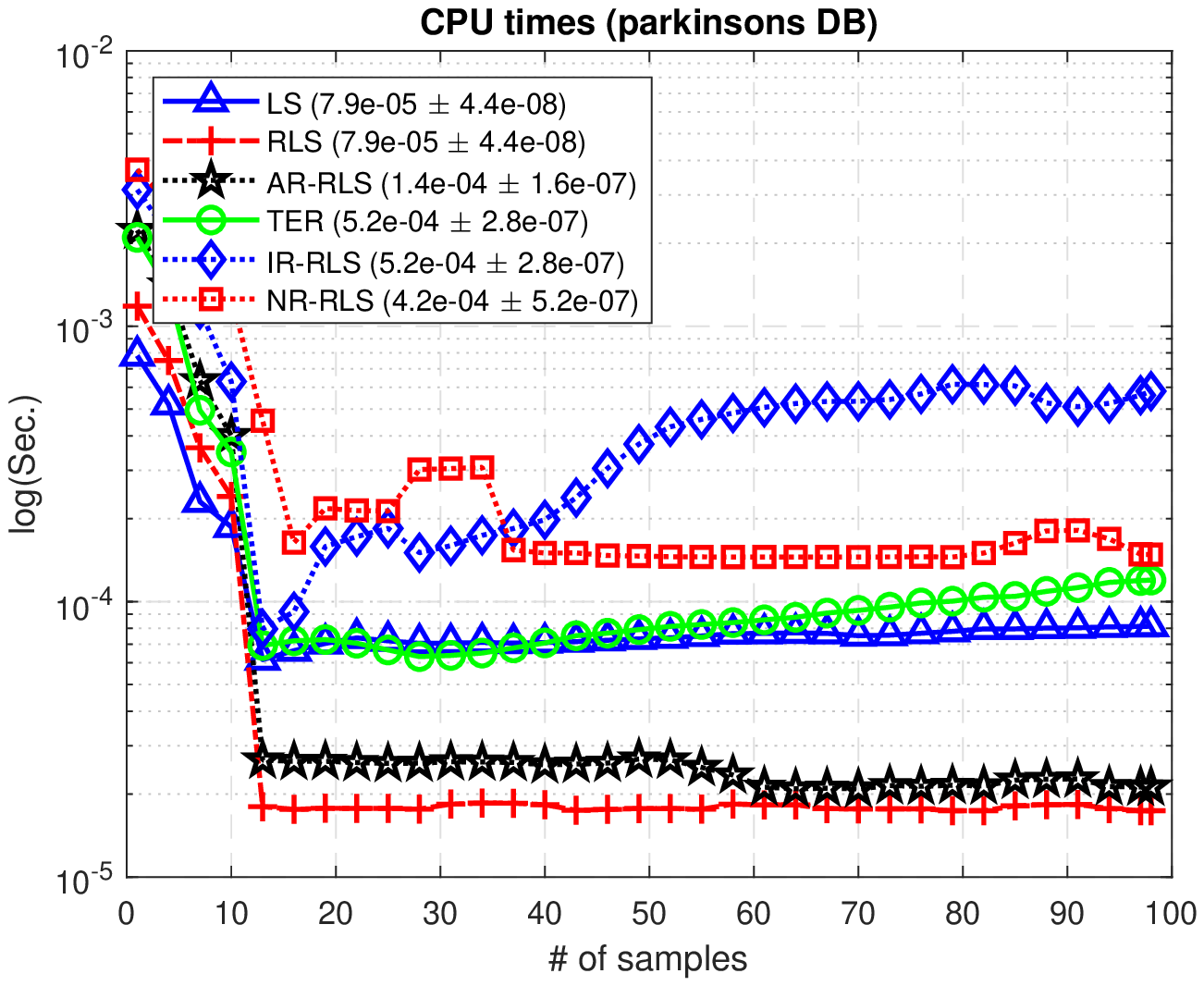}%
					} 
					\caption{Parkinsons}%
				\end{figure*} 
				
				\begin{figure*}[!h]
					\centering 
					\subfigure[The $L_2$-norm values
					]{ 
						\includegraphics[width=0.3\textwidth]{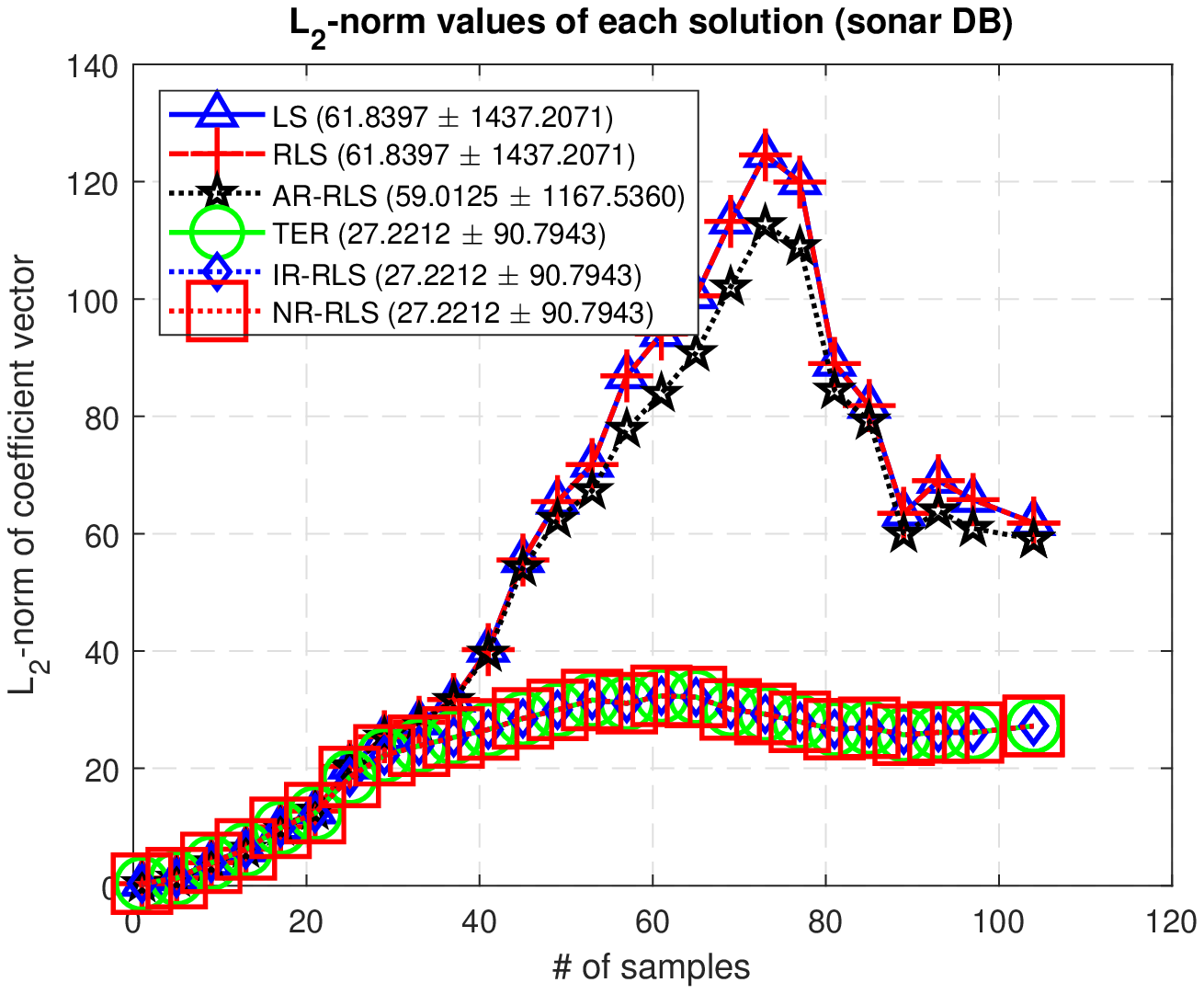}%
					} \hspace{0.4cm}
					\subfigure[The G-means
					]{ 
						\includegraphics[width=0.3\textwidth]{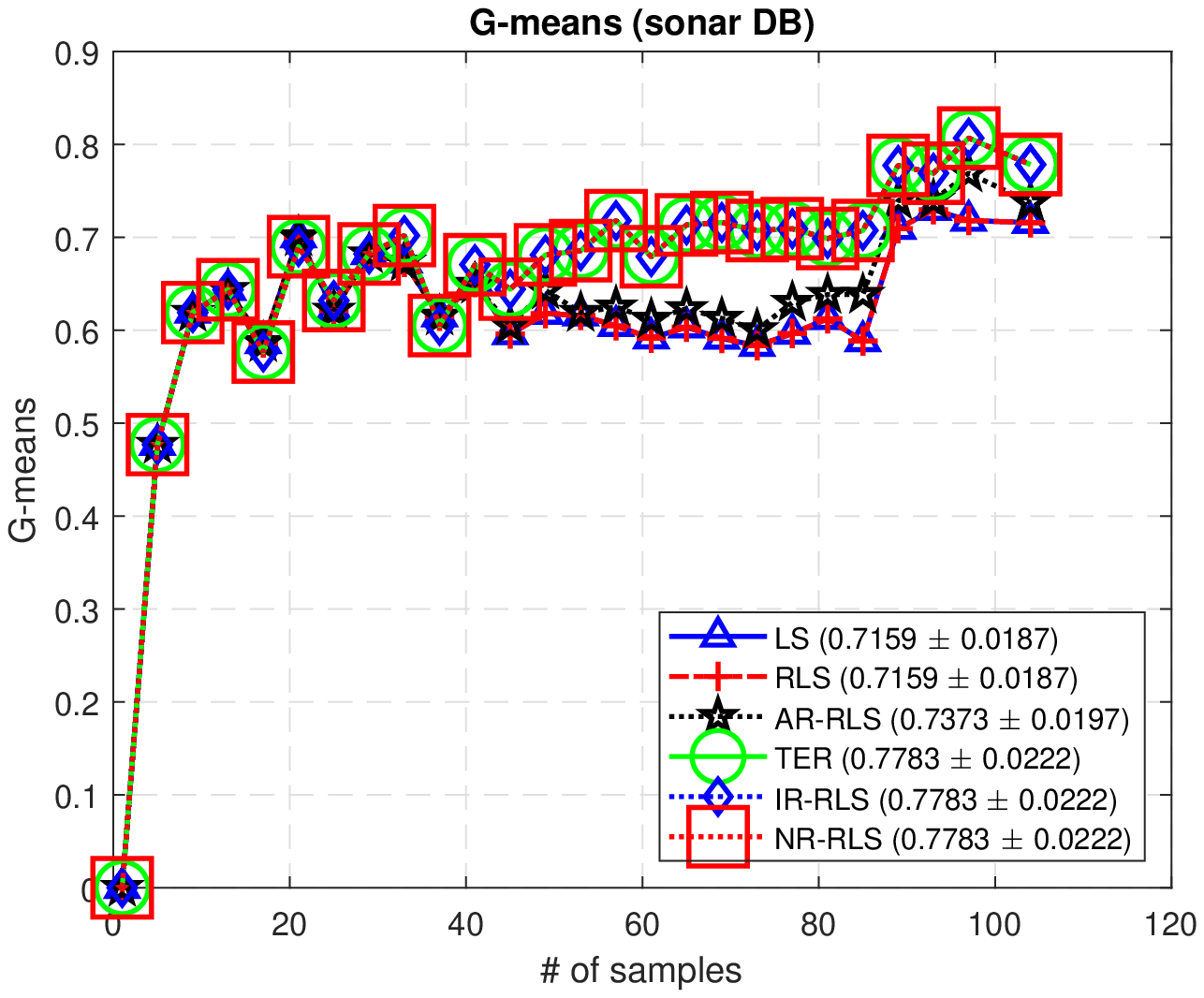}%
					} \hspace{0.4cm}
					\subfigure[The CPU times
					]{ 
						\includegraphics[width=0.3\textwidth]{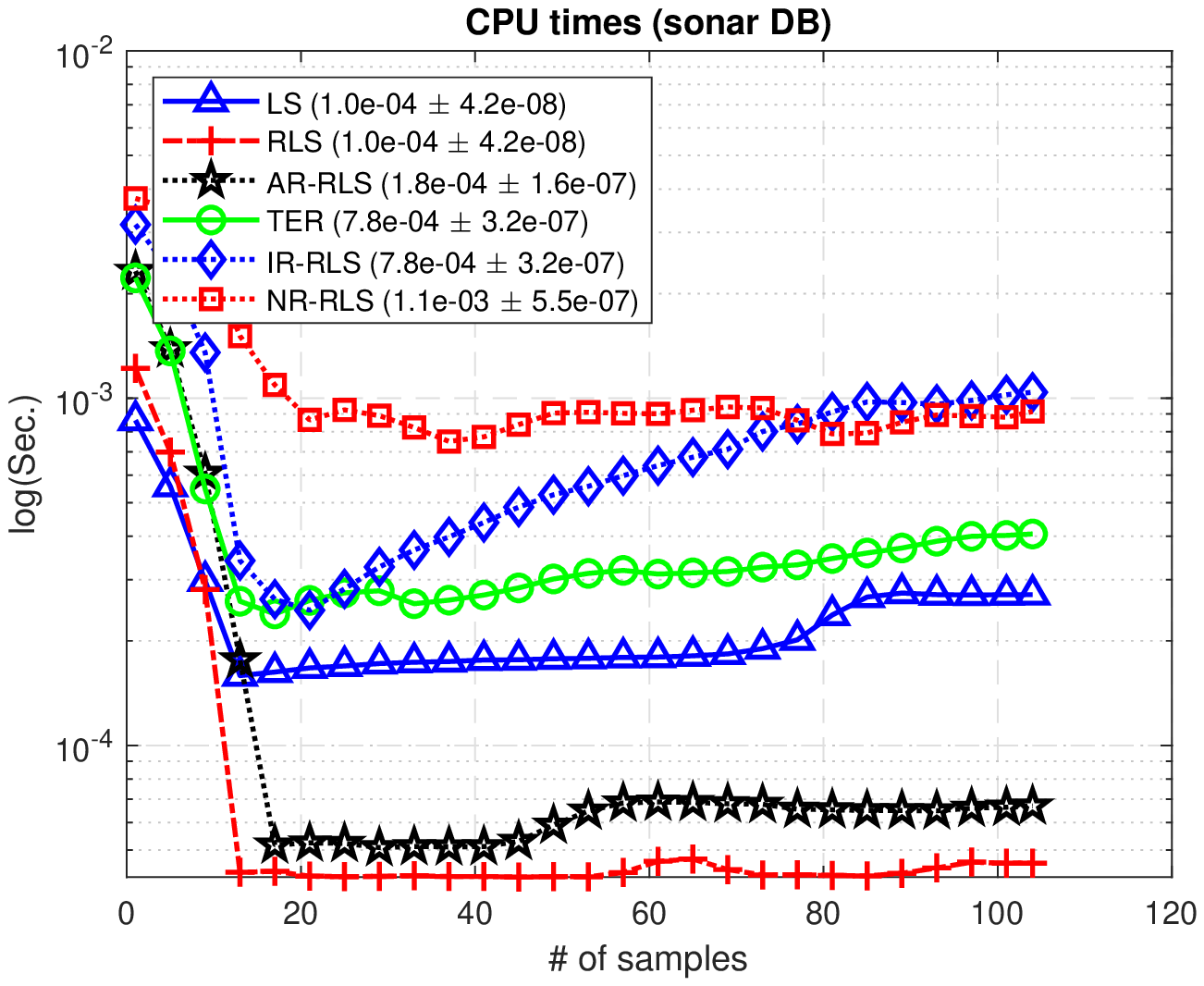}%
					} 
					\caption{Sonar}%
					
				\end{figure*} 
				
				\begin{figure*}[!h]
					\centering 
					\subfigure[The $L_2$-norm values
					]{ 
						\includegraphics[width=0.3\textwidth]{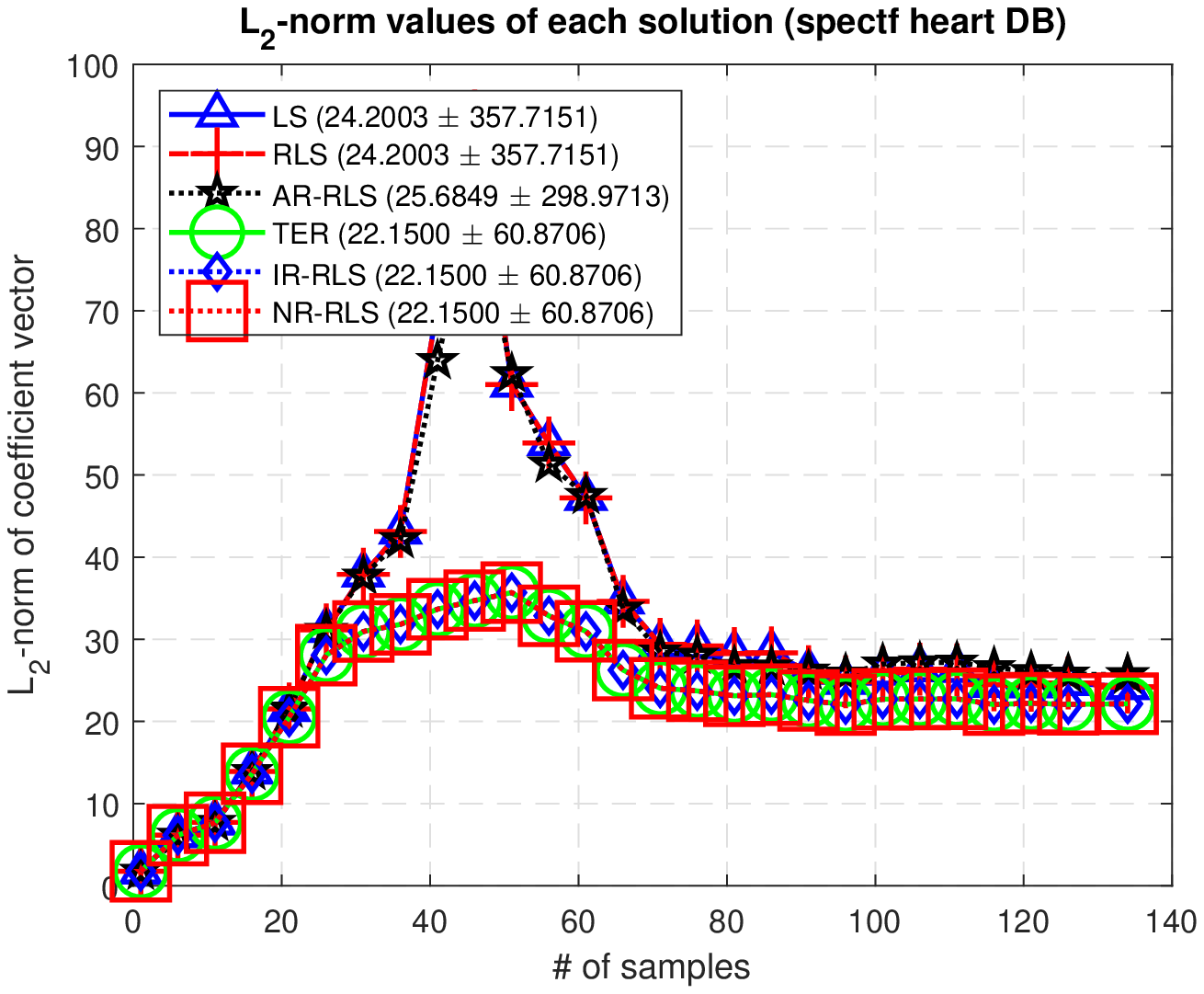}%
					} \hspace{0.4cm}
					\subfigure[The G-means
					]{ 
						\includegraphics[width=0.3\textwidth]{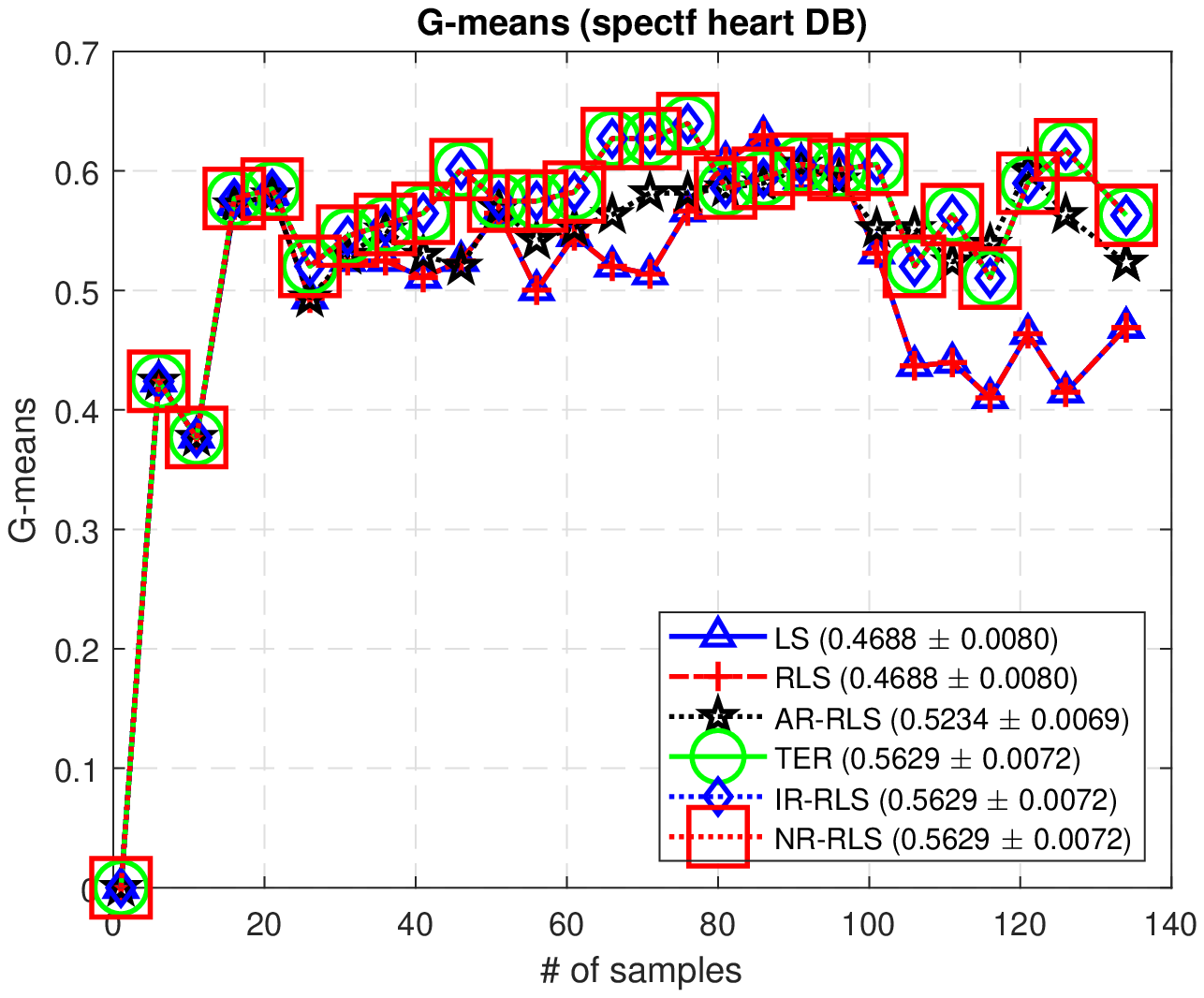}%
					} \hspace{0.4cm}
					\subfigure[The CPU times
					]{ 
						\includegraphics[width=0.3\textwidth]{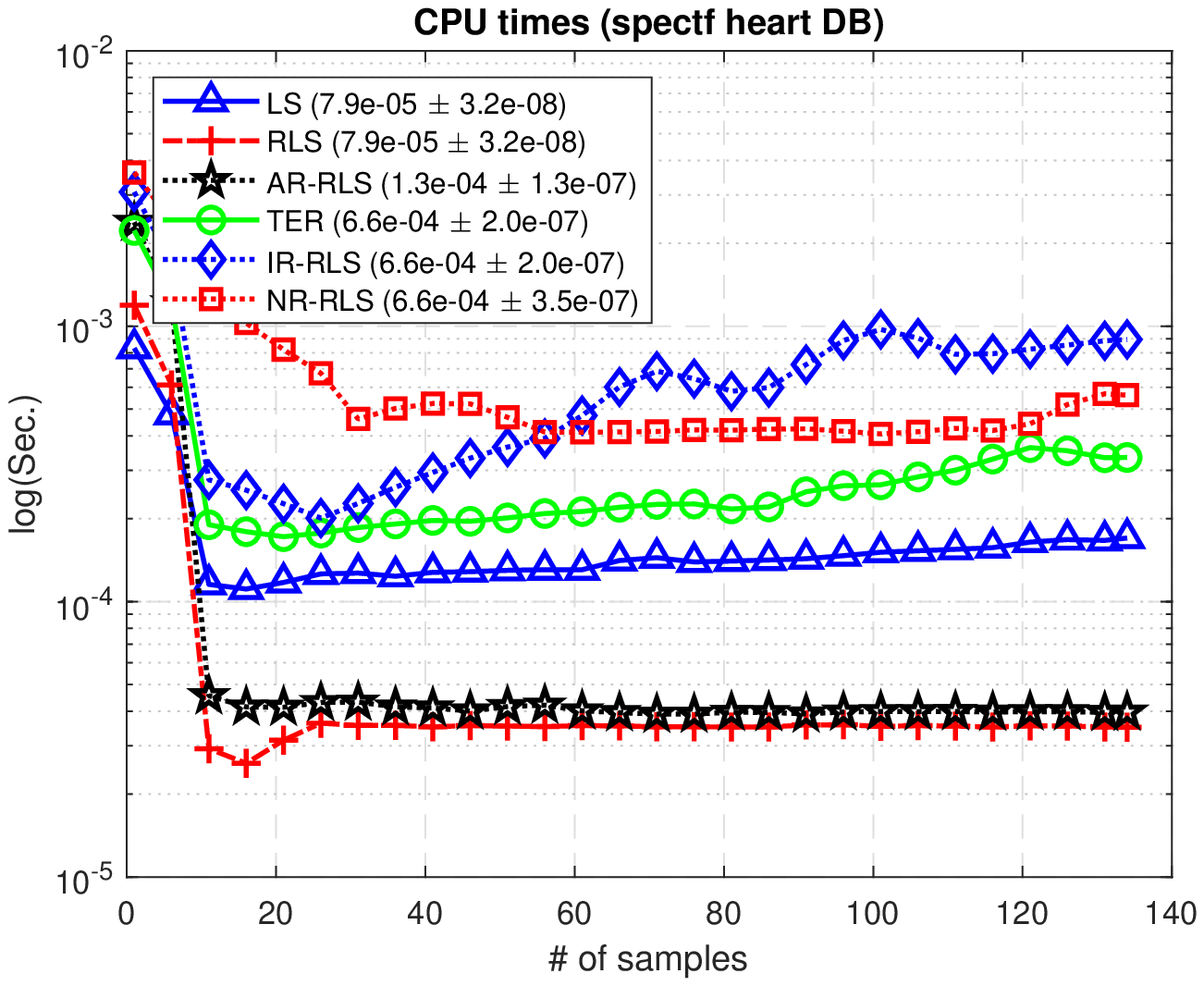}%
					} 
					\caption{SPECTF-heart}%
					
				\end{figure*} 
				
				\begin{figure*}[!h]
					\centering 
					\subfigure[The $L_2$-norm values
					]{ 
						\includegraphics[width=0.3\textwidth]{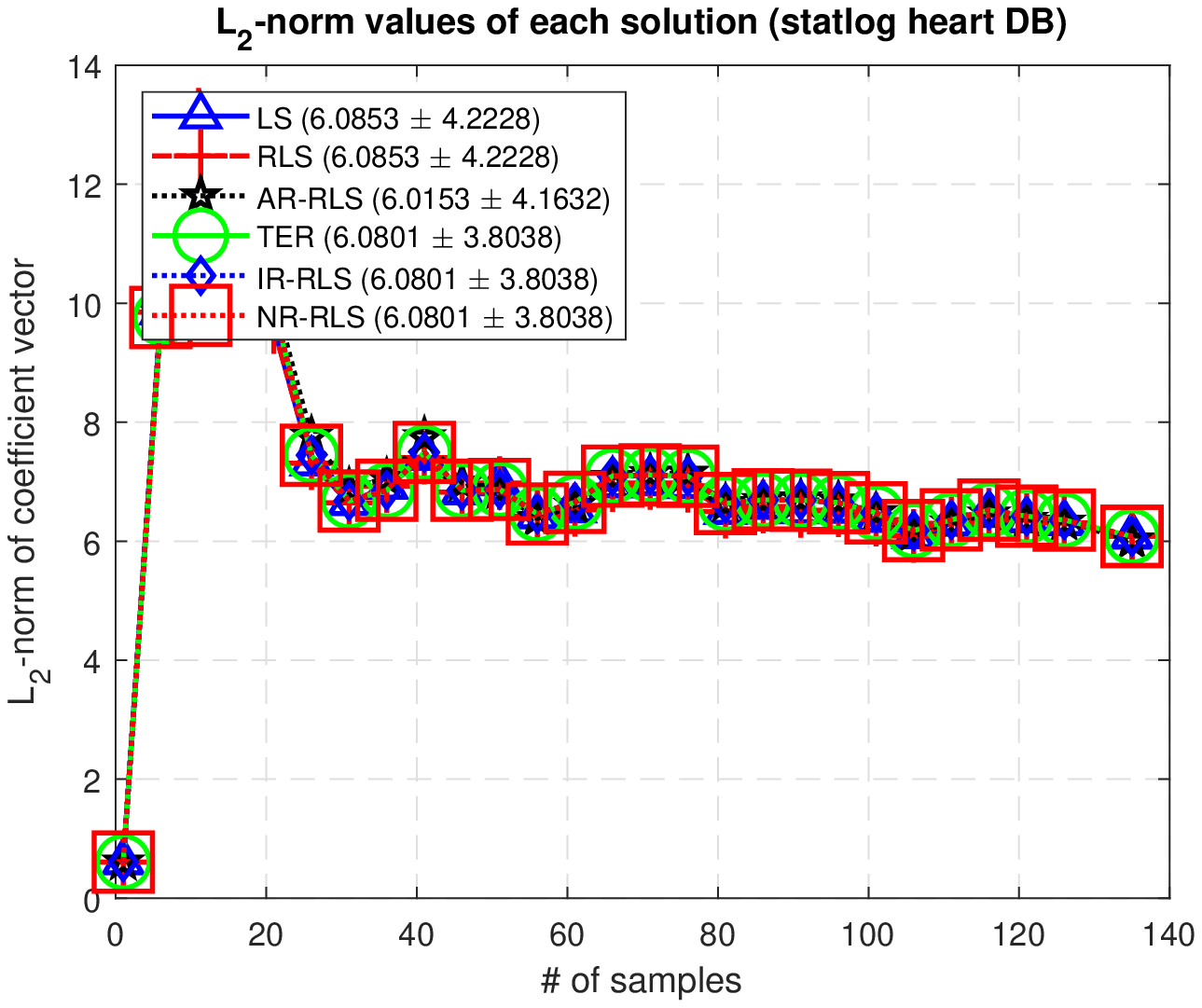}%
					} \hspace{0.4cm}
					\subfigure[The G-means
					]{ 
						\includegraphics[width=0.3\textwidth]{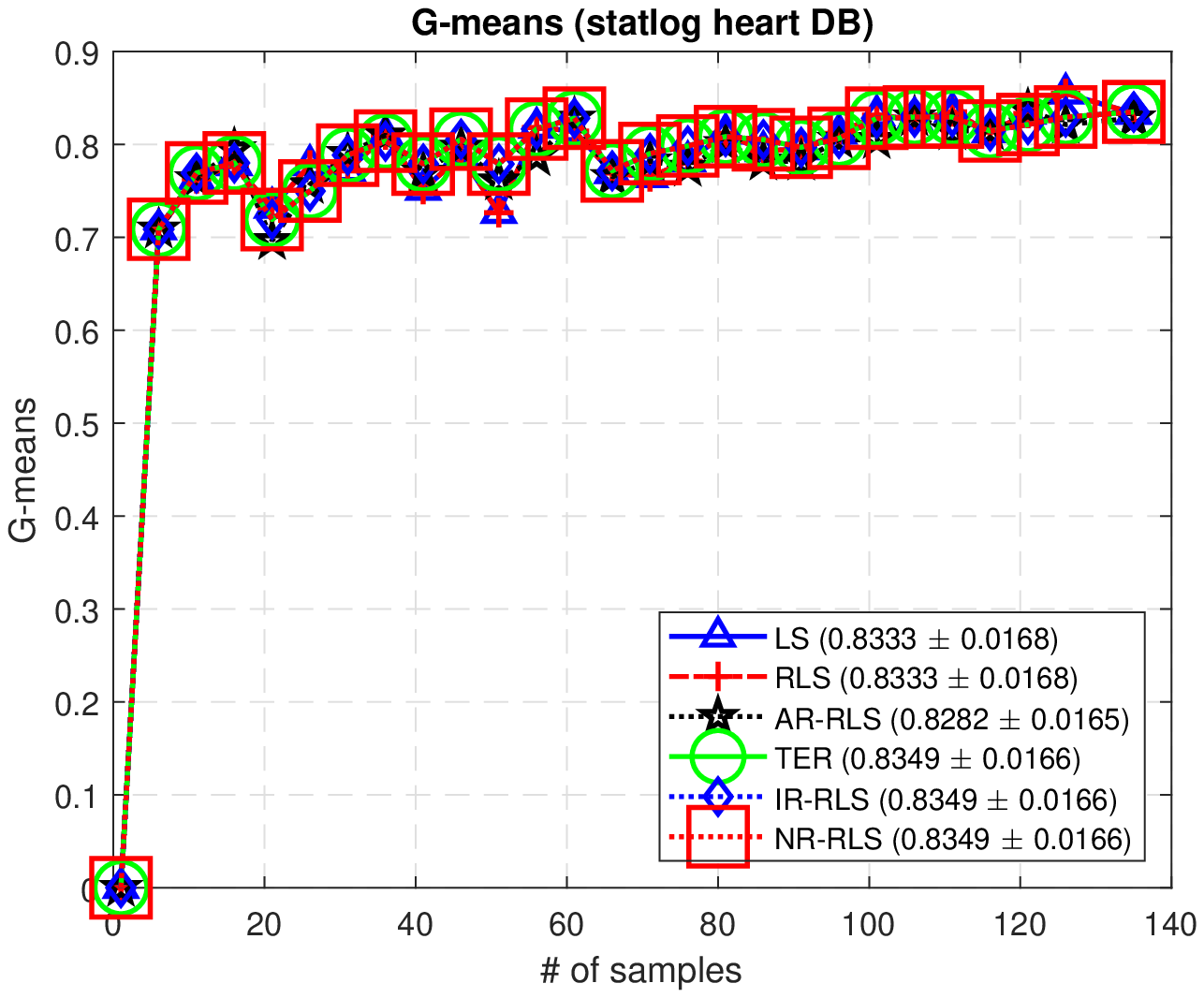}%
					} \hspace{0.4cm}
					\subfigure[The CPU times
					]{ 
						\includegraphics[width=0.3\textwidth]{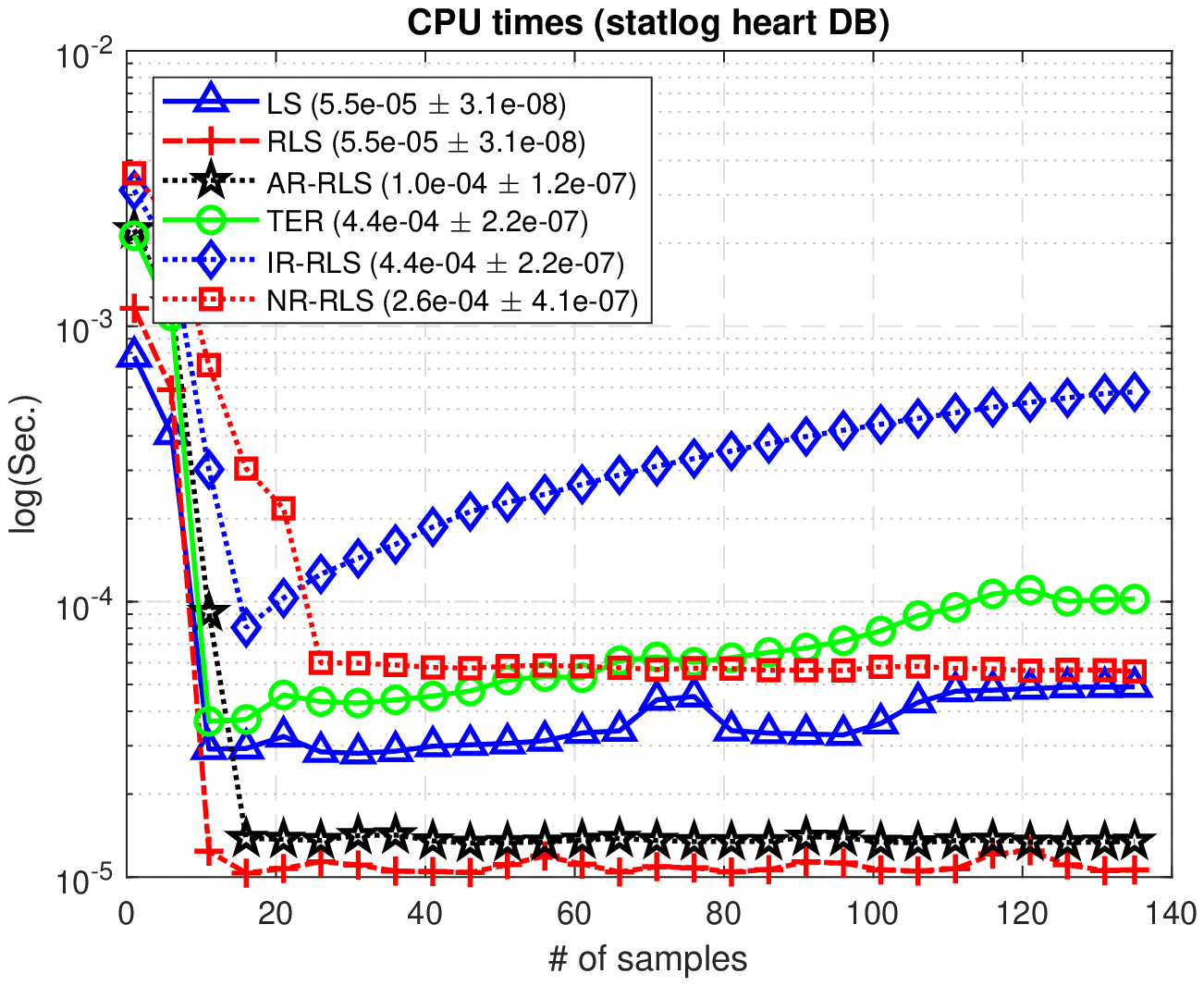}%
					} 
					\caption{StatLog-heart}%
				\end{figure*} 
				\newpage
				
				\begin{figure*}[!h]
					\centering 
					\subfigure[The $L_2$-norm values
					]{ 
						\includegraphics[width=0.3\textwidth]{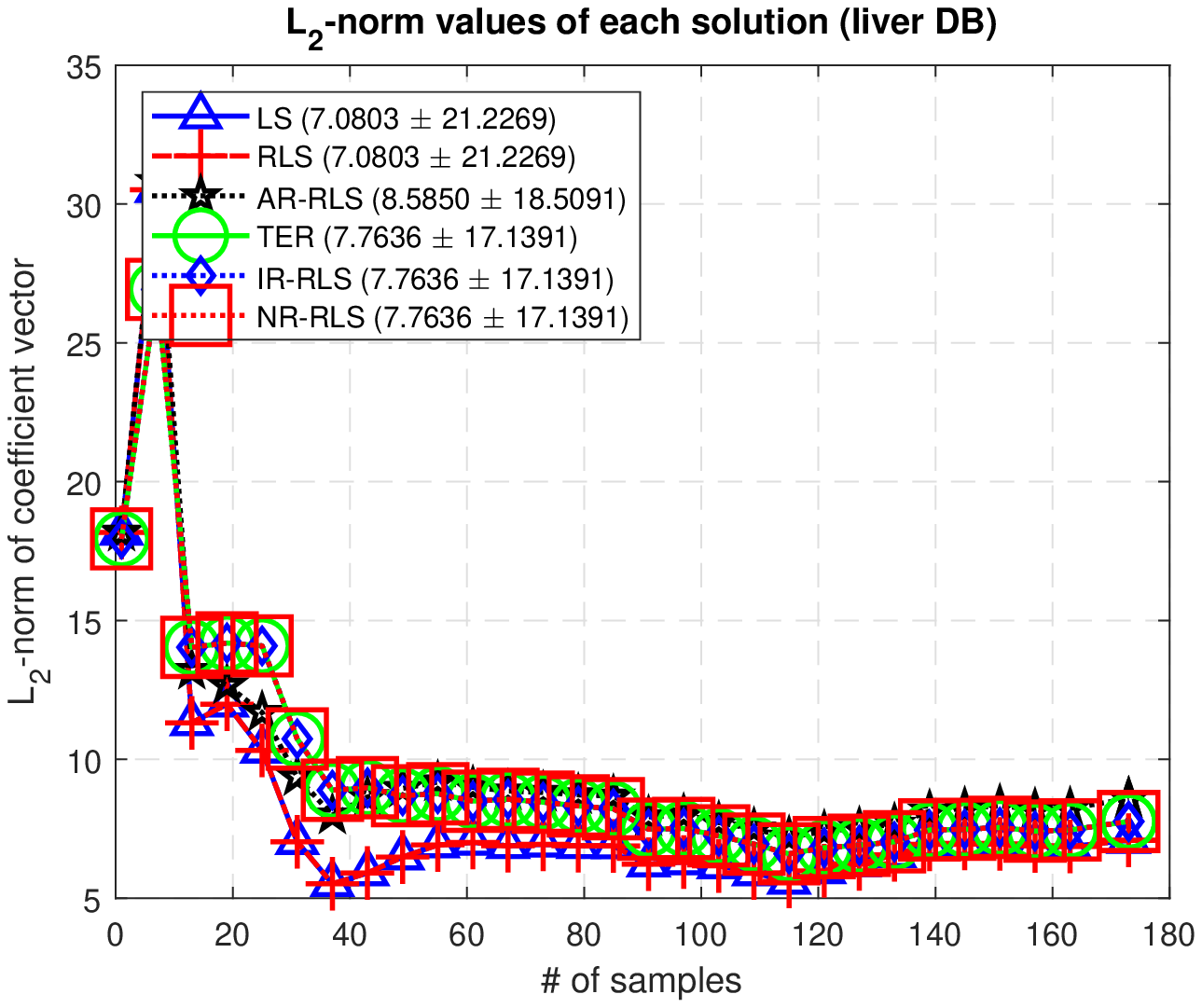}%
					} \hspace{0.4cm}
					\subfigure[The G-means
					]{ 
						\includegraphics[width=0.3\textwidth]{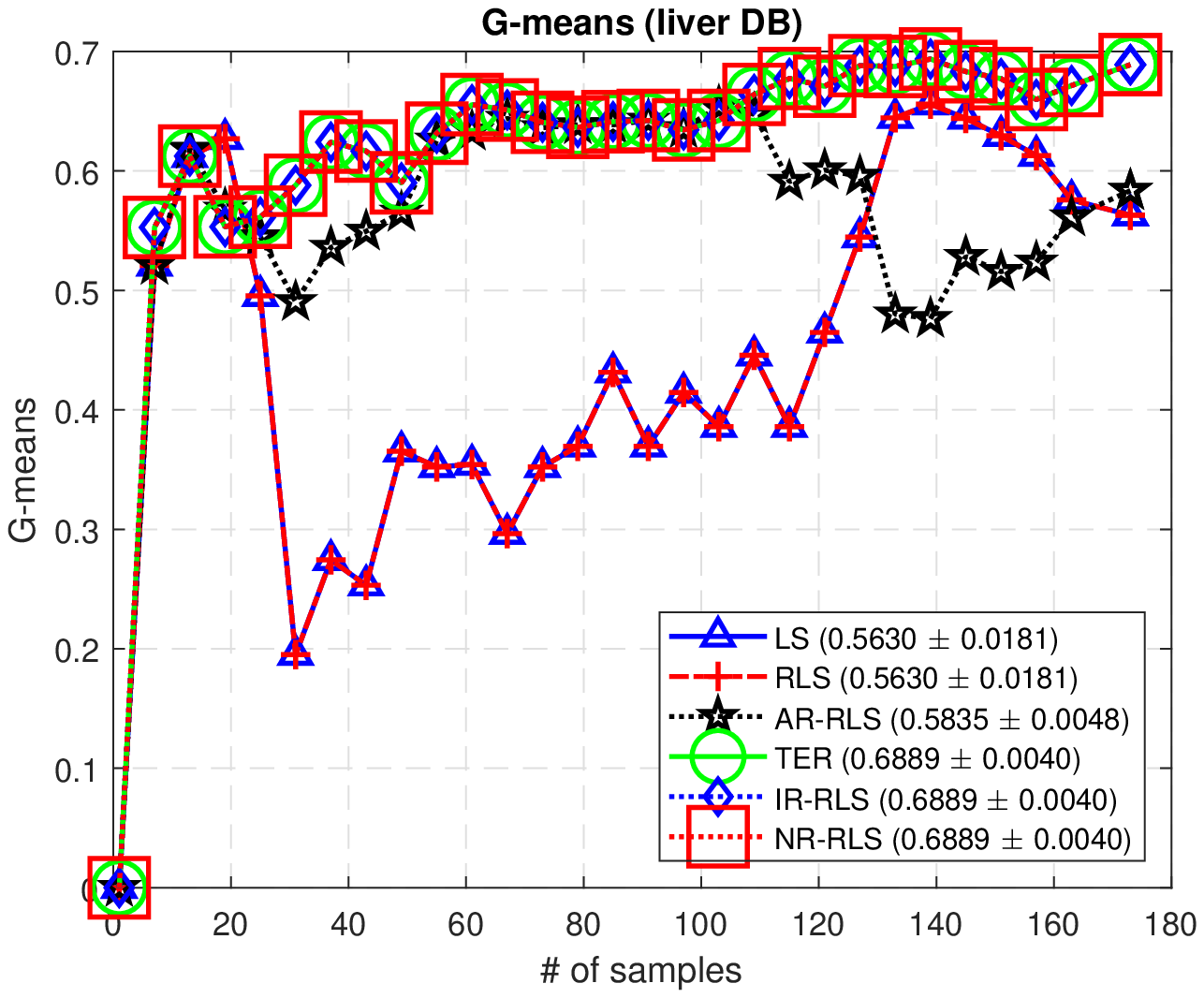}%
					} \hspace{0.4cm}
					\subfigure[The CPU times
					]{ 
						\includegraphics[width=0.3\textwidth]{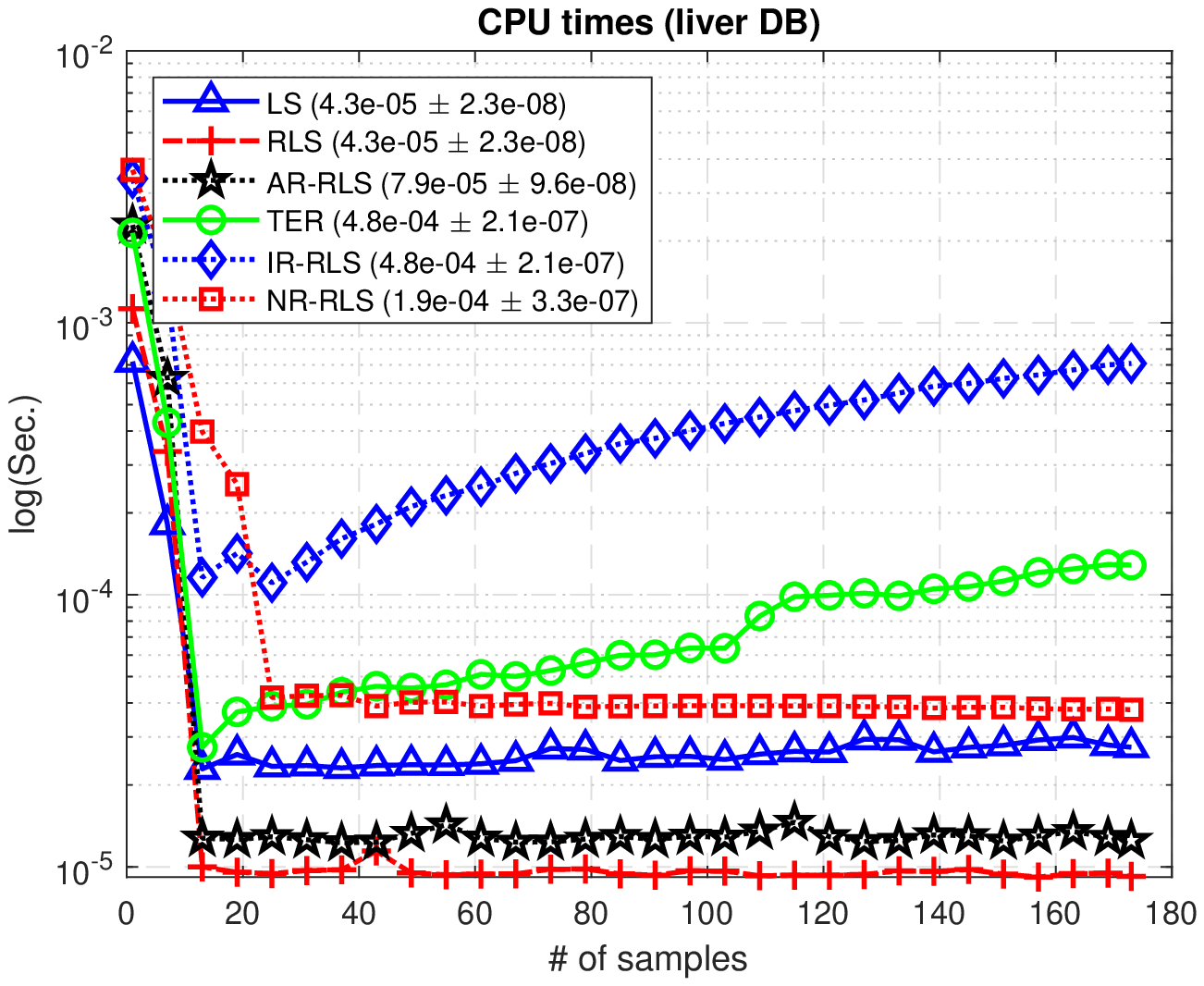}%
					} 
					\caption{BUPA-liver}%
					
				\end{figure*} 
				\vspace{0.4cm}
				\begin{figure*}[!h]
					\centering 
					\subfigure[The $L_2$-norm values
					]{ 
						\includegraphics[width=0.3\textwidth]{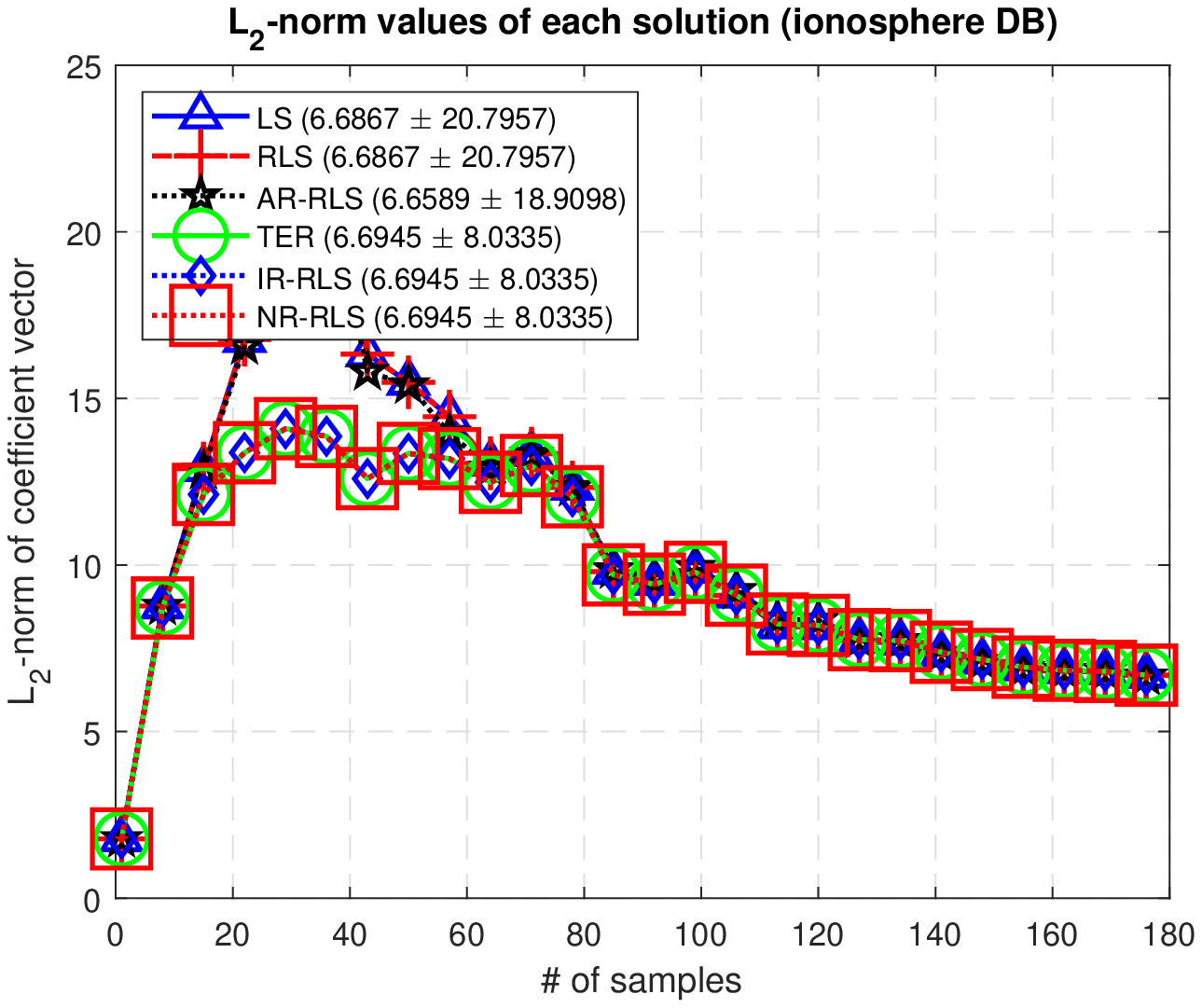}%
					} \hspace{0.4cm}
					\subfigure[The G-means
					]{ 
						\includegraphics[width=0.3\textwidth]{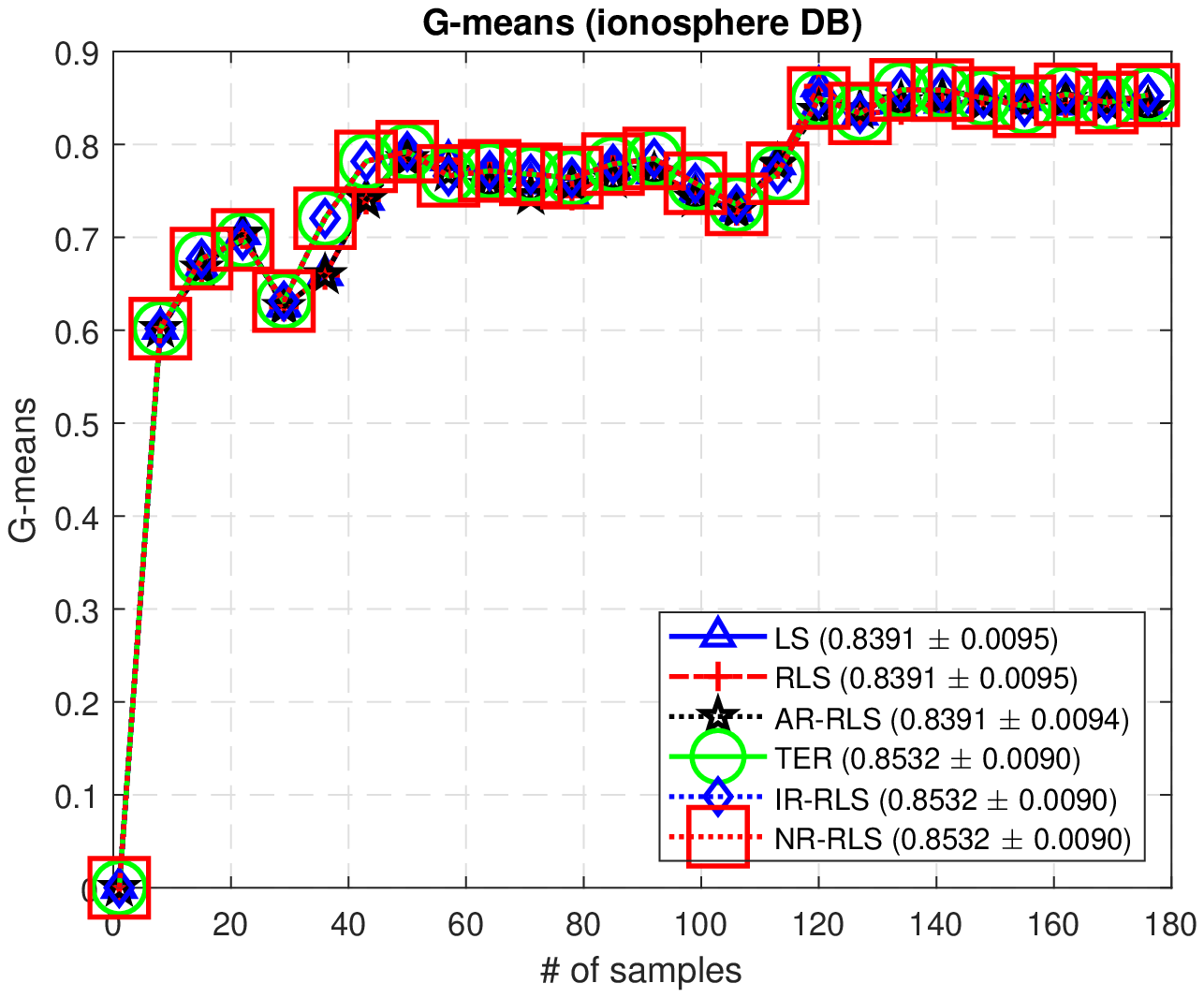}%
					} \hspace{0.4cm}
					\subfigure[The CPU times
					]{ 
						\includegraphics[width=0.3\textwidth]{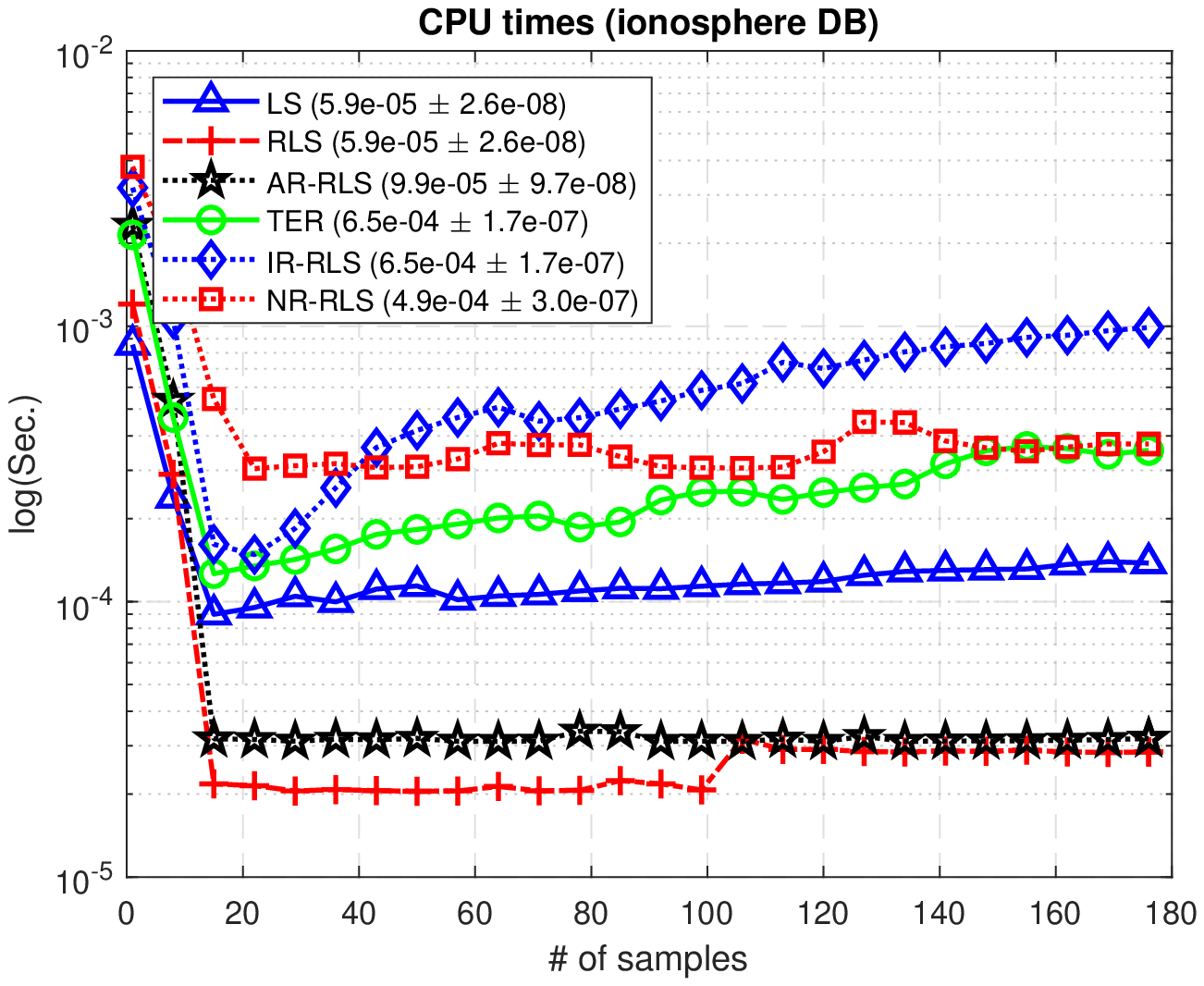}%
					} 
					\caption{Ionosphere}%
				\end{figure*} 
				\vspace{0.4cm}
				\begin{figure*}[!h]
					\centering 
					\subfigure[The $L_2$-norm values
					]{ 
						\includegraphics[width=0.3\textwidth]{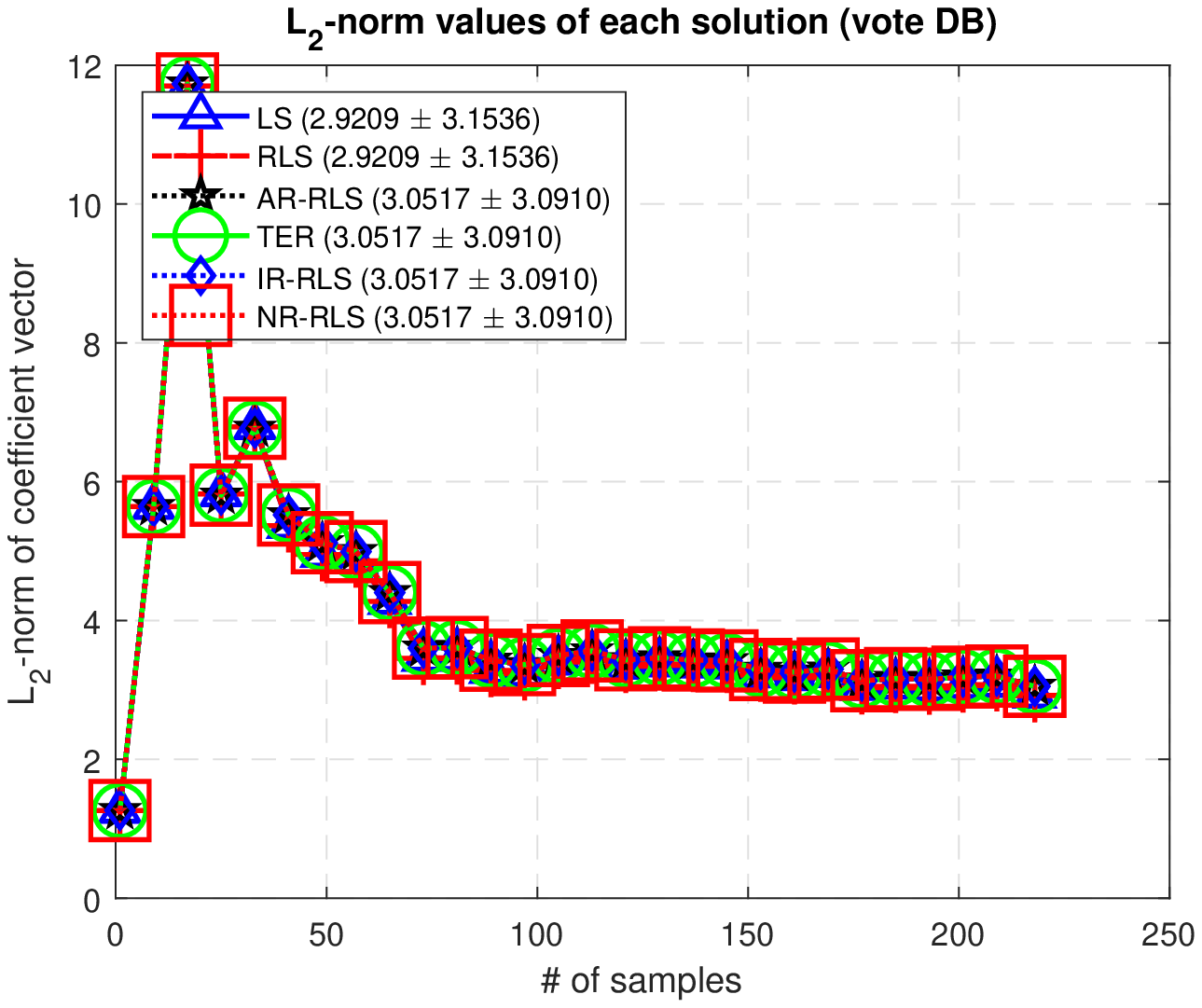}%
					} \hspace{0.4cm}
					\subfigure[The G-means
					]{ 
						\includegraphics[width=0.3\textwidth]{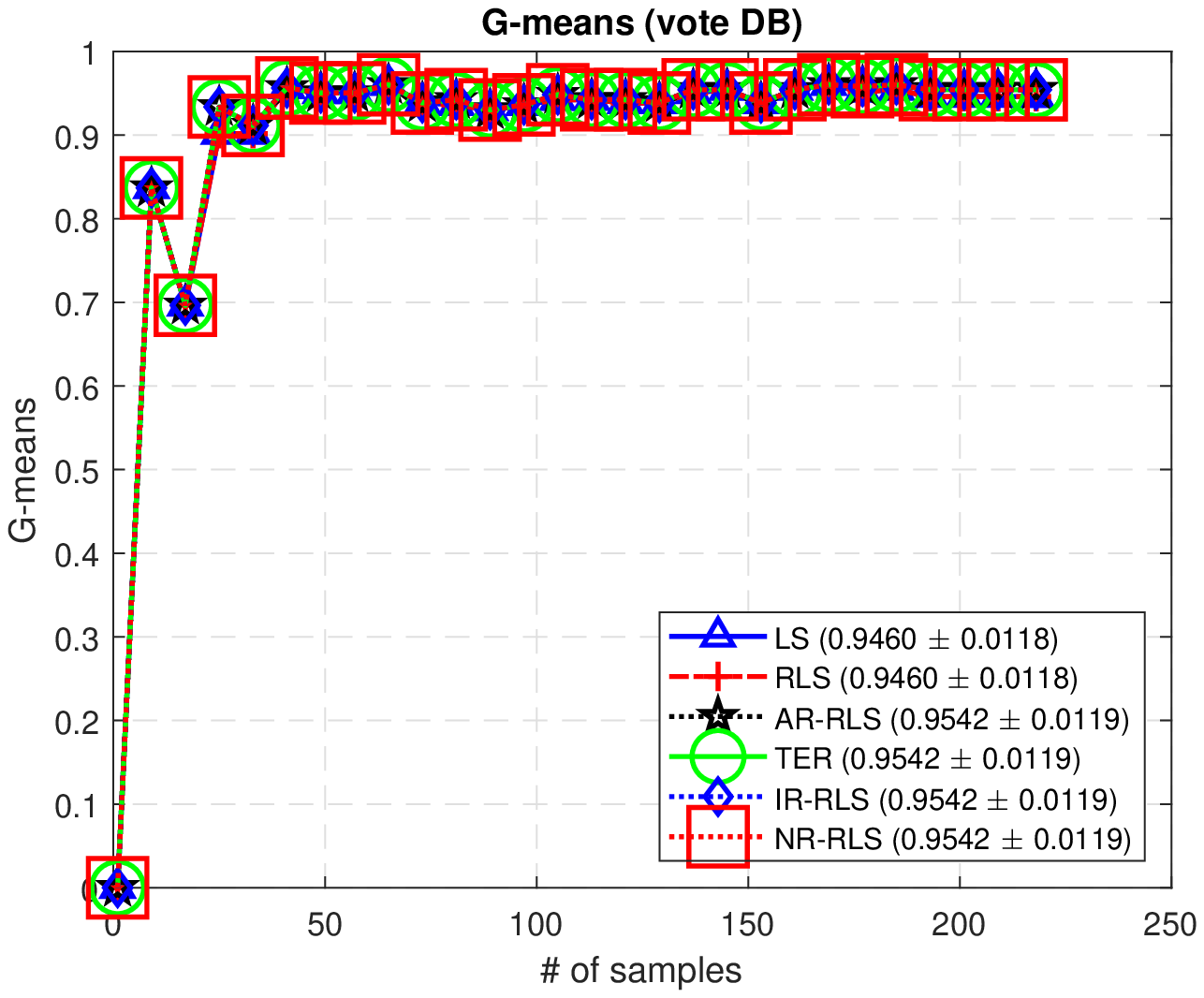}%
					} \hspace{0.4cm}
					\subfigure[The CPU times
					]{ 
						\includegraphics[width=0.3\textwidth]{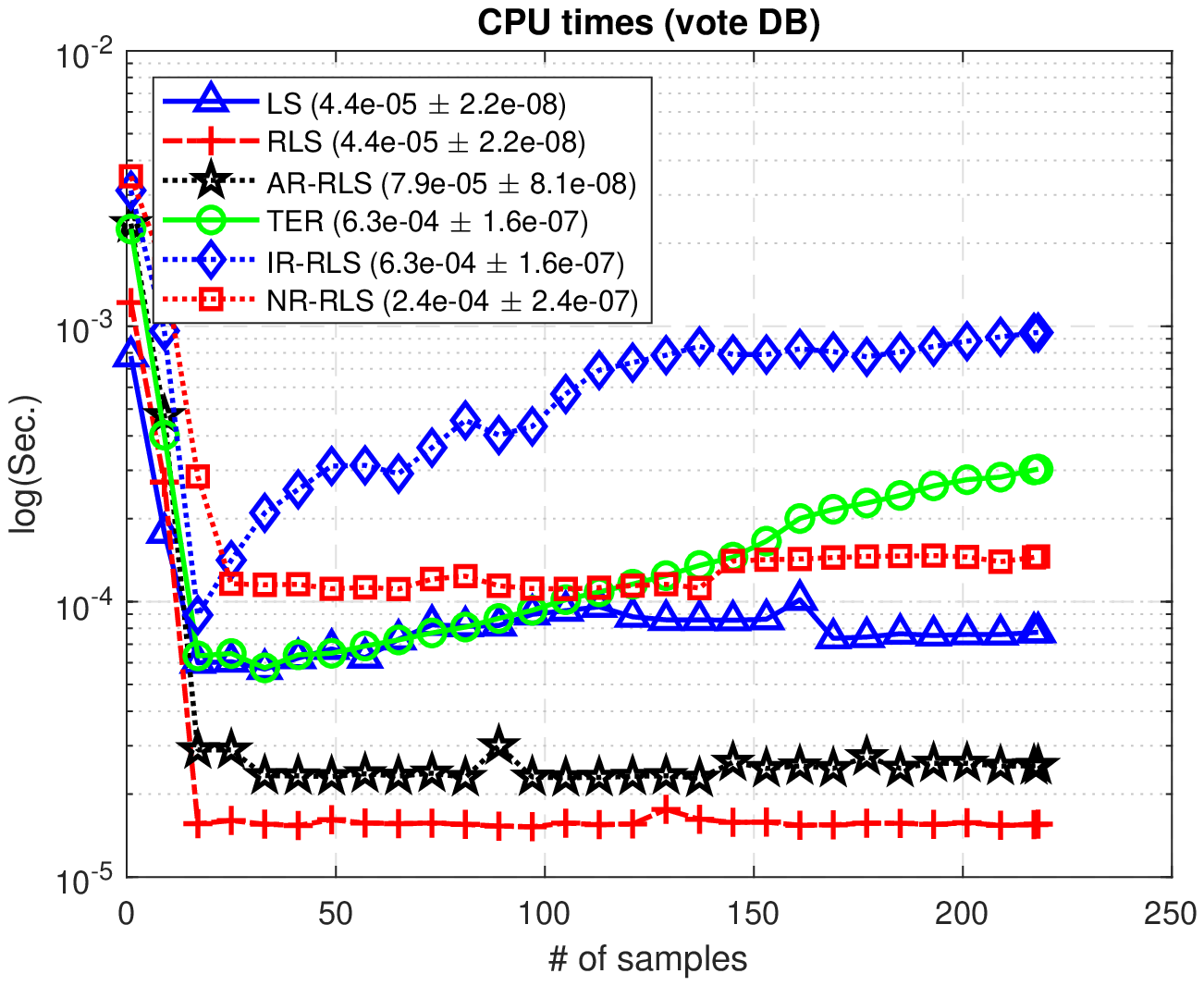}%
					} 
					\caption{Votes}%
					
				\end{figure*} 
				
				\begin{figure*}[!h]
					\centering 
					\subfigure[The $L_2$-norm values
					]{ 
						\includegraphics[width=0.3\textwidth]{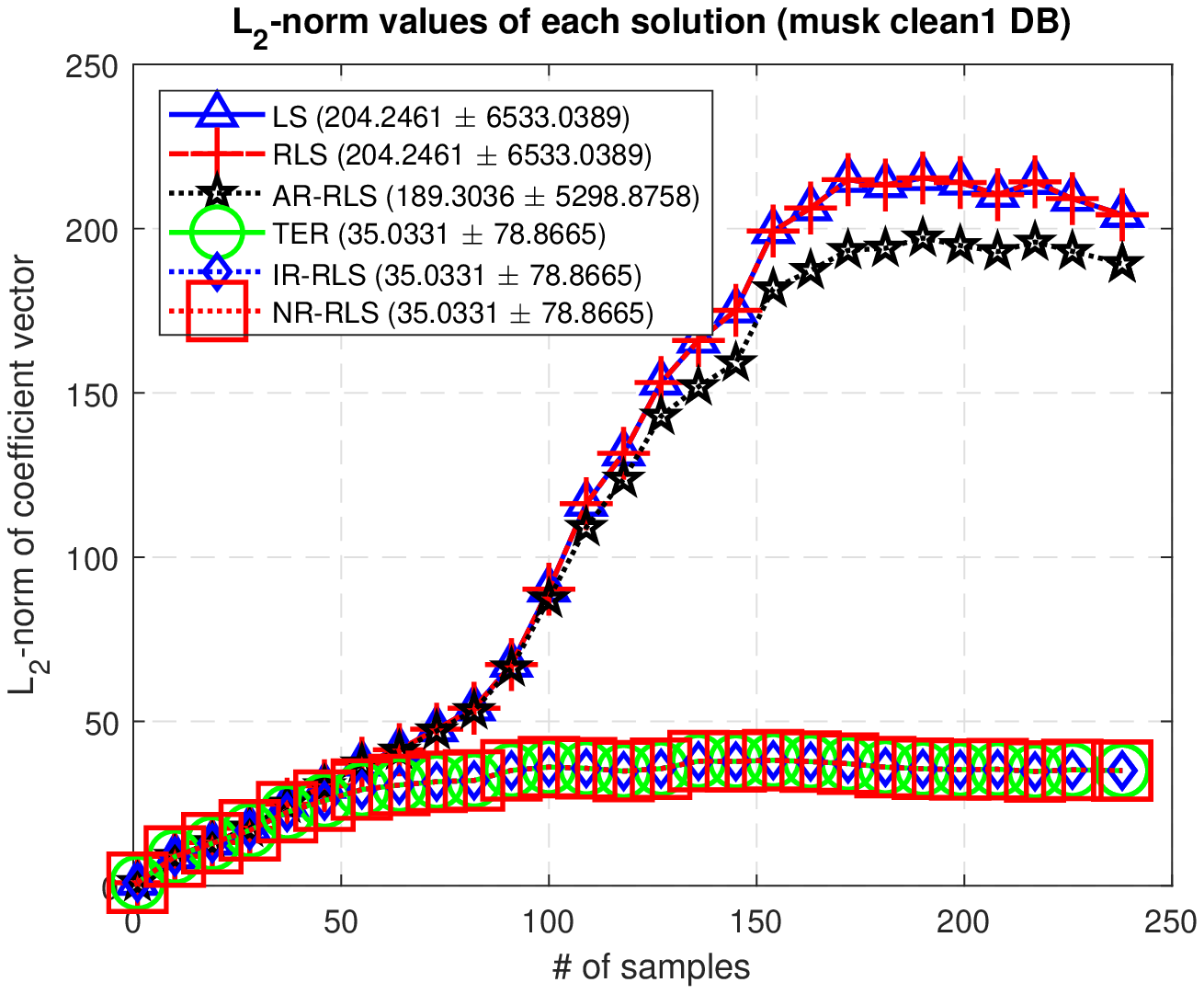}%
					} \hspace{0.4cm}
					\subfigure[The G-means
					]{ 
						\includegraphics[width=0.3\textwidth]{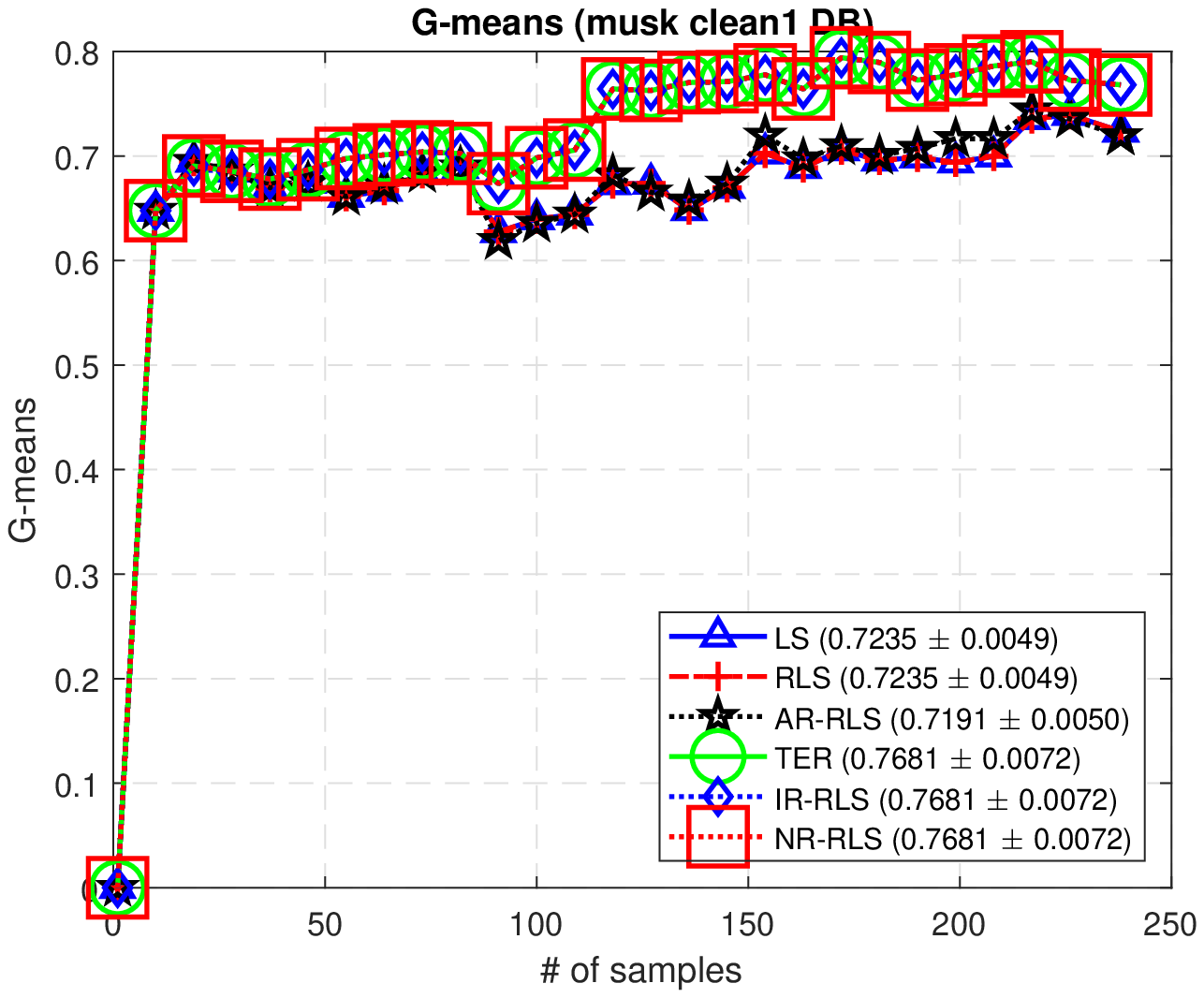}%
					} \hspace{0.4cm}
					\subfigure[The CPU times
					]{ 
						\includegraphics[width=0.3\textwidth]{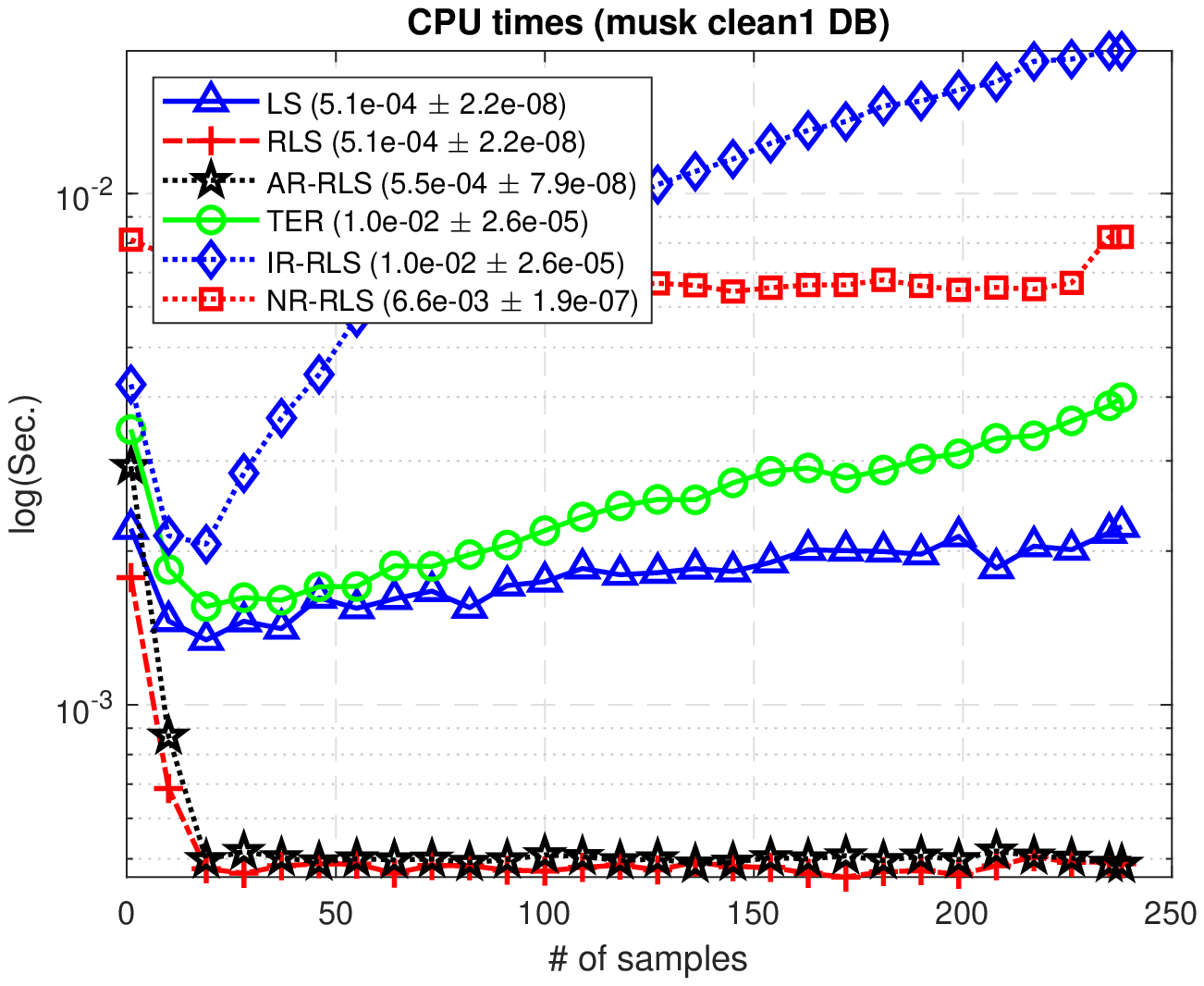}%
					} 
					\caption{Musk-clean-1}%
					
				\end{figure*} 
				\newpage
				\begin{figure*}[!h]
					\centering 
					\subfigure[The $L_2$-norm values
					]{ 
						\includegraphics[width=0.3\textwidth]{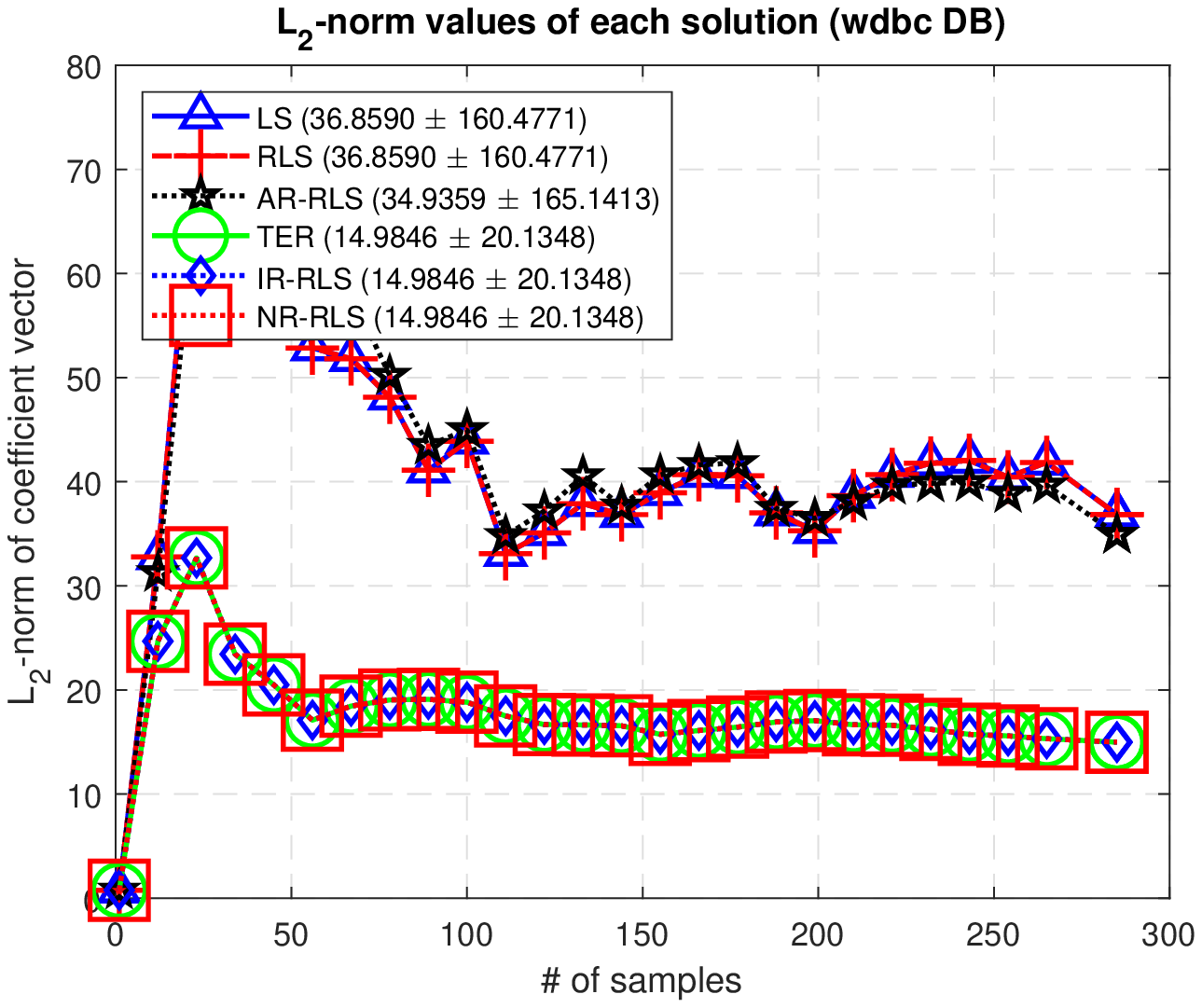}%
					} \hspace{0.4cm}
					\subfigure[The G-means
					]{ 
						\includegraphics[width=0.3\textwidth]{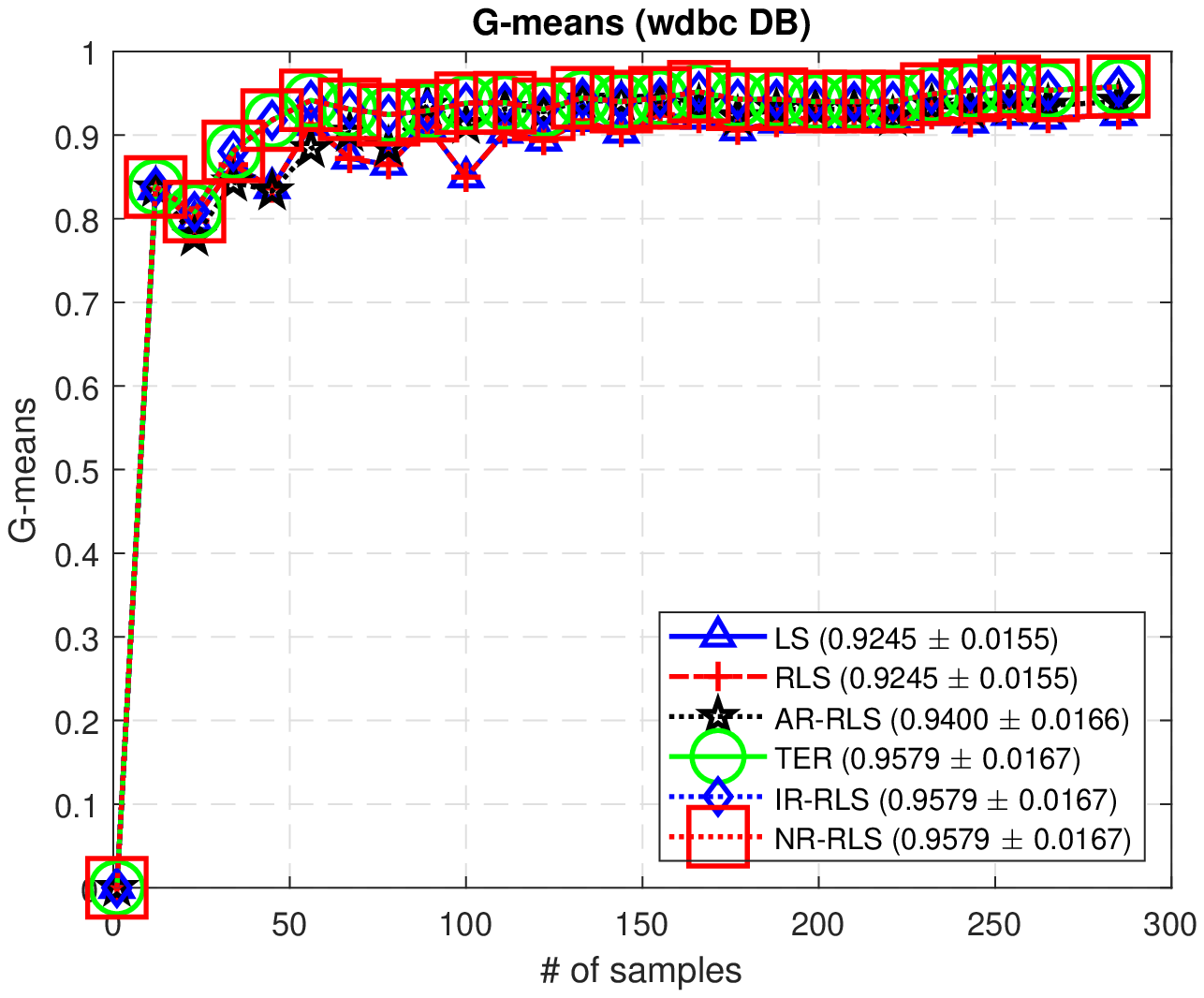}%
					} \hspace{0.4cm}
					\subfigure[The CPU times
					]{ 
						\includegraphics[width=0.3\textwidth]{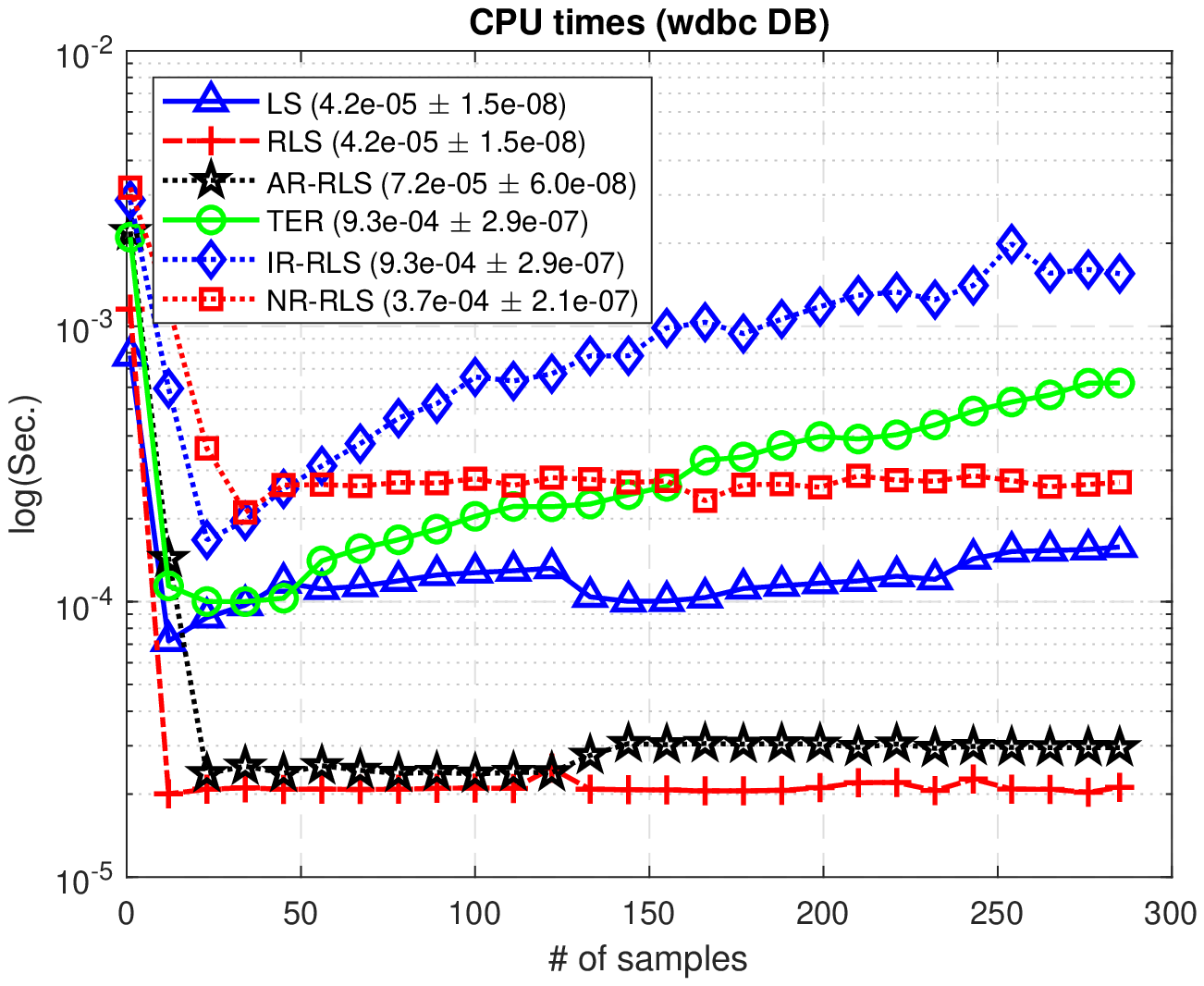}%
					} 
					\caption{Wdbc}%
				\end{figure*} 
				
				\begin{figure*}[!h]
					\centering 
					\subfigure[The $L_2$-norm values
					]{ 
						\includegraphics[width=0.3\textwidth]{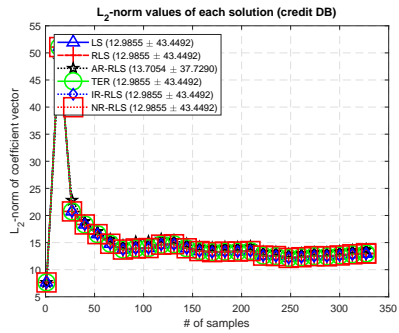}%
					} \hspace{0.4cm}
					\subfigure[The G-means
					]{ 
						\includegraphics[width=0.3\textwidth]{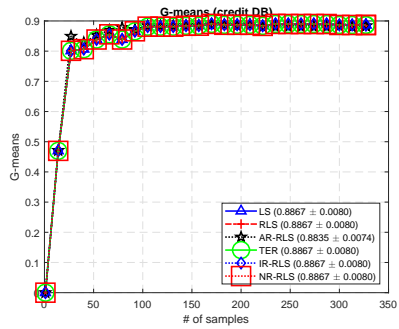}%
					} \hspace{0.4cm}
					\subfigure[The CPU times
					]{ 
						\includegraphics[width=0.3\textwidth]{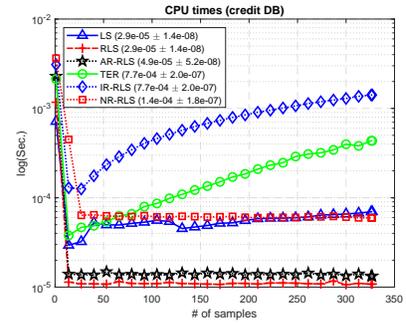}%
					} 
					\caption{Credit-app}%
				\end{figure*} 
				
				\begin{figure*}[!h]
					\centering 
					\subfigure[The $L_2$-norm values
					]{ 
						\includegraphics[width=0.3\textwidth]{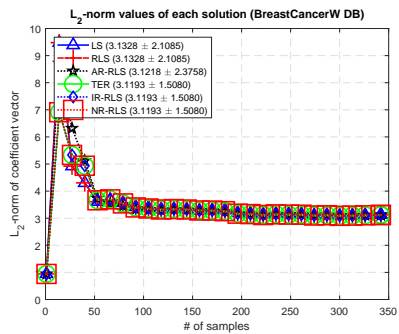}%
					} \hspace{0.4cm}
					\subfigure[The G-means
					]{ 
						\includegraphics[width=0.3\textwidth]{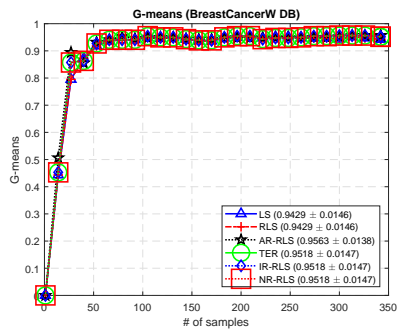}%
					} \hspace{0.4cm}
					\subfigure[The CPU times
					]{ 
						\includegraphics[width=0.3\textwidth]{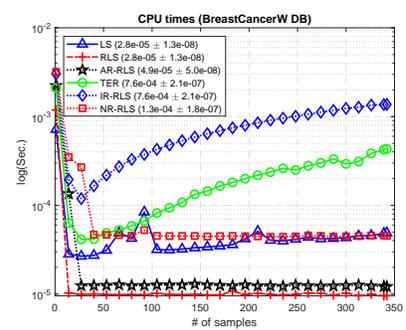}%
					} 
					\caption{Breast-cancer-W}%
					
				\end{figure*} 
				
				\begin{figure*}[!h]
					\centering 
					\subfigure[The $L_2$-norm values
					]{ 
						\includegraphics[width=0.3\textwidth]{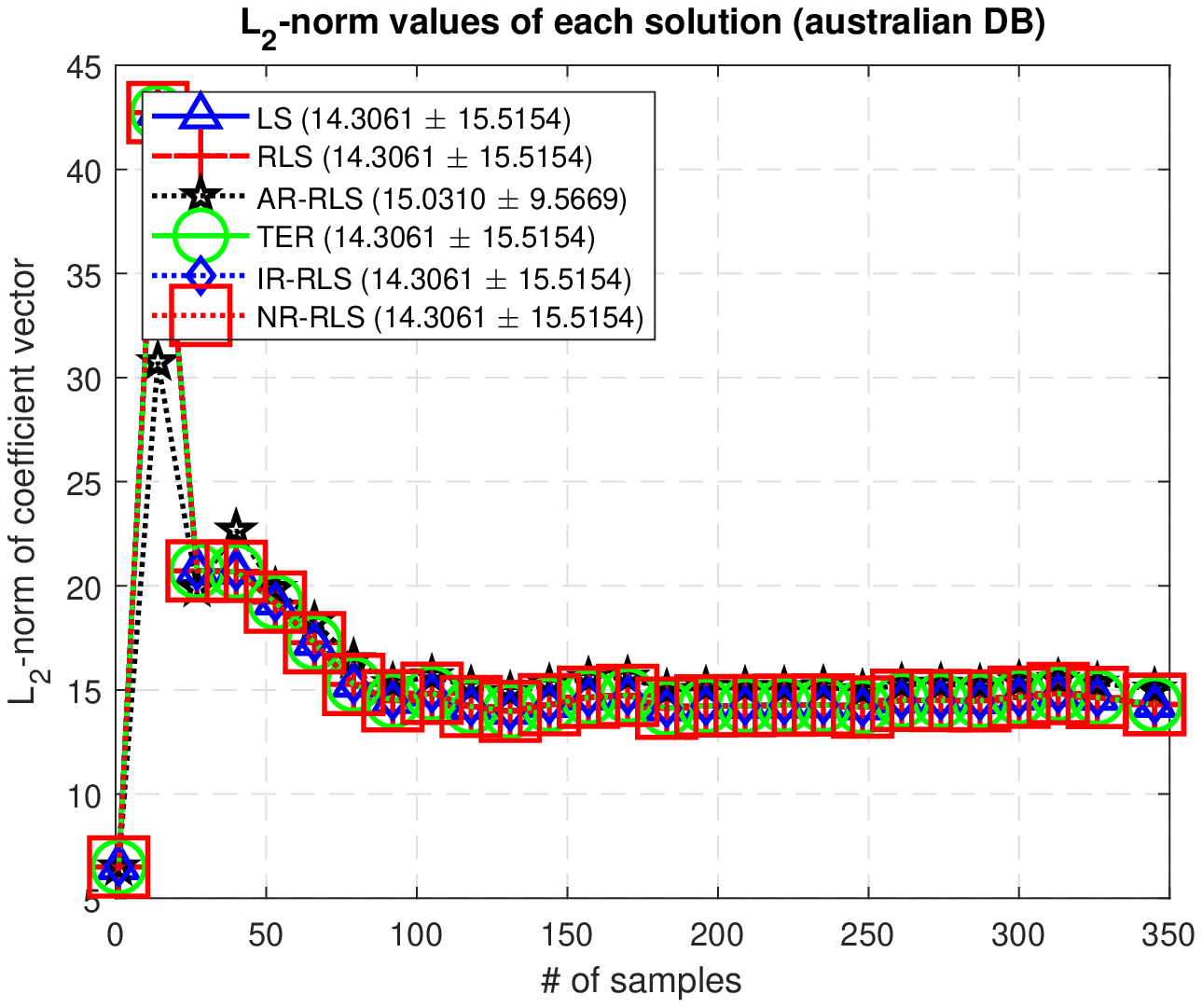}%
					} \hspace{0.4cm}
					\subfigure[The G-means
					]{ 
						\includegraphics[width=0.3\textwidth]{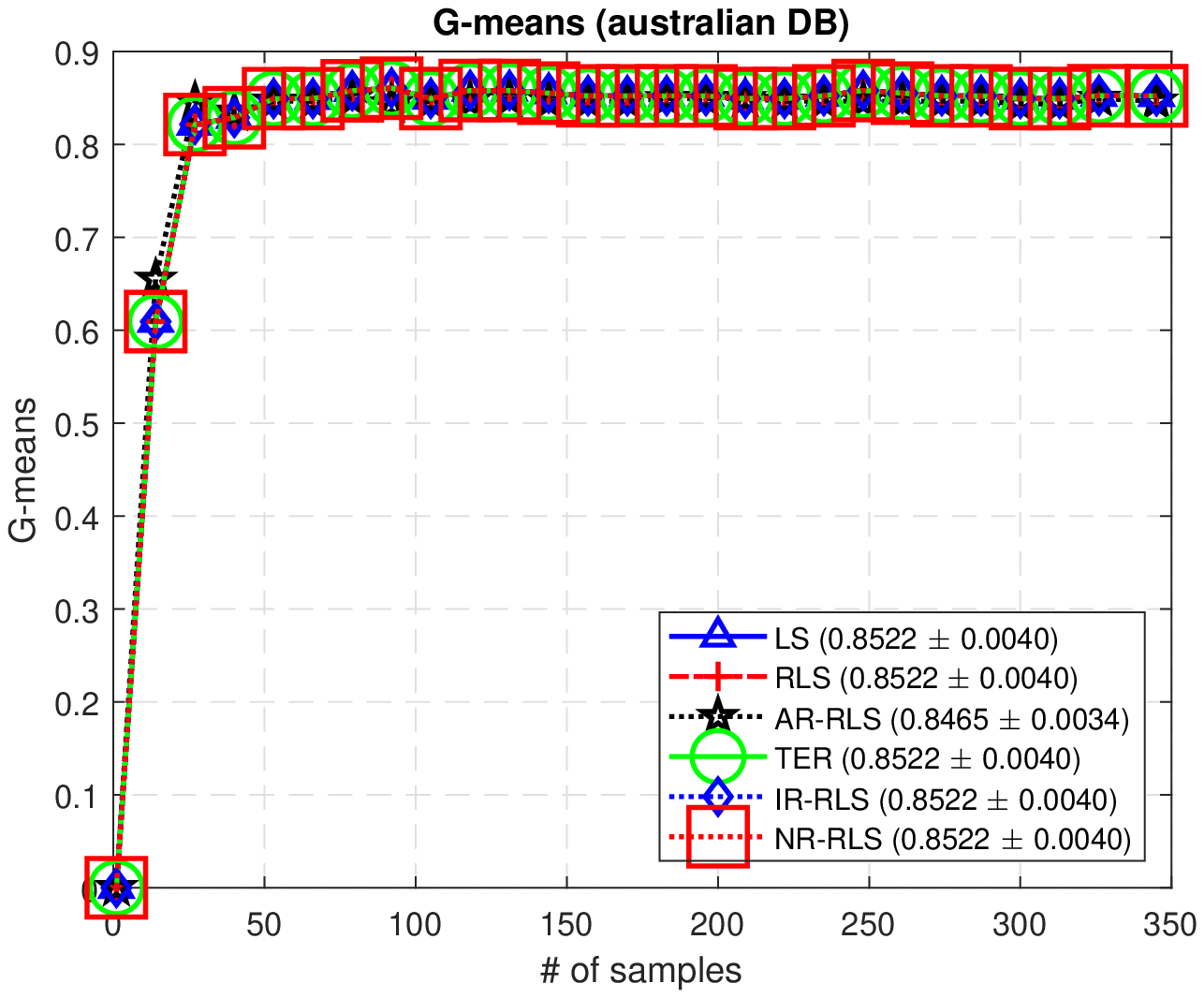}%
					} \hspace{0.4cm}
					\subfigure[The CPU times
					]{ 
						\includegraphics[width=0.3\textwidth]{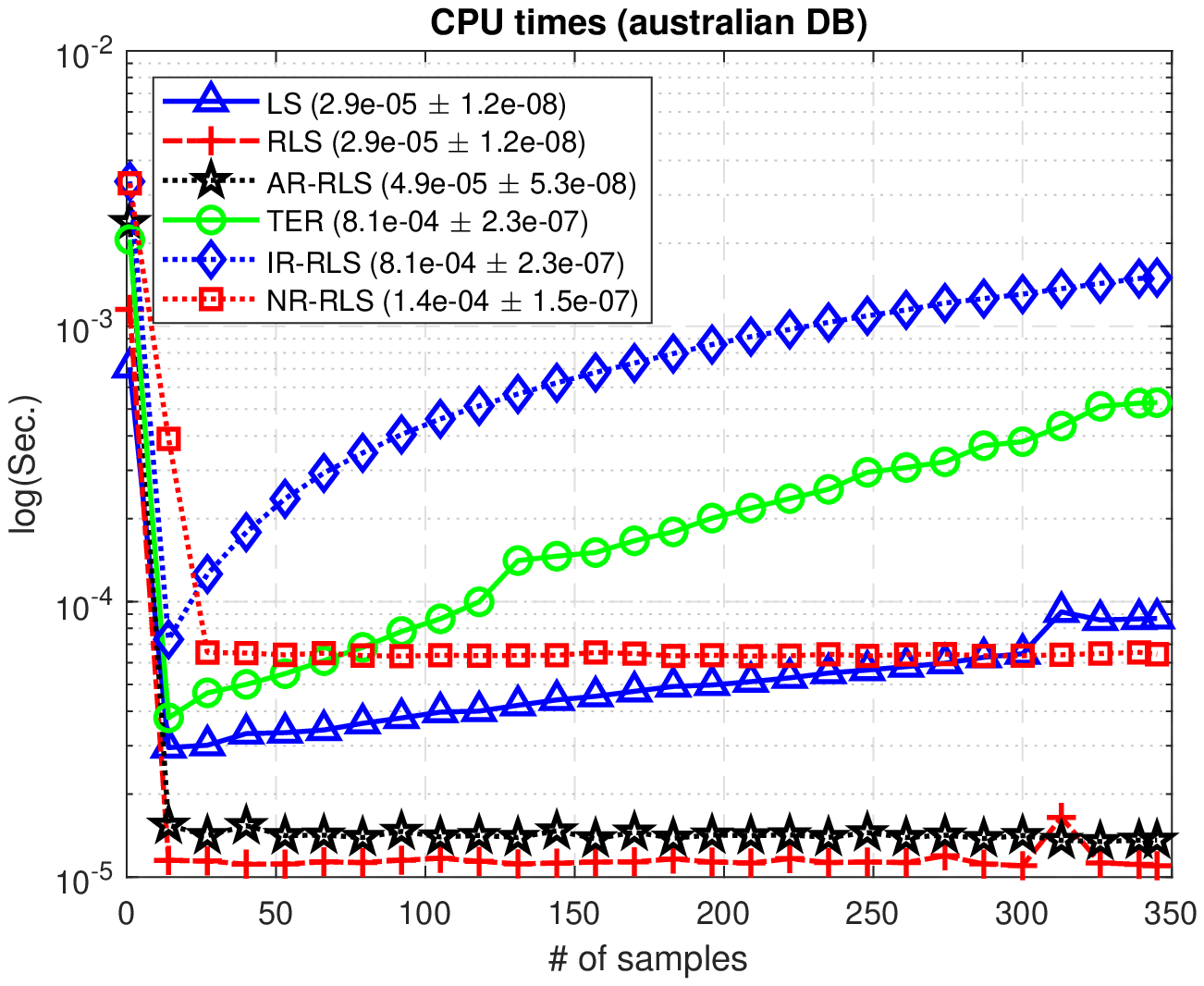}%
					} 
					\caption{Statlog-australian}%
					
				\end{figure*} 
				\newpage
				\begin{figure*}[!h]
					\centering 
					\subfigure[The $L_2$-norm values
					]{ 
						\includegraphics[width=0.3\textwidth]{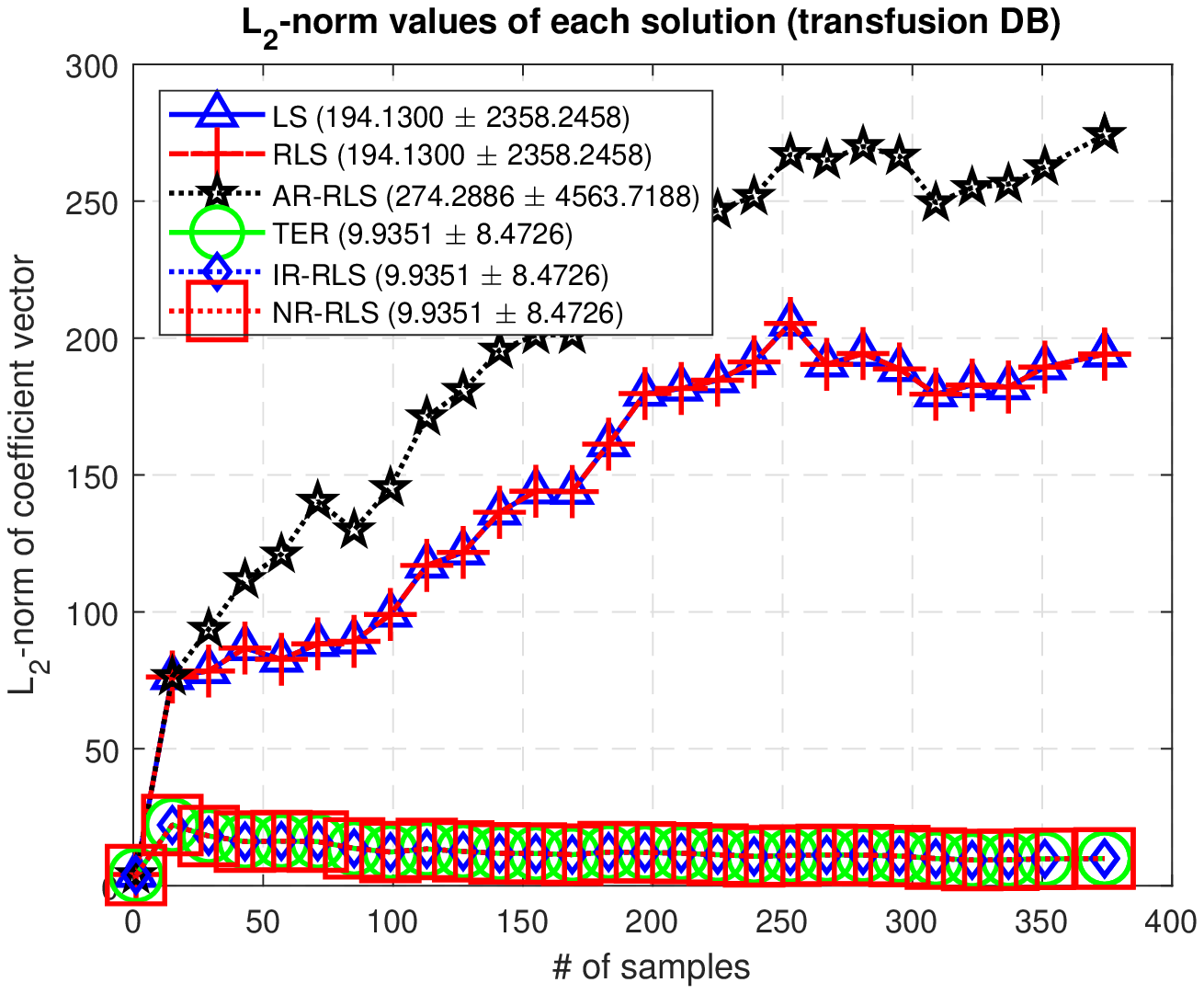}%
					} \hspace{0.4cm}
					\subfigure[The G-means
					]{ 
						\includegraphics[width=0.3\textwidth]{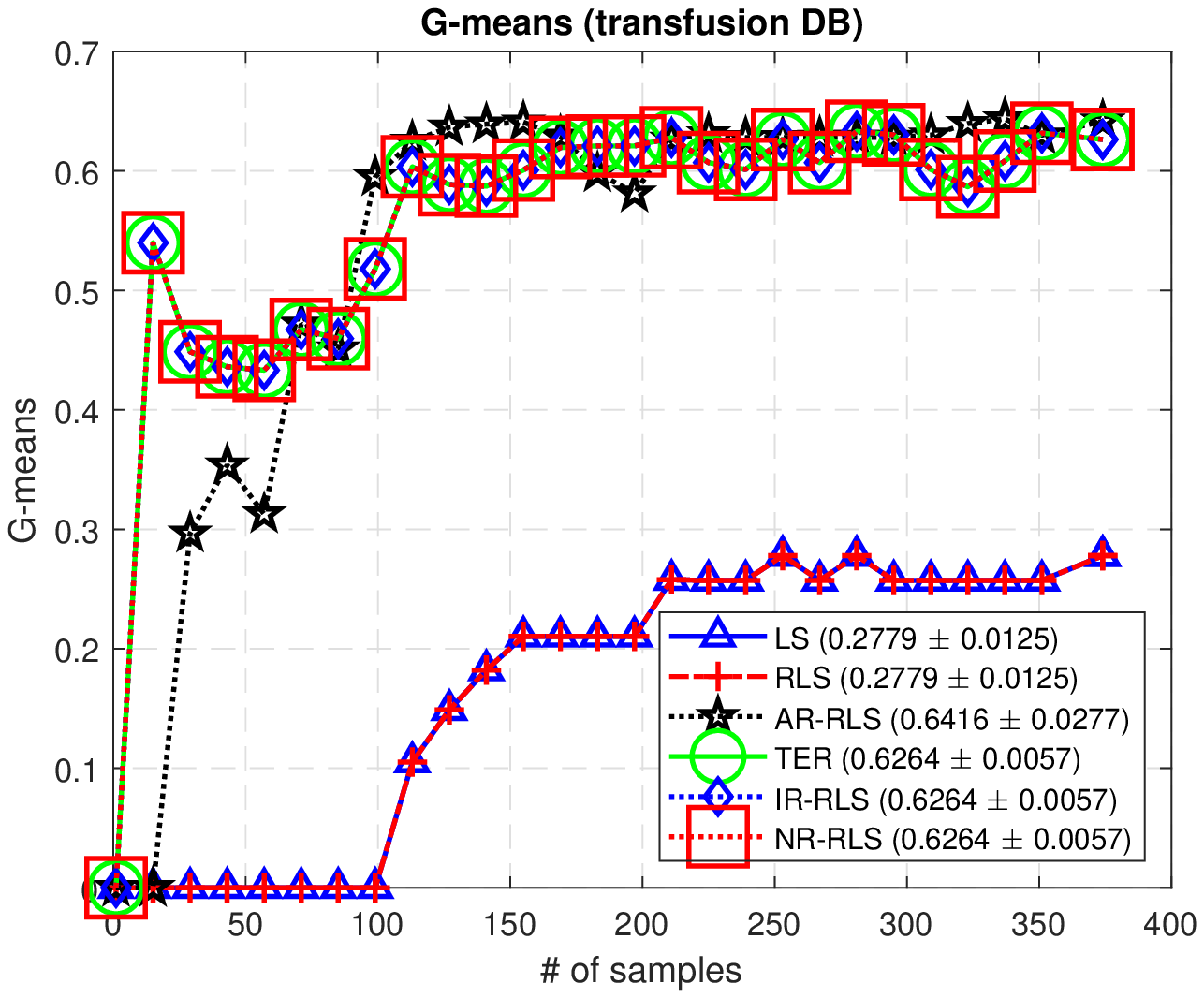}%
					} \hspace{0.4cm}
					\subfigure[The CPU times
					]{ 
						\includegraphics[width=0.3\textwidth]{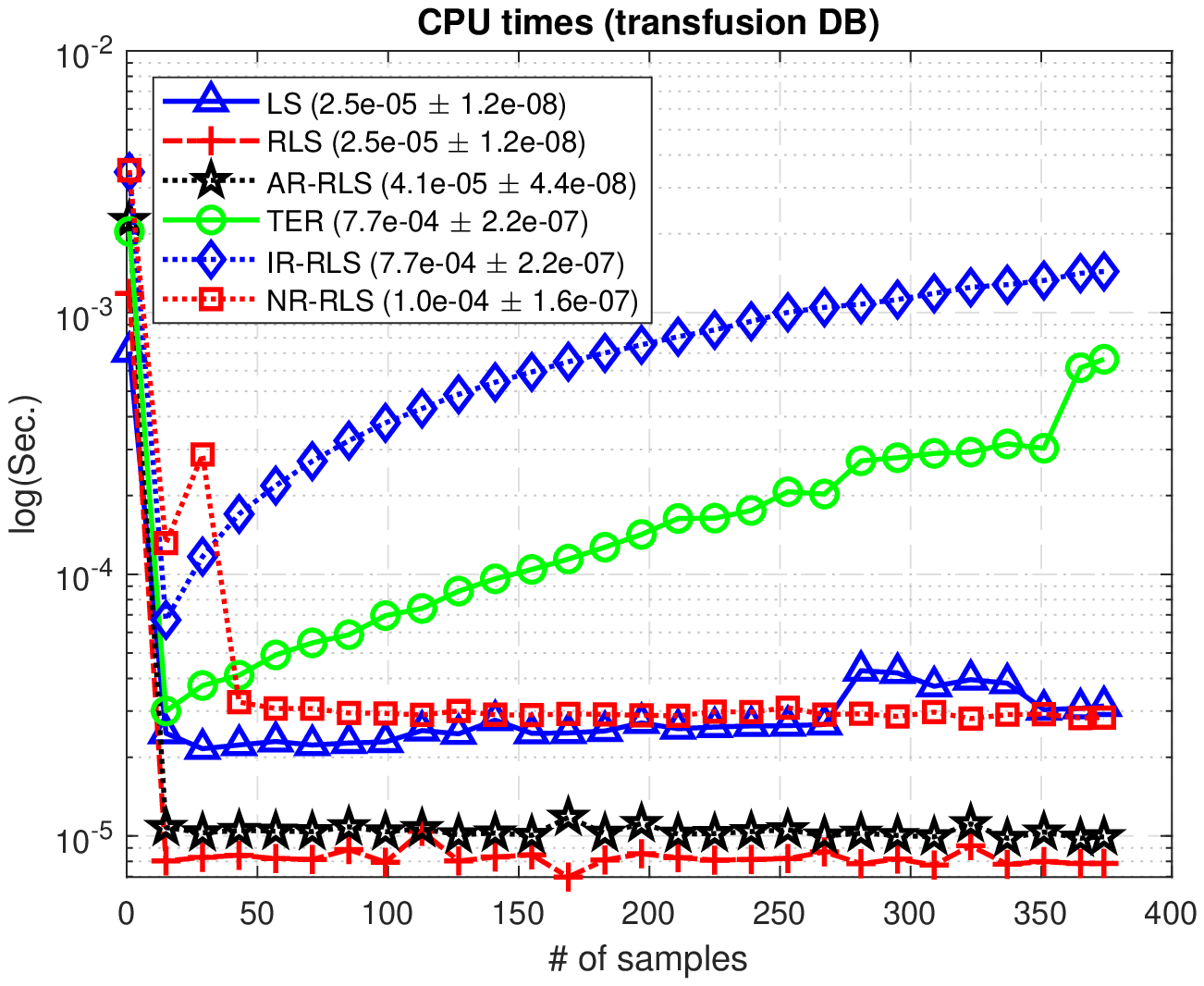}%
					} 
					\caption{Blood-transfusion}%
					
				\end{figure*} 
				
				\begin{figure*}[!h]
					\centering 
					\subfigure[The $L_2$-norm values
					]{ 
						\includegraphics[width=0.3\textwidth]{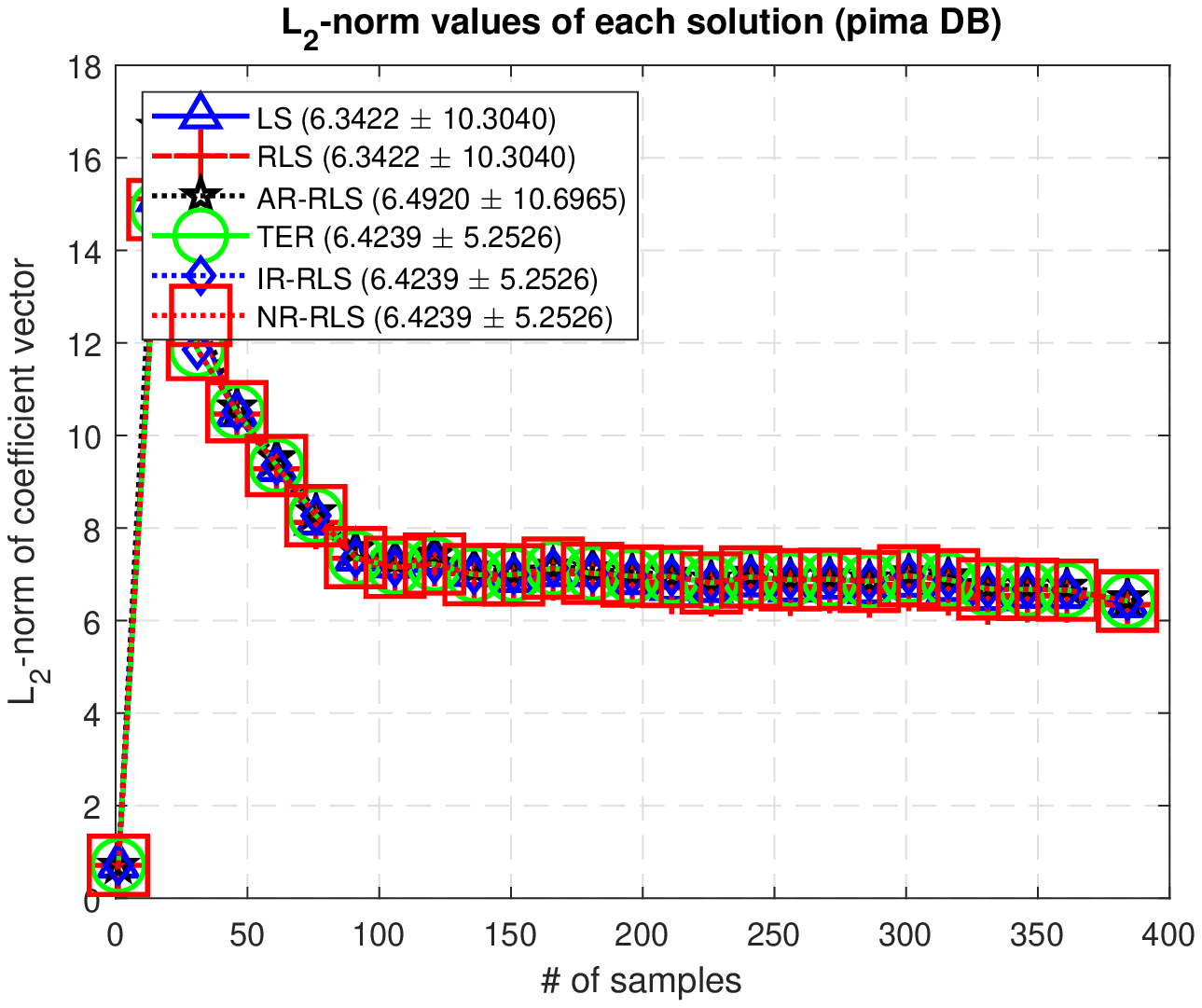}%
					} \hspace{0.4cm}
					\subfigure[The G-means
					]{ 
						\includegraphics[width=0.3\textwidth]{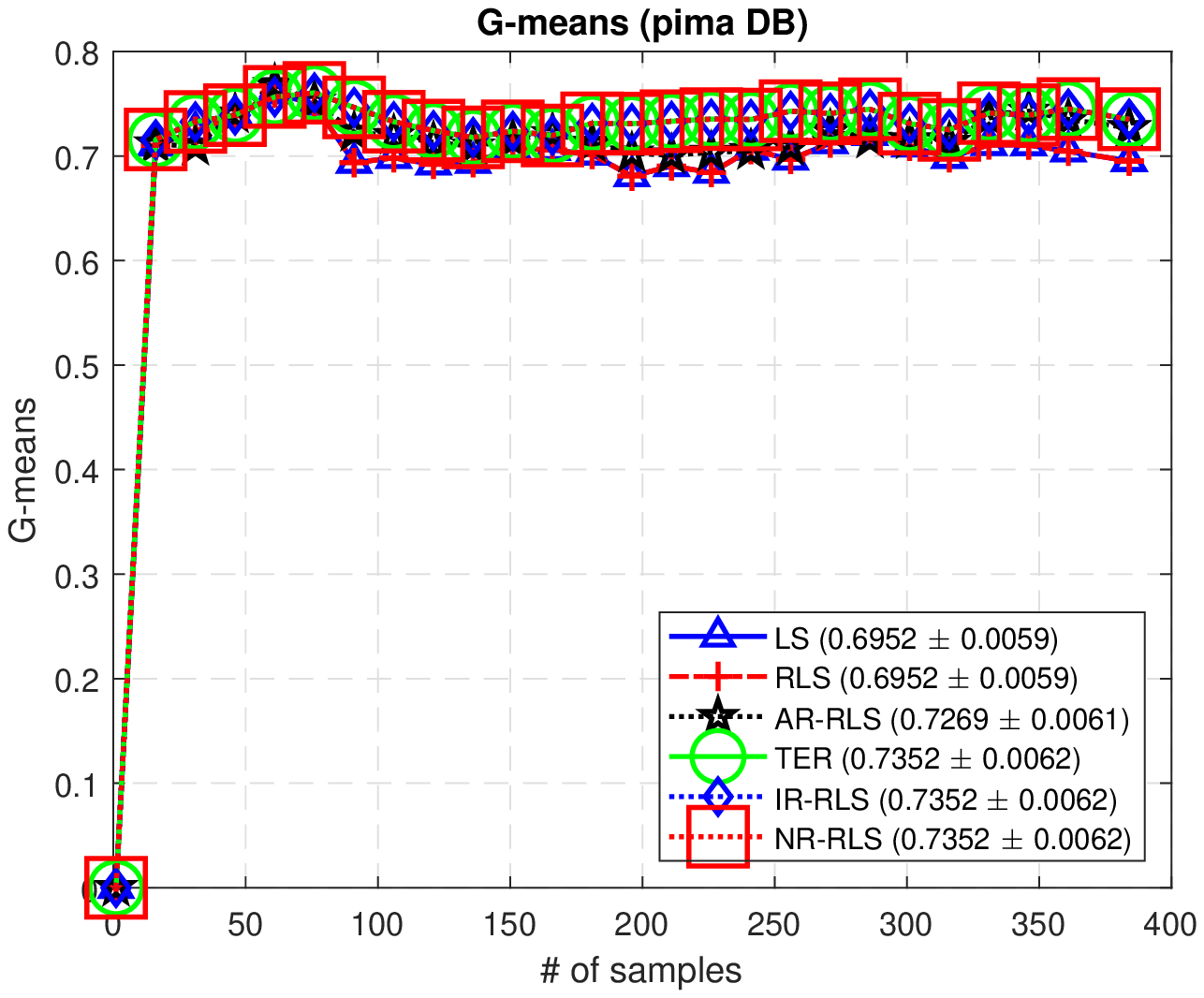}%
					} \hspace{0.4cm}
					\subfigure[The CPU times
					]{ 
						\includegraphics[width=0.3\textwidth]{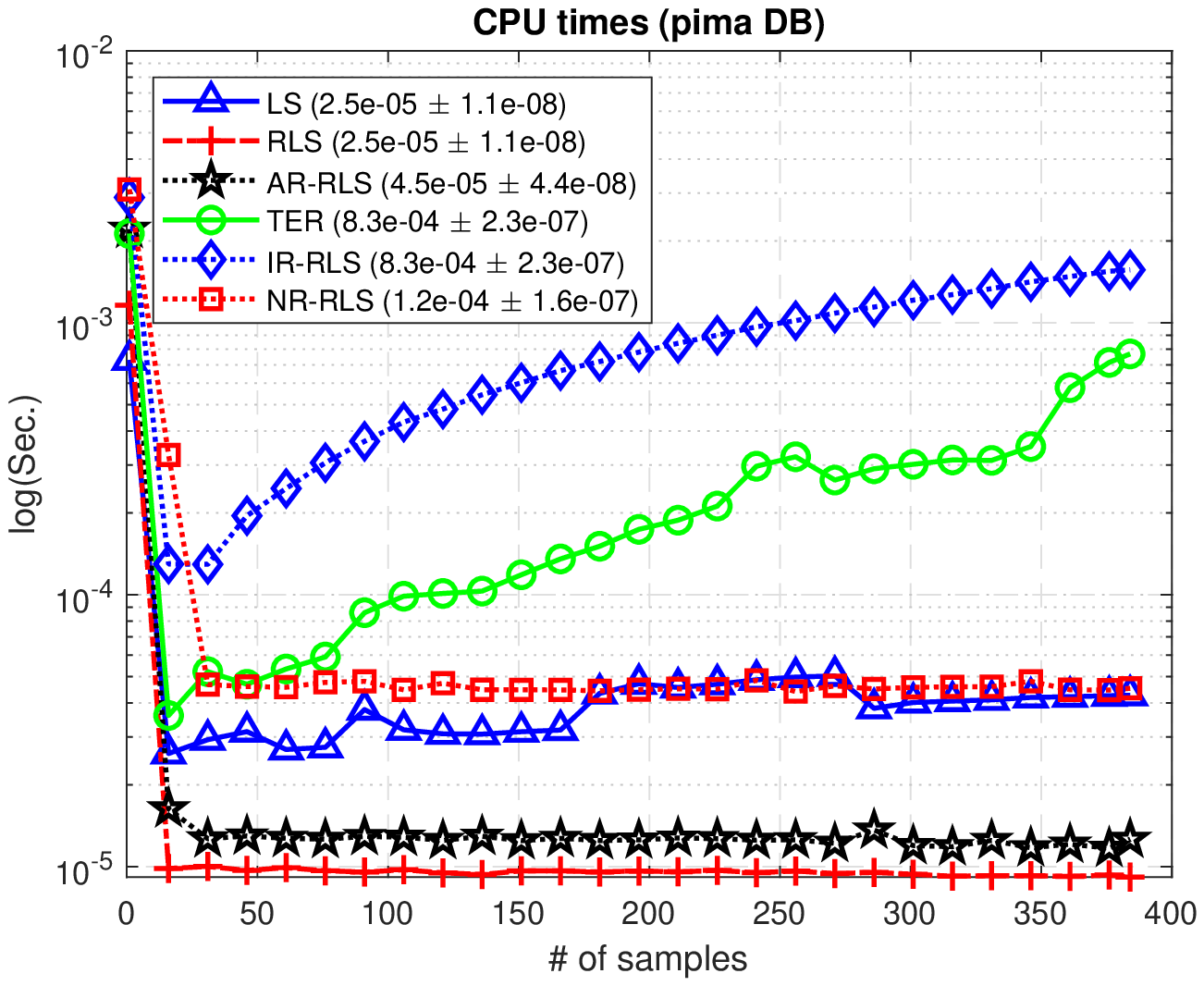}%
					} 
					\caption{Pima-diabetes}%
					
				\end{figure*} 
				
				\begin{figure*}[!h]
					\centering 
					\subfigure[The $L_2$-norm values
					]{ 
						\includegraphics[width=0.3\textwidth]{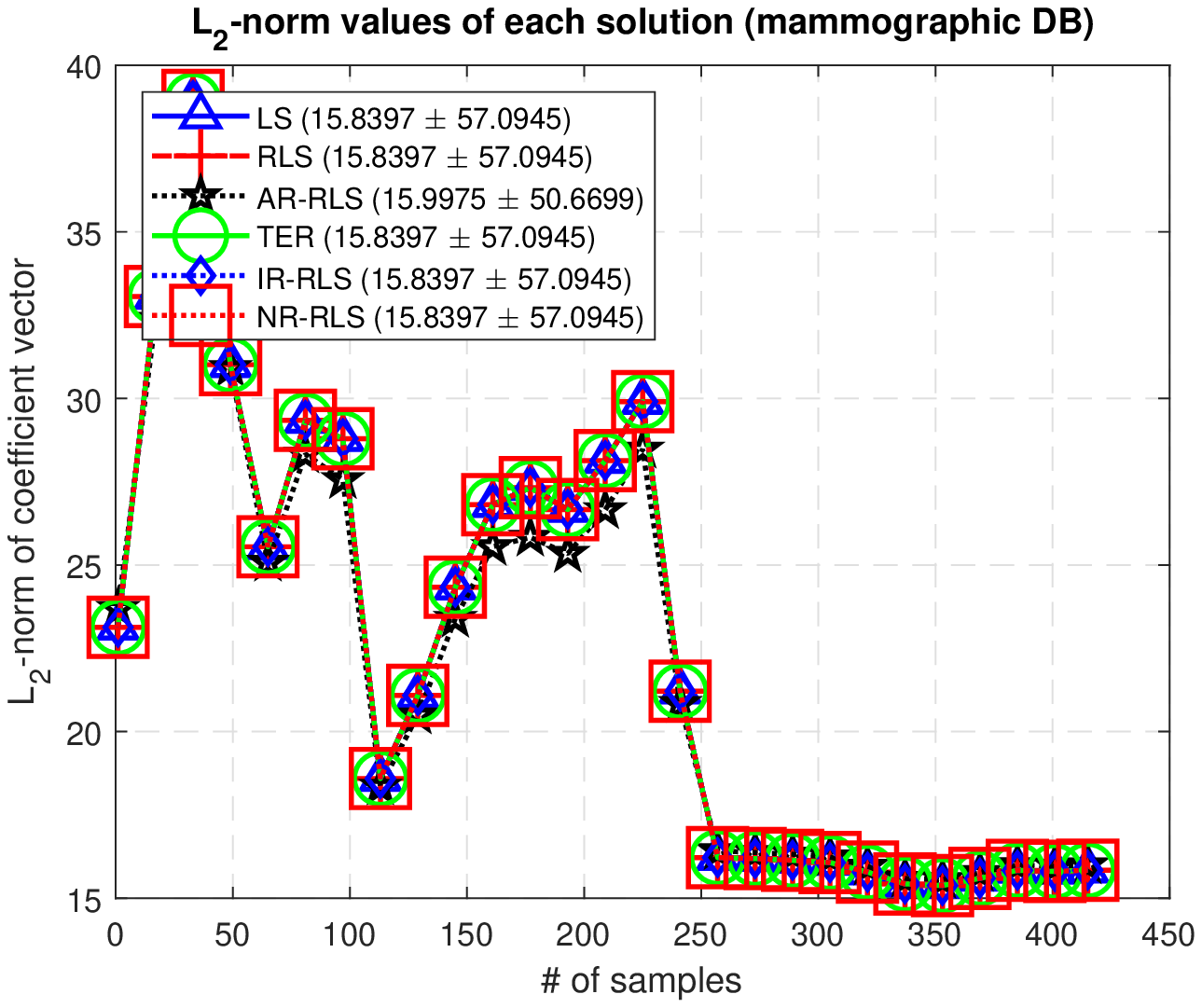}%
					} \hspace{0.4cm}
					\subfigure[The G-means
					]{ 
						\includegraphics[width=0.3\textwidth]{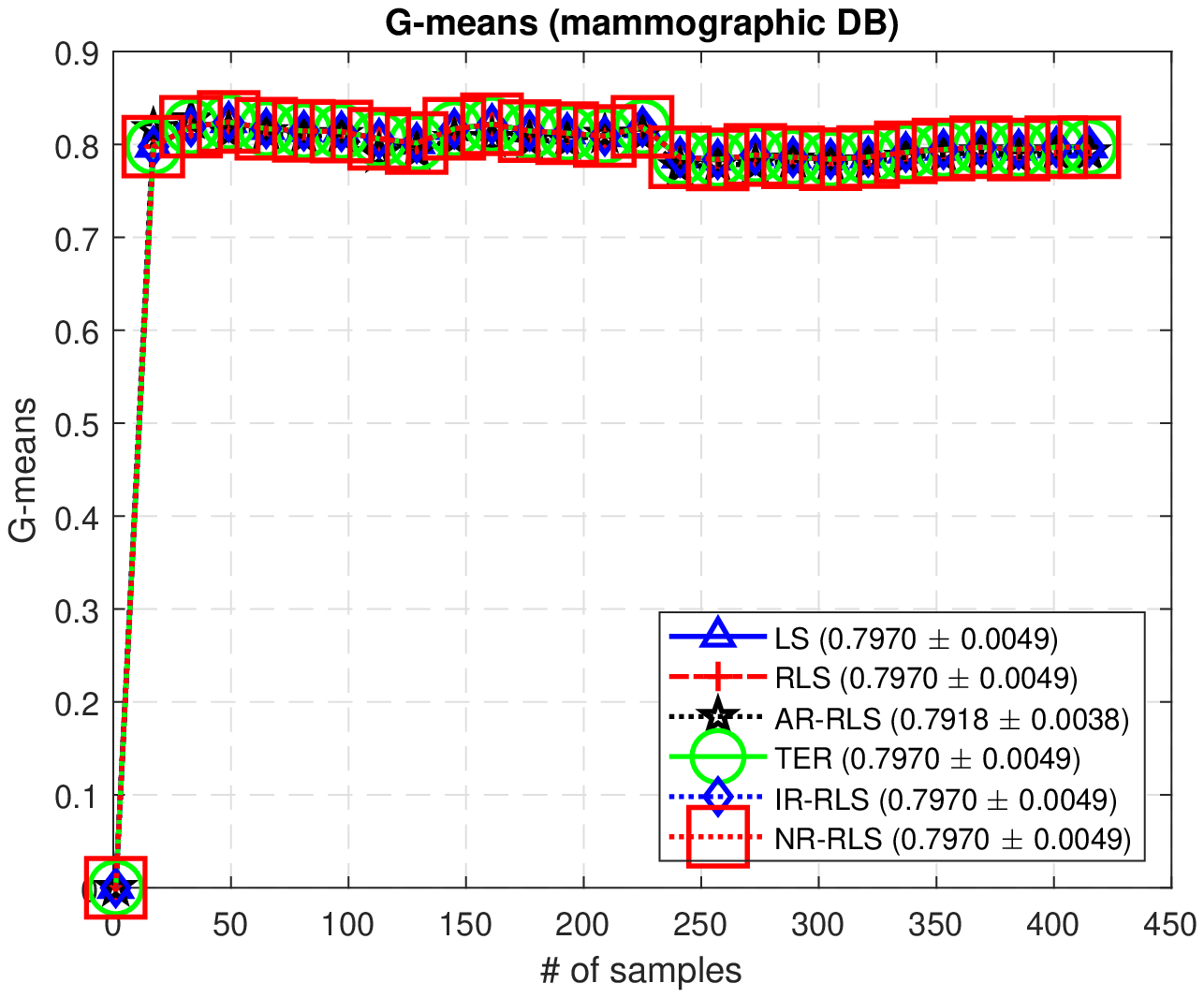}%
					} \hspace{0.4cm}
					\subfigure[The CPU times
					]{ 
						\includegraphics[width=0.3\textwidth]{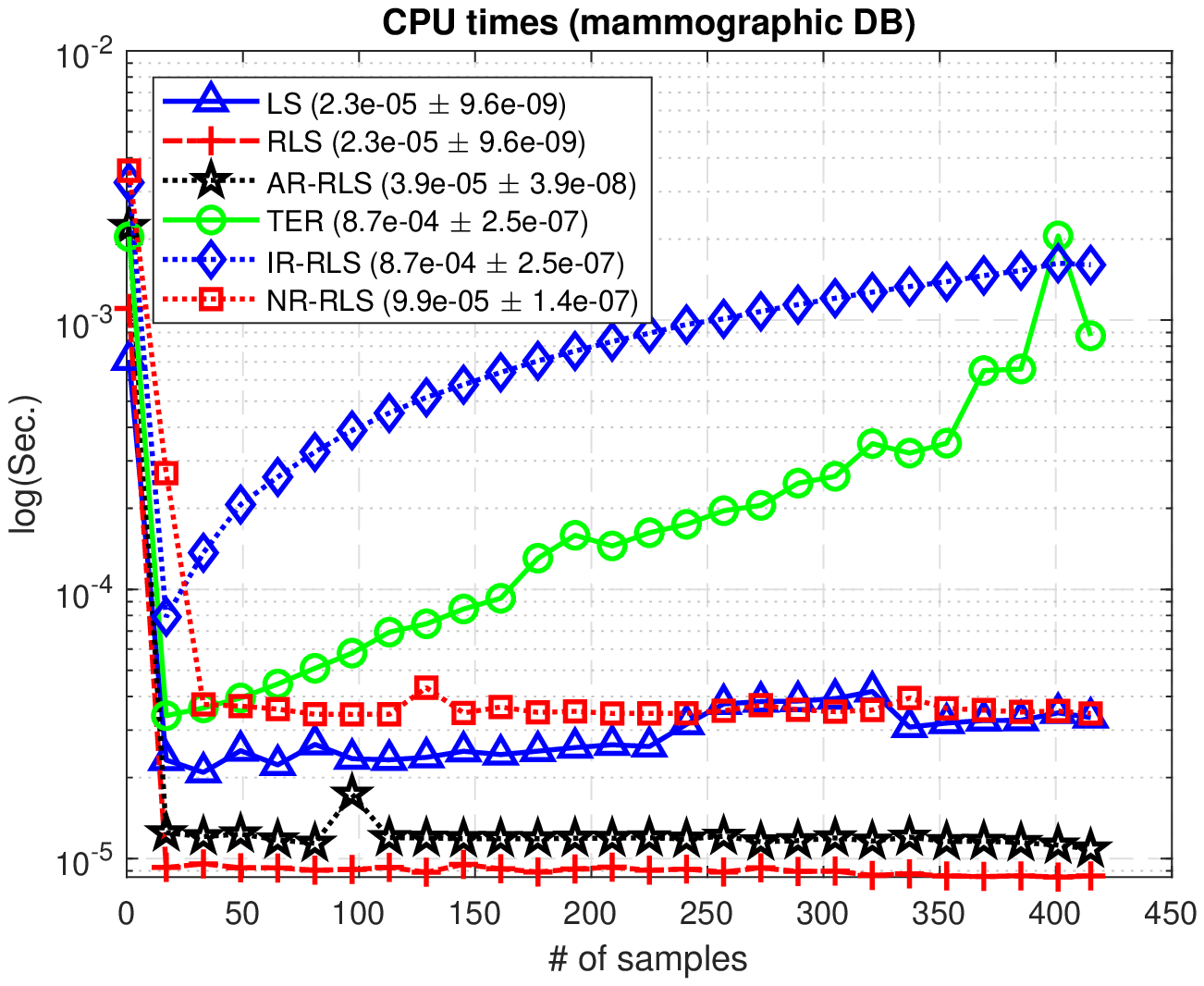}%
					} 
					\caption{Mammographic}%
					
				\end{figure*} 
				
				\begin{figure*}[!h]
					\centering 
					\subfigure[The $L_2$-norm values
					]{ 
						\includegraphics[width=0.3\textwidth]{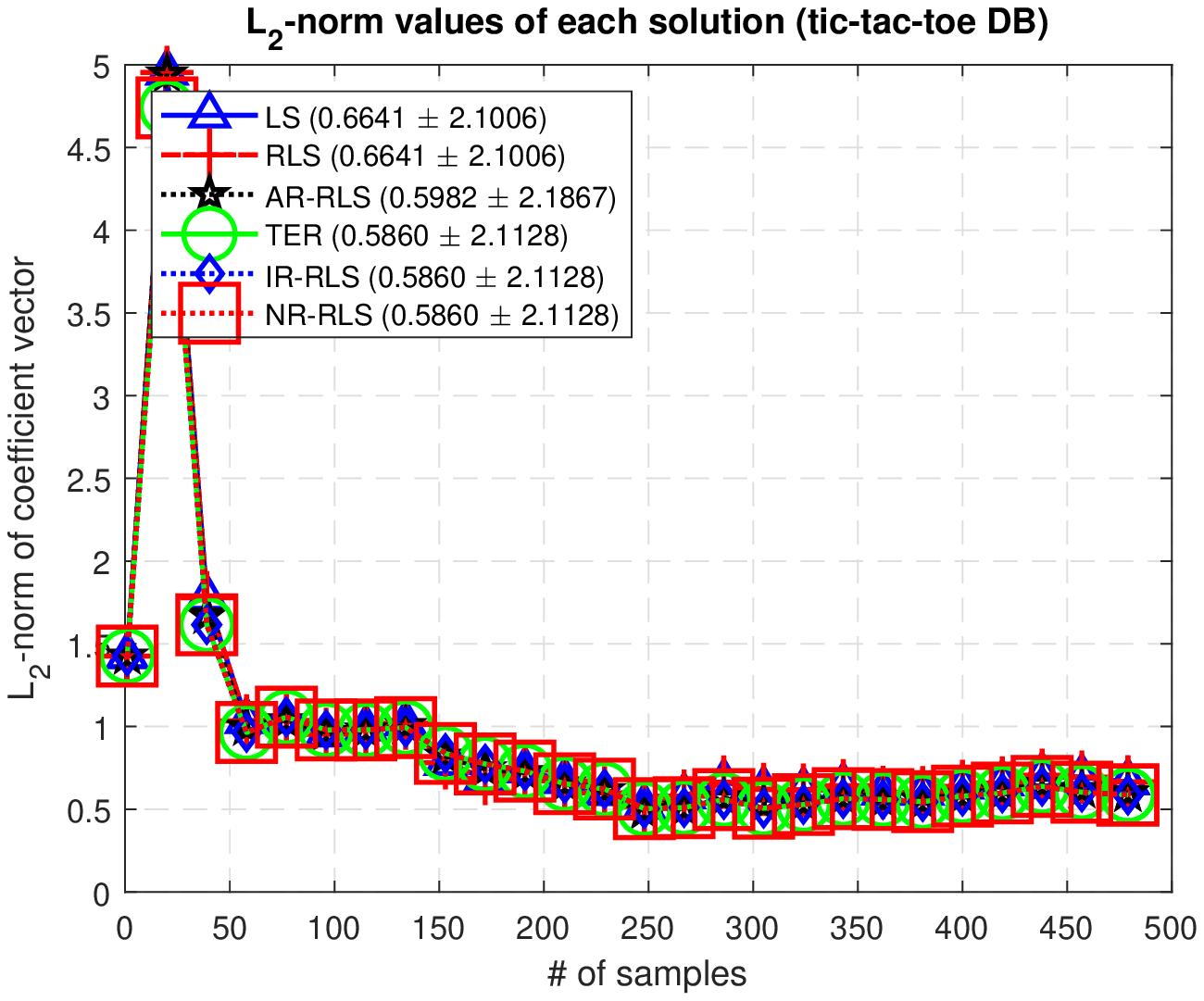}%
					} \hspace{0.4cm}
					\subfigure[The G-means
					]{ 
						\includegraphics[width=0.3\textwidth]{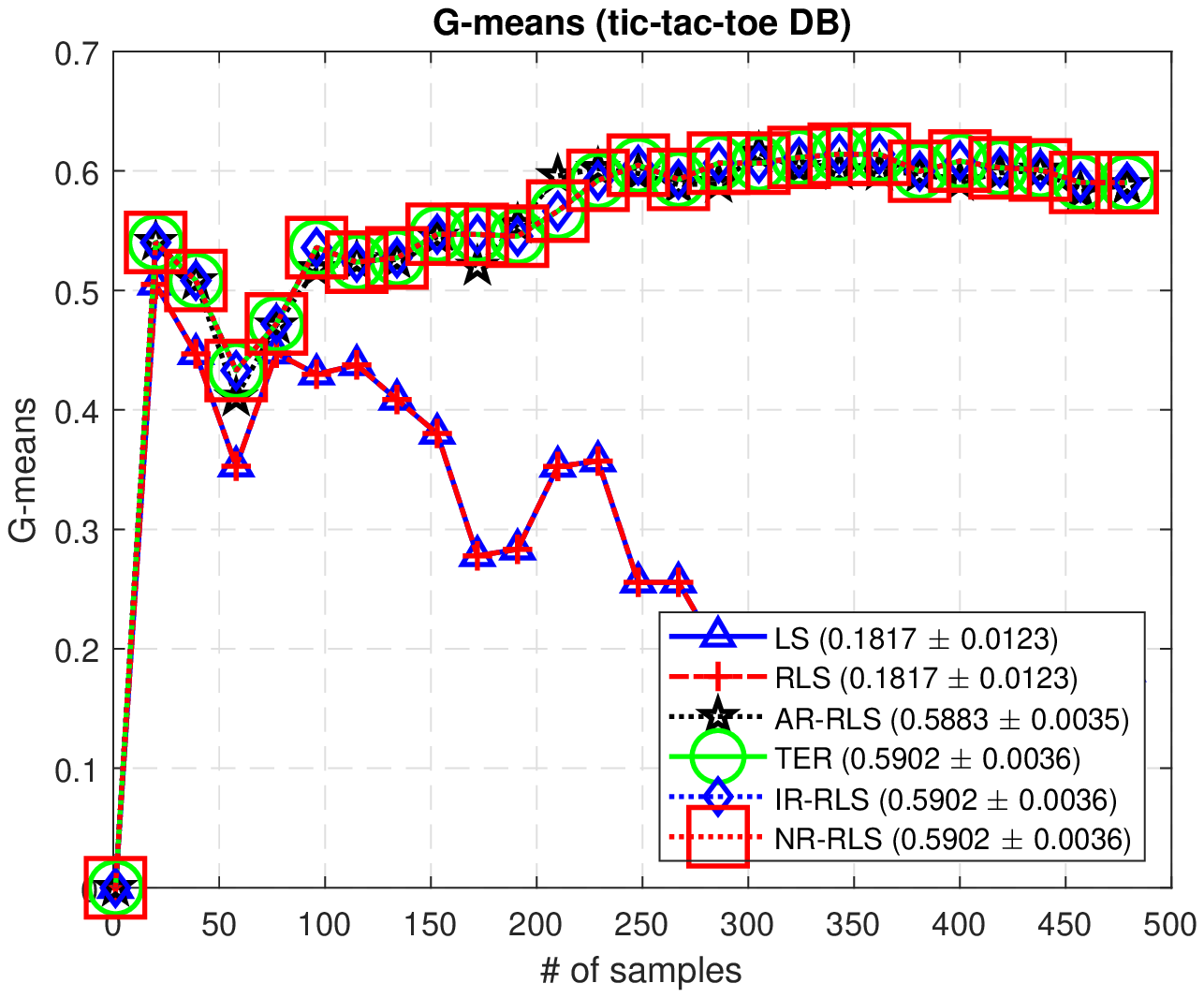}%
					} \hspace{0.4cm}
					\subfigure[The CPU times
					]{ 
						\includegraphics[width=0.3\textwidth]{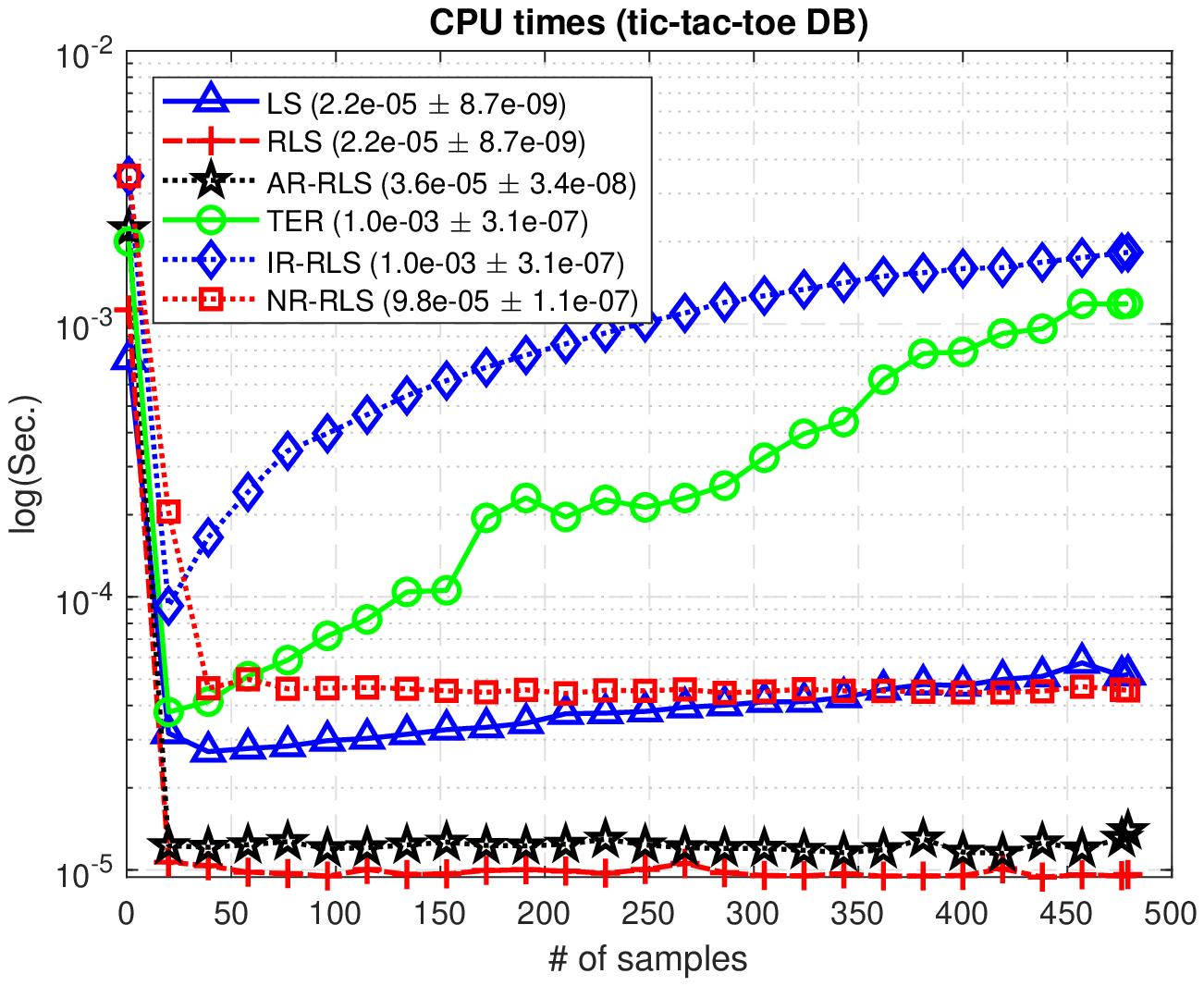}%
					} 
					\caption{Tic-tac-toe}%
					
				\end{figure*} 
				\newpage
				\begin{figure*}[!h]
					\centering 
					\subfigure[The $L_2$-norm values
					]{ 
						\includegraphics[width=0.3\textwidth]{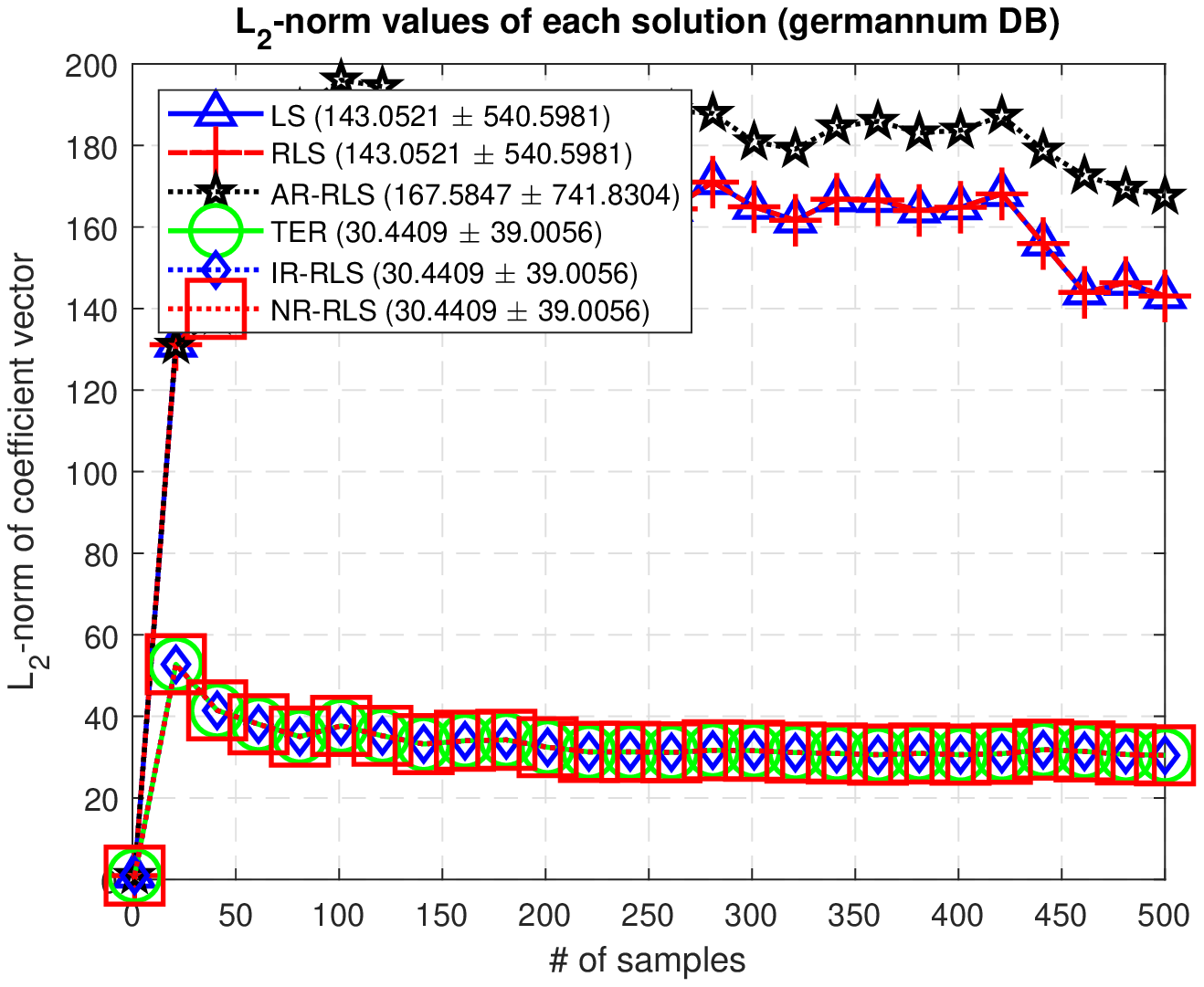}%
					} \hspace{0.4cm}
					\subfigure[The G-means
					]{ 
						\includegraphics[width=0.3\textwidth]{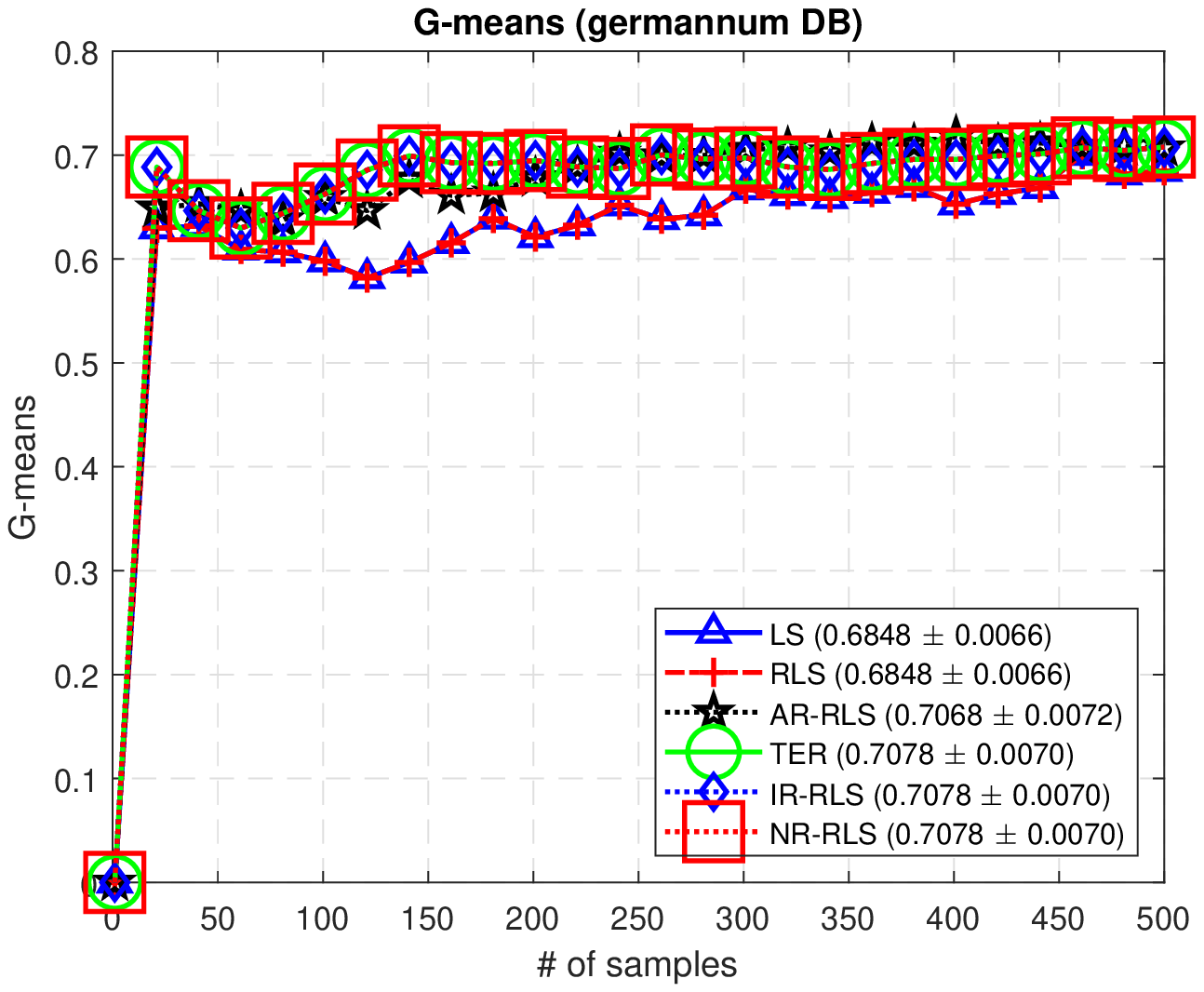}%
					} \hspace{0.4cm}
					\subfigure[The CPU times
					]{ 
						\includegraphics[width=0.3\textwidth]{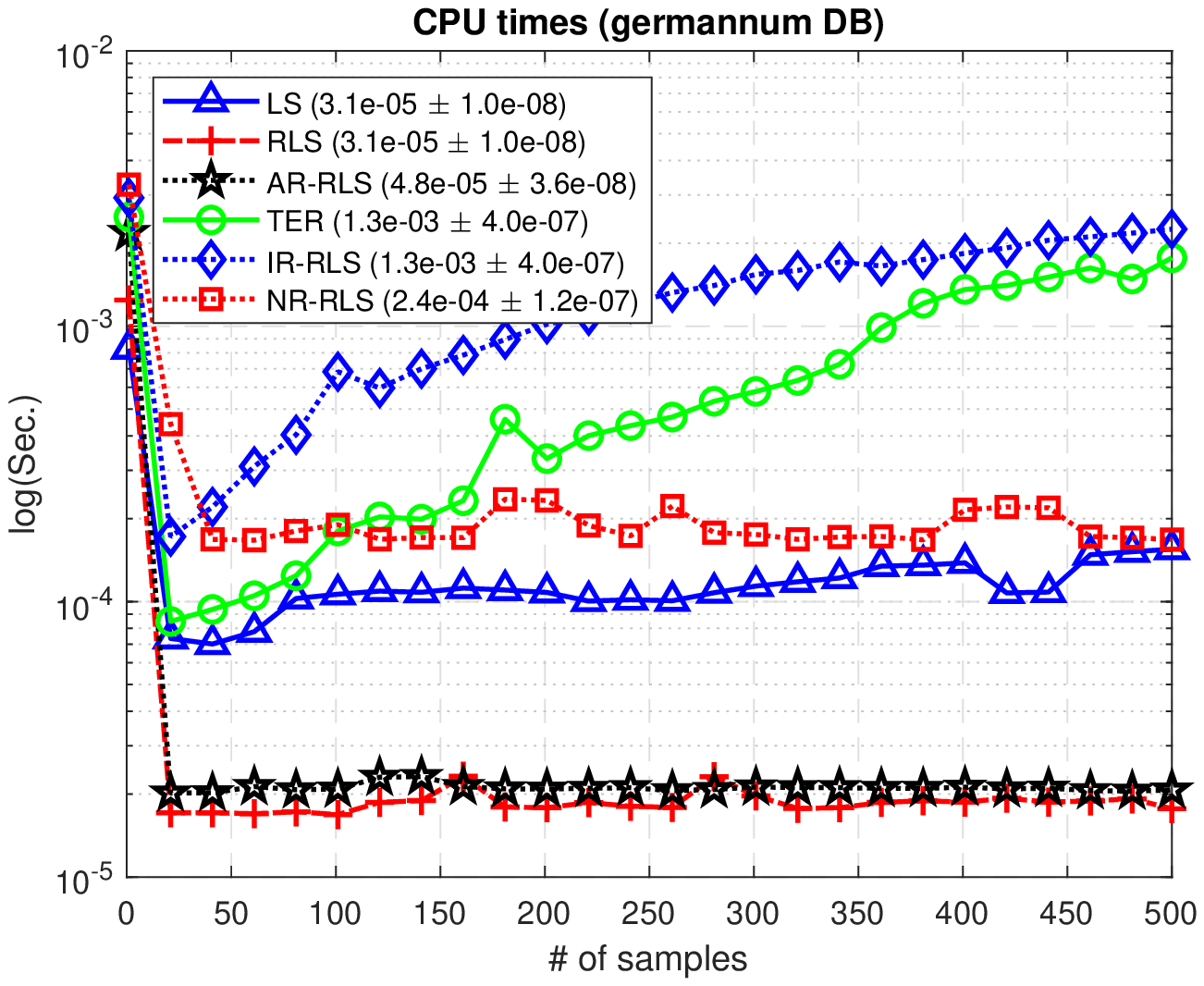}%
					} 
					\caption{Statlg-german}%
					
				\end{figure*} 
				
				\begin{figure*}[!h]
					\centering 
					\subfigure[The $L_2$-norm values
					]{ 
						\includegraphics[width=0.3\textwidth]{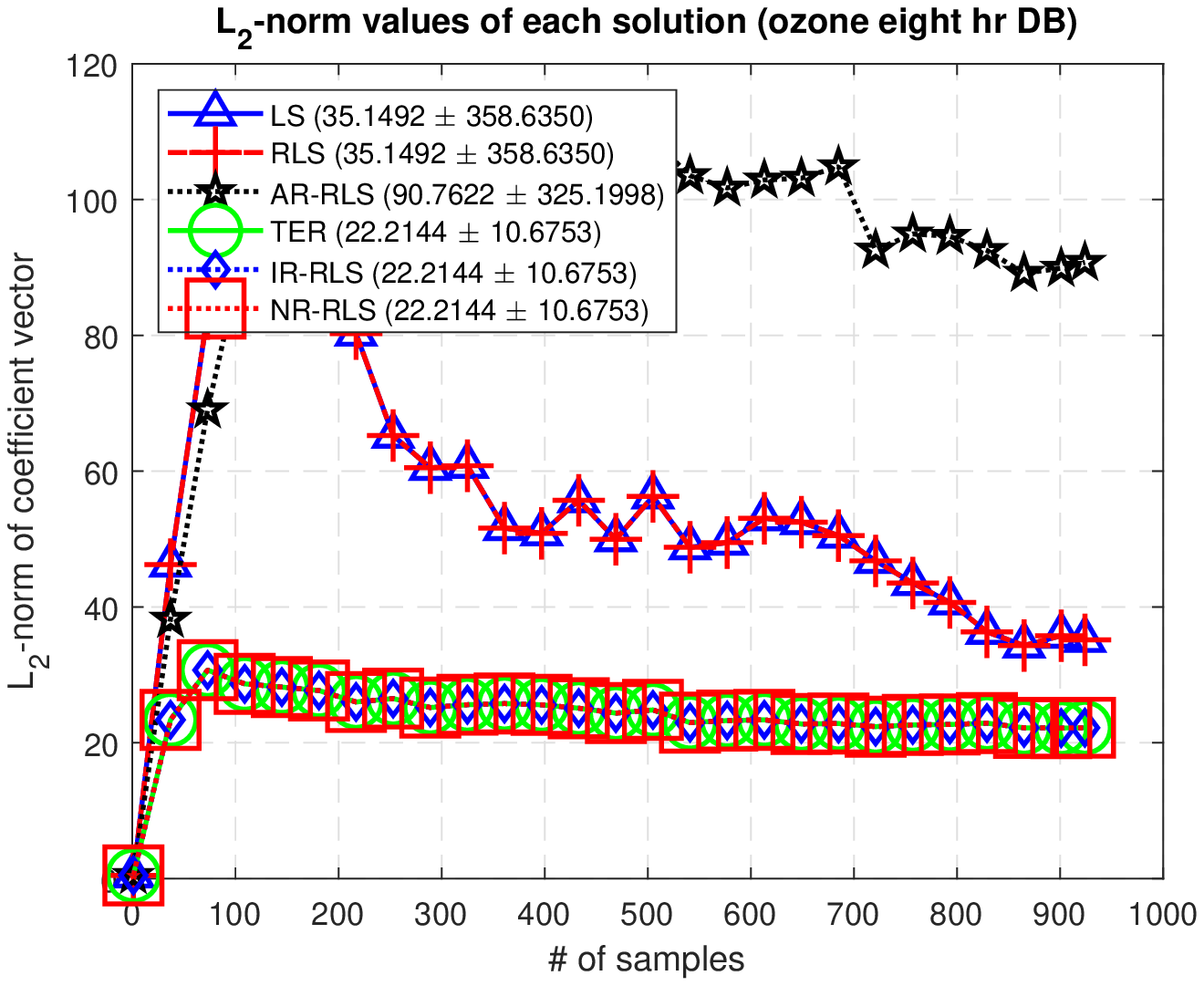}%
					} \hspace{0.4cm}
					\subfigure[The G-means
					]{ 
						\includegraphics[width=0.3\textwidth]{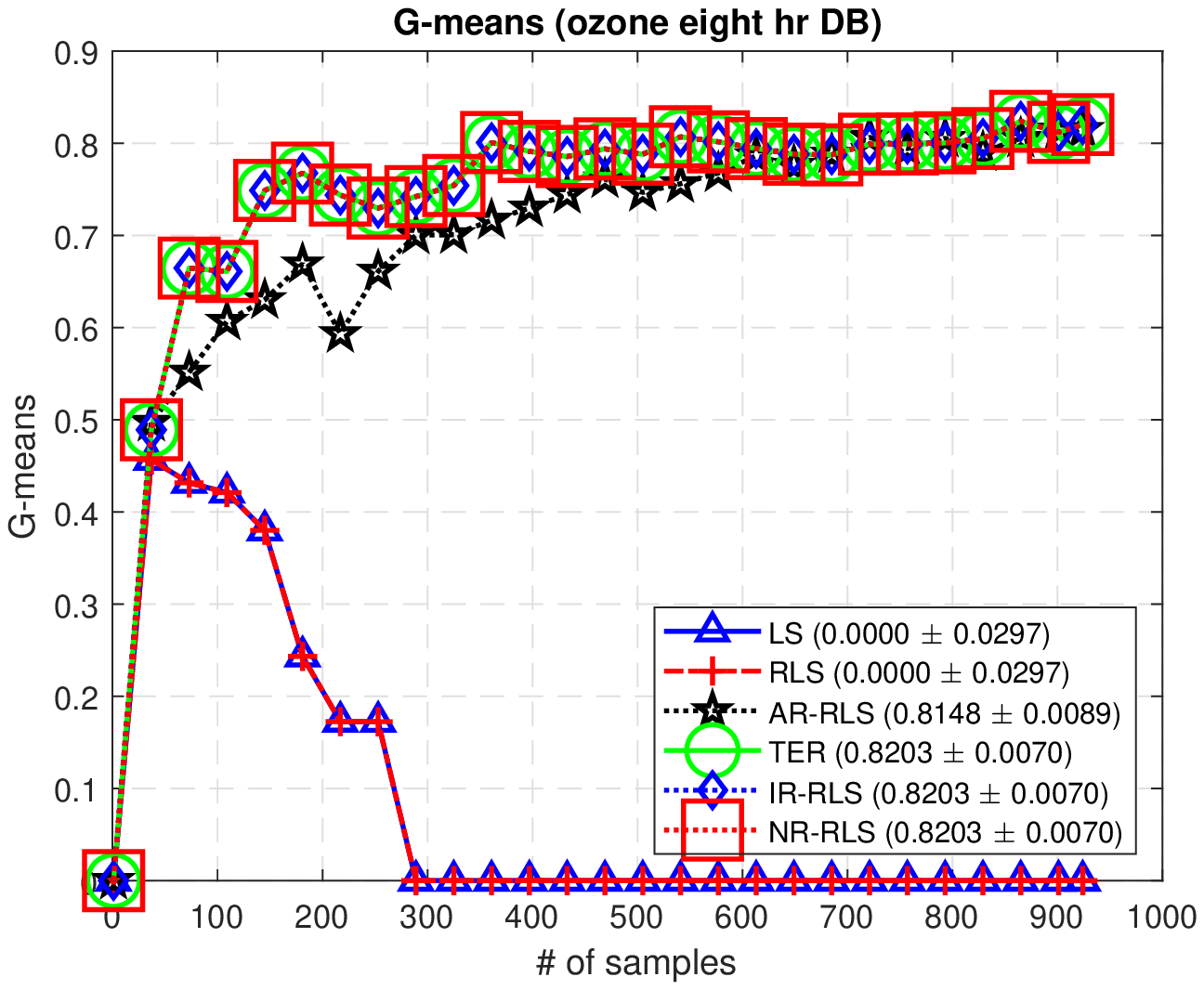}%
					} \hspace{0.4cm}
					\subfigure[The CPU times
					]{ 
						\includegraphics[width=0.3\textwidth]{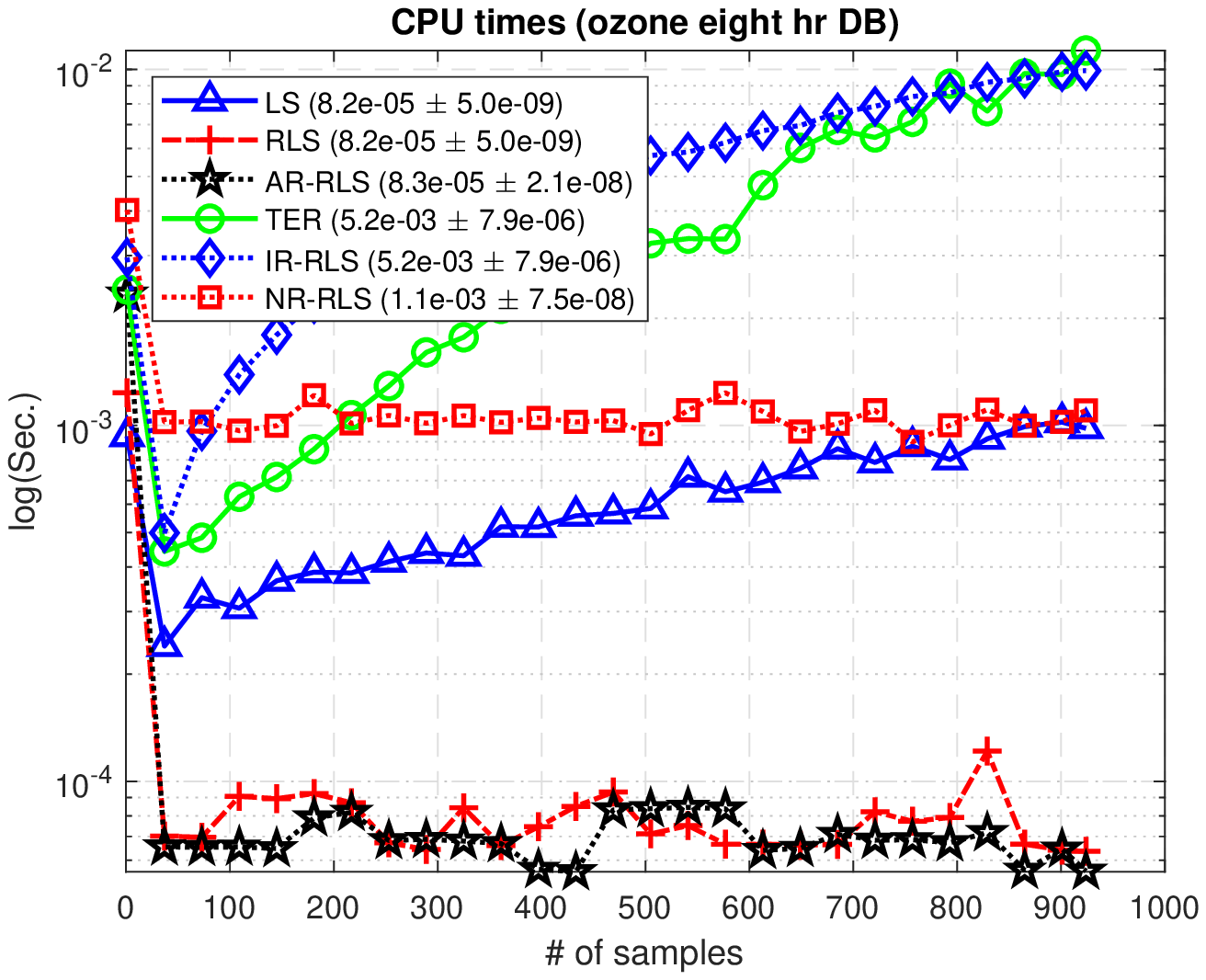}%
					} 
					\caption{Ozone-eight}%
					
				\end{figure*} 
				
				\begin{figure*}[!h]
					\centering 
					\subfigure[The $L_2$-norm values
					]{ 
						\includegraphics[width=0.3\textwidth]{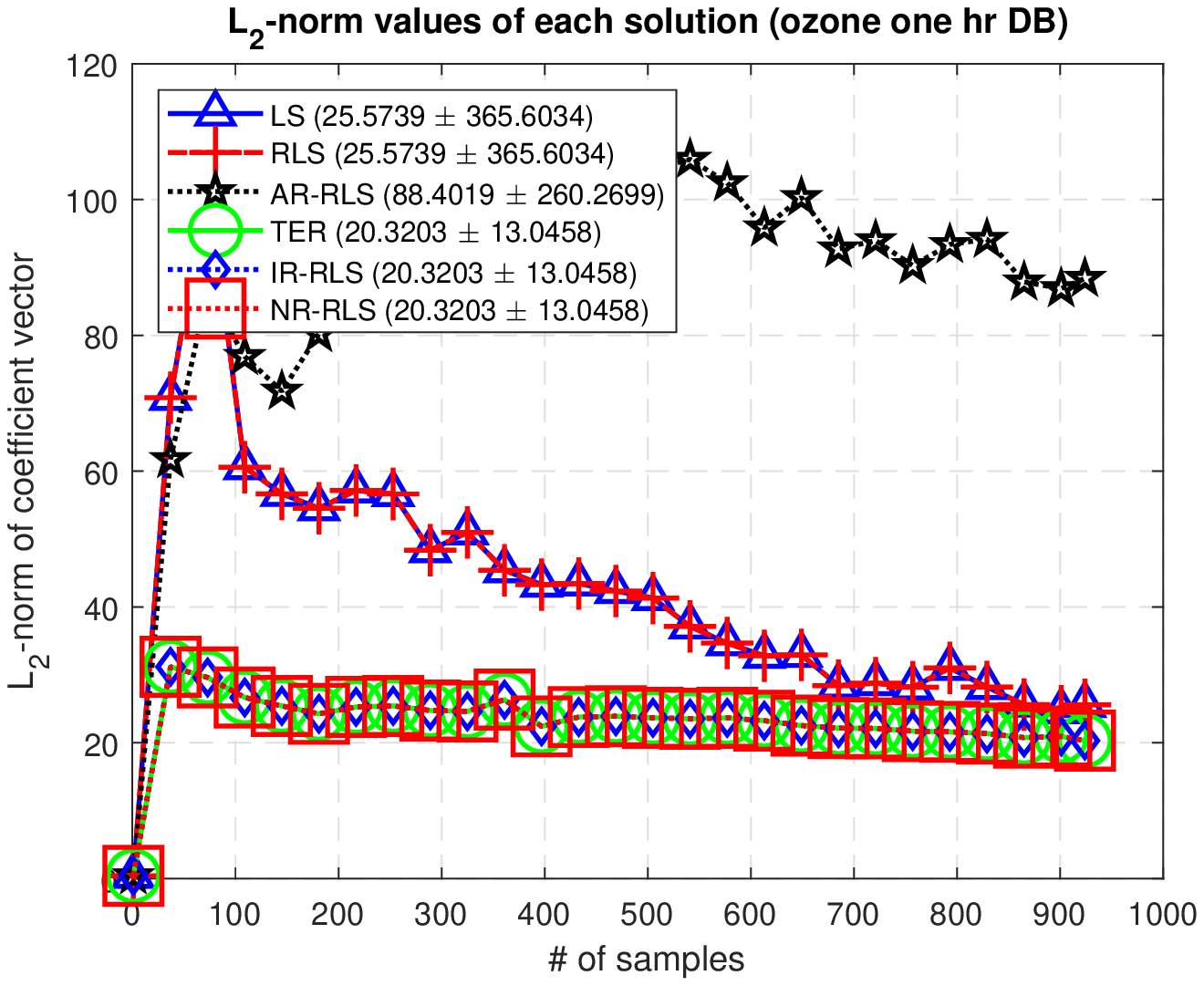}%
					} \hspace{0.4cm}
					\subfigure[The G-means
					]{ 
						\includegraphics[width=0.3\textwidth]{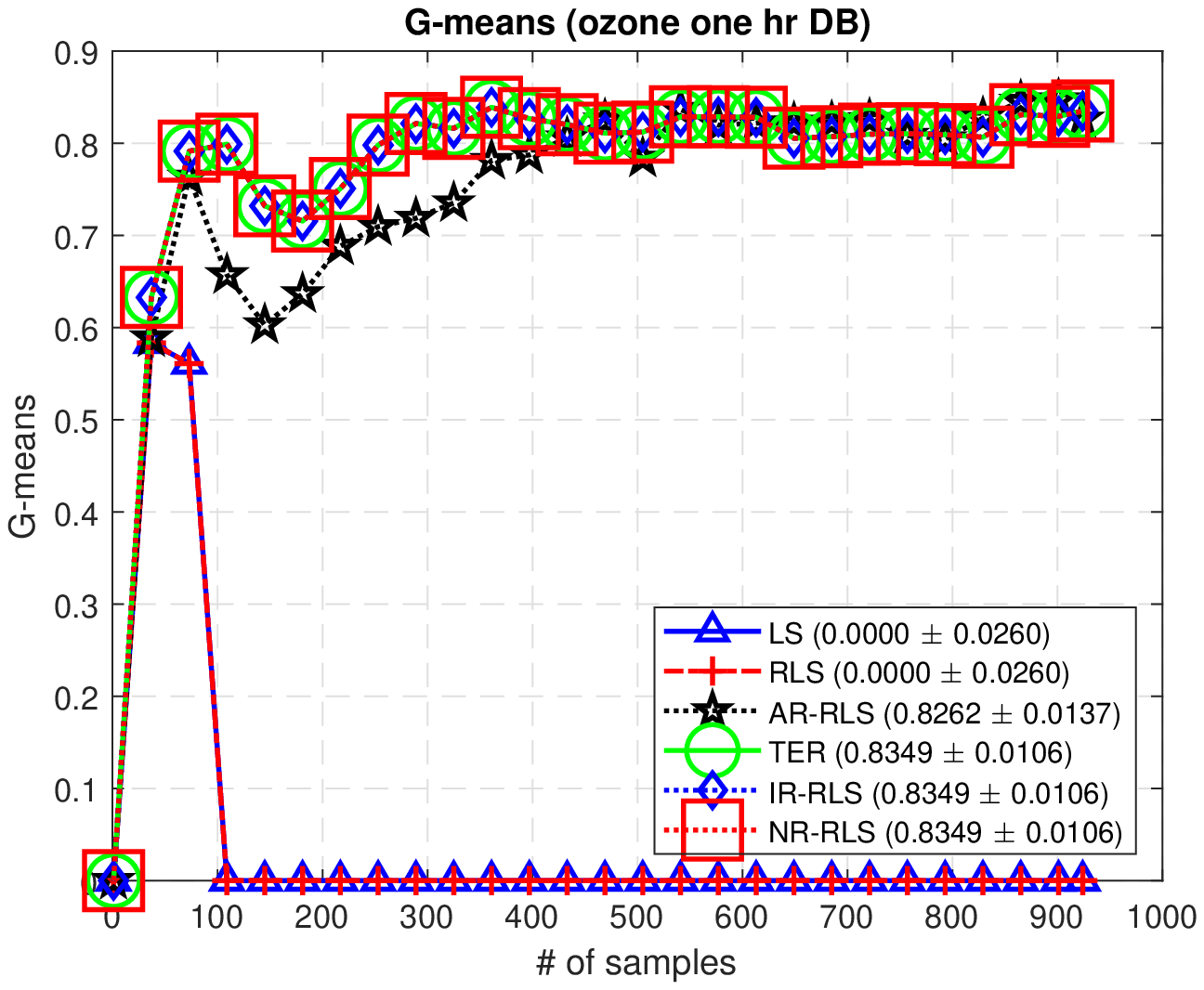}%
					} \hspace{0.4cm}
					\subfigure[The CPU times
					]{ 
						\includegraphics[width=0.3\textwidth]{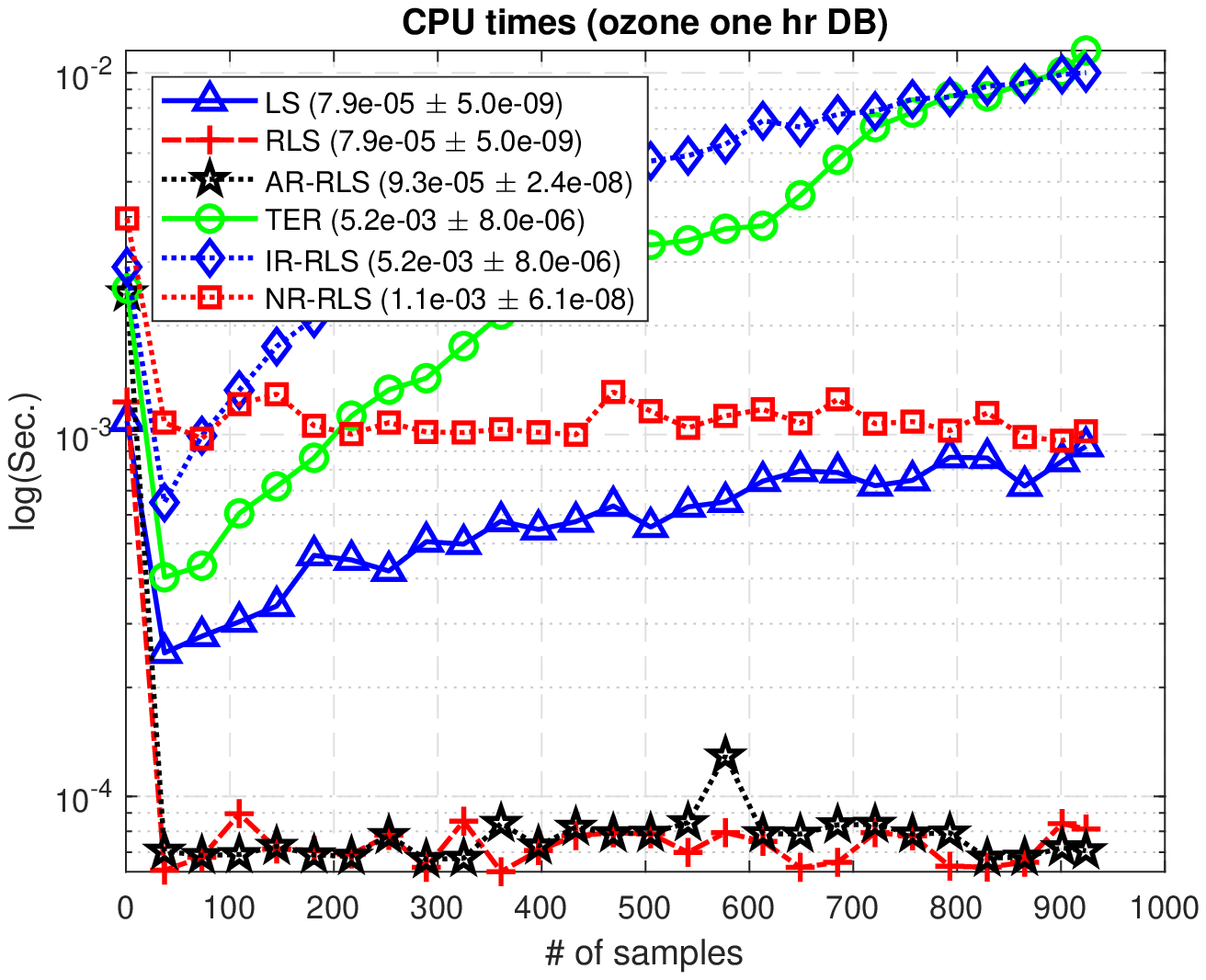}%
					} 
					\caption{Ozone-one}%
					
				\end{figure*}

				\begin{figure*}[!h]
					\centering 
					\subfigure[The $L_2$-norm values
					]{ 
						\includegraphics[width=0.3\textwidth]{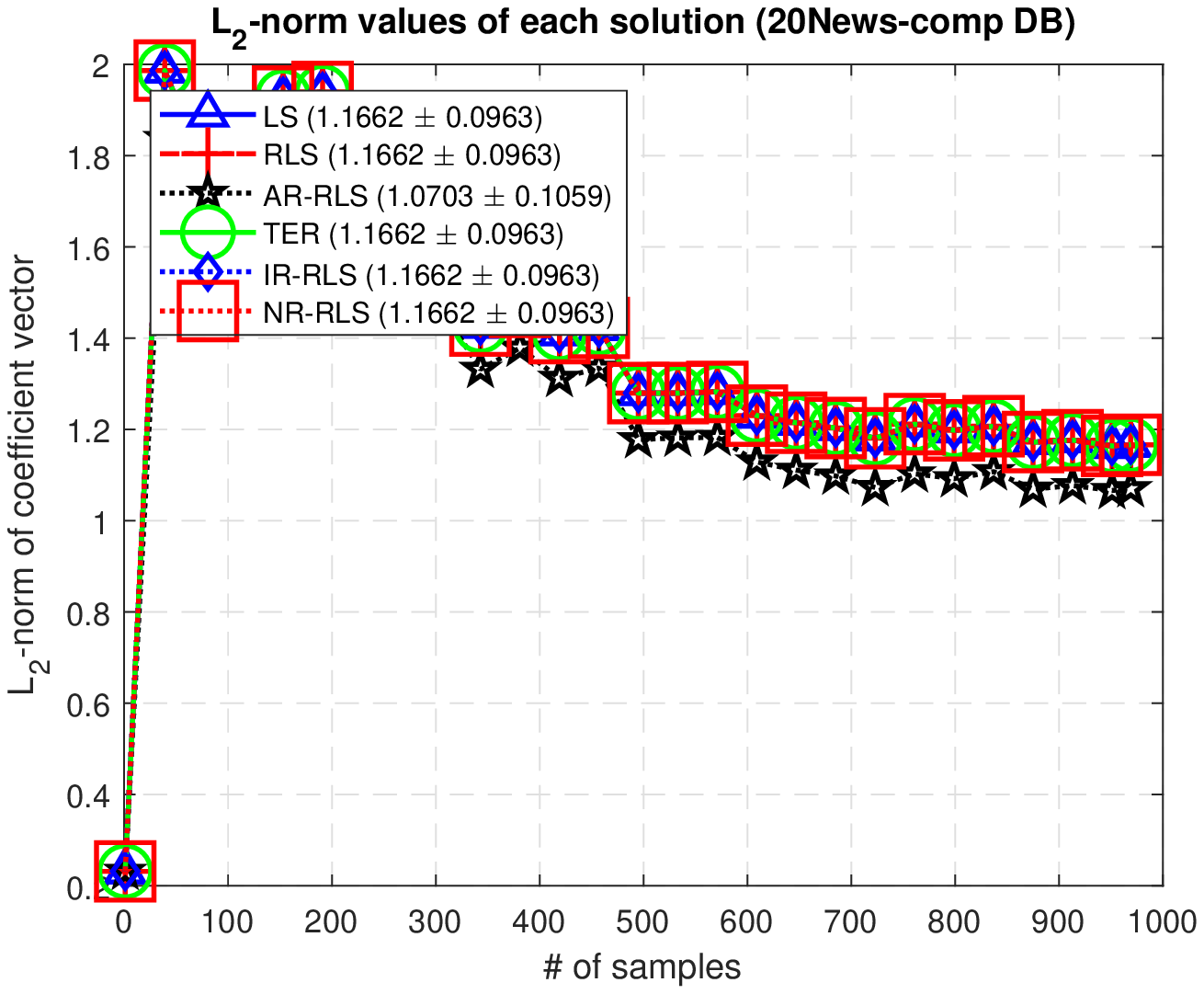}%
					} \hspace{0.4cm}
					\subfigure[The G-means
					]{ 
						\includegraphics[width=0.3\textwidth]{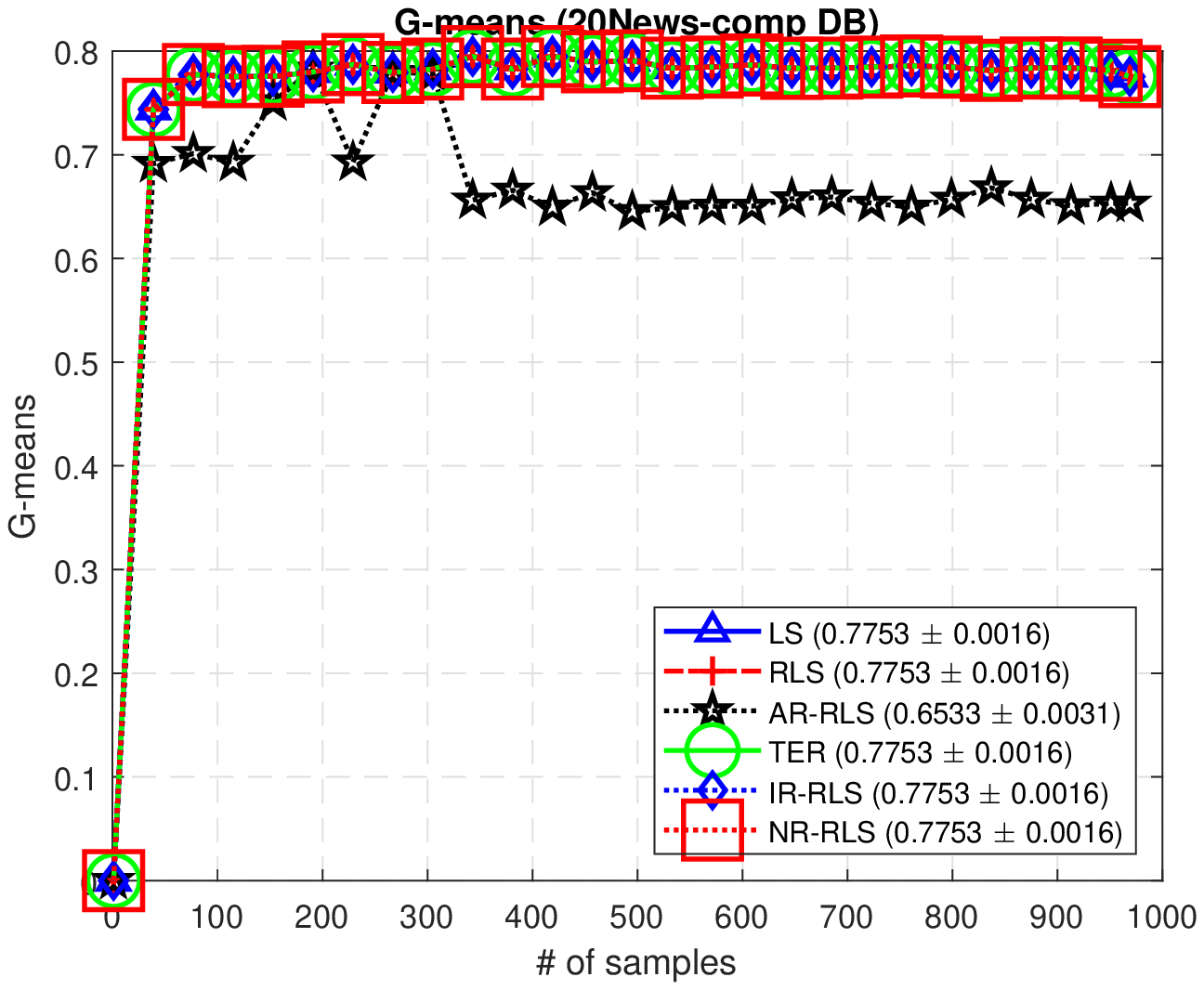}%
					} \hspace{0.4cm}
					\subfigure[The CPU times
					]{ 
						\includegraphics[width=0.3\textwidth]{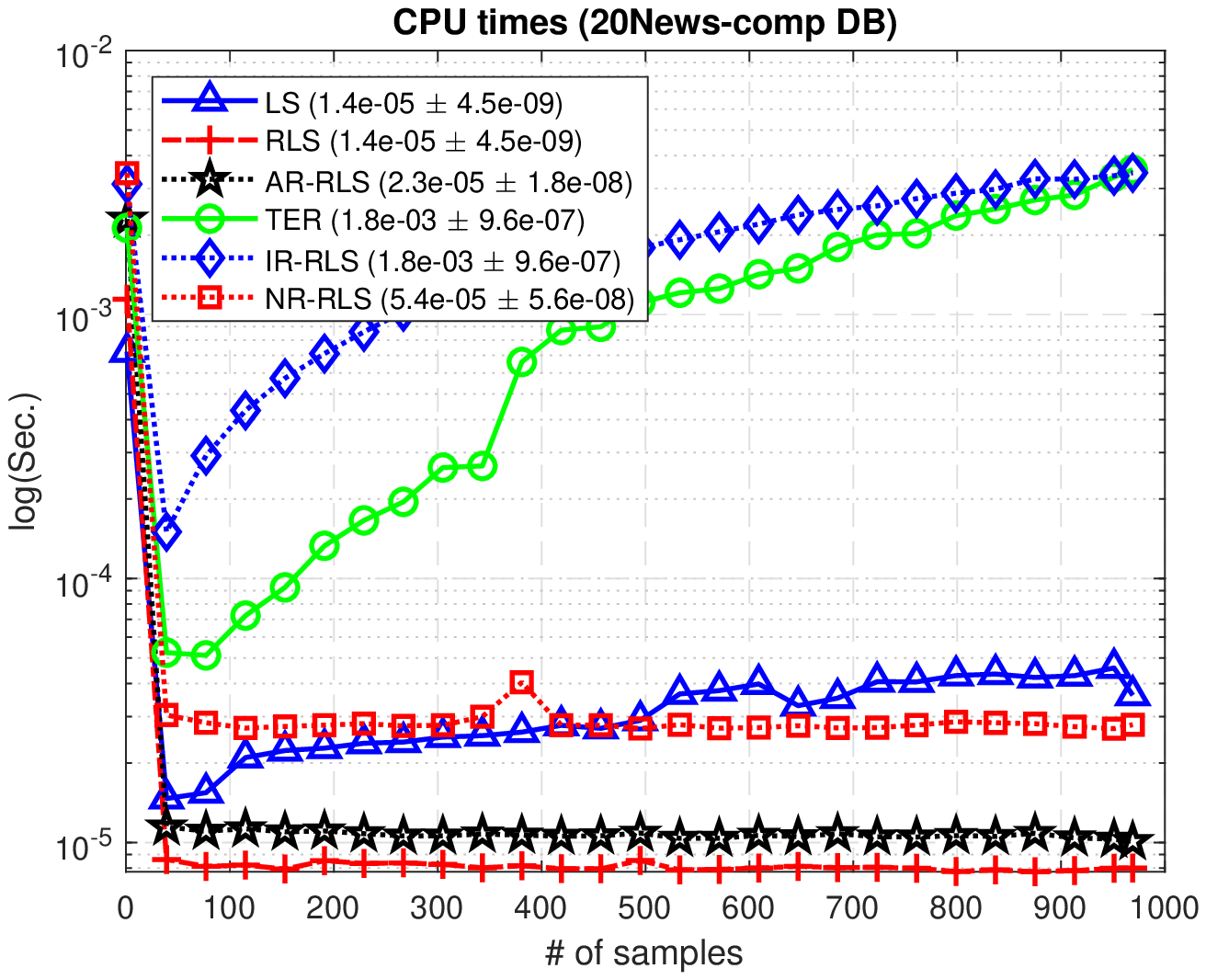}%
					} 
					\caption{20News-comp}%
					
				\end{figure*} 
				
				\begin{figure*}[!h]
					\centering 
					\subfigure[The $L_2$-norm values
					]{ 
						\includegraphics[width=0.3\textwidth]{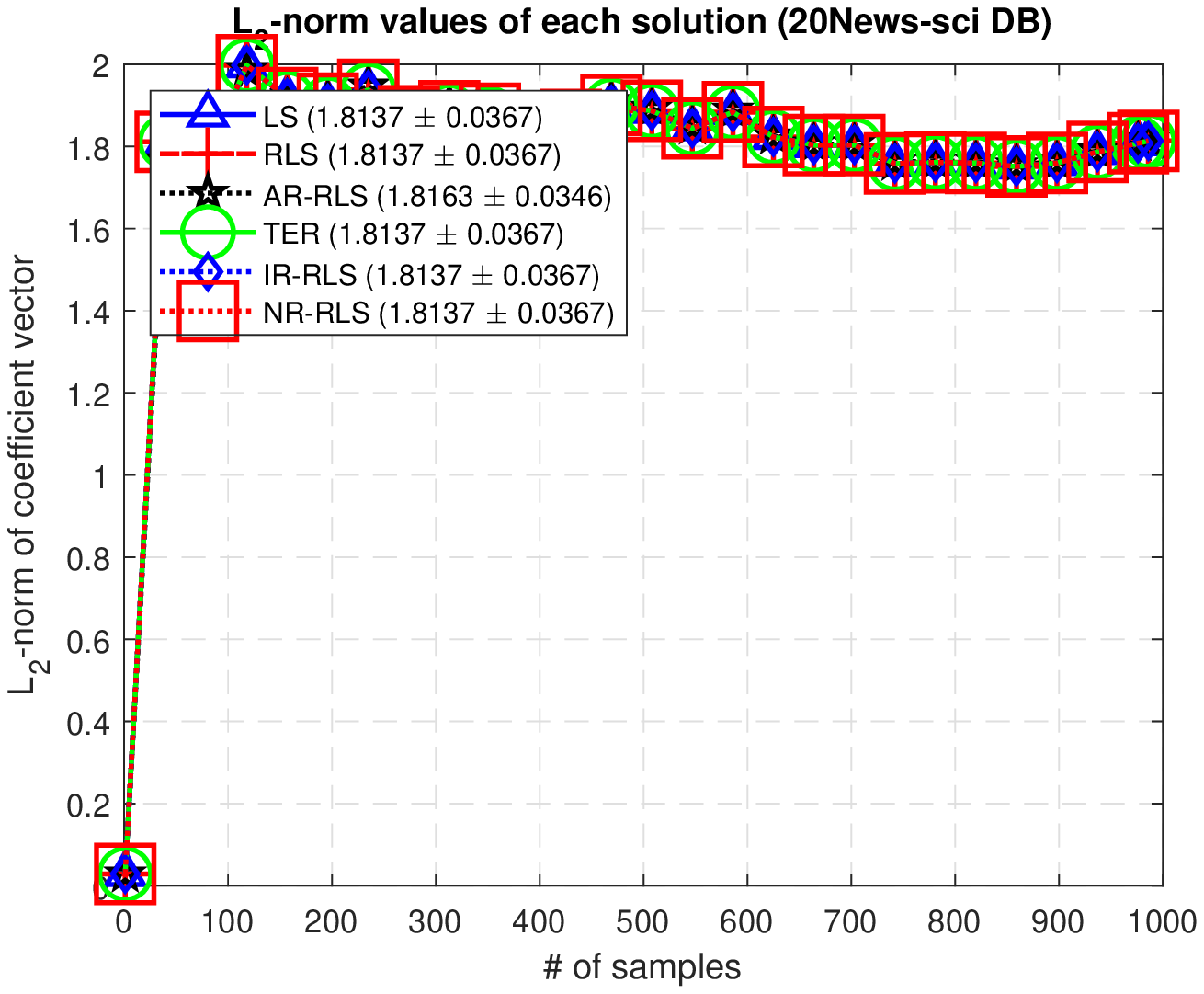}%
					} \hspace{0.4cm}
					\subfigure[The G-means
					]{ 
						\includegraphics[width=0.3\textwidth]{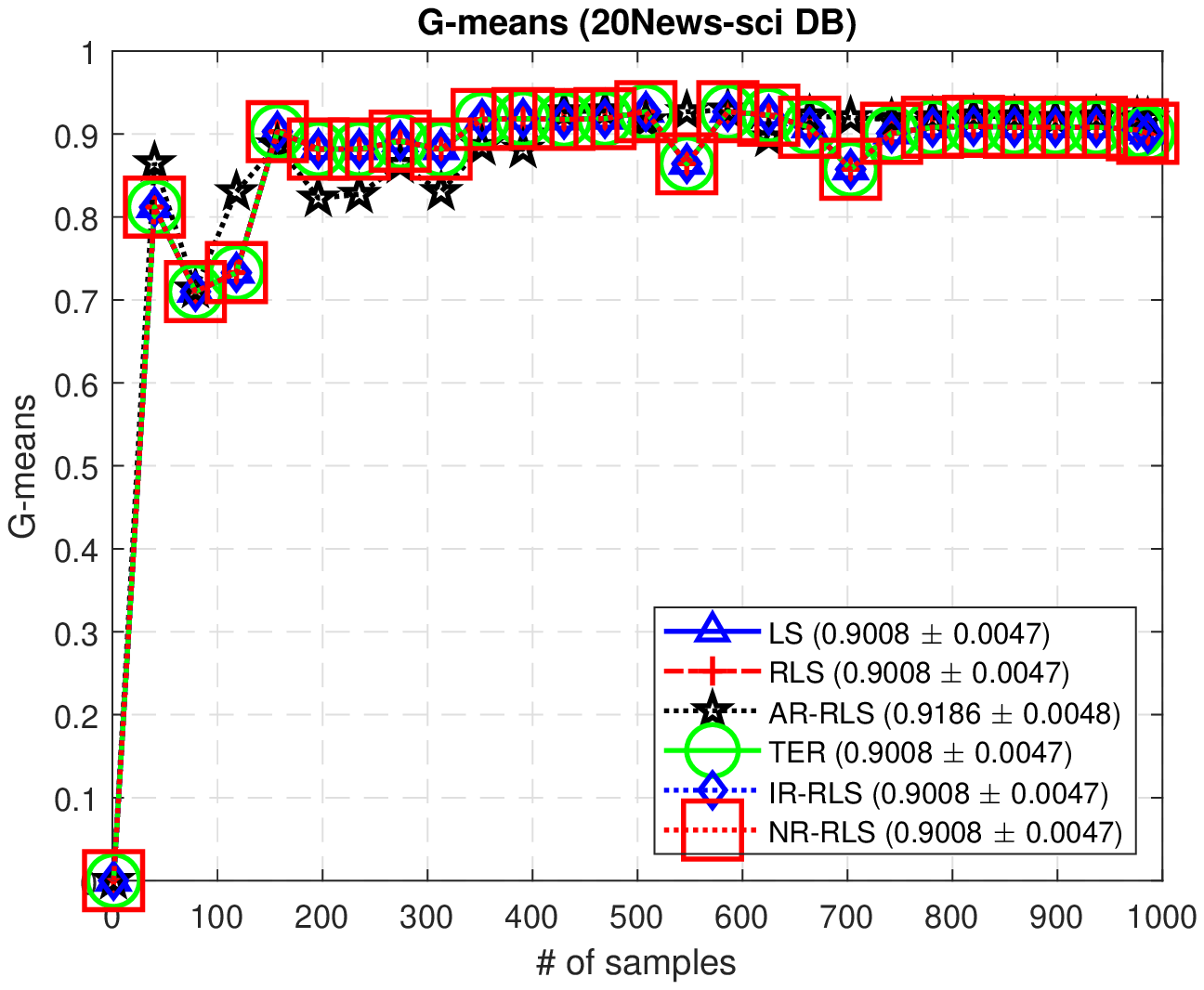}%
					} \hspace{0.4cm}
					\subfigure[The CPU times
					]{ 
						\includegraphics[width=0.3\textwidth]{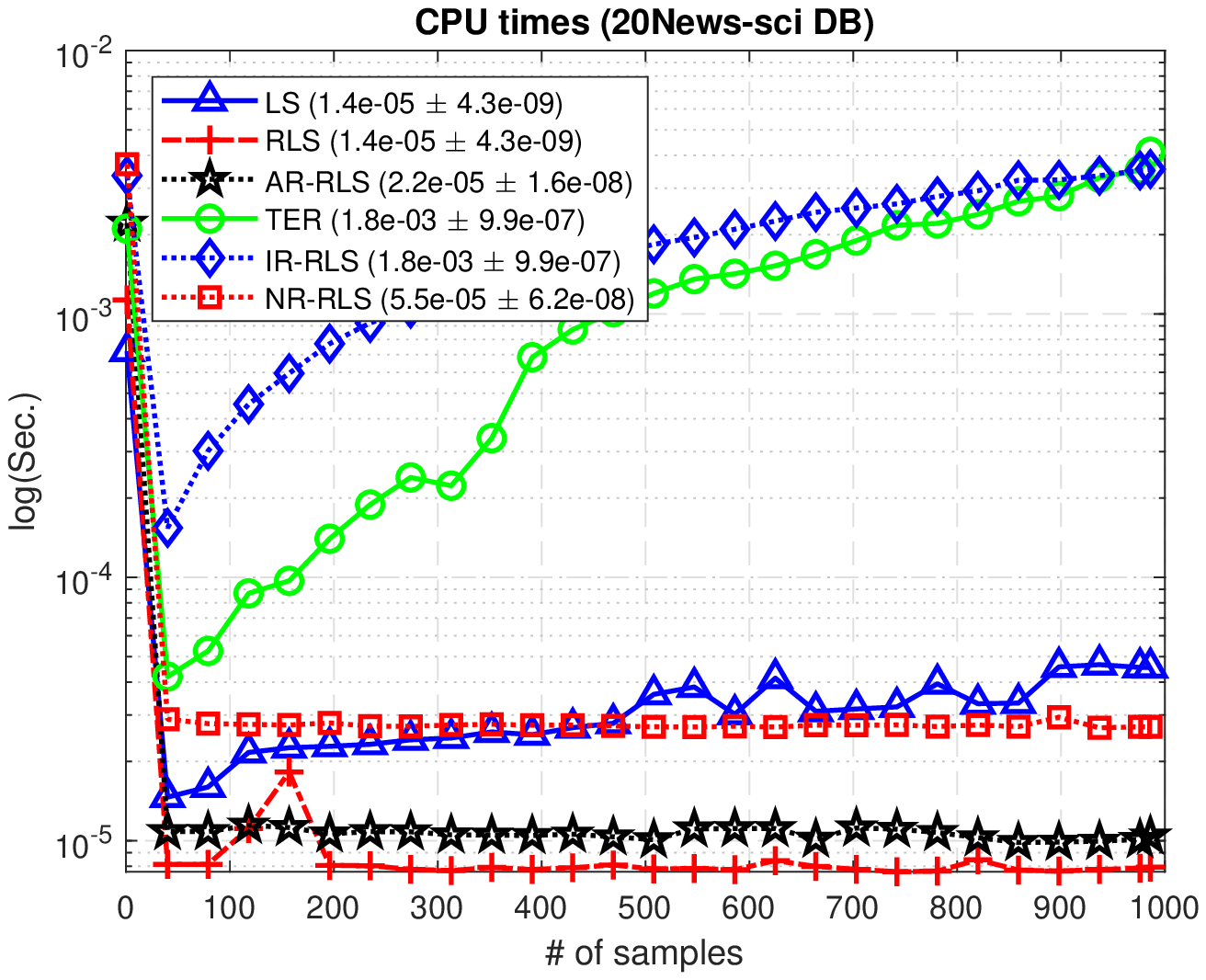}%
					} 
					\caption{20News-sci}%
					
				\end{figure*} 
				
				\begin{figure*}[!h]
					\centering 
					\subfigure[The $L_2$-norm values
					]{ 
						\includegraphics[width=0.3\textwidth]{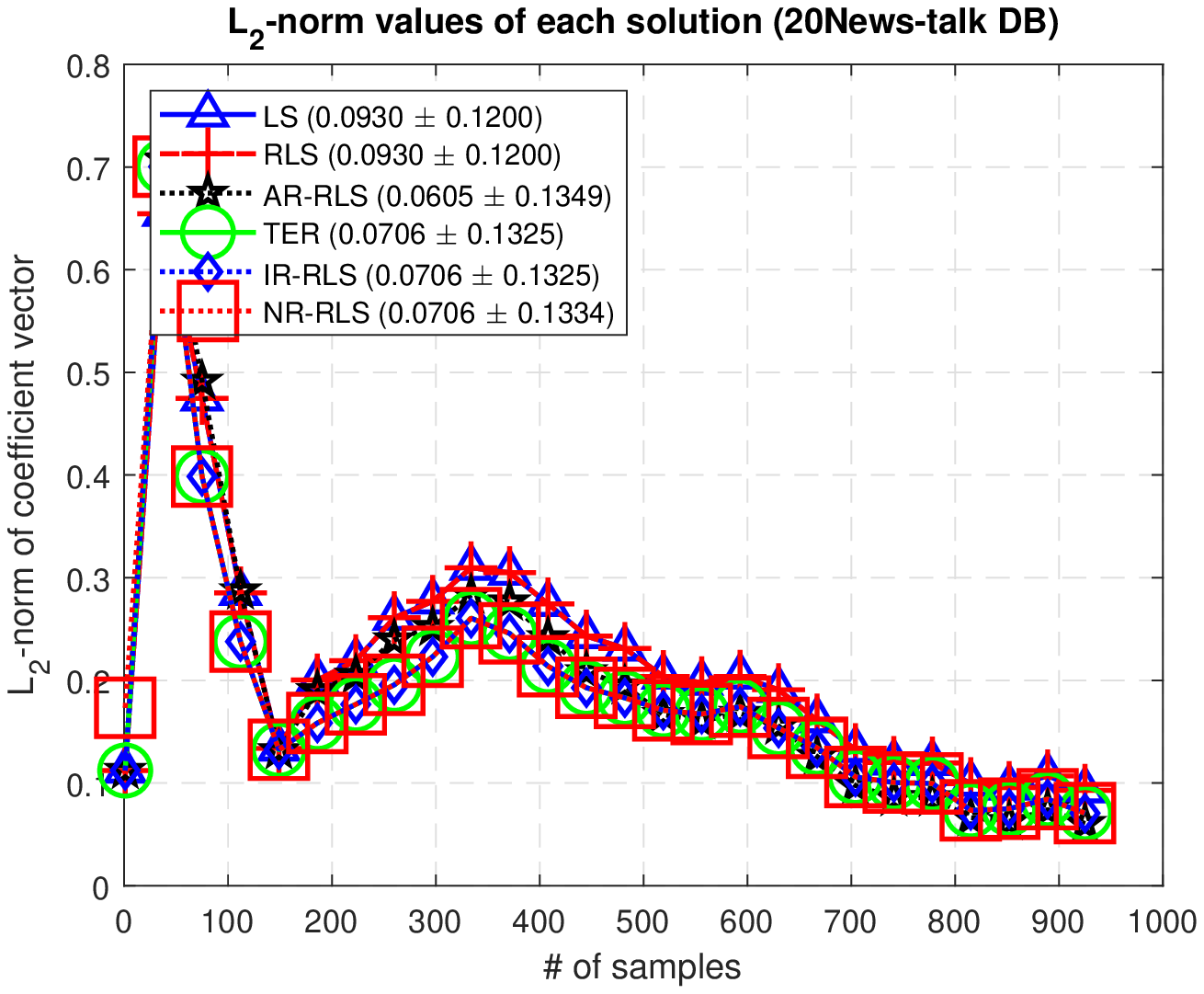}%
					} \hspace{0.4cm}
					\subfigure[The G-means
					]{ 
						\includegraphics[width=0.3\textwidth]{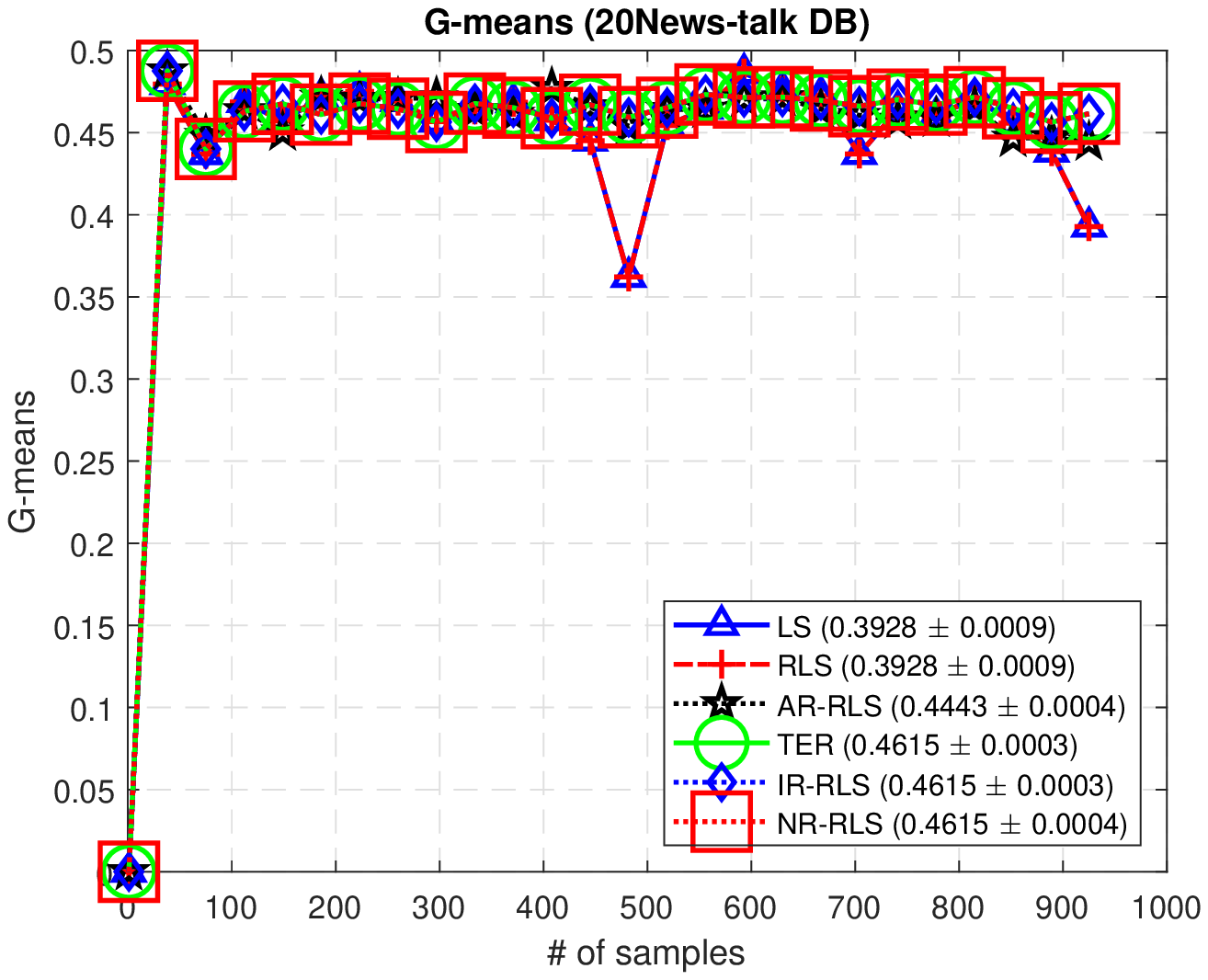}%
					} \hspace{0.4cm}
					\subfigure[The CPU times
					]{ 
						\includegraphics[width=0.3\textwidth]{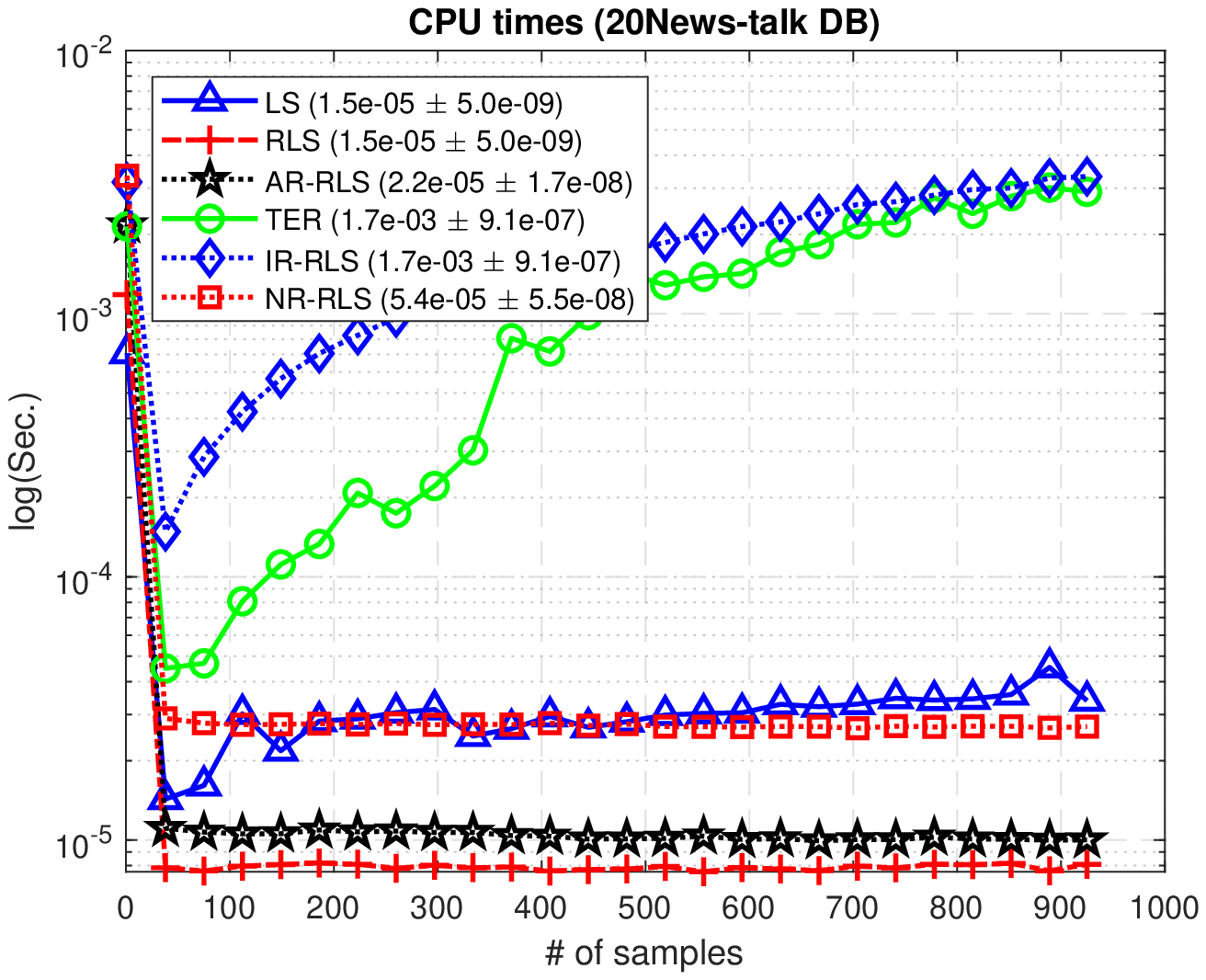}%
					} 
					\caption{20News-talk}%
					
				\end{figure*}

				\begin{figure*}[!h]
					\centering 
					\subfigure[The $L_2$-norm values
					]{ 
						\includegraphics[width=0.3\textwidth]{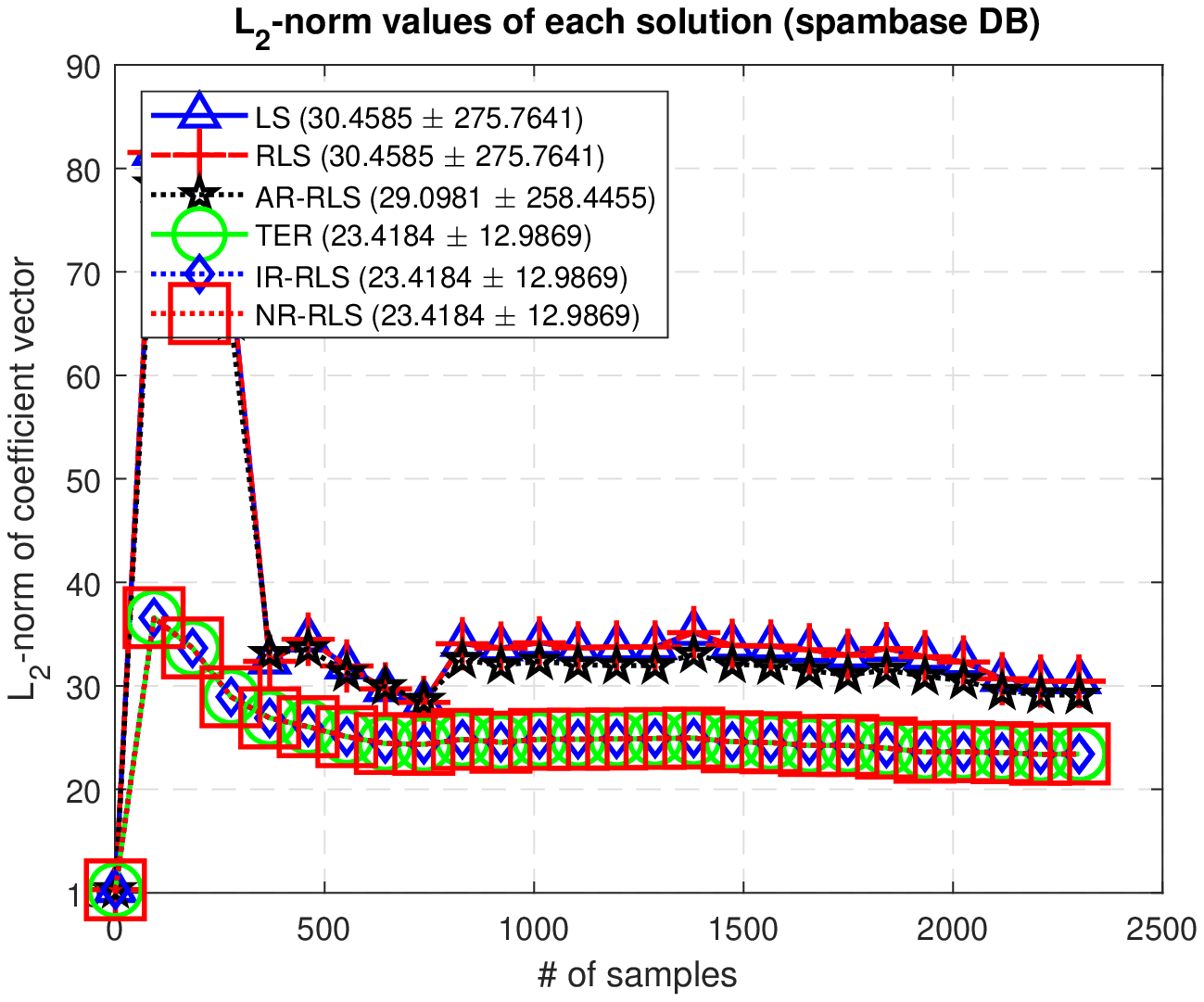}%
					} \hspace{0.4cm}
					\subfigure[The G-means
					]{ 
						\includegraphics[width=0.3\textwidth]{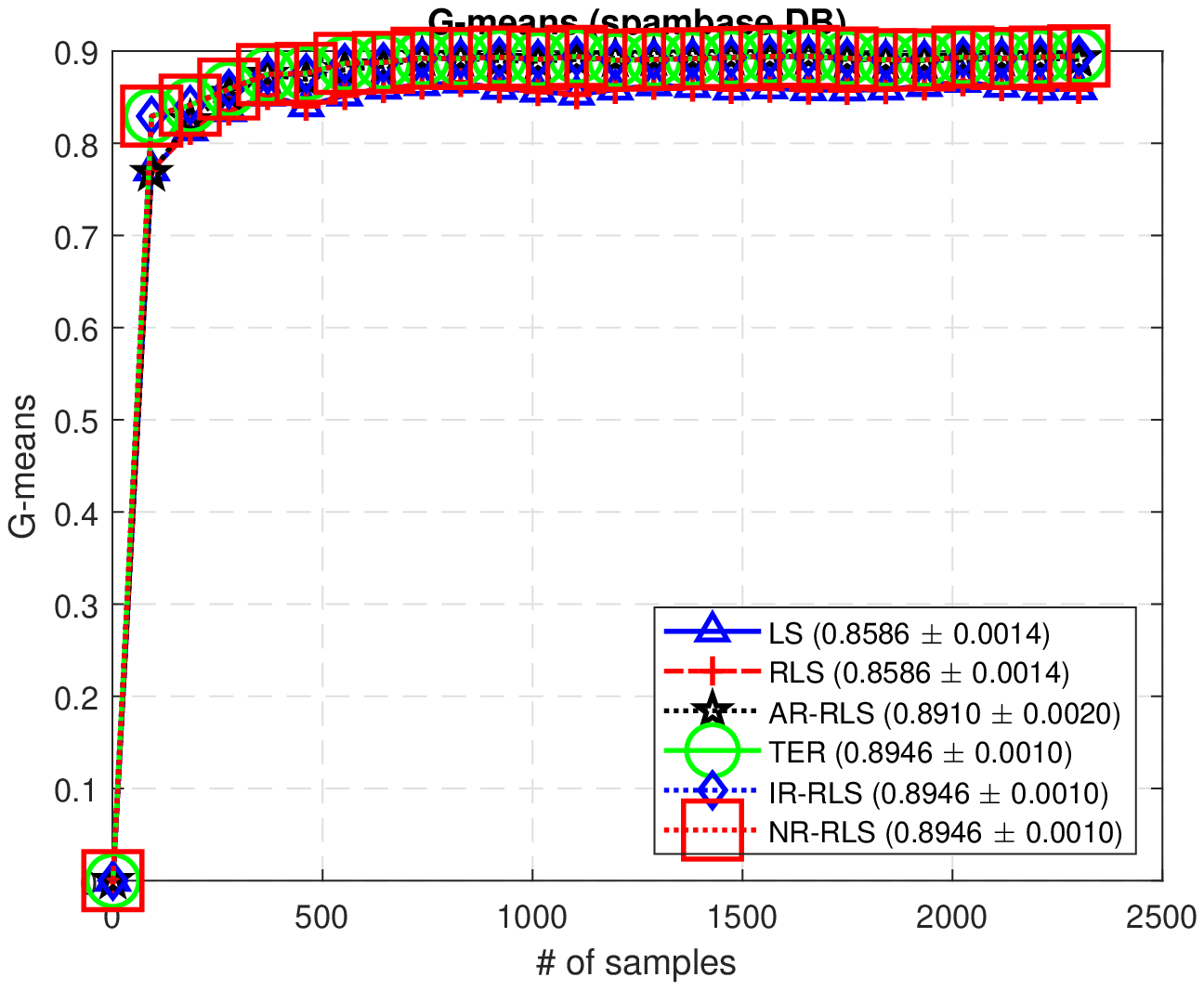}%
					} \hspace{0.4cm}
					\subfigure[The CPU times
					]{ 
						\includegraphics[width=0.3\textwidth]{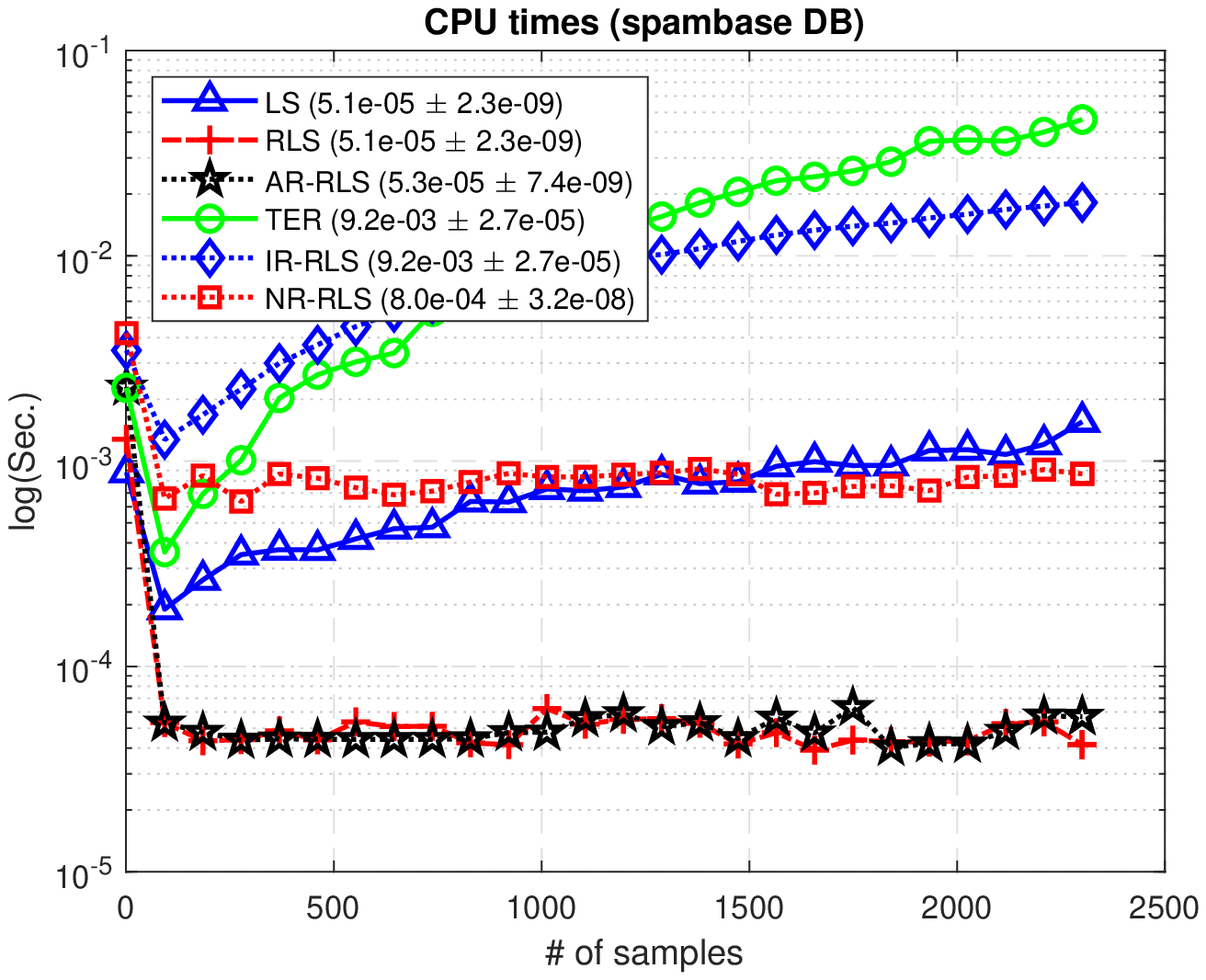}%
					} 
					\caption{Spambase}%
					
				\end{figure*} 
				\newpage
				\begin{figure*}[!h]
					\centering 
					\subfigure[The $L_2$-norm values
					]{ 
						\includegraphics[width=0.3\textwidth]{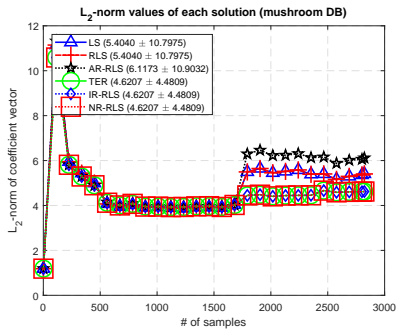}%
					} \hspace{0.4cm}
					\subfigure[The G-means
					]{ 
						\includegraphics[width=0.3\textwidth]{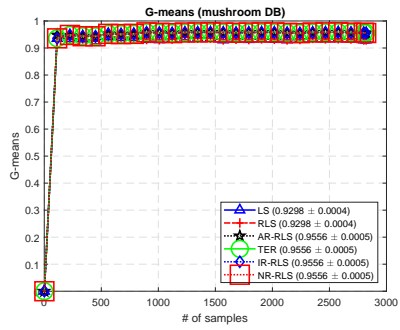}%
					} \hspace{0.4cm}
					\subfigure[The CPU times
					]{ 
						\includegraphics[width=0.3\textwidth]{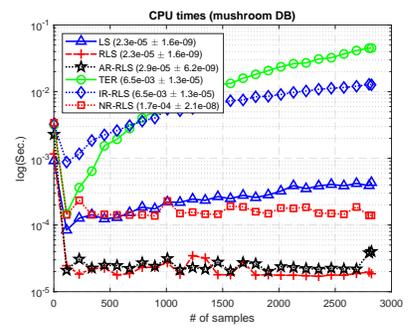}%
					} 
					\caption{Mushroom}%
					
				\end{figure*} 
				
				\begin{figure*}[!h]
					\centering 
					\subfigure[The $L_2$-norm values
					]{ 
						\includegraphics[width=0.3\textwidth]{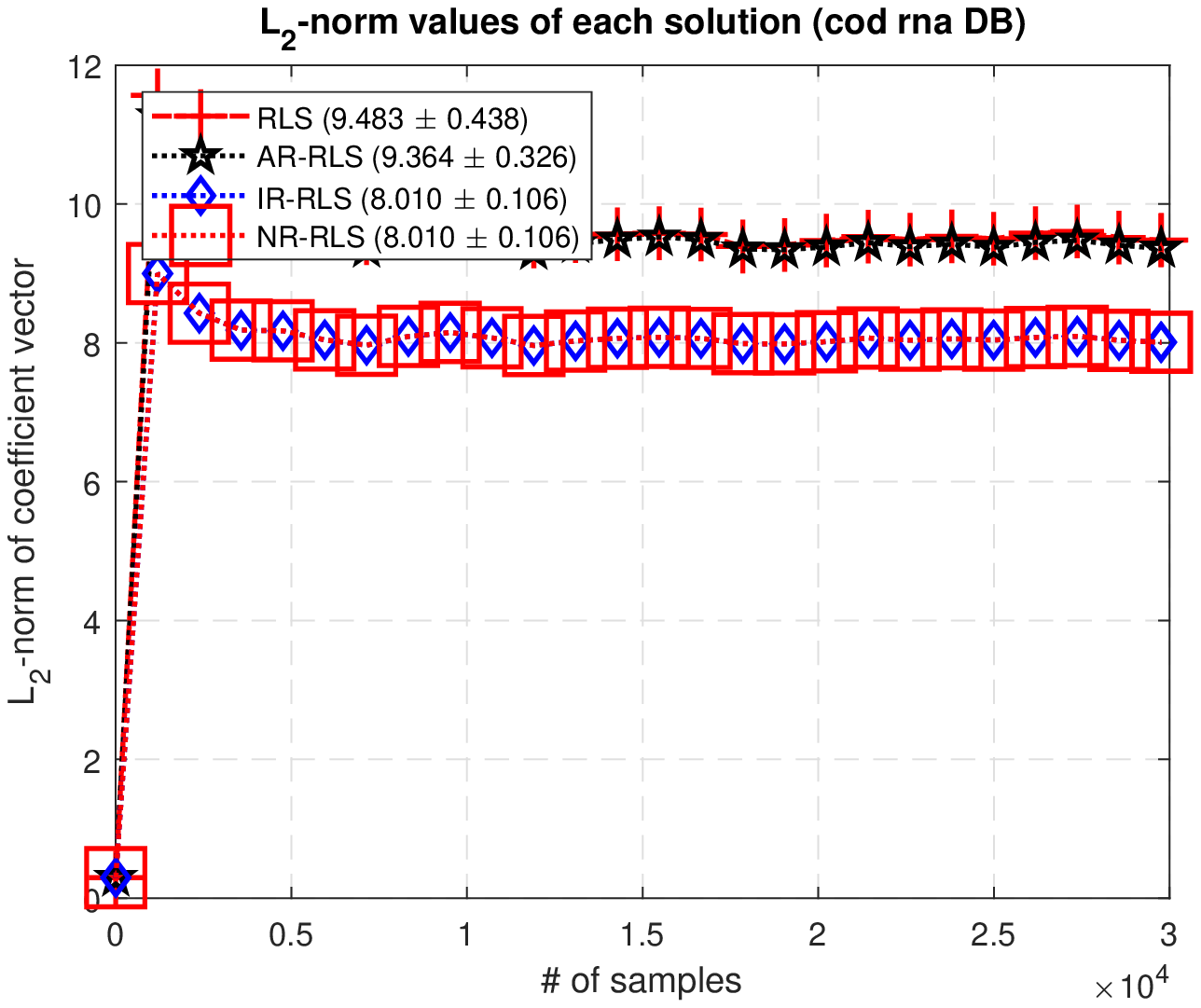}%
					} \hspace{0.4cm}
					\subfigure[The G-means
					]{ 
						\includegraphics[width=0.3\textwidth]{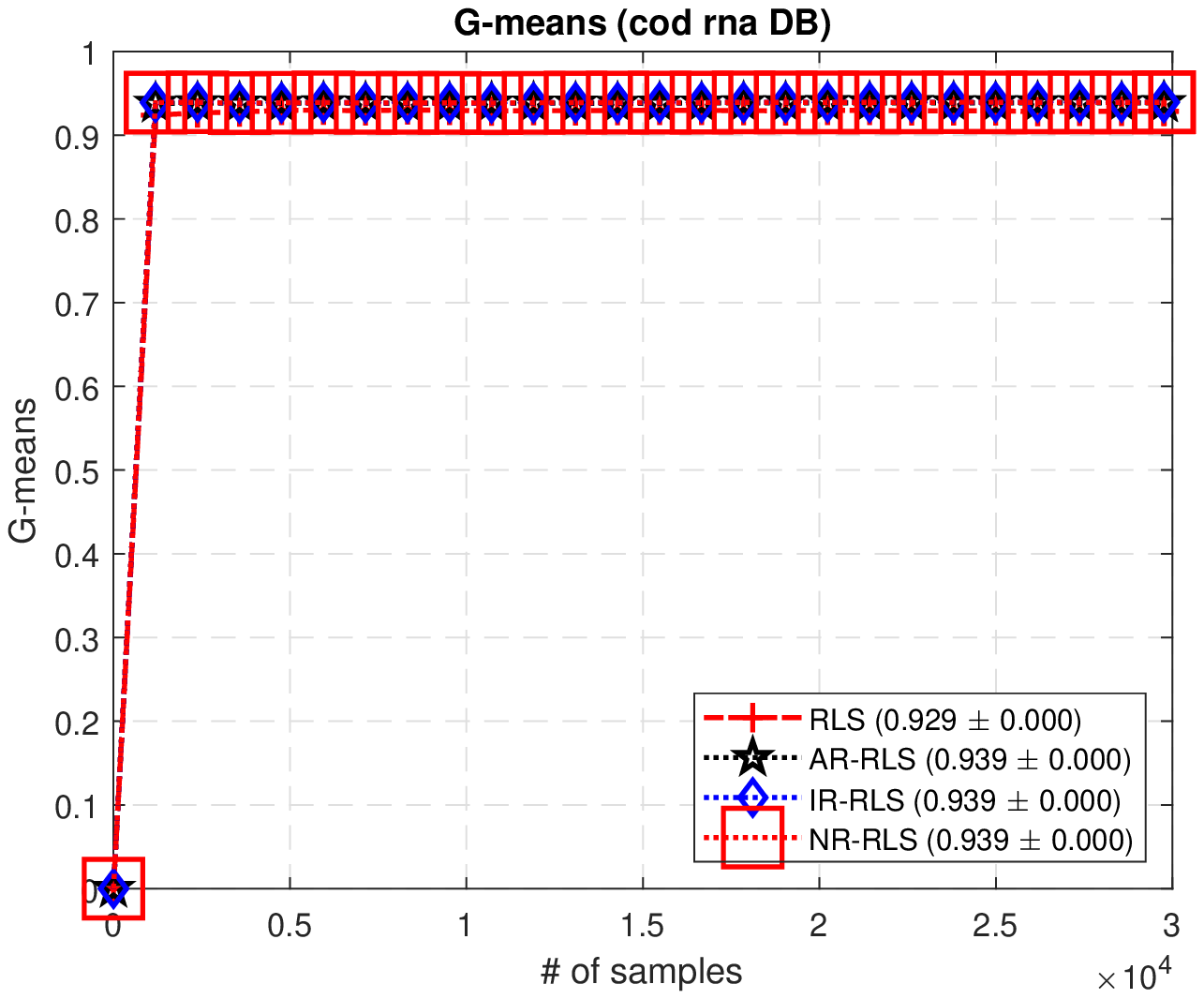}%
					} \hspace{0.4cm}
					\subfigure[The CPU times
					]{ 
						\includegraphics[width=0.3\textwidth]{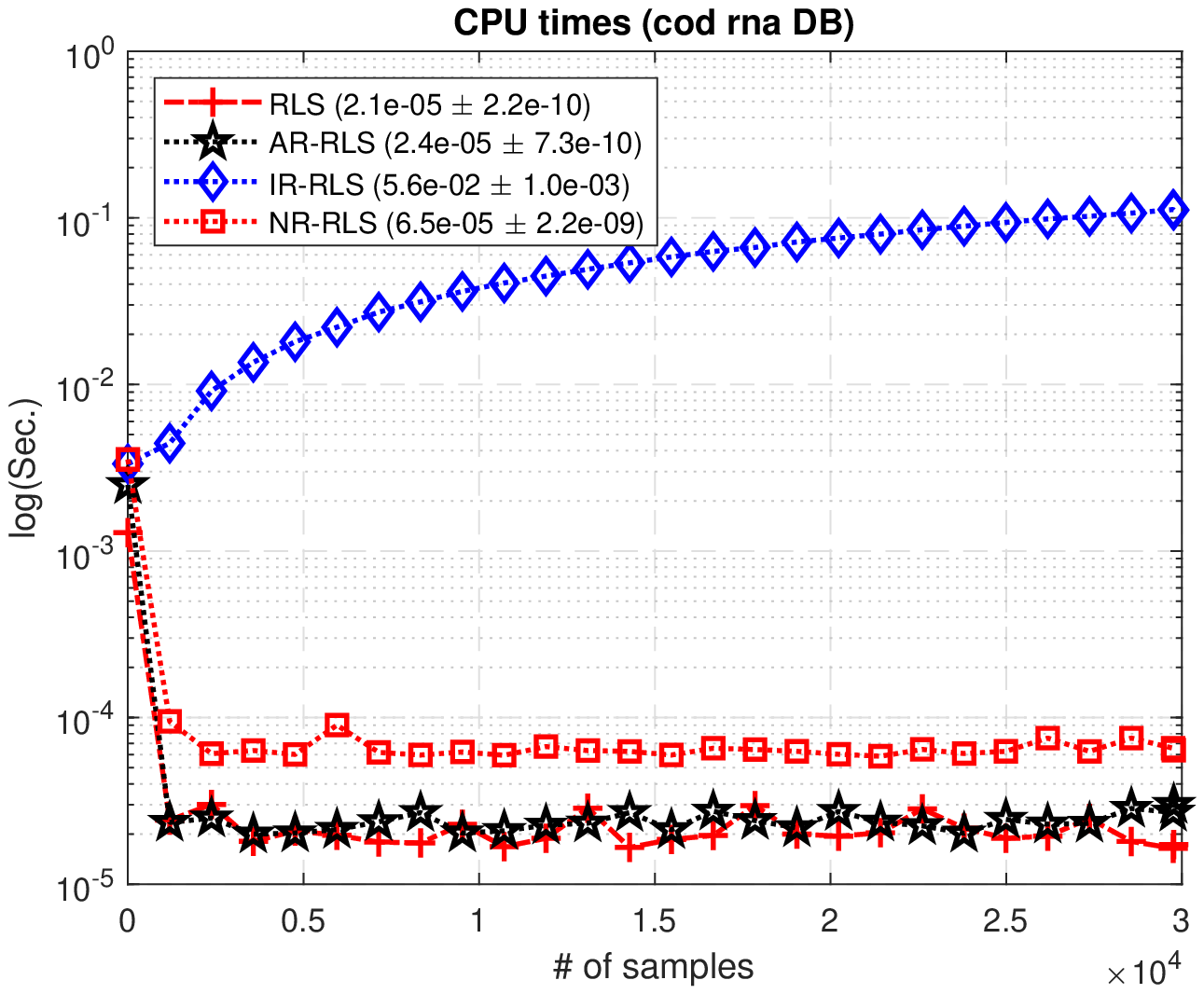}%
					} 
					\caption{Cod-rna}%
					
				\end{figure*} 
				
				\begin{figure*}[!h]
					\centering 
					\subfigure[The $L_2$-norm values
					]{ 
						\includegraphics[width=0.3\textwidth]{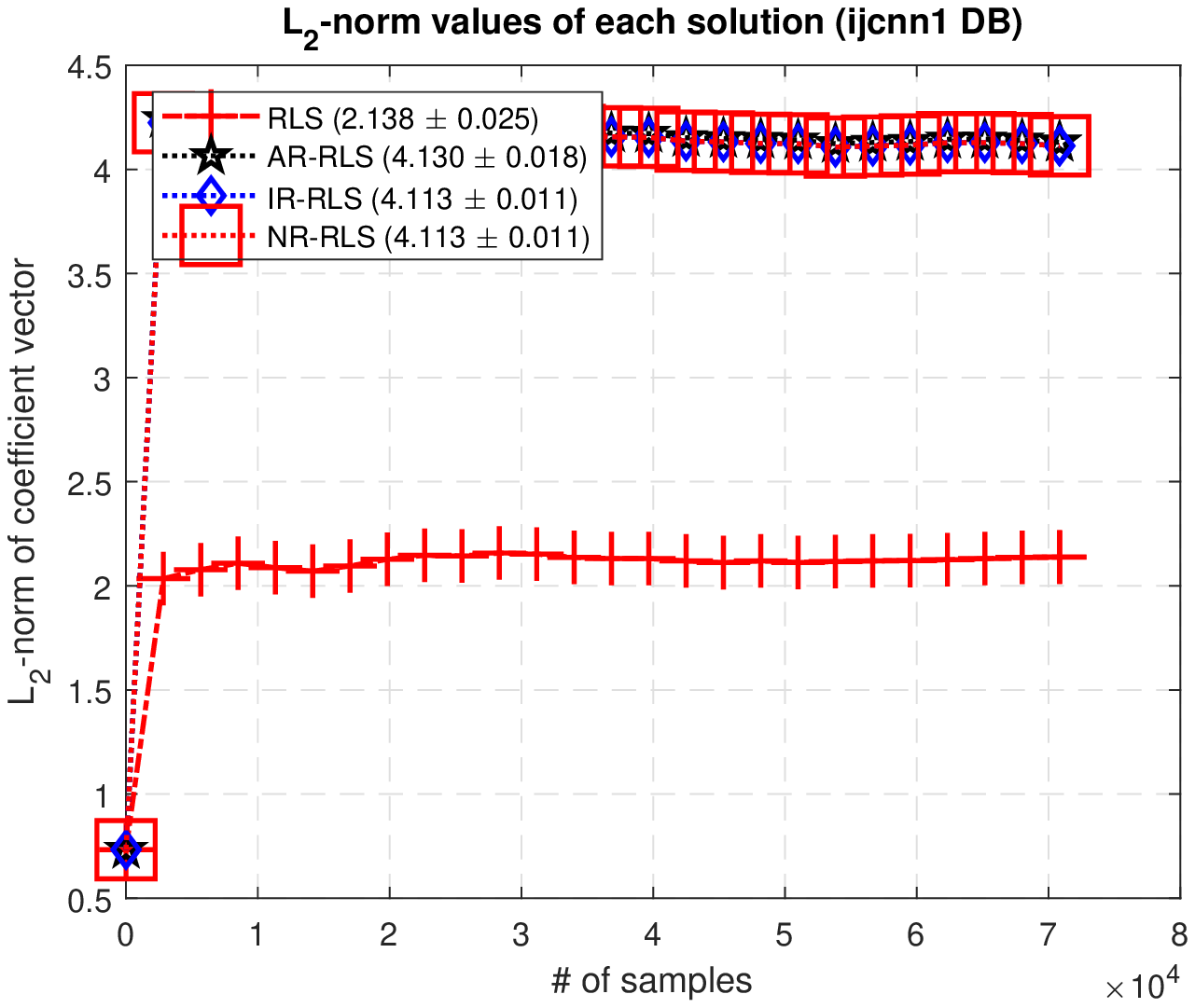}%
					} \hspace{0.4cm}
					\subfigure[The G-means
					]{ 
						\includegraphics[width=0.3\textwidth]{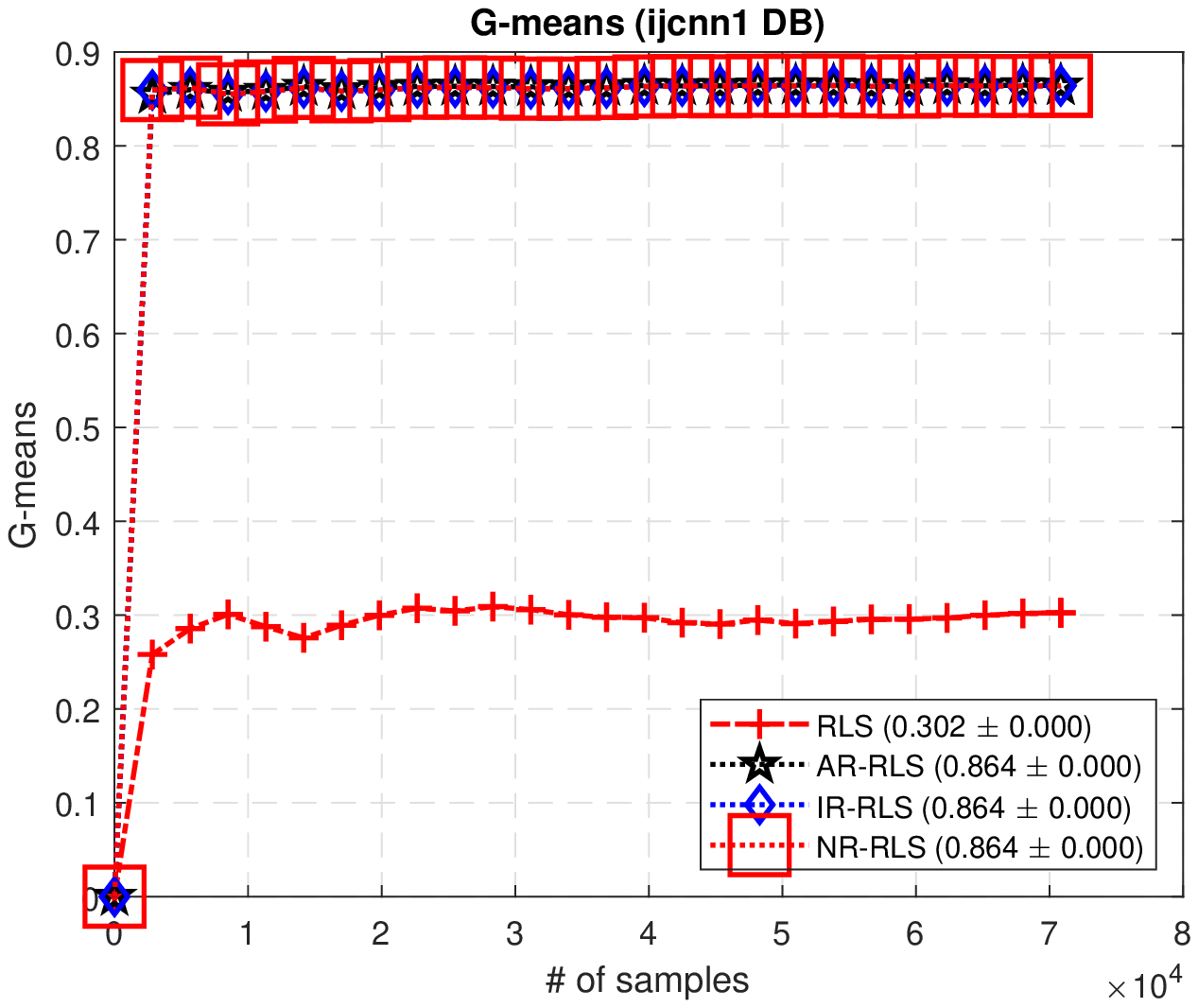}%
					} \hspace{0.4cm}
					\subfigure[The CPU times
					]{ 
						\includegraphics[width=0.3\textwidth]{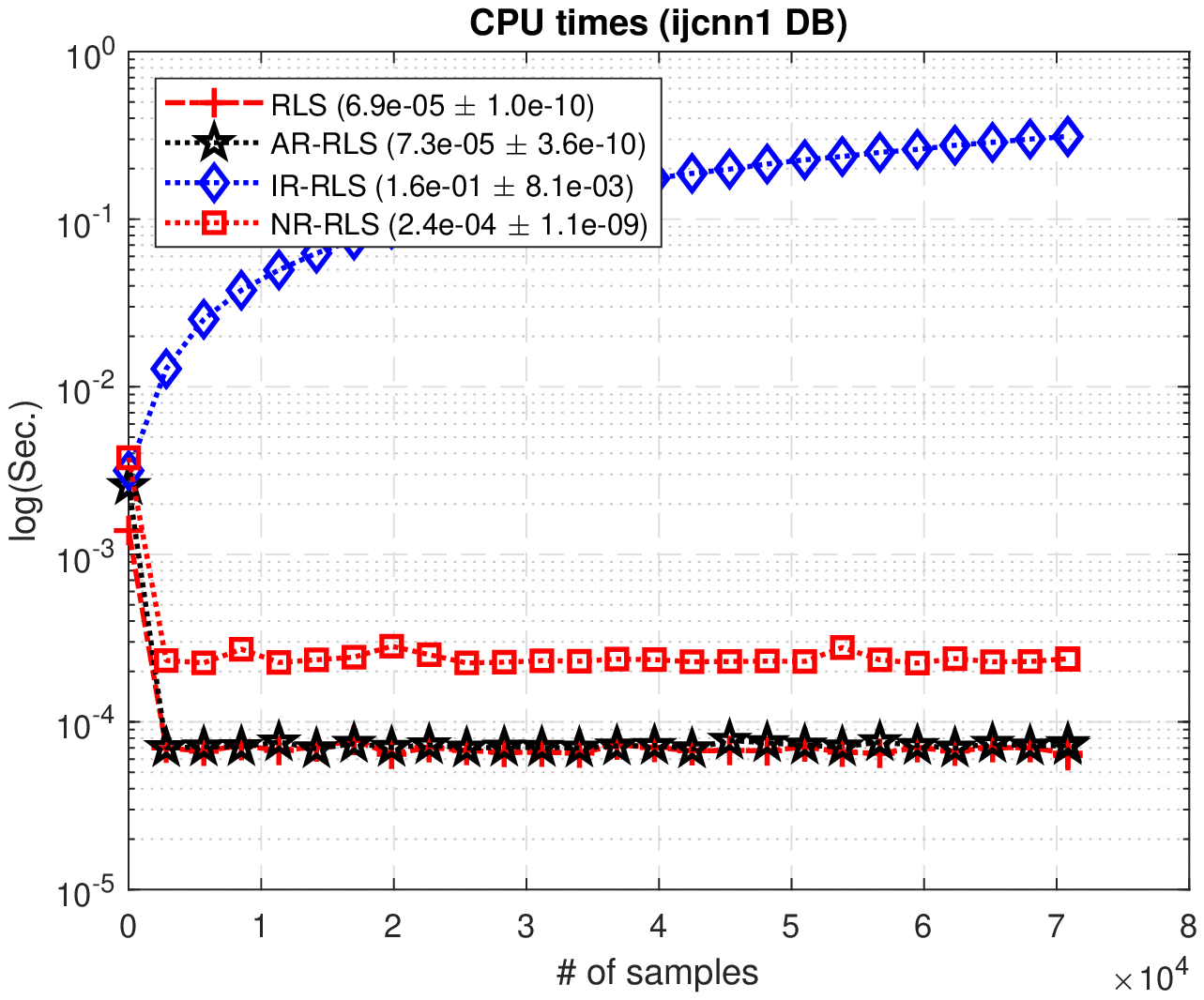}%
					} 
					\caption{Ijcnn1}%
					
				\end{figure*} 
				
				\begin{figure*}[!h]
					\centering 
					\subfigure[The $L_2$-norm values
					]{ 
						\includegraphics[width=0.3\textwidth]{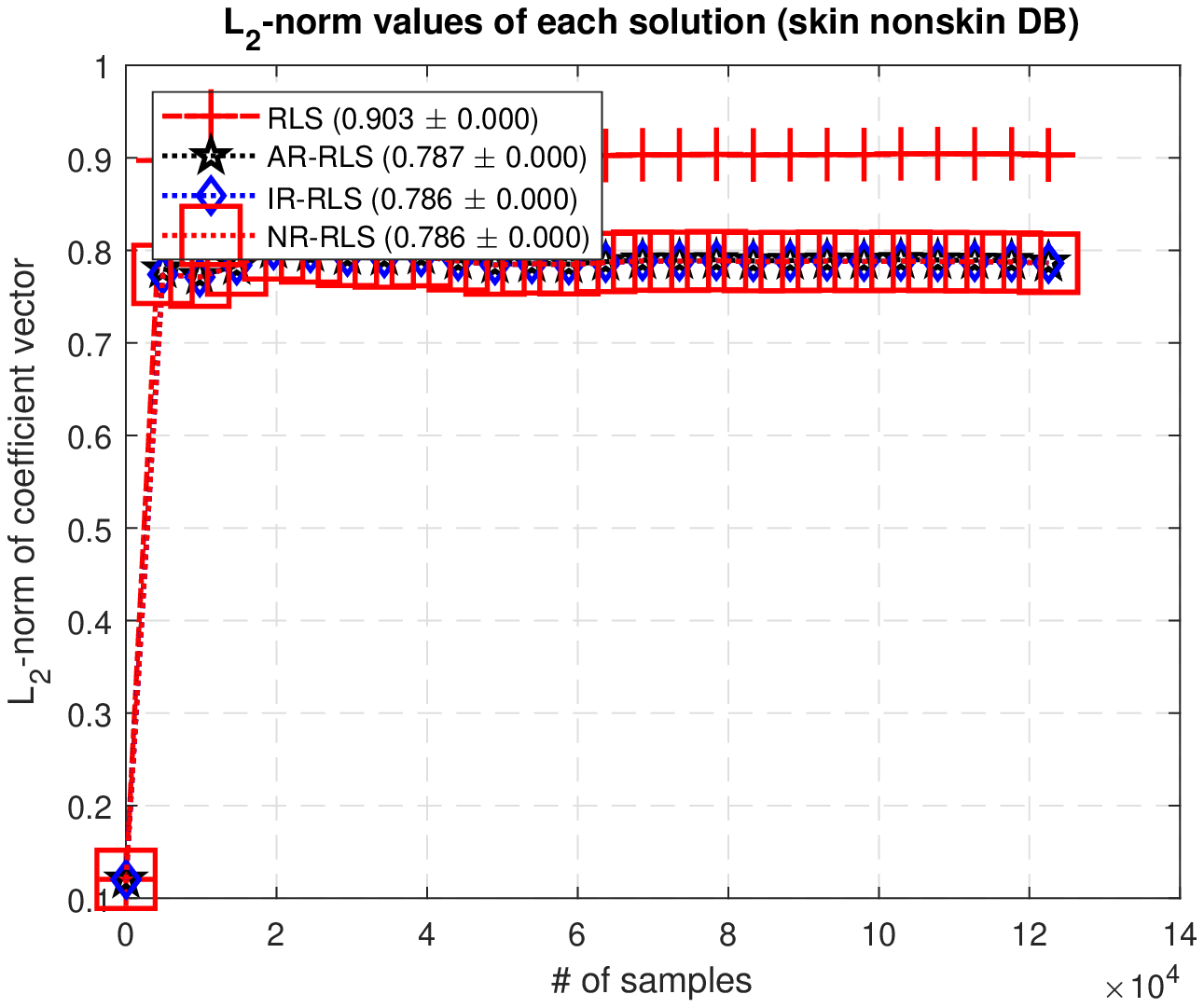}%
					} \hspace{0.4cm}
					\subfigure[The G-means
					]{ 
						\includegraphics[width=0.3\textwidth]{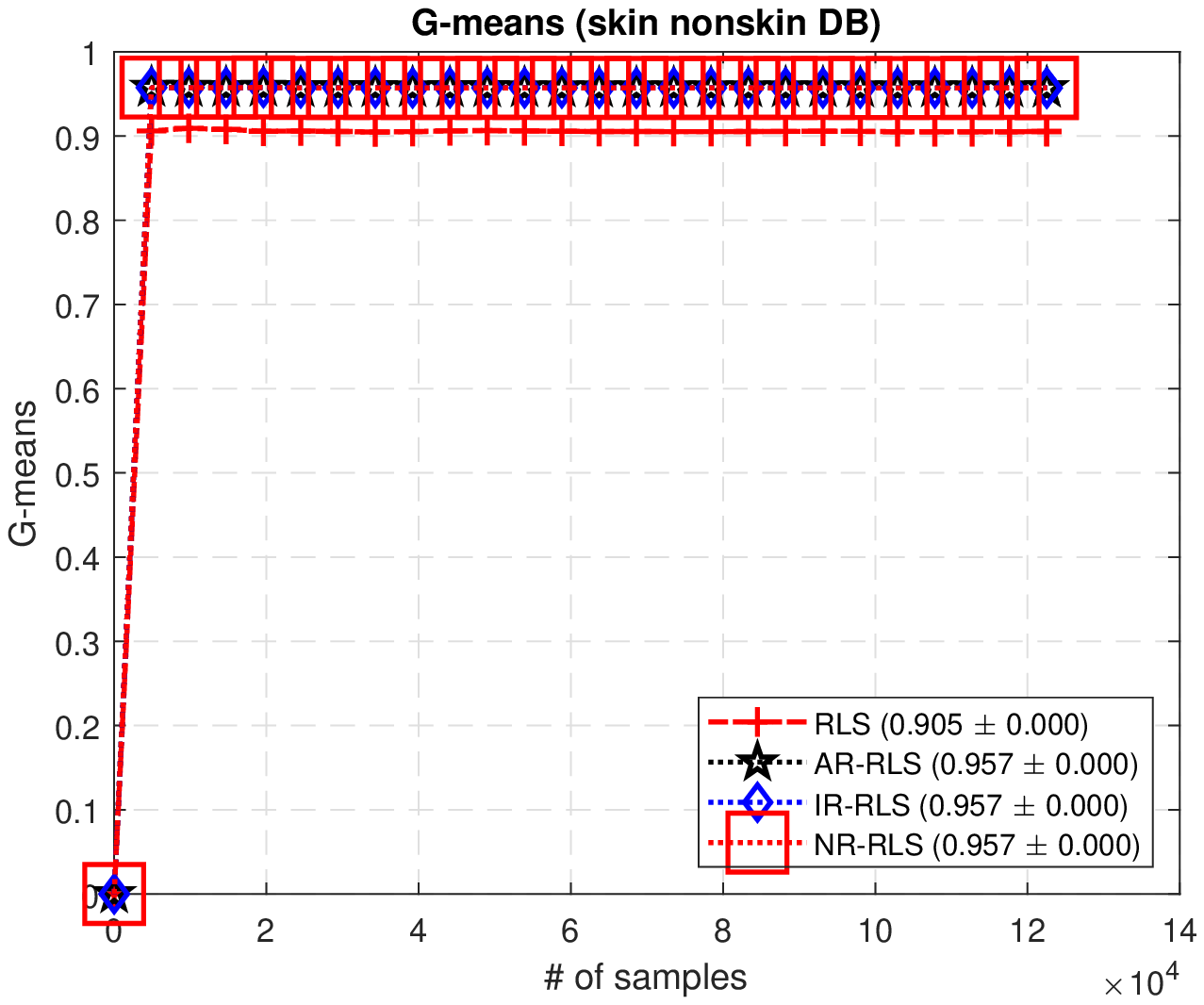}%
					} \hspace{0.4cm}
					\subfigure[The CPU times
					]{ 
						\includegraphics[width=0.3\textwidth]{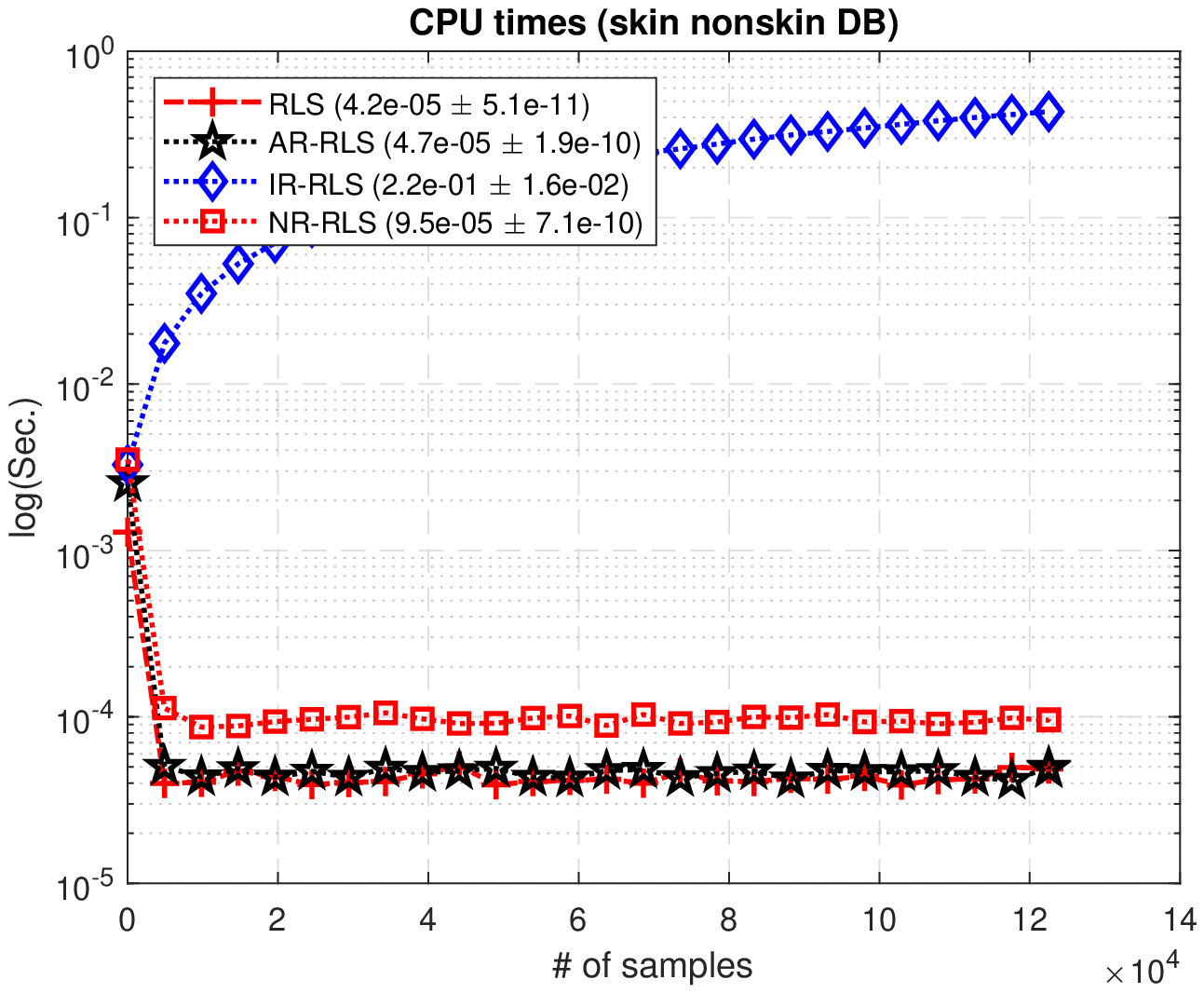}%
					} 
					\caption{Skin-nonskin}%
					
				\end{figure*} 

				\end{document}